\def\ArxivVersion{1}
\def\BlackTextVersion{1}
\ifdefined\ArxivVersion
\documentclass[10pt,a4paper]{article}
\usepackage[margin=25mm]{geometry}
\usepackage[T1]{fontenc}
\usepackage{lmodern}
\usepackage[skip=4pt plus 1pt minus 1pt,indent=0pt]{parskip}
\AddToHook{env/abstract/begin}{%

}
\usepackage{microtype}
\usepackage[authoryear,round]{natbib}
\setcitestyle{citesep={;},aysep={,},yysep={;}}
\else
  \documentclass{article} 
  \usepackage{iclr2027_conference,times}
\fi
\usepackage{booktabs}
\usepackage{amssymb}
\usepackage{tabularx}
\usepackage{placeins}
\usepackage{tikz}
\usetikzlibrary{arrows.meta,calc,fit,positioning}

\usepackage{amsmath,amsfonts,bm}

\def\eqref#1{equation~\ref{#1}}
\def\Eqref#1{Equation~\ref{#1}}

\def\1{\bm{1}}

\DeclareMathAlphabet{\mathsfit}{\encodingdefault}{\sfdefault}{m}{sl}
\SetMathAlphabet{\mathsfit}{bold}{\encodingdefault}{\sfdefault}{bx}{n}

\newcommand{\KL}{D_{\mathrm{KL}}}

\newcommand{\MASK}{\mathsf{MASK}}
\newcommand{\dTV}{d_{\mathrm{TV}}}
\newcommand{\sqHellinger}{h^2}
\newcommand{\HellingerAffinity}{\operatorname{Aff}}
\newcommand{\KLDivergence}[2]{%
  \KL\!\left(#1\middle\Vert #2\right)}
\newcommand{\ChiSquareDivergence}[2]{%
  \chi^2\!\left(#1\middle\Vert #2\right)}
\newcommand{\PositivePart}[1]{\left[#1\right]_+}

\newcommand{\PositionSet}{[N]}
\newcommand{\Vocab}{\mathcal V}
\newcommand{\VocabSize}{V}
\newcommand{\TargetLaw}{P}
\newcommand{\InstanceTargetLaw}[1]{P^{#1}}
\newcommand{\TargetForest}{F}
\newcommand{\TargetEdges}{E}
\newcommand{\TargetRoots}{\operatorname{Roots}(\TargetForest)}
\newcommand{\TokenOrder}{\prec}
\newcommand{\ParentOf}[1]{\operatorname{pa}(#1)}

\newcommand{\PosMarginal}[1]{\pi_{#1}}
\newcommand{\EndpointKernel}[2]{K_{#1\to #2}}
\newcommand{\BoundaryRow}[4]{\mu_{#1\to #2}(\cdot\mid #3,#4)}
\newcommand{\TailSet}[2]{\mathcal T_{#1}(#2)}
\newcommand{\DirectedResponse}[3]{\Omega_{#1\to #2}(#3)}

\newcommand{\ResponseConstant}{L}
\newcommand{\ResponseExponent}{\alpha}
\newcommand{\EdgeSignal}{\omega}
\newcommand{\FrequencyConstant}{C}

\newcommand{\VocabGrowthExponent}{\nu}
\newcommand{\RankTailExponent}{s}

\newcommand{\RankedToken}[2]{a_{#1,#2}}

\newcommand{\FrozenOracle}{q}
\newcommand{\ExactRow}[2]{\mu_{#2}(\cdot\mid #1)}
\newcommand{\ExactRowMass}[3]{\mu_{#2}(#3\mid #1)}
\newcommand{\OracleRow}[2]{q_{#2}(\cdot\mid #1)}
\newcommand{\OracleRowMass}[3]{q_{#2}(#3\mid #1)}
\newcommand{\InstanceOracleRow}[3]{q_{#3}^{#1}(\cdot\mid #2)}
\newcommand{\InstanceOracleRowMass}[4]{q_{#3}^{#1}(#4\mid #2)}
\newcommand{\OracleClass}{\mathcal Q^{\mathrm{Hel}}}
\newcommand{\OracleRadius}{\varepsilon^{\mathrm H}}
\newcommand{\HellingerCase}{case~(i)}
\newcommand{\HellingerCaseTitle}{Case~(i)}

\newcommand{\MaskSet}[1]{M(#1)}
\newcommand{\ObservedSet}[1]{O(#1)}
\newcommand{\AllMaskState}{y^\perp}
\newcommand{\HistoryAt}[1]{H_{#1}}
\newcommand{\History}{\HistoryAt{G}}
\newcommand{\ResidualPositionsAt}[1]{U_{#1}}
\newcommand{\ResidualPositions}{\ResidualPositionsAt{G}}
\newcommand{\ResidualForestAt}[1]{F_{#1}}
\newcommand{\ResidualForest}{\ResidualForestAt{G}}

\newcommand{\OpVisibleHistory}[1]{h^{\mathrm{op}}_{#1-}}
\newcommand{\OpSubmittedState}[1]{y^{\mathrm{op}}_{#1}}
\newcommand{\OpVisibleObservation}[1]{o^{\mathrm{vis}}_{#1}}

\newcommand{\ControllerSeed}{W}
\newcommand{\SamplerOutput}{\widehat X}
\newcommand{\SamplerOutputLaw}[2]{\widehat P_{#1}^{#2}}
\newcommand{\SafeSamplerOutputLaw}[2]{\widehat P_{#1}^{\mathrm{safe},#2}}
\newcommand{\ExactCommitOutputLaw}[1]{P_{#1}^{\otimes}}
\newcommand{\ExactBatchLaw}[1]{P^{\otimes}(\cdot\mid #1)}
\newcommand{\ExactBatchMass}[2]{P^{\otimes}(#2\mid #1)}
\newcommand{\OracleBatchLaw}[1]{\widehat P^{\otimes}(\cdot\mid #1)}
\newcommand{\OracleBatchMass}[2]{\widehat P^{\otimes}(#2\mid #1)}
\newcommand{\ExactSuffixLaw}[1]{P^{\mathrm{suf}}(\cdot\mid #1)}
\newcommand{\ExactSuffixMass}[2]{P^{\mathrm{suf}}(#2\mid #1)}
\newcommand{\OracleSuffixLaw}[1]{\widehat P^{\mathrm{suf}}(\cdot\mid #1)}
\newcommand{\OracleSuffixMass}[2]{\widehat P^{\mathrm{suf}}(#2\mid #1)}
\newcommand{\SuffixAffinity}[1]{\operatorname{Aff}_{\mathrm{suf}}(#1)}

\newcommand{\InstanceOracle}[1]{q^{#1}}
\newcommand{\SeedRisk}{\mathcal R_{\mathrm{TV}}}
\newcommand{\RobustRisk}{\overline{\mathcal R}_{\mathrm{TV}}}
\newcommand{\TargetTolerance}{\varepsilon}
\newcommand{\SharedClass}{\mathfrak F^{\mathrm{sh}}}
\newcommand{\SharedRisk}{\mathcal R_{\mathrm{TV}}^{\mathrm{sh}}}

\newcommand{\CounterfactualQueries}{Q_{\mathrm{cf}}}
\newcommand{\MarginalQueryCharge}{Q_{\mathrm{marg}}}
\newcommand{\ReportedQueries}{Q_{\mathrm{nc}}}
\newcommand{\CommitRounds}{R}
\newcommand{\QueryBudget}{\overline Q}
\newcommand{\RoundBudget}{\overline R}
\newcommand{\PackedSampler}{\mathcal A_{\mathrm{pack}}}

\newcommand{\RandomizedPackedSampler}{\mathcal A_{\mathrm{pack}}^{\mathrm{rnd}}}

\newcommand{\RowTVError}{\varepsilon_0}

\newcommand{\ScreenResolution}{\delta}
\newcommand{\TailTolerance}{\delta_{\mathrm{tail}}}
\newcommand{\ThresholdFloor}{\underline t}
\newcommand{\StandardThresholdFloor}[1]{4\max\{\RowTVError,#1/\VocabSize\}}
\newcommand{\TailThreshold}{t}

\newcommand{\DegreeCutoff}{d}
\newcommand{\ChunkCount}{J}
\newcommand{\DegreeExponent}{\gamma_{\mathrm{deg}}}

\newcommand{\PreprocessData}{\mathsf{Prep}}
\newcommand{\ThresholdGridDepth}{J^{\mathrm{grid}}}
\newcommand{\ThresholdGrid}{\mathcal G}
\newcommand{\FeasibleThresholdGrid}{\mathcal G^{\mathrm{feas}}}
\newcommand{\NoisyMarginal}[1]{\widetilde\pi_{#1}}
\newcommand{\TokenBank}[1]{\widehat{\mathcal B}_{#1}}
\newcommand{\BankToken}[2]{\widehat a_{#1,#2}}
\newcommand{\BankSize}[1]{\ell_{#1}}
\newcommand{\MaxBankSize}{\ell_{\max}}
\newcommand{\TailRepresentative}[1]{b_{#1}}

\newcommand{\DraftFiller}{f}
\newcommand{\DraftToken}[1]{f_{#1}}
\newcommand{\DraftGenerator}{\mathsf{Draft}}

\newcommand{\TailColumn}{\mathsf T}
\newcommand{\ScreenColumnSet}{\mathcal I^{\mathrm{src}}}
\newcommand{\ScreenColumnCount}{\Lambda}
\newcommand{\PackedColumn}[2]{\varphi_{#1}(#2;\DraftFiller)}

\newcommand{\OracleDepth}{D}

\newcommand{\ReadoutChunkSize}{B^{\mathrm{rd}}}

\newcommand{\ColorCount}{p}
\newcommand{\ColoringCount}{M}
\newcommand{\ColorSet}{\mathcal C^{\mathrm{val}}}
\newcommand{\ColoringIndexSet}{\mathcal I^{\mathrm{col}}}

\newcommand{\ColorHash}[2]{h_{#1}(#2)}

\newcommand{\RandomColorCount}{p_{\mathrm{rnd}}}
\newcommand{\RandomColoringCount}{M_{\mathrm{rnd}}}
\newcommand{\ScreenCallCap}{T_{\mathrm{scr}}}
\newcommand{\ScreenFailureBudget}{\delta^{\mathrm{fail}}}
\newcommand{\HardRoundCap}{\RoundBudget}
\newcommand{\ScreenSuccess}[1]{\mathsf{Iso}(#1)}
\newcommand{\ColorClassAt}[3]{\mathcal C_{#1,#2}(#3)}
\newcommand{\ColorClass}[2]{\ColorClassAt{#1}{#2}{\History}}
\newcommand{\ColorClassSize}[2]{n_{#1,#2}(\History)}
\newcommand{\ReadoutChunkAt}[4]{\mathcal C_{#1,#2,#3}(#4)}
\newcommand{\ReadoutChunk}[3]{\ReadoutChunkAt{#1}{#2}{#3}{\History}}
\newcommand{\ReadoutChunkOf}[2]{\mathcal C_{#1}(#2)}
\newcommand{\ReadoutColorValue}{c^{\mathrm{rd}}}
\newcommand{\SourceColorValue}{c^{\mathrm{src}}}

\newcommand{\ProbeStateAt}[5]{y_{#1,#2,#3}(#4,\DraftFiller,#5)}
\newcommand{\ProbeState}[4]{\ProbeStateAt{#1}{#2}{#3}{\History}{#4}}
\newcommand{\ProbeObservedRow}[4]{q^{\mathrm{scr}}_{#1,#2}(#3\!\to\!#4\mid\History,\DraftFiller)}
\newcommand{\ProbeExactRow}[4]{\mu^{\mathrm{scr}}_{#1,#2}(#3\!\to\!#4\mid\History,\DraftFiller)}
\newcommand{\ObservedRowFamily}[3]{\mathcal Q_{#1}(#2\!\to\!#3\mid\History)}
\newcommand{\ExactRowFamily}[3]{\mathcal P_{#1}(#2\!\to\!#3\mid\History)}
\newcommand{\ScreenDiameter}[3]{D_{#1}(#2,#3\mid\History)}
\newcommand{\ScreenVote}[3]{Y_{#1}(#2,#3\mid\History)}
\newcommand{\CandidateRow}[1]{\widetilde A_{#1}(\History)}
\newcommand{\RowSelector}{\mathsf S}
\newcommand{\ScreenTranscript}{\mathcal T}
\newcommand{\SelectedPackedSampler}{\RandomizedPackedSampler[\RowSelector]}
\newcommand{\ScreenRowsAt}[1]{A(#1)}
\newcommand{\ScreenRowAt}[2]{A_{#2}(#1)}
\newcommand{\ScreenRows}{\ScreenRowsAt{\History}}
\newcommand{\ScreenRow}[1]{\ScreenRowAt{\History}{#1}}

\newcommand{\PackedScreenName}{PackedScreen}
\newcommand{\CommitName}{Commit}
\newcommand{\CommitOperation}{\operatorname{\CommitName}}
\newcommand{\InfeasibleOutcome}{\textsc{Infeasible}}
\newcommand{\ClaimCount}[1]{\operatorname{cl}_{#1}(\History)}
\newcommand{\PeelSet}{\mathcal P(\History)}
\newcommand{\FrozenPeelSet}{\overline{\mathcal P}}
\newcommand{\HighDegreeSetAt}[1]{\mathcal V^{\mathrm{hi}}(#1)}
\newcommand{\HighDegreeSet}{\HighDegreeSetAt{\History}}
\newcommand{\HighDegreeCount}[1]{n_{#1}^{\mathrm{hi}}}
\newcommand{\PeelPhaseCap}{T_{\mathrm{peel}}}
\newcommand{\ChunksPerColoring}[1]{J_{#1}}
\newcommand{\ScreenQueries}{Q_{\mathrm{scr}}(\History)}
\newcommand{\TerminalHistory}{\HistoryAt{\star}}
\newcommand{\TerminalPositions}{\ResidualPositionsAt{\star}}
\newcommand{\CertifiedEdges}{\widehat E_\star}
\newcommand{\CertifiedForest}{\widehat F}

\newcommand{\ComponentCentroid}[1]{v_{#1}^{\mathrm{cen}}}

\newcommand{\HardFamily}{\mathfrak I^{\mathrm{hard}}}
\newcommand{\HardMinimax}{\mathfrak R_{\mathrm{TV}}^{\mathrm{hard}}}
\newcommand{\HardVocabSize}{\VocabSize}
\newcommand{\HardBankSize}{m}
\newcommand{\TriggerMass}{\rho}
\newcommand{\EdgeCoupling}{\eta}

\newcommand{\HardCandidate}[1]{a_{#1}}
\newcommand{\HardBackground}{\mathcal B^{\mathrm{bg}}}
\newcommand{\HardMarginal}{\pi}
\newcommand{\HiddenMatching}{\mathcal M}
\newcommand{\TriggerVector}{\Theta}
\newcommand{\TriggerIndex}[1]{\Theta_{#1}}
\newcommand{\TriggerFeature}[1]{\phi_{#1}}
\newcommand{\TriggerVariance}{\sigma_\rho^2}
\newcommand{\MateInMatching}[1]{\operatorname{mate}_{\HiddenMatching}(#1)}

\newcommand{\HardEdgeLaw}[2]{P_{#1}^{#2}}
\newcommand{\HardIndependentEdgeLaw}[1]{P_{#1}^{0}}
\newcommand{\HardTargetLaw}[1]{P^{#1}}
\newcommand{\HardKernel}[3]{K_{#1\to #2}^{#3}}
\newcommand{\HardOracle}[1]{q^{#1}}
\newcommand{\HardOracleRow}[3]{\InstanceOracleRow{#1}{#2}{#3}}
\newcommand{\HardOracleRowMass}[4]{\InstanceOracleRowMass{#1}{#2}{#3}{#4}}
\newcommand{\HardAlgorithmClass}[2]{\mathfrak A(#1,#2)}

\newcommand{\HardInstancePrior}{\Pi}
\newcommand{\HardOutputLaw}[2]{\widehat P_{#2}^{#1}}
\newcommand{\HardMidpointLaw}[2]{\mathbb M_{#2}^{#1}}
\newcommand{\HardTerminalRoundCount}[3]{T_{#2}^{#1}(#3)}
\newcommand{\HardExactBatchLaw}[2]{P_{#1}^{#2}}
\newcommand{\HardProductBatchLaw}[2]{\widehat P_{#1}^{#2}}
\newcommand{\HardLocalAffinity}[2]{a_{#1}(#2)}
\newcommand{\HardMidpointKernel}[2]{m_{#1}^{#2}}
\newcommand{\HardEdgeDistance}{d_{\mathrm{edge}}}
\newcommand{\HardTest}[2]{\mathsf{Test}_{#2}(#1)}
\newcommand{\HardHit}[2]{\mathsf{Hit}_{#2}(#1)}
\newcommand{\UnresolvedVertices}[1]{\mathcal U_{#1}}
\newcommand{\UnresolvedVertexCount}[1]{n_{#1}}
\newcommand{\ActiveBatch}[1]{B_{#1}^{\circ}}
\newcommand{\ActiveBatchSize}[1]{n_{#1}^{\circ}}
\newcommand{\ActiveRoundCount}{R^{\circ}}
\newcommand{\CollisionCount}[1]{Z_{#1}}
\newcommand{\TotalCollisions}{Z}
\newcommand{\InternalEdgeCount}[1]{Z_{#1}}
\newcommand{\CollisionExponential}[1]{\mathsf M_{#1}}
\newcommand{\GenieDeletions}{D_{\mathrm{gen}}}
\newcommand{\RemainingTriggers}[2]{\mathcal C_{#1,#2}}

\newcommand{\AnalyticHistory}[1]{\mathcal H_{#1-}}
\newcommand{\AnalyticField}[1]{\mathcal F_{#1-}}
\newcommand{\SurvivorState}[1]{\mathsf S_{#1}}

\newcommand{\RetiredRecordValue}{\mathfrak r}

\newcommand{\SurvivorConfigA}{\mathfrak z}

\newcommand{\OpAnalyticHistory}[1]{\mathcal H^{\mathrm{op}}_{#1-}}
\newcommand{\OpAnalyticField}[1]{\mathcal F^{\mathrm{op}}_{#1-}}
\newcommand{\OpUnresolvedMatching}[1]{\mathcal M_{#1}^{\mathrm{unres}}}

\newcommand{\OpUnresolvedVertices}[1]{\mathcal U^{\mathrm{op}}_{#1}}
\newcommand{\OpRemainingTriggers}[2]{\mathcal C^{\mathrm{op}}_{#1,#2}}
\newcommand{\OpSurvivorState}[1]{\mathsf S^{\mathrm{op}}_{#1}}
\newcommand{\OpRetiredRecord}[1]{\mathsf{Ret}_{#1}}
\newcommand{\OpObservation}[1]{\mathsf{Obs}_{#1}}
\newcommand{\OpSurvivorLikelihood}[4]{\ell_{#1,#2}(#3\,;#4)}

\newcommand{\KLRowBudget}{\kappa}
\newcommand{\KLTargetTolerance}{\varepsilon^{\mathrm{KL}}}

\newcommand{\KLOracleRadius}{\sqrt{\KLTargetTolerance/2}}
\newcommand{\KLCase}{case~(ii)}
\newcommand{\KLCaseTitle}{Case~(ii)}
\newcommand{\KLOracleClass}{\mathcal Q^{\mathrm{KL}}}
\newcommand{\KLSeedRisk}{\mathcal R_{\mathrm{KL}}}
\newcommand{\JointReferenceLaw}[1]{\mathbb P_{#1}^{\mathrm{joint}}}
\newcommand{\JointBatchLaw}[1]{p(\cdot\mid #1)}
\newcommand{\ScreenFailureEvent}{\mathsf{Fail}}

\newcommand{\HellingerBenchmarkSignal}{\frac12
  \left(\frac{(\OracleRadius)^2}{N}\right)^{1/3}}
\newcommand{\KLBenchmarkSignal}{\frac12
  \left(\frac{\KLTargetTolerance}{2N}\right)^{1/3}}

\newcommand{\PathOddsA}{r_{\mathsf A}}
\newcommand{\PathOddsB}{r_{\mathsf B}}

\newcommand{\EffectiveTC}{\mathcal D_{\mathrm{eff}}}
\newcommand{\MaskedPairInformation}{\mathcal I}

\usepackage{amsthm}
\usepackage{aliascnt}
\usepackage{algorithm}
\usepackage{algpseudocode}
\usepackage{float}

\makeatletter
\def\theHALG@line{\thealgorithm.\arabic{ALG@line}}
\makeatother

\algrenewcommand\algorithmicrequire{\textbf{Input:}}
\algrenewcommand\algorithmicensure{\textbf{Output:}}

\usepackage{hyperref}
\usepackage[capitalise,noabbrev,nameinlink]{cleveref}
\usepackage{amsthm}
\usepackage{needspace}

\theoremstyle{plain}
\newtheorem{theorem}{Theorem}[section]
\newaliascnt{lemma}{theorem}
\newtheorem{lemma}[lemma]{Lemma}
\aliascntresetthe{lemma}
\crefname{lemma}{Lemma}{Lemmas}
\Crefname{lemma}{Lemma}{Lemmas}
\newtheorem*{restatedlemmabase}{Lemma~\restatedlemmanumber}

\newaliascnt{proposition}{theorem}
\newtheorem{proposition}[proposition]{Proposition}
\aliascntresetthe{proposition}
\crefname{proposition}{Proposition}{Propositions}
\Crefname{proposition}{Proposition}{Propositions}
\newcommand{\AppLowDegreeScreenRef}{%
  \cref{app:lem:low-degree-screen}%
}
\newaliascnt{corollary}{theorem}
\newtheorem{corollary}[corollary]{Corollary}
\aliascntresetthe{corollary}
\crefname{corollary}{Corollary}{Corollaries}
\Crefname{corollary}{Corollary}{Corollaries}

\theoremstyle{definition}
\newaliascnt{definition}{theorem}
\newtheorem{definition}[definition]{Definition}
\aliascntresetthe{definition}
\crefname{definition}{Definition}{Definitions}
\Crefname{definition}{Definition}{Definitions}
\newtheorem*{definition*}{Definition}
\newaliascnt{assumption}{theorem}
\newtheorem{assumption}[assumption]{Assumption}
\aliascntresetthe{assumption}
\crefname{assumption}{Assumption}{Assumptions}
\Crefname{assumption}{Assumption}{Assumptions}

\theoremstyle{remark}
\newaliascnt{remark}{theorem}
\newtheorem{remark}[remark]{Remark}
\aliascntresetthe{remark}
\crefname{remark}{Remark}{Remarks}
\Crefname{remark}{Remark}{Remarks}

\crefname{assumption}{Assumption}{Assumptions}
\Crefname{assumption}{Assumption}{Assumptions}
\usepackage{url}

\theoremstyle{remark}
\newtheorem*{proposalcontextbase}{Remark}
\newenvironment{proposalcontext}[1]
  {\begin{proposalcontextbase}[#1]}
  {\end{proposalcontextbase}}

\theoremstyle{definition}
\newtheorem{example}{Example}[section]
\newtheorem*{example*}{Example}
\crefname{example}{Example}{Examples}
\Crefname{example}{Example}{Examples}

\newcommand{\DeclareAlgoBlock}[3]{%
  \expandafter\newcommand\csname AlgoBlockCode@#1\endcsname{#2}%
  \expandafter\newcommand\csname AlgoBlockName@#1\endcsname{#3}}
\DeclareAlgoBlock{preprocess}{1}{Preprocess \& initialize}
\DeclareAlgoBlock{draft}{2}{Draft \& color}
\DeclareAlgoBlock{screen}{3}{Packed screen}
\DeclareAlgoBlock{peel}{4}{Claim \& peel}
\DeclareAlgoBlock{terminal}{5}{Terminal graph}
\DeclareAlgoBlock{centroid}{6}{Centroid commits}
\DeclareAlgoBlock{chunks}{3.1}{Readout chunks}
\DeclareAlgoBlock{probes}{3.2}{Packed probes}
\DeclareAlgoBlock{aggregate}{3.3}{Votes \& row selection}

\newcommand{\AlgoStepPrefix}{Step}
\newcommand{\AlgoBlockCode}[1]{\csname AlgoBlockCode@#1\endcsname}
\newcommand{\AlgoBlockName}[1]{\csname AlgoBlockName@#1\endcsname}
\makeatletter
\newcommand{\AlgoBlockHeading}[1]{%
  \ALG@closeloops%
  \hspace*{\ALG@tlm}%
  \textbf{\AlgoStepPrefix~\AlgoBlockCode{#1}.\ \AlgoBlockName{#1}}}
\makeatother
\newcommand{\AlgoBlockTarget}[1]{%
  \hypertarget{algo-block-#1}{\AlgoBlockHeading{#1}}}
\DeclareRobustCommand{\AlgoBlockRef}[1]{%
  \hyperlink{algo-block-#1}{\AlgoStepPrefix~\AlgoBlockCode{#1}}}
\DeclareRobustCommand{\AlgoBlockNamedRef}[1]{%
  \hyperlink{algo-block-#1}{\AlgoStepPrefix~\AlgoBlockCode{#1} (\AlgoBlockName{#1})}}

\newcommand{\AlgoGuardName}{Round-cap fallback}
\DeclareRobustCommand{\AlgoGuardRef}{\hyperlink{algo-guard}{\textsc{Guard}}}
\newcommand{\AlgoGuardTarget}{%
  \hypertarget{algo-guard}{\emph{\textup{\textsc{Guard}}: \AlgoGuardName}}}

\ifdefined\BlackTextVersion
\NewCommandCopy\ManuscriptOriginalColor\color
\ExplSyntaxOn
\RenewDocumentCommand \color { o m }
  {
    \IfNoValueTF {#1}
      {
        \str_case:nnF {#2}
          {
            {red}  {\ManuscriptOriginalColor{black}}
            {blue} {\ManuscriptOriginalColor{black}}
          }
          {\ManuscriptOriginalColor{#2}}
      }
      {\ManuscriptOriginalColor[#1]{#2}}
  }
\ExplSyntaxOff
\AddToHook{env/tikzpicture/begin}{\let\color\ManuscriptOriginalColor}
\hypersetup{hidelinks}

\fi

\newif\ifincludeexponents
\includeexponentsfalse
\ifdefined\WithExponentSupplement
  \includeexponentstrue
\fi

\title{Counterfactual Probing for Parallel\\
Unmasking with Hidden Forest Structure}
\ifdefined\ArxivVersion
  \author{%
    Ryotaro Kawata\textsuperscript{1}\thanks{\href{mailto:kawata-ryotaro725@g.ecc.u-tokyo.ac.jp}{\nolinkurl{kawata-ryotaro725@g.ecc.u-tokyo.ac.jp}}}\qquad
    Satoshi Hayakawa\textsuperscript{1}\thanks{\href{mailto:hayakawa@mist.i.u-tokyo.ac.jp}{\nolinkurl{hayakawa@mist.i.u-tokyo.ac.jp}}}\qquad
    Taiji Suzuki\textsuperscript{1,2}\thanks{\href{mailto:taiji@mist.i.u-tokyo.ac.jp}{\nolinkurl{taiji@mist.i.u-tokyo.ac.jp}}}\\[0.5em]
    \normalsize\textsuperscript{1}The University of Tokyo\\
    \normalsize\textsuperscript{2}RIKEN AIP
  }
  \date{}
  \hypersetup{
    pdftitle={Counterfactual Probing for Parallel Unmasking with Hidden Forest Structure},
    pdfauthor={Ryotaro Kawata, Satoshi Hayakawa, Taiji Suzuki}
  }
\else
  \author{Anonymous Authors}
\fi

\begin{document}

\maketitle

\begin{abstract}
Masked generative models offer parallel token prediction, but accurate parallel
sampling must account for dependencies among tokens. When dependencies
are unknown, finding safe batches also costs model evaluations.
We study whether total evaluations, including discovery, can be sublinear
in sequence length $N$; sublinear sequential depth then follows.
We consider discrete distributions
with hidden forest structure, accessed through a fixed approximate conditional
oracle. Under explicit regularity conditions and uniform Hellinger error bounds,
for any fixed target accuracy $\TargetTolerance\in(0,1/8]$ and sufficiently
large $N$, our sampler achieves seed-averaged total-variation error at most
$\TargetTolerance$, with total masked-state submissions and sequential depth
both bounded by $\widetilde O(N^C\TargetTolerance^{-a})$ for constants
$0<C<1$ and $a>0$. These guarantees use polynomial vocabulary size and an
edge-response lower bound set by $N$ and $\TargetTolerance$.
The sampler shares evaluations of hypothetical reveals across dependence
tests to identify safe parallel batches without requiring full recovery
of the hidden forest. A tunable
parameter trades probing cost against irreversible commit rounds. In the same
class, any admissible irreversible product-commit sampler attaining
the same seed-averaged accuracy requires $\Omega(N^c\TargetTolerance^b)$ counterfactual submissions
or commit rounds in the worst case, for constants $c,b>0$.
\end{abstract}

\section{Introduction}
\label{sec:introduction}

Masked generative models predict missing tokens from a partially revealed
sequence, allowing many positions to be processed in one model
evaluation~\citep{austin2021structured,sahoo2024simple,nie2025large}.
Turning these predictions into a joint sample is less straightforward.
Drawing several tokens independently can miss dependencies between them.
With exact conditionals, revealing one token at a time samples exactly by
the chain rule, in any fixed order, but uses $N$ model evaluations in $N$
sequential steps for a sequence of length $N$.
The question is whether parallel unmasking can reduce this total cost.

Choosing which tokens to reveal together is itself part of the problem.
Confidence, entropy, and adaptive ordering guide practical
samplers~\textcolor{red}{\citep{chang2022maskgit,benhamu2025accelerated,kim2025train,hayakawa2026demystifying}}.
Yet the singleton predictions at one masked state do not, in general,
determine the dependence among its masked positions.
Evaluating hypothetical reveals can supply additional information:
\textcolor{red}{existing methods compare candidate continuations or test contextual
independence across positions
\citep{lee2025lookahead,azangulov2025parallel}.}
\textcolor{red}{These evaluations also contribute to the sampling cost.}
These methods do not establish a sublinear total-evaluation bound at fixed
sampling accuracy. We therefore ask:
\begin{quote}
\emph{Can a sampler discover enough dependence structure and generate an
accurate sample with $o(N)$ model evaluations in total, at fixed accuracy?}
\end{quote}

\textcolor{red}{To study this question, we consider distributions that factorize
over hidden forests, classical models of discrete dependence
\citep{chowliu1968dependence,tan2011markovforest}.
For a known forest, exact conditionals allow exact sampling in
$O(\log N)$ rounds by revealing one centroid per residual component.
Our goal is to exploit this conditional separation while accounting for
the cost of discovering safe batches.}

\textbf{\textcolor{red}{Key strategy: shared counterfactual probing.}}
\textcolor{red}{We compare singleton predictions from a fixed approximate
conditional oracle under hypothetical token fillings, without committing
the tested values
(Figure~\ref{fig:intro-counterfactual-20260920}).
Sharing evaluations across many dependence tests reduces discovery cost.
The resulting local structure guides generation, progressively simplifying
the remaining forest. To generate tokens, we independently sample a batch
from the oracle's singleton conditionals and commit it permanently.
We count evaluations used for both probing and generation.}

\begin{figure}[!t]
\centering
\includegraphics[width=\linewidth]{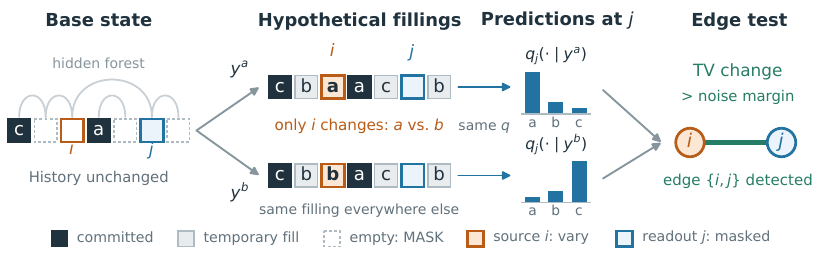}
\caption{Counterfactual edge probing; faint arcs show the hidden forest.
 Keeping the committed state unchanged,
 vary only source $i$ between two hypothetical fillings and compare the
 predicted rows at masked readout $j$. With the same fully filled background,
 a response above the oracle-noise margin detects $\{i,j\}$.
 Section~\ref{sec:main-packed-screen} packs and aggregates such tests.}
\label{fig:intro-counterfactual-20260920}
\end{figure}

\textbf{Main results, informally.}
\textcolor{red}{\emph{Under the target and oracle assumptions in
Section~\ref{sec:main-setup}, for any fixed sufficiently small accuracy
$\TargetTolerance>0$ and all sufficiently large $N$, there is a sampler
with seed-averaged total-variation (TV) error at most $\TargetTolerance$
using $\widetilde O(N^{2/3+1/(9\RankTailExponent)})$ evaluations in total,
which also gives sublinear sequential depth.
Here $\RankTailExponent>1$ is the rank-frequency exponent.
Conversely, every admissible irreversible product-commit sampler meeting
the same accuracy target uniformly over the class requires worst-case
budgets of $\Omega(N^c)$ probes or $\Omega(N^c)$ commit rounds for some
$c>0$ (Theorems~\ref{main:thm:upper-curve}--\ref{main:thm:lower-envelope}).}}

\textcolor{red}{The model assumes detectable edge responses, controlled tail responses,
a rank-frequency envelope, and polynomial vocabulary.
Its uniform oracle accuracy also suffices
for the $N$-step baseline to meet the output target
(Appendix~\ref{app:comparison-unmasking-errors}).}

Existing distribution-general samplers achieve smaller depth exponents
with sufficient total budgets of $O(N)$ or $O(N\log N)$
\citep{anari2024parallel,anari2026autospeculation}.
Dependence-adaptive guarantees can be sublinear on weakly dependent
instances \citep{li2025convergence,zhao2026adaptation}.
\textcolor{red}{Our guarantee also covers paths with constant dependence per
edge and total correlation $\Theta(N)$
(Proposition~\ref{app:prop:path-in-class}; Section~\ref{sec:related-work}).}

\textcolor{red}{Section~\ref{sec:main-algorithm} combines random-color probes,
majority certification of low-degree neighborhoods, singleton peeling of the
high-degree core, and parallel centroid commits, without first recovering
the original forest.}
A public degree cutoff trades probing work against commit rounds.

\section{Problem formulation}
\label{sec:main-setup}

\textbf{Public and hidden information.}
\textcolor{red}{The vocabulary and its order, class parameters, accuracy target,
and algorithm parameters are public: known to the sampler and fixed before
the hidden instance is chosen.}
The target, forest, exact marginals, ranks, and response witnesses are hidden.
For the main results, fix $0<\TargetTolerance\le1/8$ before taking
$N\to\infty$; we suppress dependence on $N$.
The appendix retains the general finite setup.

For laws $p,q$ on a finite alphabet, write
$\dTV(p,q)=\frac12\sum_a|p(a)-q(a)|$,
$\sqHellinger(p,q)=1-\sum_a\sqrt{p(a)q(a)}$, and
$\KLDivergence{p}{q}=\sum_a p(a)\log(p(a)/q(a))$ (natural logarithms).
Write $\Delta^\circ$ for the strictly positive probability simplex.

\subsection{Hidden forest and regularity}
\label{sec:main-target-law}
\label{sec:main-response}

\begin{definition}[Forest target and committed history]
\label{main:def:forest-target}
Let $\PositionSet=\{1,\ldots,N\}$ and let $\Vocab$ be a vocabulary of
size $\VocabSize$ with public order $\TokenOrder$.
The target $X\in\Vocab^N$ has a strictly positive law
$\TargetLaw\in\Delta^\circ(\Vocab^N)$ and a hidden undirected forest
$\TargetForest=(\PositionSet,\TargetEdges)$\textcolor{red}{~\citep{tan2011markovforest,bhattacharyya2023chowliu}}.
Write $\PosMarginal{i}$ for the marginal of $X_i$.
A committed history $\History=(G,x_G)$ leaves residual positions
$\ResidualPositions=\PositionSet\setminus G$ and forest
$\ResidualForest=\TargetForest[\ResidualPositions]$
(\cref{app:def:target-law,app:eq:history}).
\end{definition}

\begin{assumption}[Forest factorization]
\label{main:ass:forest-structure}
Choose one root per component, with root set $\TargetRoots$ and parent
$\ParentOf{j}$ for each nonroot. Assume the factorization
(\cref{app:ass:forest-structure}):
\begin{equation}
\TargetLaw(x)
 =\prod_{u\in\TargetRoots}\PosMarginal{u}(x_u)
  \prod_{j\notin\TargetRoots}
  \TargetLaw(X_j=x_j\mid X_{\ParentOf{j}}=x_{\ParentOf{j}}).
\label{main:eq:forest-factorization}
\end{equation}
\end{assumption}

\begin{definition}[Tail sets and directed response]
\label{main:def:response-quantities}
\textcolor{red}{For distinct positions $i,j$, call $i$ the \emph{source}
whose token value is varied and $j$ the masked \emph{readout} whose
conditional distribution is inspected. For a complete boundary}
$z\in\Vocab^{\PositionSet\setminus\{i,j\}}$, write
$\TailSet{i}{t}=\{a:\PosMarginal{i}(a)\le t\}$ and
$\BoundaryRow{i}{j}{a}{z}
 =\mathcal L_{\TargetLaw}(X_j\mid X_i=a,\,
 X_{\PositionSet\setminus\{i,j\}}=z)$.
The directed response is
$\DirectedResponse{i}{j}{z}
 =\max_{a,b\in\Vocab}\dTV(\BoundaryRow{i}{j}{a}{z},
 \BoundaryRow{i}{j}{b}{z})$ (\cref{app:def:response-quantities}).
\end{definition}

{\color{red}
\begin{assumption}[Regularity and scaling for the main results]
\label{main:ass:regularity-scaling}
\label{main:ass:response-regularity}
Let $\RankTailExponent>1$ and $\EdgeSignal\in(0,1]$ be public.
\begingroup
\setlength{\leftmargini}{2pc}
\begin{itemize}
\item \phantomsection\label[assumption]{main:ass:rt}
\textbf{(RT) Tail response.}
For all distinct $i,j$, $t\in[0,1]$, and complete boundaries $z$,
$\max_{a,b\in\TailSet{i}{t}}
 \dTV\bigl(\BoundaryRow{i}{j}{a}{z},\BoundaryRow{i}{j}{b}{z}\bigr)
 \le t$.
For tails of size at most one, take the maximum as zero.
Tail tokens are thus interchangeable within TV error $t$
(\cref{app:ass:response-regularity}).

\item \phantomsection\label[assumption]{main:ass:uen}
\textbf{(UEN) Edge nondegeneracy.}
Every edge
$\{u,v\}\in\TargetEdges$ and complete boundary $z$ satisfy
$\min\{\DirectedResponse{u}{v}{z},\DirectedResponse{v}{u}{z}\}
 \ge\EdgeSignal$.
The witnessing token pair may depend on the direction and boundary
(\cref{app:ass:response-regularity}).

\item \phantomsection\label[assumption]{main:ass:frequency-envelope}
\textbf{(RF) Rank-frequency envelope.}
Rank tokens $\RankedToken{i}{k}$ by decreasing $\PosMarginal{i}$
(ties use $\TokenOrder$). For every position $i$ and rank $k$,
$\PosMarginal{i}(\RankedToken{i}{k})
 \le k^{-\RankTailExponent}+\VocabSize^{-1}$
(\cref{app:ass:rate-regime}). No lower frequency bound or adjacent-rank
separation is assumed. This bounds tokens above thresholds exceeding
$\VocabSize^{-1}$ and, with RT, explicit token choices
(\cref{sec:main-algorithm-flow}).

\item \phantomsection\label[assumption]{main:ass:main-scaling}
\textbf{Scaling.}
For a public $\VocabGrowthExponent>1/3$, take
$\VocabSize=N^{\VocabGrowthExponent+o(1)}$
and set $\EdgeSignal=(\TargetTolerance^2/(32N))^{1/3}$
(\cref{app:cor:fixed-accuracy-main}).
\end{itemize}
\endgroup
\end{assumption}
}

\subsection{Oracle access and sampling objective}
\label{sec:main-oracle}

One oracle evaluation is one \emph{masked-state submission}; a
\emph{probe} is a noncommit counterfactual submission.
Hypothetical-reveal tests also appear in PUNT~\citep{azangulov2025parallel};
Definition~\ref{main:def:algorithm-risk-resources} charges probes to the sampling budget.

\begin{definition}[Frozen conditional oracle]
\label{main:def:frozen-oracle}
A masked state is $y\in(\Vocab\cup\{\MASK\})^N$ with observed set
$\ObservedSet{y}=\{i:y_i\ne\MASK\}$.
For a masked readout $j\notin\ObservedSet{y}$, its exact row is
$\ExactRow{y}{j}
 =\mathcal L_{\TargetLaw}(X_j\mid
 X_{\ObservedSet{y}}=y_{\ObservedSet{y}})$.
\begingroup\color{red}
The oracle $\FrozenOracle=\{\FrozenOracle_j\}_{j\in\PositionSet}$ is fixed
and deterministic (\cref{app:def:frozen-oracle}). One submission returns
probability vectors for a requested subset $S$ of masked positions:
\[
 (y,S)\ \xrightarrow{\text{one masked-state submission}}\quad
 \bigl\{\OracleRow{y}{j}\bigr\}_{j\in S},
 \qquad S\subseteq\PositionSet\setminus\ObservedSet{y}.
\]
Each row is the full vector
$\OracleRow{y}{j}=(\OracleRowMass{y}{j}{a})_{a\in\Vocab}\in\Delta^\circ(\Vocab)$;
the submission costs one evaluation, regardless of $|S|$.
\endgroup
\end{definition}

\begingroup\color{red}
\begin{definition}[Sampling operations, resources, and risk]
\label{main:def:algorithm-risk-resources}
An admissible algorithm uses $\FrozenOracle$ adaptively without remasking
committed positions (\cref{app:def:admissible-algorithm}).
At $\History=(G,x_G)$, both operations preserve $y_G=x_G$:
\begingroup
\setlength{\leftmargini}{2pc}
\begin{itemize}
\item \emph{Probe.} Submit $(y,S)$ with
$S\subseteq\PositionSet\setminus\ObservedSet{y}$ and inspect the probabilities
$\{\OracleRowMass{y}{j}{a}:j\in S,\ a\in\Vocab\}$ returned together.
The state $y$ may temporarily reveal residual coordinates;
only the transcript changes, while $(G,x_G)$ stays fixed.
\item \emph{Commit.} Choose a nonempty batch $B\subseteq\ResidualPositions$
from prior information, before the reply. Submit $(y,B)$ with all residual
positions masked, receive $\{\OracleRow{y}{j}\}_{j\in B}$, and sample
$z_B\sim\bigotimes_{j\in B}\OracleRow{y}{j}$ to commit permanently.
\end{itemize}
\endgroup
Batches eventually partition $\PositionSet$.
Committing counterfactual values or using verify-and-accept steps is not allowed.

\emph{Resources (\cref{app:def:resources}).}
Let $s=1,2,\ldots$ index the submissions actually performed, and let
$s_{\mathrm{pre}}$ index the distinguished preprocessing all-mask readout
($0$ if absent). Define
\[
\begin{aligned}
 \CounterfactualQueries
 &:=\bigl|\{s:s\ne s_{\mathrm{pre}},\ \text{submission }s\text{ is a probe}\}\bigr|,\\
 \CommitRounds
 &:=\bigl|\{s:\text{submission }s\text{ performs a nonempty product commit}\}\bigr|.
\end{aligned}
\]
The construction's total is $1+\CounterfactualQueries+\CommitRounds$,
including its one initial all-mask readout. Repeated states and parallel
submissions count separately; positions and tests sharing one submission do not.
The oracle depth $\OracleDepth$ counts sequential oracle stages; returned probabilities and
local computation are not charged separately.

\emph{Output law and risk.}
The algorithm's decision rule chooses submissions and batches from available
information and a seed $\ControllerSeed$ independent of the hidden instance
and primitive commit-draw randomness. Conditional on $\ControllerSeed=w$,
the output law $\SamplerOutputLaw{w}{\FrozenOracle}$ integrates the commit draws
(\cref{app:def:output-risk}). Its seed-averaged risks are
$\SeedRisk(\mathcal A;\TargetLaw,\FrozenOracle)
 =\mathbb E_{\ControllerSeed}
   \dTV(\TargetLaw,\SamplerOutputLaw{\ControllerSeed}{\FrozenOracle})$
and $\KLSeedRisk(\mathcal A;\TargetLaw,\FrozenOracle)
 =\mathbb E_{\ControllerSeed}\KLDivergence{\TargetLaw}
   {\SamplerOutputLaw{\ControllerSeed}{\FrozenOracle}}$
(\cref{app:def:kl-oracle-risk}).
\end{definition}
\endgroup

The risks average divergence over decision-rule seeds, before mixing their
output laws. Convexity gives
$\dTV(\TargetLaw,\mathbb E_{\ControllerSeed}
 \SamplerOutputLaw{\ControllerSeed}{\FrozenOracle})
 \le\SeedRisk(\mathcal A;\TargetLaw,\FrozenOracle)$:
the upper bound also controls the mixed output, while the lower bound
concerns the stated seed-averaged criterion.

\begin{assumption}[Uniform frozen-oracle accuracy: two cases]
\label{main:ass:oracle-accuracy}
For the fixed public accuracy $\TargetTolerance$, assume one of the
following bounds uniformly over all valid state--readout pairs.

\emph{\HellingerCaseTitle\ (Hellinger, A2; \cref{app:ass:oracle-accuracy}):}
$\sqHellinger(\ExactRow{y}{j},\OracleRow{y}{j})
 \le\TargetTolerance^2/(8N)$.

\emph{\KLCaseTitle\ (forward KL; \cref{app:def:kl-oracle-risk}):}
$\KLDivergence{\ExactRow{y}{j}}{\OracleRow{y}{j}}
 \le\TargetTolerance^2/(4N)$.
\end{assumption}
\textcolor{red}{Both cases imply row-TV error at most
$\TargetTolerance/(2\sqrt N)$, uniformly over counterfactual and commit states.
This uniform requirement is stronger than average prediction accuracy
of a trained model.}

\begin{definition}[Target--oracle classes]
\label{main:def:shared-rate-class}
\textcolor{red}{For each case, take all target--oracle tuples satisfying
\cref{main:ass:forest-structure,main:ass:regularity-scaling}
and the corresponding condition of \cref{main:ass:oracle-accuracy}.}
For \HellingerCase, denote this class by $\SharedClass$ and write
$(\InstanceTargetLaw{I},\InstanceOracle{I})$ for the target and oracle of
instance $I\in\SharedClass$ (\cref{app:cor:fixed-accuracy-main}).
\end{definition}

\section{Main results}
\label{sec:main-results}

\subsection{Sublinear sampling is possible}
\label{sec:main-upper}

The positive result controls total submissions and sequential depth.
The public algorithm parameter $\DegreeCutoff$ trades discovery submissions
against irreversible commit rounds; target degrees may be arbitrary.

\begin{theorem}[Existence of a sublinear oracle sampler]
\label{main:thm:upper-curve}
For either target--oracle class in Definition~\ref{main:def:shared-rate-class}
and every public integer $9\le\DegreeCutoff\le N-1$, there is an admissible
sampler $\PackedSampler$, depending only on the public class data and
$\DegreeCutoff$, with the following guarantees for all sufficiently large $N$.
Uniformly over the respective target--oracle class, on every execution path,
\begin{equation}
 \CounterfactualQueries
 =\widetilde O\!\left(
   \DegreeCutoff^2(N/\TargetTolerance^2)^{1/(3\RankTailExponent)}
 \right),
 \qquad
 \CommitRounds,\OracleDepth=\widetilde O(N/\DegreeCutoff).
 \label{main:eq:benchmark-upper}
\end{equation}
Here $\widetilde O$ hides powers of $\log N$, with $\TargetTolerance$ fixed.
In \HellingerCase,
$\sup_{I\in\SharedClass}
 \SeedRisk(\PackedSampler;\InstanceTargetLaw{I},\InstanceOracle{I})
 \le\TargetTolerance$.
In \KLCase, uniformly over its class,
$\KLSeedRisk(\PackedSampler;\TargetLaw,\FrozenOracle)\le\TargetTolerance^2$.
\end{theorem}

In particular, choosing
$\DegreeCutoff=\lceil(N\EdgeSignal^{1/\RankTailExponent})^{1/3}\rceil$
gives the common balanced bound
\begin{equation}
 1+\CounterfactualQueries+\CommitRounds,\ \OracleDepth
 =\widetilde O\!\left(
 N^{2/3+1/(9\RankTailExponent)}/\TargetTolerance^{2/(9\RankTailExponent)}
 \right)=o(N).
 \label{main:eq:balanced-upper}
\end{equation}
\textcolor{red}{Since $\RankTailExponent>1$, both total submissions and depth
are sublinear at fixed admissible accuracy.
Section~\ref{sec:main-algorithm} gives the construction and
\cref{app:cor:fixed-accuracy-main} the finite calibration and logarithmic factors.
Table~\ref{tab:resource-comparison} compares bounds under the respective
target and oracle assumptions.}
\begin{table}[!htb]
\centering
\caption{Sufficient bounds at fixed positive accuracy and polynomial vocabulary.
Mean: seed-averaged divergence; exp.: expected resources (ours are pathwise).
TC/DTC: total/dual total correlation; $\widehat T$: a supplied bound on either.
Path: Proposition~\ref{app:prop:path-in-class}.
Oracle conditions, accuracy dependence, and markers $a,b$:
Appendices~\ref{app:comparison}--\ref{app:oracle-comparison}.}
\label{tab:resource-comparison}
\setlength{\tabcolsep}{3pt}
\begin{tabularx}{\linewidth}{@{}>{\raggedright\arraybackslash}p{2.4cm}
 >{\raggedright\arraybackslash}X
 >{\raggedright\arraybackslash}p{2.85cm}
 >{\raggedright\arraybackslash}p{1.85cm}
 >{\raggedright\arraybackslash}p{1.45cm}@{}}
\toprule
Sampler / variant & Target / oracle & Submissions & Depth & Output\\
\midrule
Sequential chain rule & Arbitrary / exact & $N$ & $N$ & Exact\\
One product batch & Arbitrary / exact & $1$ & $1$ & \textcolor{red}{No error bound}$^a$\\
\midrule
\citet{anari2024parallel} & Arbitrary / exact
 & $O(N)$, exp. & $\widetilde O(N^{2/3})$ & Exact\\
\citet{anari2026autospeculation} & Arbitrary / noisy$^b$
 & $O(N\log N)$, exp. & $\widetilde O(N^{1/2})$ & TV\\
\citet{li2025convergence}
 & TC/DTC / averaged prediction error
 & $O(1+\mathrm{TC}+\mathrm{DTC})$ (path: $O(N)$) & same & Mean KL\\
\citet{chen2026optimal}
 & Supplied $\widehat T$ / exact
 & $\widetilde O(1+\widehat T)$ (path: $\widetilde O(N)$) & same & Mean KL\\
\midrule
This paper, \HellingerCase & Regular hidden forests / uniform Hellinger
 & $\widetilde O(N^{2/3+1/(9\RankTailExponent)})$ & same & Mean TV\\
This paper, \KLCase & Regular hidden forests / uniform forward KL
 & $\widetilde O(N^{2/3+1/(9\RankTailExponent)})$ & same & Mean KL\\
\bottomrule
\end{tabularx}

\end{table}

\subsection{A query--round lower bound}
\label{sec:main-lower}

The lower bound applies to every admissible sampler in the Hellinger/TV
class $\SharedClass$: each commit samples a product batch from the
current singleton predictions and permanently fixes its tokens.
Accuracy is measured by the seed-averaged TV risk in
Definition~\ref{main:def:algorithm-risk-resources}.

\begin{theorem}[Fixed-accuracy query--round lower bound]
\label{main:thm:lower-envelope}
For all sufficiently large even $N$, every admissible algorithm with
\textcolor{red}{deterministic bounds
$\CounterfactualQueries\le\QueryBudget$ and $\CommitRounds\le\RoundBudget$
on every instance, seed, transcript, and execution path} and
$\sup_{I\in\SharedClass}
 \SeedRisk(\mathcal A;\InstanceTargetLaw{I},\InstanceOracle{I})
 \le\TargetTolerance$ satisfies
\begin{equation}
 \QueryBudget
 =\Omega\!\left(
     (N/\TargetTolerance^2)^{1/(3\RankTailExponent)}
   \right)
 \quad\text{or}\quad
 \RoundBudget
 =\Omega\!\left(
     N^{1/3}\TargetTolerance^{1/3}
   \right).
 \label{main:eq:canonical-lower}
\end{equation}
The constants may depend on $\RankTailExponent$, but not on the hidden
instance or the algorithm. The sufficiently-large-$N$ threshold may
depend on the fixed public accuracy $\TargetTolerance$.
\end{theorem}

\textbf{Lower-bound construction.}
The hard family is a hidden matching whose oracle reveals dependence
only when a source is set to its private trigger token.
One submitted state tests at most one candidate per source, even when
many readout rows are returned.
The proof converts collisions of unresolved matched endpoints within
product-commit batches into a TV lower bound, showing that accuracy
requires enough submissions or commit rounds.
\textcolor{red}{The hard family uses
$\EdgeSignal\asymp(\TargetTolerance^2/N)^{1/3}$: its nontrigger
squared-Hellinger error is of order $\EdgeSignal^3\asymp\TargetTolerance^2/N$.
Thus the signal floor vanishes as $N^{-1/3}$ at fixed accuracy;
changing $\TargetTolerance$ also changes the target class.}
\Cref{app:cor:signal-calibrated-finite-lower,app:cor:fixed-accuracy-main}
give the finite bound and its specialization to
\cref{main:thm:lower-envelope}.

\section{Algorithmic construction}
\label{sec:main-algorithm}

\newcommand{\MainAlgoBlockTarget}[1]{%
  \hypertarget{main-algo-block-#1}{\AlgoBlockHeading{#1}}}
\newcommand{\MainAlgoInlineTarget}[1]{%
  \hypertarget{main-algo-block-#1}{%
    \textbf{\AlgoStepPrefix~\AlgoBlockCode{#1}. \AlgoBlockName{#1}.}}}
\newcommand{\MainAlgoShortTarget}[1]{%
  \hypertarget{main-algo-block-#1}{\textbf{\AlgoStepPrefix~\AlgoBlockCode{#1}.}}}
\DeclareRobustCommand{\MainAlgoBlockRef}[1]{%
  \hyperlink{main-algo-block-#1}{\AlgoStepPrefix~\AlgoBlockCode{#1}}}

{\color{red}
We now construct the sampler in \cref{main:thm:upper-curve}.
The key is to learn the residual forest as sampling removes vertices.
\MainAlgoBlockRef{draft}--\MainAlgoBlockRef{screen} share counterfactual
probes to recover low-degree neighborhoods with high probability.
When all screens succeed, \MainAlgoBlockRef{peel} uses these reports to
shrink the high-degree core by singleton commits, and
\MainAlgoBlockRef{terminal}--\MainAlgoBlockRef{centroid} recover the
remaining forest for parallel centroid commits
(Algorithm~\ref{main:alg:upper-bound-sampler}).

\begin{algorithm}[H]
\caption{\textcolor{red}{Counterfactual-probing sampler: shared tests, claim-based peeling, and centroid commits. Full safeguards: \cref{app:alg:complete-sampler}.}}
\label{main:alg:upper-bound-sampler}
{\color{red}\small
\algrenewtext{Loop}{\textbf{repeat}}
\algtext*{EndLoop}
\algtext*{EndFor}
\begin{algorithmic}[1]
\Statex \textbf{Input:} $N$, ordered vocabulary $\Vocab$, frozen oracle $\FrozenOracle$, accuracy $\TargetTolerance$
\Statex \textbf{Parameters:} $\DegreeCutoff,\ChunkCount,\ThresholdFloor,\TailTolerance,p,\RandomColoringCount$
  (\cref{app:cor:fixed-accuracy-main})
\State \MainAlgoInlineTarget{preprocess}
\Loop
  \State \MainAlgoInlineTarget{draft} Fix a draft; draw fresh independent colorings
  \State \MainAlgoInlineTarget{screen}
  \For{each coloring $\tau=1,\ldots,\RandomColoringCount$}
    \State \MainAlgoShortTarget{chunks} Split each readout color class into chunks
    \For{each nonempty chunk $\mathcal C$ and other source color}
      \For{each bank-or-tail column $\kappa$}
        \State \MainAlgoShortTarget{probes} Submit the packed state; retain replies at $j\in\mathcal C$
      \EndFor
    \EndFor
  \EndFor
  \State \MainAlgoShortTarget{aggregate} Compute votes, strict majorities, and capped rows $\ScreenRow{j}$
  \State \MainAlgoInlineTarget{peel} Form $\PeelSet$; if empty or at the phase cap, \textbf{break}
  \State Freeze $\PeelSet$ and commit its vertices singly in public order
\EndLoop
\State \MainAlgoInlineTarget{terminal}\enspace \MainAlgoInlineTarget{centroid}
\end{algorithmic}
}
\end{algorithm}

\begin{figure}[H]
\centering
\includegraphics[width=\linewidth]{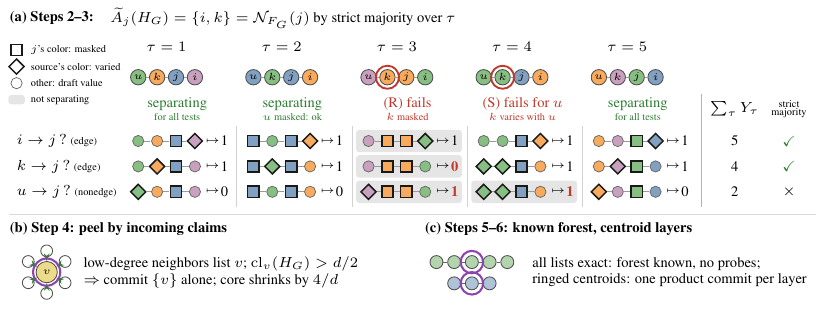}
\caption{\textcolor{red}{Shared tests and majority recovery.
(a) Each cell represents all bank/tail columns; each readout color has one
chunk. Equal-color pairs (no diamond) get vote $0$ without probing.
Grey: nonseparation; red: illustrative errors.
(b) Singleton peels contract the core per phase.
(c) One centroid per component forms a terminal batch.}}
\label{fig:algorithm-overview-20260918}
\end{figure}

\subsection{Building a single-source edge test}
\label{sec:main-algorithm-flow}

We prepare the edge test that \cref{sec:main-packed-screen} will share across pairs.
Fix a history $\History=(G,x_G)$, distinct residual positions $i,j$, and
a query-free draft $\DraftFiller$.
The single-source state $y^a$ reveals $x_G$, sets
$y_i^a=a$, masks $j$, and fixes all other residual coordinates to their draft
values (Figure~\ref{fig:intro-counterfactual-20260920}).
Its exact reply $\BoundaryRow{i}{j}{a}{z}$ has boundary
$z=(x_G,\DraftFiller_{\ResidualPositions\setminus\{i,j\}})$.
For a nonedge this row is constant in $a$ by forest separation
(\cref{app:lem:residual-markov});
for an edge its TV diameter is at least $\EdgeSignal$ by
\cref{main:ass:uen} (UEN).
Testing all ordered pairs and tokens separately costs
$|\ResidualPositions|(|\ResidualPositions|-1)\VocabSize$ submissions.

\textbf{Banks and tail representatives (\MainAlgoBlockRef{preprocess}).}
To reduce token choices, we test a local vocabulary bank and one
representative of the omitted tokens.
Write $\RowTVError:=\TargetTolerance/(2\sqrt N)$ for the row-TV error
bound in either case (\cref{app:eq:oracle-tv-radius,app:eq:kl-oracle-inclusion}).
On the threshold grid $\ThresholdGrid$ with floor $\ThresholdFloor$,
choose $\TailThreshold=\max\{u\in\ThresholdGrid:u\le\TailTolerance\}$,
the largest threshold allowed by the tail budget.
One all-mask submission gives the local vocabulary banks
$\TokenBank{i}:=\{a\in\Vocab:\OracleRowMass{\MASK^N}{i}{a}
\ge\TailThreshold-\RowTVError\}$ and tail representatives
$\TailRepresentative{i}=\min_{\TokenOrder}(\Vocab\setminus\TokenBank{i})$
(\cref{app:def:preprocessing}).
The standard draft reuses these rows without another submission
(\cref{app:eq:standard-draft}).

\cref{main:ass:rt} (RT) allows each omitted token to be replaced by $\TailRepresentative{i}$
within TV error $\TailTolerance=\EdgeSignal/2$.
\cref{main:ass:frequency-envelope} (RF) bounds
$\MaxBankSize=\max_i|\TokenBank{i}|=O(\EdgeSignal^{-1/\RankTailExponent})$.

\textbf{A resolvable edge signal.}
The test queries every bank token and the tail representative, then compares
$\max_{a,a'\in\TokenBank{i}\cup\{\TailRepresentative{i}\}}
\dTV(\OracleRow{y^a}{j},\OracleRow{y^{a'}}{j})$ with $2\RowTVError$.
This observed diameter is at most $2\RowTVError$ for a nonedge;
for an edge it is at least
$\EdgeSignal-\TailTolerance-2\RowTVError>2\RowTVError$
when $\ScreenResolution:=4\RowTVError+\TailTolerance<\EdgeSignal$.
Under the scaling in \cref{main:ass:main-scaling}, our parameter choices
make preprocessing feasible and ensure $\ScreenResolution<\EdgeSignal$
for all sufficiently large $N$ (\cref{app:cor:fixed-accuracy-main}).
\AppLowDegreeScreenRef\ proves the diameter bound.

\subsection{Packing tests with random colorings}
\label{sec:main-packed-screen}

We share the tests of \cref{sec:main-algorithm-flow} across source and readout groups.
For each coloring, \MainAlgoBlockRef{chunks} splits color classes into
chunks of at most $\lceil N/\ChunkCount\rceil$ positions.
\MainAlgoBlockRef{probes} masks one chunk and varies another color class
through bank-or-tail columns $\kappa$: each source takes its own token value,
while committed values and the remaining draft stay fixed.
There are $\ScreenColumnCount\le\MaxBankSize+1$ submitted columns;
\cref{app:eq:packed-column,app:eq:screen-column-count} specify padding
and the empty-bank case.
The same replies test every source--readout pair between the two groups.
\MainAlgoBlockRef{draft} draws $\RandomColoringCount$ independent colorings
with independent uniform colors $\ColorHash{\tau}{i}\in[p]$, $p=8(\DegreeCutoff+1)$.
All screen probes are fixed before replies: one oracle stage
(\cref{app:alg:packed-screen}).
Coloring also shares evaluations in sparse Jacobian
estimation~\citep{coleman1983estimation}.

For $i\to j$, the single-source test is preserved if every neighbor of $j$
is revealed and all except $i$ stay fixed. Call a coloring \emph{separating} when:
\begingroup
\setlength{\leftmargini}{2pc}
\begin{itemize}
\item \emph{Readout separation (R).} No neighbor of $j$ has $j$'s color,
so none is masked.
\item \emph{Source separation (S).} No neighbor of $j$ other than $i$ has
$i$'s color, so none varies with $i$.
\end{itemize}
\endgroup
For separating colorings with distinct source/readout colors, exact replies
equal their single-source counterparts (\cref{app:lem:fixed-draft-rows}).
Figure~\ref{fig:algorithm-overview-20260918}(a) shows possible collision errors
and majority voting.

\textbf{Voting and capped candidate lists (\MainAlgoBlockRef{aggregate}).}
Each coloring casts one vote on whether $i$ and $j$ are adjacent, using
the shared replies from \MainAlgoBlockRef{probes}.
For differently colored $i,j$, compare the conditional distributions
returned at $j$ across the columns probing $i$'s color class.
Let $\ScreenDiameter{\tau}{i}{j}$ be their maximum pairwise TV distance
(\cref{app:eq:observed-screen-diameter}).
Applying the threshold from \cref{sec:main-algorithm-flow} gives the vote
$\ScreenVote{\tau}{i}{j}
=\mathbf1\{\ScreenDiameter{\tau}{i}{j}>2\RowTVError\}$.
Equal-color pairs have no corresponding probe and receive vote zero.
We retain candidates supported by a strict majority of all
$\RandomColoringCount$ colorings:
\[
\CandidateRow{j}:=\left\{i\in\ResidualPositions\setminus\{j\}:
\sum_{\tau=1}^{\RandomColoringCount}\ScreenVote{\tau}{i}{j}>
\RandomColoringCount/2\right\}.
\]
The returned list $\ScreenRow{j}$ equals $\CandidateRow{j}$ when it contains
at most $\DegreeCutoff$ candidates; otherwise, it retains the $\DegreeCutoff$
candidates with the highest vote totals, breaking ties by public position order.
Random hashing followed by majority aggregation also appears in junta
learning~\citep{bshouty2018exact}.

\textbf{Low-degree recovery (\MainAlgoBlockRef{draft}--\MainAlgoBlockRef{screen}).}
For a readout $j$ with $d_{\ResidualForest}(j)\le\DegreeCutoff$,
conditions (R) and (S) forbid at most $2\DegreeCutoff$ color equalities.
Conditional on the past and the fixed draft, each equality has probability
$1/p$. Thus, each fresh coloring separates any fixed pair $(i,j)$ with
probability at least $1-2\DegreeCutoff/p>3/4$.

By the single-source margin $\ScreenResolution<\EdgeSignal$ from
\cref{sec:main-algorithm-flow}, every separating coloring gives the correct
edge/nonedge vote.
This includes same-color separating pairs: condition (R) excludes an edge,
so their zero vote is correct.
A strict separating majority therefore classifies the pair correctly,
regardless of the remaining votes.
Let $\ScreenSuccess{\History}$ denote the analysis-only event that such a
majority holds for every $i\ne j$ with
$d_{\ResidualForest}(j)\le\DegreeCutoff$.
On this event, each low-degree candidate list equals the true neighborhood.
Its size is already at most $\DegreeCutoff$, so truncation leaves it
unchanged (\cref{app:lem:low-degree-screen}):
\begin{equation}
 \ScreenRow{j}=\mathcal N_{\ResidualForest}(j)
 \quad\text{for every }j\in\ResidualPositions
 \text{ with }d_{\ResidualForest}(j)\le\DegreeCutoff.
 \label{main:eq:exact-low-degree-row}
\end{equation}

By independence across colorings, Hoeffding's inequality bounds the
probability that a fixed pair lacks a strict separating majority by
$e^{-\RandomColoringCount/8}$. A union bound over pairs gives
\begin{equation}
 \Pr\bigl(\ScreenSuccess{\History}^{\mathsf c}
       \mid\text{past including the current draft}\bigr)
 \le N^2e^{-\RandomColoringCount/8}.
 \label{main:eq:conditional-screen-success}
\end{equation}
Fresh colors make this bound valid at every adaptively reached screen.
A further union bound over the capped number of screen calls controls
failure along the full execution (\cref{app:lem:random-screen-success}).

\subsection{Peeling the high-degree core, then centroid layers}
\label{sec:main-peeling-terminal}

\MainAlgoBlockRef{peel} commits vertices with many incoming neighbor reports singly.
When all screens succeed, repeated peeling removes the high-degree core,
enabling forest recovery and parallel centroid commits
(\MainAlgoBlockRef{terminal}--\MainAlgoBlockRef{centroid}).

Define the incoming claim count and the peel set
(also \cref{app:eq:claim-and-peel-set}) by
\begin{equation}
 \ClaimCount{v}
 :=|\{u\in\ResidualPositions:v\in\ScreenRow{u}\}|,
 \qquad
 \PeelSet:=\{v\in\ResidualPositions:\ClaimCount{v}>\DegreeCutoff/2\}.
 \label{main:eq:claim-and-peel-set}
\end{equation}

\textbf{Why peel by incoming claims (\MainAlgoBlockRef{peel})?}
A high-degree vertex's own report may be unreliable, but each low-degree
neighbor reports it correctly.
Peeling may also select low-degree vertices; singleton
commits introduce no within-batch dependence error.
The size cap controls their cost: on successful screens,
$|\ScreenRow{u}|\le d_{\ResidualForest}(u)$, so
\[
 (\DegreeCutoff/2)|\PeelSet|
 \le\sum_v\ClaimCount{v}
 =\sum_u|\ScreenRow{u}|
 \le\sum_u d_{\ResidualForest}(u)<2|\ResidualPositions|.
\]
Thus a phase uses fewer than $4|\ResidualPositions|/\DegreeCutoff$
singleton rounds.

\textbf{The high-degree core shrinks.}
Let $\HighDegreeSet=\{v:d_{\ResidualForest}(v)>\DegreeCutoff\}$.
On a successful screen, an unpeeled high-degree vertex has at most $\DegreeCutoff/2$ low-degree
neighbors, since each contributes a claim; it therefore has more than
$\DegreeCutoff/2$ neighbors in $\HighDegreeSet$.
The forest induced on $\HighDegreeSet$ has degree sum below
$2|\HighDegreeSet|$, so fewer than
$(4/\DegreeCutoff)|\HighDegreeSet|$ high-degree vertices remain unpeeled
when $\HighDegreeSet$ is nonempty.
Deleting vertices cannot raise degrees; the next history $\History'$ satisfies
\begin{equation}
 |\HighDegreeSetAt{\History'}|
 \le(4/\DegreeCutoff)|\HighDegreeSet|.
 \label{main:eq:peeling-contraction}
\end{equation}

\textbf{Finish without further probes (\MainAlgoBlockRef{terminal}--\MainAlgoBlockRef{centroid}).}
On successful-screen paths, the $O(\log N)$ phase cap or an empty peel set
ensures residual maximum degree at most $\DegreeCutoff$
(\cref{app:lem:successful-path-structure}).
Since screening precedes either test,
\cref{main:eq:exact-low-degree-row} makes every terminal row exact.
Joining reported neighbors therefore recovers the true residual forest.
The exact joint conditional of one centroid per component factorizes.
Centroid deletion halves component sizes, giving logarithmically many layers
(\cref{app:lem:residual-markov,app:lem:forest-basics}).
These layers form a vertex ranking~\citep{iyer1988optimal}, with higher ranks
assigned to earlier commits.
Only commit submissions remain: later steps delete vertices of the known
forest rather than rediscovering it.

\subsection{Resources and accuracy}
\label{sec:main-algorithm-resources}

Each screen uses at most
$\ScreenColumnCount\RandomColoringCount p(p+\ChunkCount)$ submissions:
columns, colorings, source colors, and readout chunks.
With $\ChunkCount=\DegreeCutoff$, the bank bound, phase count, and peel-size
bound yield
\[
 \CounterfactualQueries
 =\widetilde O(\EdgeSignal^{-1/\RankTailExponent}\DegreeCutoff^2),
 \qquad
 \CommitRounds,\OracleDepth=\widetilde O(N/\DegreeCutoff).
\]
Larger $\DegreeCutoff$ increases screening cost but reduces commit rounds;
balancing gives
\cref{main:eq:balanced-upper}.
Public caps enforce the bounds even on failed-screen paths
(\cref{app:lem:upper-resources}); \cref{app:alg:complete-sampler} specifies
singleton cycle repair and the round-cap all-residual product-commit fallback.

On paths where all screens succeed, the target conditional factorizes over
every commit batch.
The appendix bounds oracle error and the contribution of failed screens
using separate TV and forward-KL arguments
(\cref{app:lem:exact-hybrid,app:lem:adaptive-hellinger-upper,app:lem:safe-controller-comparison,app:thm:finite-kl-upper}).
The calibration in \cref{app:cor:fixed-accuracy-main} gives the guarantees
of \cref{main:thm:upper-curve} for the sampler in
Algorithm~\ref{main:alg:upper-bound-sampler}.
}

\section{Related work}
\label{sec:related-work}

{\color{red}
\textbf{Masked generation and dependence-aware decoding.}
Under standard absorbing masks, the denoising target can be expressed using
time-independent clean-data conditionals
\citep{ou2024your,zheng2025timeagnostic}, motivating our frozen-oracle model.
Most closely related, PUNT~\mbox{\citep{azangulov2025parallel}} uses hypothetical
reveals and shares evaluations across contextual-independence tests, but does
not establish a sublinear total-evaluation guarantee or an end-to-end
sampling-error bound. Our random-color screen instead certifies local
neighborhoods and yields sublinear total submissions under the hidden-forest
target--oracle assumptions of Theorem~\ref{main:thm:upper-curve}.
\Citet{fu2025bits} lower-bound rounds under a confidence-threshold restriction;
our lower bound trades probes against commit rounds over the full admissible
interface.

\textbf{Parallel sampling guarantees.}
Distribution-general conditional-oracle samplers provide exact- and
noisy-oracle guarantees with smaller depth exponents
\citep{anari2024parallel,anari2026autospeculation};
Table~\ref{tab:resource-comparison} compares their row-query bounds in our
submission units. Dependence-based schedule analyses control error through
TC/DTC, effective TC, or adaptive estimates
\citep{li2025convergence,chen2026optimal,dmitriev2026efficient,zhao2026adaptation}.
Our class contains a path on which TC, DTC, and effective TC are all
$\Theta(N)$, while the construction exploits local conditional separation.
Appendices~\ref{app:comparison}--\ref{app:oracle-comparison} give the source
conditions and path calculations; these substitutions compare sufficient
upper bounds, not lower bounds on those samplers.

\textbf{Learning hidden structure.}
Structure learning uses subset queries~\citep{angluin2008hidden,abasi2019edge},
samples~\citep{chowliu1968dependence,dasarathy2016active}, or Gaussian
covariance entries~\citep{lugosi2021covariancequeries}.
Tree Ising laws can be learned in global TV without exact tree
recovery~\citep{daskalakis2021treeising}.
Our conditional-row probes yield local certificates and charge discovery to
sampling; Appendix~\ref{app:related-work-supplement} gives further
historical context.
}

\section{Finite-size experiments}
\label{sec:main-experiments}

We evaluate four hidden forest families with an exact frozen conditional
oracle at five sizes from $N=8192$ to $16384$.
\textcolor{red}{Figure~\ref{fig:ideal-fixed-accuracy} compares total oracle submissions}
under the fixed empirical batch-KL criterion
$K+2\operatorname{SE}_{\mathrm{MC}}\le10^{-10}L_0$,
\textcolor{red}{where $K$ sums product-to-joint batch KL along the sampled history,
not the target-to-output risk. The family-specific reference $L_0$ is the
one-batch KL cost at $N=2048$ (Appendix~\ref{app:experiments}).
We use empirical tuning; the targets and parameters need not satisfy
the theorem's sufficient conditions.}
\textcolor{red}{Our sampler's submission counts are consistent with sublinear
scaling over the tested finite-size range. These exact-oracle experiments isolate
submission cost; learned-model accuracy and runtime are outside their scope.}

\begin{figure}[H]
 \centering
 \ifdefined\ArxivVersion
  \includegraphics[width=0.88\linewidth]{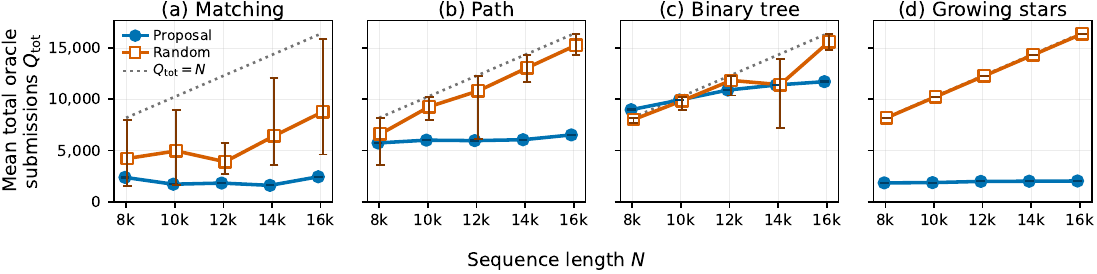}
 \else
  \includegraphics[width=0.88\linewidth]{../../../experiments/forest_toy/results/ideal_capless_fixed_accuracy_scaling/random_per_run_threshold_search_k1em10L0/ideal_strict_fixed_accuracy_linear_errorbars_1x4_large_text.pdf}
 \fi
 \caption{\textcolor{red}{Total oracle submissions
 $Q_{\mathrm{tot}}=Q_{\mathrm{pre}}+\CounterfactualQueries+\CommitRounds$,
 including preprocessing, probes, and nonempty commit rounds.
 Points: six-run means; whiskers: min--max ranges
 (not standard errors or confidence intervals).
 Proposal settings are fixed per forest and $N$;
 random budgets are selected per run on coupled curves.
 Proposal ranges ($\le78$ submissions) may lie inside markers.
 Dotted line: exact singleton $Q_{\mathrm{tot}}=N$, not a cap.
 Horizontal ticks: $1\mathrm{k}=1024$.}}
 \label{fig:ideal-fixed-accuracy}
\end{figure}

\section{Conclusion}
\label{sec:conclusion}

\textcolor{red}{We establish sublinear parallel unmasking under hidden-forest
assumptions, including dependence-discovery costs. The main practical
challenge is to turn shared probing into computational savings with learned
denoisers. Open directions include closing the probe--commit tradeoff gap
for irreversible sampling, understanding whether remasking improves this
tradeoff when resampling costs are counted, and extending the guarantees
beyond forests.}

\label{proposal:main-text-end}
\clearpage
\ifdefined\ArxivVersion
  \section*{Acknowledgments}

RK was partially supported by JSPS KAKENHI (24K02905) and JST BOOST
(JPMJBS2418).

SH was partially supported by JSPS KAKENHI (26K25548) and JST PRESTO
(JPMJPR26K8), and received access to ChatGPT for Academic Researchers program.

TS was partially supported by JST CREST (JPMJCR2115) and JSPS KAKENHI
(25H01107). This work was supported by JST ERATO Grant Number JPMJER2601.
This research is supported by the National Research Foundation,
Singapore and the Ministry of Digital Development and Information under the
AI Visiting Professorship Programme (award number AIVP-2024-004).
Any opinions, findings and conclusions or recommendations expressed in this
material are those of the author(s) and do not reflect the views of National
Research Foundation, Singapore and the Ministry of Digital Development and
Information.

\fi
\section*{AI use statement}
\label{sec:ai-use}

Generative AI tools were used extensively during this work, not only for
language editing.  GPT-5.6 Sol and GPT-6 Astra were used throughout the
research and writing process.  \textcolor{red}{Claude Fable 5.1 and
Claude Opus 5.5 were additionally used for writing assistance and the generation of
Figures~\ref{fig:intro-counterfactual-20260920}
and~\ref{fig:algorithm-overview-20260918}.}
Iterative brainstorming and technical discussions
with the GPT models contributed to the development of central proof ideas.
The GPT models also assisted in developing and refining mathematical arguments,
making proof steps explicit and rigorous, and drafting and revising
the proofs.  \textcolor{red}{AI tools also assisted with experimental design.}
The GPT models were also used to write code for the numerical
experiments.  Further assistance included \textcolor{red}{literature search,
translation, and the preparation of scientific figures, as well as}
manuscript organization,
expository revision, and \LaTeX{} preparation.
AI assistance was also used to formalize the manuscript's named mathematical
claims in Lean~4, whose kernel checked the resulting proofs.
The correspondence between the formal statements and the manuscript was
audited separately.
The authors verified the proofs and cross-checked the cited sources.
The authors take responsibility for the final content of this work,
including its mathematical claims, proofs, citations, and AI-assisted text.

\ifdefined\ArxivVersion
  \section*{Code availability}
\label{sec:code-availability}

The Lean~4 formalization, verification records, and experiment code and saved
results are available in the \texttt{arxiv-v1} snapshot at
\url{https://github.com/ryotaro-kawata-wa/counterfactual-parallel-unmasking}.
The repository includes pinned dependencies and reproduction instructions.

\fi

\bibliography{src/iclr2027_conference}
\ifdefined\ArxivVersion
  \bibliographystyle{plainnat}
\else
  \bibliographystyle{iclr2027_conference}
\fi

\appendix
\section{Full problem setup}
\label{app:setup}

We use the size convention of Section~\ref{sec:main-setup}: $N$ is fixed
within each finite statement, and size dependence is suppressed, including
in $\TargetLaw,\VocabSize,\EdgeSignal$ and the accuracy parameters.
In rate statements, parameters declared constant are independent of $N$;
every public sequence and its deterministic remainder is fixed uniformly
over the class before any hidden target or oracle is chosen.

\paragraph{Reading route to the main claims.}
The formal setup below fixes the objects used by all proofs.
The following routes give the dependencies of the main claims:
\begin{itemize}
 \item \emph{Lower bound in \HellingerCase, \cref{main:thm:lower-envelope}.}
 Appendix~\ref{app:results} defines the matching family and states its
 finite lower bound.  Appendix~\ref{app:lower-proof} proves the bound and
 verifies class inclusion in \cref{app:prop:hard-family-inclusion}.
 Then the adjacent
 \cref{app:cor:signal-calibrated-finite-lower,app:cor:fixed-accuracy-main}
 in Appendix~\ref{app:benchmark-specializations} give the finite
 calibration and the fixed-accuracy deduction of the main theorem.
 \item \emph{Upper bounds in both cases, \cref{main:thm:upper-curve}.}
 Appendix~\ref{app:algorithm} specifies the executable sampler.
 Appendix~\ref{app:upper-proof} proves the shared screen, peeling, and
 resource bounds, and the TV guarantee in \HellingerCase\
 (\cref{app:thm:finite-upper}).
 Appendix~\ref{app:kl-extension} proves the forward-KL guarantee in
 \KLCase\ (\cref{app:thm:finite-kl-upper}), using
 \cref{app:lem:adaptive-joint-kl}.
 \Cref{app:lem:benchmark-feasibility,app:cor:fixed-accuracy-main}
 verify the finite calibration and derive \cref{main:thm:upper-curve}
 with accuracy fixed before $N\to\infty$.
 \ifincludeexponents
 \Cref{app:cor:benchmark-hellinger,app:cor:benchmark-kl} give the further
 joint-limit consequences when accuracy also varies with $N$.
 \fi
 \item \emph{Comparison in Table~\ref{tab:resource-comparison}.}
 Appendix~\ref{app:comparison} proves the row-to-state accounting and
 the path's TC/DTC scales; Appendix~\ref{app:oracle-comparison} records
 the assumptions of the source results being substituted.
 \item \emph{Finite-size experiments, Section~\ref{sec:main-experiments}.}
 Appendix~\ref{app:experiments} gives the targets, accuracy test, tuning,
 resource accounting, and evaluation protocol.
\end{itemize}
The elementary auxiliary facts in Appendix~\ref{app:auxiliary} are
referenced at their point of use.  That appendix also collects the
supplementary selector-invariance result.
\ifincludeexponents
Appendix~\ref{app:independent-signal} gives the additional rates when
the signal floor is specified independently; these are not needed for
the two fixed-accuracy main theorems.
\fi

Short intuition paragraphs precede the main proof steps.
There, $\simeq$ compares scales up to fixed positive multiplicative
constants in the stated regime.  An upper bound alone is written with
$\lesssim$ or an explicit inequality, and subpolynomial factors are
displayed when relevant.  All formal statements retain their finite
constants and domains.

\begin{proposalcontext}{Supplementary context}
Unnumbered remarks give supplementary context beyond the proof route above.
Assumptions, resource definitions, conditioning arguments, and restrictions
needed to apply a result are stated in the main exposition.
\end{proposalcontext}

For each $N\ge 3$, let $\PositionSet:=\{1,\ldots,N\}$ be the set of positions and
let $\Vocab$ be a finite vocabulary of size $\VocabSize\ge 2$, equipped with
a public total order $\TokenOrder$.  For a finite set $\mathcal A$, write
$\Delta(\mathcal A)$ for its probability simplex and
$\Delta^\circ(\mathcal A)$ for the strictly positive part.  For
\(p,q\in\Delta(\mathcal A)\), we use the normalizations
\begin{equation}
 \dTV(p,q)
 :=\frac12\sum_{a\in\mathcal A}|p(a)-q(a)|,
 \qquad
 \sqHellinger(p,q)
 :=1-\sum_{a\in\mathcal A}\sqrt{p(a)q(a)}.
 \label{app:eq:distances}
\end{equation}
For \(q\in\Delta^\circ(\mathcal A)\), also write
\begin{equation}
 \ChiSquareDivergence{p}{q}
 :=\sum_{a\in\mathcal A}\frac{(p(a)-q(a))^2}{q(a)}.
 \label{app:eq:chi-square-divergence}
\end{equation}
Also define $\KLDivergence{p}{q}:=\sum_a p(a)\log(p(a)/q(a))$, with
natural logarithms and the convention $0\log(0/q(a))=0$.
For \(x\in\mathbb R\), let \(\PositivePart{x}:=\max\{x,0\}\).

All quantities explicitly declared public below, together with the algorithm
and the law of its internal random seed, are fixed before the target instance
is chosen.  The hidden target law and, subsequently, the frozen oracle may be
chosen adversarially.  Local vocabulary banks $\TokenBank{i}$, tail
representatives, drafts, and graph certificates are constructed from oracle
replies by the algorithm.

\subsection{Target law and hidden forest}
\label{app:setup-target}

We expand \cref{main:def:forest-target,main:ass:forest-structure}.

\begin{definition}[Forest-structured target law]
\label{app:def:target-law}
Let $X=(X_1,\ldots,X_N)\in\Vocab^N$ have a hidden strictly positive law
$\TargetLaw\in\Delta^\circ(\Vocab^N)$, and let
$\TargetForest=(\PositionSet,\TargetEdges)$ be a hidden undirected forest.  The exact position
marginal is
\begin{equation}
 \PosMarginal{i}:=\mathcal L_{\TargetLaw}(X_i)
 \in\Delta^\circ(\Vocab),
 \qquad i\in\PositionSet.
 \label{app:eq:position-marginal}
\end{equation}
For each orientation $i\to j$ of an edge $\{i,j\}\in\TargetEdges$, let
$\EndpointKernel{i}{j}(\cdot\mid a)\in\Delta^\circ(\Vocab)$ be an endpoint
kernel for every $a\in\Vocab$.
\end{definition}

\begin{assumption}[Structural forest law (S0--C0)]
\label{app:ass:forest-structure}
\textbf{Forest factorization (S0).}
Fix one hidden reference root in each connected component.  If
$\TargetRoots$ is the resulting root set and $\ParentOf{j}$ is the
induced parent of a nonroot vertex $j$, assume
\begin{equation}
 \TargetLaw(x)
 =\prod_{u\in\TargetRoots}\PosMarginal{u}(x_u)
  \prod_{j\notin\TargetRoots}
  \EndpointKernel{\ParentOf{j}}{j}
  (x_j\mid x_{\ParentOf{j}}).
 \label{app:eq:S0}
\end{equation}

\textbf{Endpoint compatibility (C0).}
For every $\{i,j\}\in\TargetEdges$ and $a,b\in\Vocab$, assume
\begin{equation}
 \PosMarginal{i}(a)\EndpointKernel{i}{j}(b\mid a)
 =\PosMarginal{j}(b)\EndpointKernel{j}{i}(a\mid b).
 \label{app:eq:C0}
\end{equation}
Moving a root across one edge replaces the left-hand side of
\ref{app:eq:C0} by its right-hand side and leaves every other factor
unchanged.  Iterating this replacement along the unique path between two roots
shows that (S0)--(C0) yield the same factorization for every componentwise root
choice.
\end{assumption}

To obtain these kernels from \cref{main:ass:forest-structure}, set
$\EndpointKernel{i}{j}(b\mid a):=\TargetLaw(X_j=b\mid X_i=a)$
for each oriented edge. Then (S0) is the main-text factorization, and
(C0) follows from
\[
 \PosMarginal{i}(a)\EndpointKernel{i}{j}(b\mid a)
 =\TargetLaw(X_i=a,X_j=b)
 =\PosMarginal{j}(b)\EndpointKernel{j}{i}(a\mid b).
\]

For a committed set $G\subseteq\PositionSet$ and realized values
$x_G\in\Vocab^G$, define the history, uncommitted set, and residual
forest by
\begin{equation}
 \History:=(G,x_G),
 \qquad
 \ResidualPositions:=\PositionSet\setminus G,
 \qquad
 \ResidualForest:=\TargetForest[\ResidualPositions].
 \label{app:eq:history}
\end{equation}
Strict positivity makes every such history admissible.
For any graph $H=(W,E_H)$, we write
$\mathcal N_H(v):=\{u\in W:\{u,v\}\in E_H\}$ for its neighborhood and
$d_H(v):=|\mathcal N_H(v)|$ for its degree.  The induced graph $H[S]$,
for $S\subseteq W$, has vertex set $S$ and edge set
$\{\{u,v\}\in E_H:u,v\in S\}$; $V(T)$ denotes the vertex set of a
connected component $T$.

\subsection{Response regularity}
\label{app:setup-response}

\textcolor{red}{The response quantities below correspond to
\cref{main:def:response-quantities}; the general regularity assumptions
correspond to the RT, UEN, and RF items of
\cref{main:ass:regularity-scaling}.}
For simplicity, the main text and \cref{app:cor:fixed-accuracy-main}
use $\ResponseExponent=\ResponseConstant=\FrequencyConstant=1$.

\begin{definition}[Marginal tails and intrinsic responses]
\label{app:def:response-quantities}
For a source position $i\in\PositionSet$ and threshold $t\in[0,1]$, define the exact
marginal tail
\begin{equation}
 \TailSet{i}{t}
 :=\{a\in\Vocab:\PosMarginal{i}(a)\le t\}.
 \label{app:eq:tail}
\end{equation}
For distinct positions $i,j$ and a complete boundary
$z\in\Vocab^{\PositionSet\setminus\{i,j\}}$, define the boundary-conditioned
row
\begin{equation}
 \BoundaryRow{i}{j}{a}{z}
 :=\mathcal L_{\TargetLaw}\!\left(
 X_j\mid X_i=a,
 X_{\PositionSet\setminus\{i,j\}}=z
 \right).
 \label{app:eq:boundary-row}
\end{equation}
For an edge $e=\{u,v\}\in\TargetEdges$ and such a complete boundary, define
the directed intrinsic response
\begin{equation}
 \DirectedResponse{v}{u}{z}
 :=\max_{a,b\in\Vocab}
 \dTV\bigl(
 \BoundaryRow{v}{u}{a}{z},
 \BoundaryRow{v}{u}{b}{z}
 \bigr).
 \label{app:eq:intrinsic-response}
\end{equation}
\end{definition}

\begin{assumption}[Response regularity (RT--UEN)]
\label{app:ass:response-regularity}
Let $\ResponseExponent>0$, $\ResponseConstant\in[0,\infty)$, and
$\EdgeSignal\in(0,1]$ be public.

\textbf{Power tail-response continuity (RT).}
For all distinct $i,u\in\PositionSet$, every $t\in[0,1]$, and every complete assignment
$z\in\Vocab^{\PositionSet\setminus\{i,u\}}$, assume
\begin{equation}
 \max_{a,b\in\TailSet{i}{t}}
 \dTV\bigl(
 \BoundaryRow{i}{u}{a}{z},
 \BoundaryRow{i}{u}{b}{z}
 \bigr)
 \le \ResponseConstant t^{\ResponseExponent}.
 \label{app:eq:RT}
\end{equation}
The maximum is defined to be zero when the tail is empty or a singleton.

\textbf{Two-sided edge nondegeneracy (UEN).}
For every edge
$e=\{u,v\}\in\TargetEdges$ and every complete boundary $z$, assume
\begin{equation}
 \min\bigl\{
  \DirectedResponse{v}{u}{z},
  \DirectedResponse{u}{v}{z}
 \bigr\}
 \ge \EdgeSignal.
 \label{app:eq:UEN}
\end{equation}
Only the scale $\EdgeSignal$ is public; the forest, edge direction,
boundary, and maximizing token pair remain hidden.
\end{assumption}

\begin{assumption}[Asymptotic vocabulary and frequency envelope (RF)]
\label{app:ass:rate-regime}
For asymptotic rate statements, let $\VocabGrowthExponent>0$,
$\RankTailExponent>1$, and $\FrequencyConstant\ge1$ be public constants, and
assume $\VocabSize=N^{\VocabGrowthExponent+o(1)}$.  Order the vocabulary at each position by exact
marginal mass,
\begin{equation}
 \PosMarginal{i}(\RankedToken{i}{1})\ge\cdots\ge
 \PosMarginal{i}(\RankedToken{i}{\VocabSize}),
 \label{app:eq:hidden-ranks}
\end{equation}
with ties broken by $\TokenOrder$, and assume uniformly in $i\in\PositionSet$ and
$k\in[\VocabSize]$ that
\begin{equation}
 \PosMarginal{i}(\RankedToken{i}{k})
 \le \FrequencyConstant
      (k^{-\RankTailExponent}+\VocabSize^{-1}).
 \label{app:eq:RF}
\end{equation}
No lower Zipf bound or adjacent-rank separation is assumed.  This assumption
is used only to turn finite bounds into uniform asymptotic rates; it is not
needed for finite correctness once threshold feasibility and the realized bank
sizes are fixed.
\end{assumption}

\begin{samepage}
\subsection{Frozen conditional oracle}
\label{app:setup-oracle}

The two cases of Assumption~\ref{main:ass:oracle-accuracy} share the frozen oracle
interface below, with different accuracy conditions and output objectives:
\begin{itemize}
 \item \emph{\HellingerCaseTitle\ (Hellinger/TV).}
 The uniform squared-Hellinger condition is (A2) below; the output objective
 is the seed-averaged TV risk in \textcolor{red}{\cref{app:def:output-risk}}.
 The normalized Hellinger radius $\OracleRadius$ remains independent
 of the output-TV tolerance $\TargetTolerance$.
 \item \emph{\KLCaseTitle\ (forward KL).}
 \Cref{app:def:kl-oracle-risk} defines the uniform forward row-KL
 condition and the seed-averaged forward-KL output objective.
\end{itemize}
\end{samepage}
The main text uses $\OracleRadius=\TargetTolerance/2$ in \HellingerCase\
and $\KLTargetTolerance=\TargetTolerance^2/2$ in \KLCase, with the latter
calibrated by \cref{app:cor:kl-epsilon}.
\Cref{app:cor:fixed-accuracy-main} gives both main-text specializations.

\begin{definition}[Frozen conditional oracle]
\label{app:def:frozen-oracle}
As in \cref{main:def:frozen-oracle}, let $\MASK\notin\Vocab$. For a masked state
$y\in(\Vocab\cup\{\MASK\})^N$, define
\begin{equation}
 \MaskSet{y}:=\{i:y_i=\MASK\},
 \qquad
 \ObservedSet{y}:=\PositionSet\setminus\MaskSet{y}.
 \label{app:eq:masked-revealed}
\end{equation}
For $j\in\MaskSet{y}$, the hidden exact singleton conditional is
\begin{equation}
 \ExactRowMass{y}{j}{a}
 :=\TargetLaw\bigl(X_j=a\mid
 X_{\ObservedSet{y}}=y_{\ObservedSet{y}}\bigr),
 \qquad a\in\Vocab.
 \label{app:eq:exact-row}
\end{equation}
The all-mask state is $\AllMaskState:=(\MASK,\ldots,\MASK)$, so
$\ExactRow{\AllMaskState}{i}=\PosMarginal{i}$.

The oracle is one deterministic frozen map
\begin{equation}
 \FrozenOracle:\{(y,j):y\in(\Vocab\cup\{\MASK\})^N,
                 \ j\in\MaskSet{y}\}
 \longrightarrow\Delta^\circ(\Vocab).
 \label{app:eq:oracle-type}
\end{equation}
Equivalently, write
$\FrozenOracle=\{\FrozenOracle_j\}_{j\in\PositionSet}$, with
$\FrozenOracle(y,j)=\OracleRow{y}{j}$ and
$\OracleRowMass{y}{j}{a}$ its probability mass at token $a$.

\emph{\textcolor{red}{One submission, a family of probability vectors.}}
One oracle evaluation submits a single masked state $y$ and returns full
vocabulary probability vectors at the requested masked positions:
\[
 (y,S)\ \xrightarrow{\text{one masked-state submission}}\quad
 \bigl\{\OracleRow{y}{j}\bigr\}_{j\in S},
 \qquad S\subseteq\MaskSet{y}.
\]
Here $S$ selects the readouts; each returned row is the whole vector
$\OracleRow{y}{j}=(\OracleRowMass{y}{j}{a})_{a\in\Vocab}$.
The charge is one submission, regardless of $|S|$ or how many tests use
the returned probabilities. Every occurrence of the same $(y,j)$ returns
the same row, across readout sets and operations; repeated submissions
still count as separate evaluations.
Probes and commits in \cref{app:def:admissible-algorithm} use this same
interface and the same frozen rows.
\end{definition}

\begin{assumption}[\HellingerCaseTitle: frozen-oracle accuracy (A2)]
\label{app:ass:oracle-accuracy}
For a public $\OracleRadius\in[0,1]$, assume uniformly over all valid $(y,j)$ that
\begin{equation}
 \sqHellinger\bigl(\ExactRow{y}{j},\OracleRow{y}{j}\bigr)
 \le \frac{(\OracleRadius)^2}{2N}.
 \label{app:eq:A2}
\end{equation}
Oracle errors may be adversarially correlated across states, positions, and
calls; no independence, unbiasedness, concentration, or fresh-noise condition
is imposed.  Let $\OracleClass(\TargetLaw;\OracleRadius)$ denote the class
of frozen maps satisfying (A2).
\end{assumption}

\begin{remark}[Oracle metric calibration]
\label{app:rem:oracle-metric-calibration}
The parameter $\OracleRadius$ is a normalized oracle budget: the per-row
Hellinger distance in (A2) is at most $\OracleRadius/\sqrt{2N}$, not
$\OracleRadius$.  Lemma~\ref{app:lem:finite-hellinger-calculus} therefore
gives the row-TV bound $\OracleRadius/\sqrt N$ used by the screen; this is
a consequence of (A2), not an additional assumption.  Conversely, that
row-TV bound alone does not imply (A2).  For example, for sufficiently small
$e>0$, the positive binary laws $(1-2e,2e)$ and $(1-e,e)$ have TV distance
$e$ but squared Hellinger distance at least
$e(\sqrt2-1)^2/2>e^2/2$.

By Lemma~\ref{app:lem:kl-hellinger}, either of the following uniform bounds
is sufficient for (A2):
\begin{equation}
 \sup_{(y,j)}\KLDivergence{\ExactRow{y}{j}}{\OracleRow{y}{j}}
 \le\frac{(\OracleRadius)^2}{N}
 \quad\text{or}\quad
 \sup_{(y,j)}\KLDivergence{\OracleRow{y}{j}}{\ExactRow{y}{j}}
 \le\frac{(\OracleRadius)^2}{N}.
 \label{app:eq:row-kl-sufficient}
\end{equation}
These are \emph{sufficient conditions}, not necessary bounds on the KL error of an
oracle satisfying (A2): Remark~\ref{app:rem:no-hellinger-kl-converse} gives
a strictly positive counterexample to the converse.  Imposing either KL
condition leaves the upper theorem valid with its stated output-TV bound;
it does not assert an output-KL bound for the randomized sampler.
\end{remark}

\subsection{Algorithms, risk, and resources}
\label{app:setup-algorithms}

Both operations below use the row family of
\cref{app:def:frozen-oracle}. A \emph{probe} inspects its probabilities
without changing the committed history; a \emph{commit} samples a chosen
batch from these rows and permanently fixes the sampled values.
The algorithm requirements and resource counts below apply to both cases.
They give the full interface of \cref{main:def:algorithm-risk-resources}.

\begin{definition}[\textcolor{red}{Admissible algorithms}]
\label{app:def:admissible-algorithm}
\textbf{Algorithm requirements.}
An admissible algorithm accesses
$\FrozenOracle=\{\FrozenOracle_j\}_{j\in\PositionSet}$ adaptively and
satisfies the following conditions.
\begin{itemize}
\item \emph{Observable decisions and independent randomness.}
Let $\ControllerSeed$ be decision-rule randomness independent of the hidden
instance and of the primitive random sources used for product-commit draws.
The \emph{decision rule} (also called the controller) uses public data,
previous oracle replies, committed values, and its internal state and seed
to choose the next operation (probe or commit), its submitted state, and
any commit batch. Once the seed is fixed,
these choices are deterministic functions of the observable history.
The sampled token values may depend on $\ControllerSeed$ through the
decision rule's choices, but the primitive commit-draw randomness is
separate from that seed. The forest, exact marginals, endpoint kernels,
hidden ranks, edge directions, and response witnesses remain hidden.

\item \emph{No remasking and history-compatible submissions.}
At a realized history $\History=(G,x_G)$, every submitted state must agree
with the committed values: $y_i=x_i$ for all $i\in G$.
Committed positions stay revealed in all subsequent submissions, with the
same values. Uncommitted positions in $\ResidualPositions$ may be masked
or temporarily assigned vocabulary tokens for a probe.

\item \emph{Probes (counterfactual submissions).}
\textcolor{red}{Choose a history-compatible state $y$ and a subset of masked
positions $S\subseteq\MaskSet{y}$. One submission $(y,S)$ returns
$\{\OracleRow{y}{j}\}_{j\in S}$ together, allowing the algorithm to inspect
all probabilities $\{\OracleRowMass{y}{j}{a}:j\in S,\ a\in\Vocab\}$.}
This operation updates the observable transcript but leaves
$\History=(G,x_G)$ unchanged; temporary assignments are not committed.
Probes may be adaptive and interleaved with commits.
They may also use $\ObservedSet{y}=G$, with every residual position masked:
additional hypothetical reveals are optional.

\item \emph{Irreversible product commits.}
Let $t$ index commits only, and let $G_{t-1}$ be the positions committed
before round $t$, with $G_0=\emptyset$.
Choose a nonempty batch $B_t\subseteq\ResidualPositionsAt{G_{t-1}}$
from the observable past and $\ControllerSeed$, \emph{before} receiving
this commit's oracle reply. Submit the state
\begin{equation}
 (y_t)_i=
 \begin{cases}
  x_i, & i\in G_{t-1},\\
  \MASK, & i\in\ResidualPositionsAt{G_{t-1}}.
 \end{cases}
 \label{app:eq:commit-state}
\end{equation}
\textcolor{red}{One submission $(y_t,B_t)$ returns the probability vectors
$\{\OracleRow{y_t}{j}\}_{j\in B_t}$ from the same interface.}
Draw $z_{B_t}\in\Vocab^{B_t}$ according to
\begin{equation}
 \bigotimes_{j\in B_t}\OracleRow{y_t}{j},
 \label{app:eq:product-commit}
\end{equation}
and permanently update $G_t:=G_{t-1}\cup B_t$ and $x_{B_t}:=z_{B_t}$,
retaining the earlier values $x_{G_{t-1}}$.
The product in \eqref{app:eq:product-commit} draws the batch coordinates
independently conditional on the current history and selected batch.
Temporary probe assignments cannot replace this draw, and the draw is
not filtered by an accept--reject or verify-and-accept step.

\item \emph{Completion.}
The batches are disjoint and satisfy
$\bigcup_{t=1}^{\CommitRounds}B_t=\PositionSet$.
\end{itemize}

\smallskip
\noindent\textbf{Operation-indexed actions and visible history.}
\textcolor{red}{Index the submissions actually performed, probes and commits
together, by consecutive positive integers $s=1,2,\ldots$; repeated or
parallel submissions have separate indices. This operation index is
distinct from the rank-frequency exponent in \eqref{app:eq:RF}.}
Let $s_t$ be the operation index of commit $t$, and write $t_s=t$ when
$s=s_t$. Fix the decision-rule seed $w$.
The complete visible history immediately before operation $s$ is
$\OpVisibleHistory{s}$, initialized by the public data and seed-fixed
initial state, $\OpVisibleHistory{1}:=h_{\mathrm{init}}(w)$.
The selected action has the form
\begin{equation}
 A_s=\mathsf{Next}_w(\OpVisibleHistory{s})=
 \begin{cases}
  (\mathsf{query},\OpSubmittedState{s},J_s),
    &\text{probe},\\
  (\mathsf{commit},\OpSubmittedState{s},B_{t_s}),
    &\text{commit}.
 \end{cases}
 \label{app:eq:operation-action}
\end{equation}
Here $\mathsf{query}$ denotes a probe, $J_s$ is its requested readout set
($S$ in the probe rule), with
$J_s\subseteq\MaskSet{\OpSubmittedState{s}}$.
At the current committed history $(G,x_G)$, a probe state satisfies
\[
 (\OpSubmittedState{s})_i=x_i\quad(i\in G),
 \qquad
 (\OpSubmittedState{s})_i\in\Vocab\cup\{\MASK\}\quad(i\notin G).
\]
The decision rule in \cref{app:eq:operation-action} selects these
uncommitted entries and $J_s$ from the visible history and fixed seed;
any token assignments there are temporary and may all be omitted.
At a commit, the state instead has all uncommitted positions masked:
\[
 \OpSubmittedState{s}=y_{t_s},
 \qquad
 (y_{t_s})_i=
 \begin{cases}
  x_i,&i\in G,\\
  \MASK,&i\notin G,
 \end{cases}
\]
as in \cref{app:eq:commit-state}, with $G=G_{t_s-1}$.

\begin{samepage}
At the current committed history $(G,x_G)$,
the vertex roles are:
\begin{center}
\begin{tabular}{@{}lll@{}}
\toprule
Role & Probe & Commit $t$\\
\midrule
Temporarily assigned positions
 & $\ObservedSet{\OpSubmittedState{s}}\setminus G$ & $\varnothing$\\
Requested readouts & $J_s$ & $B_t$\\
Newly committed positions & $\varnothing$ & $B_t$\\
\bottomrule
\end{tabular}
\end{center}
\end{samepage}
If $z_s$ is the vector drawn at a commit, the ordinary observation and
the visible-history update are
\begin{equation}
 \OpVisibleObservation{s}:=
 \begin{cases}
  (\OracleRow{\OpSubmittedState{s}}{j})_{j\in J_s},
    &\text{query},\\
  ((\OracleRow{\OpSubmittedState{s}}{j})_{j\in B_{t_s}},z_s),
    &\text{commit},
 \end{cases}
 \label{app:eq:operation-visible-observation}
\end{equation}
\begin{equation}
 \OpVisibleHistory{(s+1)}
 :=(\OpVisibleHistory{s},A_s,\OpVisibleObservation{s}).
 \label{app:eq:operation-visible-history}
\end{equation}
The decision-rule internal state is a deterministic function of this
history and $w$. This full-history representation does not require the
implementation to store the entire transcript.
\end{definition}

\begin{definition}[\textcolor{red}{Output law and \HellingerCase\ risk}]
\label{app:def:output-risk}
\textcolor{red}{For an admissible algorithm from
\cref{app:def:admissible-algorithm}, let $\SamplerOutput$ be its output.}
Conditional on $\ControllerSeed=w$, let
$\SamplerOutputLaw{w}{\FrozenOracle}$ be its law after integrating all
product-commit draws.
Define the seed-averaged total-variation risk by
\begin{equation}
 \SeedRisk(\mathcal A;\TargetLaw,\FrozenOracle)
 :=\mathbb E_{\ControllerSeed}
 \dTV\bigl(\TargetLaw,
            \SamplerOutputLaw{\ControllerSeed}{\FrozenOracle}\bigr),
 \label{app:eq:seed-risk}
\end{equation}
and the oracle-robust risk by
\begin{equation}
 \RobustRisk(\mathcal A;\TargetLaw)
 :=\sup_{\FrozenOracle\in\OracleClass(\TargetLaw;\OracleRadius)}
 \SeedRisk(\mathcal A;\TargetLaw,\FrozenOracle).
 \label{app:eq:robust-risk}
\end{equation}
The expectation in \eqref{app:eq:seed-risk} is the expectation of the
conditional-output TV, rather than the TV after first mixing over
$\ControllerSeed$.
In \KLCase, the same conditional output law is evaluated by the
seed-averaged forward-KL risk of \cref{app:def:kl-oracle-risk}.
\end{definition}

\begin{definition}[\textcolor{red}{Resources and all-path budgets}]
\label{app:def:resources}
\textcolor{red}{These definitions apply to both accuracy cases.
In the counts below, $s$ ranges over the positive-integer indices of
submissions actually performed, as in \cref{app:eq:operation-action};
$t$ indexes commits only.}
Let $s_{\mathrm{pre}}$ be the index of the one distinguished preprocessing
all-mask readout, with $s_{\mathrm{pre}}=0$ if it is not used. Define
\[
\begin{aligned}
 \CounterfactualQueries
 &:=\bigl|\{s:s\ne s_{\mathrm{pre}},\
                 \text{submission }s\text{ is a probe}\}\bigr|,\\
 \CommitRounds
 &:=\bigl|\{s:\text{submission }s\text{ performs a nonempty product commit}\}\bigr|.
\end{aligned}
\]
Each probe other than that preprocessing readout contributes one to
$\CounterfactualQueries$; each commit
contributes one to $\CommitRounds$ and zero to $\CounterfactualQueries$,
regardless of its batch size. A probe at the commit state
\eqref{app:eq:commit-state} still counts toward $\CounterfactualQueries$:
the operation performed, rather than the masked state alone, determines
its counter. Every other all-mask probe also counts toward
$\CounterfactualQueries$. Reusing returned rows in several tests adds
no submission, whereas submitting the same state again does.
The actual total number of submissions is therefore
$\CounterfactualQueries+\CommitRounds+\mathbf 1\{s_{\mathrm{pre}}>0\}$.

The one distinguished all-mask readout is kept separate through the
accounting flag $\MarginalQueryCharge\in\{0,1\}$: set it to one when
that readout is charged, and to zero when it is omitted from the reported
count (or is not used). The reported \emph{noncommit} count is
\begin{equation}
 \ReportedQueries
 :=\CounterfactualQueries+\MarginalQueryCharge.
 \label{app:eq:reported-query-count}
\end{equation}
Thus $\ReportedQueries$ excludes commits. When the preprocessing readout
is used and charged, the full count is
$\ReportedQueries+\CommitRounds=1+\CounterfactualQueries+\CommitRounds$,
as in the main construction. Setting $\MarginalQueryCharge=0$ omits only
that preprocessing charge; it does not make other probes free.

\emph{Sequential depth and bandwidth.}
$\OracleDepth$ counts sequential oracle stages. A stage may contain
multiple submissions whose inputs are fixed before receiving that
stage's replies; later inputs that depend on those replies require
another stage. Parallel submissions therefore share depth but each
retains its own submission charge. Returned-row counts, scalar bandwidth,
and local computation are not included in these interaction counts;
\cref{app:lem:upper-resources} separately accounts for the construction's rows.

\smallskip
\noindent\textbf{Deterministic all-path budgets.}
The budgets $\QueryBudget,\RoundBudget$ satisfy
$\CounterfactualQueries\le\QueryBudget$ and $\CommitRounds\le\RoundBudget$
on every instance, seed, transcript, and realized path.  A lower bound on
these budgets does not assert the same lower bound on every realized counter.
\end{definition}

\begin{example*}[Submission counting]
One preprocessing all-mask readout, two probes of the same state, and
one commit of all $N$ positions use
$\CounterfactualQueries=2$ and $\CommitRounds=1$.
Charging preprocessing gives $\MarginalQueryCharge=1$,
$\ReportedQueries=3$, and $\ReportedQueries+\CommitRounds=4$ evaluations.
Returning several rows per probe or using each row in several tests
does not change these counts.
\end{example*}

\paragraph{Quantifier order.}
The public class data and the algorithm
are fixed first.  A hidden target law satisfying
Assumptions~\ref{app:ass:forest-structure} and
\ref{app:ass:response-regularity}, together with
Assumption~\ref{app:ass:rate-regime} where asymptotic rates are claimed, is
then chosen. In \HellingerCase, a single frozen oracle satisfying
Assumption~\ref{app:ass:oracle-accuracy} is chosen after the target law; the algorithm
is then executed and the robust risk in \eqref{app:eq:robust-risk} is evaluated.
\KLCaseTitle\ uses the same order with the oracle class and output risk
of \cref{app:def:kl-oracle-risk}.

\section{Finite matching lower-bound statements}
\label{app:results}

This section defines the hard matching family and states its finite TV
query--round lower bound and target-error tradeoff.
Appendix~\ref{app:lower-proof} proves these results and verifies class
membership in \cref{app:prop:hard-family-inclusion}.
Appendix~\ref{app:benchmark-specializations} then places the finite
main-text calibration (\cref{app:cor:signal-calibrated-finite-lower})
immediately before the fixed-accuracy derivation
(\cref{app:cor:fixed-accuracy-main}) of the \HellingerCase\ lower bound
in \cref{main:thm:lower-envelope}.
The finite upper bounds for the two cases are
\cref{app:thm:finite-upper,app:thm:finite-kl-upper}.

\subsection{Finite matching lower bound}
\label{app:results-lower}

\begin{definition}[Finite hard matching family]
\label{app:def:hard-matching-family}
Let $N$ be even, let
$1\le\HardBankSize<\HardVocabSize$ be integers, and suppose
\begin{equation}
 \HardBankSize\TriggerMass<1,
 \qquad
 0<\EdgeCoupling\le\TriggerMass\le\frac12.
 \label{app:eq:hard-basic-conditions}
\end{equation}
Use the public vocabulary and common marginal
\begin{equation}
 \Vocab=\{\HardCandidate{1},\ldots,\HardCandidate{\HardBankSize}\}
        \sqcup\HardBackground,
 \qquad |\HardBackground|=\HardVocabSize-\HardBankSize,
 \qquad
 \HardMarginal(\HardCandidate{c})=\TriggerMass,
 \qquad
 \HardMarginal(b)=
 \frac{1-\HardBankSize\TriggerMass}
      {\HardVocabSize-\HardBankSize}.
 \label{app:eq:hard-common-marginal}
\end{equation}
The candidate bank
$\{\HardCandidate{1},\ldots,\HardCandidate{\HardBankSize}\}$ and the
background set are public.  An instance
$I=(\HiddenMatching,\TriggerVector)$ consists of a hidden perfect matching
$\HiddenMatching$ on $\PositionSet$ and hidden trigger indices
$\TriggerVector=(\TriggerIndex{i})_{i=1}^N
\in[\HardBankSize]^N$.  Put
\begin{equation}
 \TriggerFeature{i}(a)
 :=\mathbf 1\{a=\HardCandidate{\TriggerIndex{i}}\}-\TriggerMass,
 \qquad
 \TriggerVariance:=\TriggerMass(1-\TriggerMass).
 \label{app:eq:hard-features}
\end{equation}
For $e=\{i,j\}\in\HiddenMatching$, define
\begin{align}
 \HardIndependentEdgeLaw{e}(a,b)
 &:=\HardMarginal(a)\HardMarginal(b),
 \label{app:eq:hard-independent-edge-law}\\
 \HardEdgeLaw{e}{I}(a,b)
 &:=\HardMarginal(a)\HardMarginal(b)
 \left[1+\EdgeCoupling
 \frac{\TriggerFeature{i}(a)\TriggerFeature{j}(b)}{\TriggerVariance}\right],
 \label{app:eq:hard-edge-law}\\
 \HardTargetLaw{I}&:=\bigotimes_{e\in\HiddenMatching}\HardEdgeLaw{e}{I},
 \qquad
 \HardKernel{i}{j}{I}(b\mid a)
 :=\HardMarginal(b)
 \left[1+\EdgeCoupling
 \frac{\TriggerFeature{i}(a)\TriggerFeature{j}(b)}{\TriggerVariance}\right].
 \label{app:eq:hard-target-and-kernel}
\end{align}
The bounds in \eqref{app:eq:hard-basic-conditions} make these laws strictly
positive: the three possible likelihood-ratio values are displayed in
\eqref{app:eq:hard-likelihood-values} and checked in Step~1 of the proof of
Lemma~\ref{app:lem:hard-local-facts}.
If $\MateInMatching{j}$ denotes the mate of $j$, the selected
frozen oracle is, for every masked readout $j$ and
$i:=\MateInMatching{j}$,
\begin{equation}
 \HardOracleRow{I}{y}{j}
 :=\begin{cases}
 \HardMarginal,
   &i\in\MaskSet{y},\\
 \HardKernel{i}{j}{I}(\cdot\mid y_i),
   &i\in\ObservedSet{y}\ \text{and}\
     y_i=\HardCandidate{\TriggerIndex{i}},\\
 \HardMarginal,
   &i\in\ObservedSet{y}\ \text{and}\
     y_i\ne\HardCandidate{\TriggerIndex{i}}.
 \end{cases}
 \label{app:eq:hard-oracle}
\end{equation}
Let
$\HardFamily[\HardVocabSize,\HardBankSize,\TriggerMass,\EdgeCoupling]$ be the
family obtained by ranging over all $\HiddenMatching$ and $\TriggerVector$.
The public candidate labels are granted to the algorithm;
the matching and trigger indices remain hidden.
\end{definition}

For integers $\QueryBudget{},\RoundBudget{}$, let $\HardAlgorithmClass{\QueryBudget{}}{\RoundBudget{}}$ be the admissible algorithms
whose counterfactual-submission and nonempty-commit counters are at most $\QueryBudget{}$
and $\RoundBudget{}$, respectively, on every instance, seed, transcript, and realized
path.  Define
\begin{equation}
 \HardMinimax(\QueryBudget{},\RoundBudget{};\HardVocabSize,\HardBankSize,
              \TriggerMass,\EdgeCoupling)
 :=\inf_{\mathcal A\in\HardAlgorithmClass{\QueryBudget{}}{\RoundBudget{}}}
   \sup_{I\in\HardFamily[\HardVocabSize,\HardBankSize,
                           \TriggerMass,\EdgeCoupling]}
   \SeedRisk(\mathcal A;\HardTargetLaw{I},\HardOracle{I}).
 \label{app:eq:hard-minimax-risk}
\end{equation}

The theorem below uses only the finite hard-family conditions.
\Cref{app:cor:signal-calibrated-finite-lower} later chooses these parameters
to satisfy the \HellingerCase\ oracle condition for the main lower bound.

\begin{theorem}[Nonlinear finite adaptive query--round lower bound]
\label{app:thm:nonlinear-finite-lower}
Assume only the basic hard-family conditions
\eqref{app:eq:hard-basic-conditions}; in particular, no restriction on
\(N\EdgeCoupling^2\) is imposed. Then
\begin{equation}
 \begin{aligned}
 &0\le \QueryBudget{}\le\HardBankSize/8,
 \qquad
 1\le \RoundBudget{}\le N/16384\\
 &\quad\Longrightarrow\quad
 \HardMinimax(\QueryBudget{},\RoundBudget{};\HardVocabSize,\HardBankSize,
              \TriggerMass,\EdgeCoupling)\\
 &\hspace{7em}\ge
 \frac14\left[
  1-\exp\left\{-\frac{N\EdgeCoupling^2}{98304\RoundBudget{}}\right\}
 \right].
 \end{aligned}
 \label{app:eq:nonlinear-finite-lower}
\end{equation}
\end{theorem}

\begin{corollary}[Nonlinear finite target-error tradeoff]
\label{app:cor:nonlinear-finite-target-tradeoff}
Assume \eqref{app:eq:hard-basic-conditions}, fix
\(\varepsilon\in(0,1/8]\), and suppose an admissible algorithm has
worst-case pathwise budgets $(\QueryBudget{},\RoundBudget{})$ and
\begin{equation}
 \sup_{I\in\HardFamily[\HardVocabSize,\HardBankSize,
                         \TriggerMass,\EdgeCoupling]}
 \SeedRisk(\mathcal A;\HardTargetLaw{I},\HardOracle{I})\le\varepsilon.
 \label{app:eq:finite-target-risk}
\end{equation}
Then at least one of
\begin{equation}
 \QueryBudget{}>\HardBankSize/8,
 \qquad
 \RoundBudget{}>N/16384,
 \qquad
 \RoundBudget{}\ge\frac{N\EdgeCoupling^2}{786432\varepsilon}
 \label{app:eq:nonlinear-finite-target-alternative}
\end{equation}
holds.
\end{corollary}


\begin{remark}[Where the oracle slack enters]
\label{app:rem:oracle-slack}
The selected oracle \eqref{app:eq:hard-oracle} is exact except in its third
case, where a revealed nontrigger source value $a$ at the mate $i$ of $j$
produces the row $\HardMarginal$ instead of the exact row
$\HardKernel{i}{j}{I}(\cdot\mid a)$.  The difference of these two rows is
$\HardMarginal(b)\EdgeCoupling\TriggerFeature{i}(a)\TriggerFeature{j}(b)
/\TriggerVariance$ with $\TriggerFeature{i}(a)=-\TriggerMass$, so by
\eqref{app:eq:hard-feature-moments} they are at total-variation distance
exactly $\EdgeCoupling\TriggerMass$, and the squared Hellinger bound
\eqref{app:eq:hard-local-oracle-error} is the only place where
Assumption~\ref{app:ass:oracle-accuracy} is consumed.  The resulting
constraint $\EdgeCoupling^2\TriggerMass\le(\OracleRadius)^2/(2N)$ in
\eqref{app:eq:finite-hard-inclusion} yields the cubic signal calibration
when $\EdgeCoupling=\TriggerMass=\EdgeSignal$.
\end{remark}



\section{Proof of the lower bound}
\label{app:lower-proof}

This section proves the finite matching lower bound in
\cref{app:thm:nonlinear-finite-lower} and states and proves its class
inclusion in \cref{app:prop:hard-family-inclusion}.
The matching proof allows randomized,
adaptive, value-dependent algorithms, interleaved counterfactual submissions,
and unrestricted returned bandwidth.  Only deterministic envelopes for the two pathwise counters in
\textcolor{red}{\cref{app:def:resources}} are bounded.
The next step toward \cref{main:thm:lower-envelope} is in
Appendix~\ref{app:benchmark-specializations}:
\cref{app:cor:signal-calibrated-finite-lower} gives the finite calibration,
immediately followed by the fixed-accuracy derivation in
\cref{app:cor:fixed-accuracy-main}.
\ifincludeexponents
Independent-signal rate results are in Appendix~\ref{app:independent-signal}.
\fi

\subsection{Proof overview}
\label{app:lower-proof-overview}

The goal is to show that few counterfactual submissions and few commit
rounds cannot both yield small sampling error.  We use a hidden perfect
matching: each position has one hidden trigger among \(\HardBankSize\)
candidates, and the oracle conceals an edge's dependence until a trigger
is hit.  The argument has three steps.

\paragraph{1. At most one candidate per source per submission.}
A state can return many rows but tests at most one candidate at each
source.  Conditional uniformity bounds the probability of hitting a given
source before retirement within \(\QueryBudget{}\) submissions by
\(\QueryBudget{}/\HardBankSize\). Summing over positions gives
\(\mathbb E\GenieDeletions\le N\QueryBudget{}/\HardBankSize\), where
\(\GenieDeletions\) counts query-retired edges.

\paragraph{2. Few commit rounds force collisions.}
Every edge not retired by a query must be retired at its first endpoint
commit.  Few rounds therefore require large active batches.  Conditional
uniformity makes such batches likely to contain both endpoints of an
unresolved edge: a \emph{collision}.  Edge counting and a conditional
Laplace bound give \(\TotalCollisions\gtrsim N/\RoundBudget{}\) collisions
with constant probability under the budget conditions of
\cref{app:thm:nonlinear-finite-lower}.

\paragraph{3. Collisions turn hidden dependence into sampling error.}
At a collision, sampling independently from the correct marginals
\(\HardMarginal\) misses the edge's dependence.  The one-edge squared
Hellinger loss \(\HardEdgeDistance\asymp\EdgeCoupling^2\) gives a local
affinity factor \(1-\HardEdgeDistance\).  The midpoint identity combines
these factors along adaptive paths to bound output error.
Here query-disclosure and collision probabilities use the
instance--midpoint law \eqref{app:eq:joint-instance-midpoint-law}, with
the decision-rule seed fixed; the final conclusion concerns the sampler's
seed-averaged TV risk.  The table locates the results that implement these
three steps.

\begin{table}[H]
\centering
\small
\begin{tabularx}{\linewidth}{@{}>{\raggedright\arraybackslash}p{0.23\linewidth}
  >{\raggedright\arraybackslash}X
  >{\raggedright\arraybackslash}p{0.23\linewidth}@{}}
\toprule
Proof step & Key idea and conclusion & References \\
\midrule
\textbf{Construct a hard matching family}
& Nontrigger oracle replies hide dependence while each edge retains
loss \(\HardEdgeDistance\asymp\EdgeCoupling^2\).  The inclusion calculation
checks the target assumptions and oracle error under the stated calibration.
&
\cref{app:lower-proof-family};
\cref{app:lem:hard-local-facts,app:prop:hard-family-inclusion} \\
\addlinespace
\textbf{Relate adaptive commits to output error}
& Later batches depend on sampled values.  The midpoint identity expresses
output affinity as an expectation of products of local affinities along
these adaptive paths.
&
\cref{app:lower-proof-midpoint};
\cref{app:lem:adaptive-midpoint-lower} \\
\addlinespace
\textbf{Limit what probes can reveal}
& The unresolved matching and triggers remain conditionally uniform.
One candidate per source per submission gives
\(\mathbb E\GenieDeletions\le N\QueryBudget{}/\HardBankSize\).
&
\cref{app:lower-proof-invariant};
\cref{app:lem:unresolved-invariant,app:lem:query-disclosure} \\
\addlinespace
\textbf{Force collisions with few rounds}
& Edge counting forces large active batches; a conditional Laplace bound
then gives \(\TotalCollisions\gtrsim N/\RoundBudget{}\) with constant
probability under the budgets in \eqref{app:eq:lower-proof-spine}.
&
\cref{app:lower-proof-collisions,app:lower-proof-matching-laplace,app:lower-proof-collision-tail};
\cref{app:lem:pathwise-edge-counting,app:lem:matching-collision-laplace,app:prop:collision-lower-tail} \\
\addlinespace
\textbf{Convert collisions into sampling error}
& Each collision contributes a factor \(1-\HardEdgeDistance\).
The midpoint identity gives a prior-averaged loss; averaging over seeds
and selecting a worst-case instance yields the TV lower bound.
&
\cref{app:lower-proof-affinity,app:lower-proof-nonlinear};
\cref{app:lem:collision-affinity,app:thm:nonlinear-finite-lower} \\
\addlinespace
\textbf{Recover the main-text lower bound}
& Calibrate the matching witness at the public parameters, then take the
fixed-accuracy limit to obtain \cref{main:thm:lower-envelope}.
&
\cref{app:cor:signal-calibrated-finite-lower,app:cor:fixed-accuracy-main} \\
\bottomrule
\end{tabularx}
\caption{Proof steps for the lower bound.}
\label{app:tab:lower-proof-map}
\end{table}

\ifincludeexponents
The joint-limit result in \cref{app:cor:benchmark-hellinger} and
independent-signal rates in \cref{app:lower-proof-rates} are also separate
extensions.
\fi

\paragraph{Quantitative summary.}
For every fixed decision-rule seed \(w\), the finite hypotheses and
constants give the following chain:
\begin{equation}
 \begin{aligned}
 &\QueryBudget{}\le\HardBankSize/8,\quad
 1\le\RoundBudget{}\le N/16384\\
 &\Longrightarrow\quad
 \mathbb E\GenieDeletions\le N/8\\
 &\Longrightarrow\quad
 \mathbb P\left(\TotalCollisions\ge
                 \frac{N}{8192\RoundBudget{}}\right)\ge\frac14\\
 &\Longrightarrow\quad
 \mathbb E_{I\sim\HardInstancePrior}\sqHellinger
   (\HardTargetLaw{I},\HardOutputLaw{I}{w})
 \ge\frac14\left[1-\exp\left\{
       -\frac{N\EdgeCoupling^2}{98304\RoundBudget{}}\right\}\right].
 \end{aligned}
 \label{app:eq:lower-proof-spine}
\end{equation}
The collision step combines the unresolved-matching invariant, the
edge-counting identity, and the conditional Laplace bound.
The last step uses the adaptive midpoint identity and
$\HardEdgeDistance\ge\EdgeCoupling^2/12$.

\subsection{Hard-family local facts and class membership}
\label{app:lower-proof-family}

This subsection proves the one-edge calculations in
\cref{app:lem:hard-local-facts}: positivity and common marginals, the edge
response, the selected-oracle error, and the one-edge Hellinger loss.  The
same calculations then verify \cref{app:prop:hard-family-inclusion}, placing
the finite hard family inside the shared target/oracle class used by both
bounds.

\paragraph{Intuition.}
There are two different error scales.  Removing one dependence edge costs
$\HardEdgeDistance\simeq\EdgeCoupling^2$, whereas replacing a nontrigger
conditional by the neutral marginal costs at most order
$\EdgeCoupling^2\TriggerMass$.  The extra factor $\TriggerMass$ lets the
frozen oracle hide failed trigger tests while preserving a larger loss
when two unresolved endpoints are committed together.

For the moment calculations, if
\(A\sim\HardMarginal\), then the trigger indicator at position \(i\) is
Bernoulli with mean \(\TriggerMass\).  Hence
\begin{equation}
 \mathbb E\TriggerFeature{i}(A)=0,
 \qquad
 \mathbb E\TriggerFeature{i}(A)^2=\TriggerVariance,
 \qquad
 \mathbb E|\TriggerFeature{i}(A)|=2\TriggerVariance.
 \label{app:eq:hard-feature-moments}
\end{equation}
Indeed, the feature is $1-\TriggerMass$ with probability $\TriggerMass$
and $-\TriggerMass$ otherwise, so
\[
\begin{aligned}
 \mathbb E\TriggerFeature{i}
 &=\TriggerMass(1-\TriggerMass)
   -(1-\TriggerMass)\TriggerMass=0,\\
 \mathbb E\TriggerFeature{i}^2
 &=\TriggerMass(1-\TriggerMass)^2
   +(1-\TriggerMass)\TriggerMass^2
   =\TriggerMass(1-\TriggerMass),\\
 \mathbb E|\TriggerFeature{i}|
 &=\TriggerMass(1-\TriggerMass)
   +(1-\TriggerMass)\TriggerMass
   =2\TriggerMass(1-\TriggerMass).
\end{aligned}
\]

\begin{lemma}[Local facts for the hard matching family]
\label{app:lem:hard-local-facts}
Under \eqref{app:eq:hard-basic-conditions}, every member of
\(\HardFamily[\HardVocabSize,\HardBankSize,
\TriggerMass,\EdgeCoupling]\) has the following properties.
\begin{enumerate}
 \item \textup{\textbf{Positivity, marginals, and factorization.}}
 Each edge law in \eqref{app:eq:hard-edge-law} is strictly positive,
 has both marginals equal to \(\HardMarginal\), and the product law satisfies
 the structural forest assumptions.
 \item \textup{\textbf{Edge response.}}
 For either orientation \(i\to j\) of a matching edge, the largest TV
 change in the endpoint kernel equals \(\EdgeCoupling\).
 \item \textup{\textbf{Oracle error.}}
 Let $j$ be any readout, let $i:=\MateInMatching{j}$, and let
 $a\in\Vocab\setminus\{\HardCandidate{\TriggerIndex{i}}\}$ be a revealed
 nontrigger value at $i$.  The selected oracle is exact outside this case,
 while in this case
 \begin{equation}
  \sqHellinger\bigl(
   \HardKernel{i}{j}{I}(\cdot\mid a),\HardMarginal
  \bigr)
  \le
  \frac{\EdgeCoupling^2\TriggerMass}
       {2(1-\TriggerMass)}
  \le \EdgeCoupling^2\TriggerMass.
  \label{app:eq:hard-local-oracle-error}
 \end{equation}
 \item \textup{\textbf{One-edge Hellinger loss.}}
 For every matching edge \(e\), uniformly over its two private trigger
 indices, the value below depends only on
 \((\TriggerMass,\EdgeCoupling)\), not on \(e\) or \(I\):
 \begin{equation}
  \frac{\EdgeCoupling^2}{12}
  \le
  \HardEdgeDistance
  :=\sqHellinger\bigl(
     \HardEdgeLaw{e}{I},\HardIndependentEdgeLaw{e}
    \bigr)
  \le\frac{\EdgeCoupling^2}{2}.
  \label{app:eq:hard-edge-distance}
 \end{equation}
\end{enumerate}
\end{lemma}

\begin{proof}

\emph{1. Positivity, common marginals, and forest factorization.}
Fix \(e=\{i,j\}\).  Summing \eqref{app:eq:hard-edge-law} over its second
coordinate and using \eqref{app:eq:hard-feature-moments} gives
\[
 \sum_b\HardEdgeLaw{e}{I}(a,b)
 =\HardMarginal(a)
  \left[1+
   \EdgeCoupling\frac{\TriggerFeature{i}(a)}{\TriggerVariance}
   \mathbb E\TriggerFeature{j}(B)
  \right]
 =\HardMarginal(a),
\]
where \(B\sim\HardMarginal\).  The other marginal follows symmetrically.
The likelihood ratio with respect to
\(\HardIndependentEdgeLaw{e}\) takes the three values
\begin{equation}
 1+\EdgeCoupling
   \frac{\TriggerFeature{i}(a)\TriggerFeature{j}(b)}
        {\TriggerVariance}
 =
 \begin{cases}
  1+\EdgeCoupling(1-\TriggerMass)/\TriggerMass,
    &a,b\text{ are their respective triggers},\\
  1-\EdgeCoupling,
    &\text{exactly one of }a,b\text{ is its trigger},\\
  1+\EdgeCoupling\TriggerMass/(1-\TriggerMass),
    &a,b\text{ are both nontriggers}.
 \end{cases}
 \label{app:eq:hard-likelihood-values}
\end{equation}
The conditions
\(0<\EdgeCoupling\le\TriggerMass\le1/2\) place every value in
\([1/2,2]\), proving strict positivity.  The marginal identity also gives
the endpoint compatibility
\[
 \HardMarginal(a)\HardKernel{i}{j}{I}(b\mid a)
 =\HardEdgeLaw{e}{I}(a,b)
 =\HardMarginal(b)\HardKernel{j}{i}{I}(a\mid b).
\]
Every component is one edge, so multiplying these identities over the hidden
matching proves the required forest factorization for either root choice.

\emph{2. Directed endpoint response.}
For source values \(a,a'\),
\begin{align}
 &\dTV\bigl(
   \HardKernel{i}{j}{I}(\cdot\mid a),
   \HardKernel{i}{j}{I}(\cdot\mid a')
  \bigr)\notag\\
 &\qquad=
 \frac{\EdgeCoupling
       |\TriggerFeature{i}(a)-\TriggerFeature{i}(a')|}
      {2\TriggerVariance}
 \sum_b\HardMarginal(b)|\TriggerFeature{j}(b)|
 =\EdgeCoupling
  |\TriggerFeature{i}(a)-\TriggerFeature{i}(a')|.
 \label{app:eq:hard-directed-response}
\end{align}
The feature equals \(1-\TriggerMass\) at the unique trigger and
\(-\TriggerMass\) elsewhere.  Its largest difference is therefore one, so
the maximum in \eqref{app:eq:hard-directed-response} is exactly
\(\EdgeCoupling\).

\emph{3. Oracle error at a revealed nontrigger.}
Because the target is a product over matching edges, the exact row at a
readout \(j\) is \(\HardMarginal\) when its mate is masked and is
\(\HardKernel{i}{j}{I}(\cdot\mid y_i)\) when the mate \(i\) is revealed.
Thus \eqref{app:eq:hard-oracle} is exact in its first two cases.  In the third
case, $\TriggerFeature{i}(y_i)=-\TriggerMass$, so the likelihood ratio is
\[
 \frac{\HardKernel{i}{j}{I}(b\mid y_i)}{\HardMarginal(b)}
 =1-\frac{\EdgeCoupling}{1-\TriggerMass}\TriggerFeature{j}(b).
\]
Substitution into the definition of chi-square divergence gives
\begin{align}
 \ChiSquareDivergence{
  \HardKernel{i}{j}{I}(\cdot\mid y_i)}{\HardMarginal}
 &=\sum_b\HardMarginal(b)
       \left(\frac{\EdgeCoupling\TriggerFeature{j}(b)}
                   {1-\TriggerMass}\right)^2\notag\\
 &=\frac{\EdgeCoupling^2}{(1-\TriggerMass)^2}
       \TriggerMass(1-\TriggerMass)
 =\EdgeCoupling^2\frac{\TriggerMass}{1-\TriggerMass}.
 \label{app:eq:hard-oracle-chi-square}
\end{align}
For a likelihood ratio \(r\),
\((\sqrt r-1)^2\le(r-1)^2\).  The normalization in
\eqref{app:eq:distances} therefore implies
\(\sqHellinger(p,q)\le\ChiSquareDivergence{p}{q}/2\).  Combining this with
\eqref{app:eq:hard-oracle-chi-square} and
\(\TriggerMass\le1/2\) proves \eqref{app:eq:hard-local-oracle-error}.

\emph{4. One-edge squared-Hellinger loss.}
Finally, put
\[
 Z_e(a,b):=
 \frac{\TriggerFeature{i}(a)\TriggerFeature{j}(b)}
      {\TriggerVariance}.
\]
Under \(\HardIndependentEdgeLaw{e}\), each endpoint independently equals
its trigger token with probability \(\TriggerMass\).  Thus
\[
 Z_e=
 \begin{cases}
 (1-\TriggerMass)/\TriggerMass,
   &\text{with probability }\TriggerMass^2
      \quad\text{(both triggers)},\\
 -1,
   &\text{with probability }2\TriggerMass(1-\TriggerMass)
      \quad\text{(one trigger)},\\
 \TriggerMass/(1-\TriggerMass),
   &\text{with probability }(1-\TriggerMass)^2
      \quad\text{(neither trigger)}.
 \end{cases}
 \]
This law depends only on \(\TriggerMass\), not on the edge or its trigger
indices.
\Eqref{app:eq:hard-feature-moments} and independence under
\(\HardIndependentEdgeLaw{e}\) give
\[
 \mathbb E_{\HardIndependentEdgeLaw{e}}Z_e
 =\frac{\mathbb E\TriggerFeature{i}\,
        \mathbb E\TriggerFeature{j}}{\TriggerVariance}=0,\qquad
 \mathbb E_{\HardIndependentEdgeLaw{e}}Z_e^2
 =\frac{\mathbb E\TriggerFeature{i}^2\,
        \mathbb E\TriggerFeature{j}^2}{(\TriggerVariance)^2}=1.
\]  Since
\(\mathrm d\HardEdgeLaw{e}{I}/\mathrm d\HardIndependentEdgeLaw{e}
=1+\EdgeCoupling Z_e\),
the definition of squared Hellinger distance gives the first line below.
For the second, multiply
$\sqrt{1+\EdgeCoupling Z_e}-1$ by
$(\sqrt{1+\EdgeCoupling Z_e}+1)/
 (\sqrt{1+\EdgeCoupling Z_e}+1)$:
\begin{align}
 \HardEdgeDistance
 &=\frac12\mathbb E_{\HardIndependentEdgeLaw{e}}
       \bigl(\sqrt{1+\EdgeCoupling Z_e}-1\bigr)^2\notag\\
 &=\frac{\EdgeCoupling^2}{2}
  \mathbb E_{\HardIndependentEdgeLaw{e}}
  \frac{Z_e^2}
       {(\sqrt{1+\EdgeCoupling Z_e}+1)^2}.
 \label{app:eq:hard-edge-distance-formula}
\end{align}
Consequently \(\HardEdgeDistance\) depends only on
\((\TriggerMass,\EdgeCoupling)\), so it is the same for every matching
component and every choice of its trigger indices.
By \eqref{app:eq:hard-likelihood-values}, the denominator lies between
one and \((\sqrt2+1)^2<6\).
\Eqref{app:eq:hard-edge-distance} follows.
\end{proof}

\paragraph{Intuition for class inclusion.}
The four checks have separate jobs:
$\HardVocabSize\TriggerMass>1$ puts the candidate bank above the background,
$\TriggerMass\le\HardBankSize^{-\RankTailExponent}$ gives the frequency
envelope, $\EdgeSignal\le\EdgeCoupling$ gives the edge signal, and
$\EdgeCoupling^2\TriggerMass\le(\OracleRadius)^2/(2N)$ hides nontrigger replies
within the oracle budget.  In the main calibration,
$\EdgeCoupling=\TriggerMass=\EdgeSignal$, so the last check becomes the
cubic constraint $\EdgeSignal^3\le(\OracleRadius)^2/(2N)$.

\begin{proposition}[Finite hard-family inclusion]
\label{app:prop:hard-family-inclusion}
In addition to \eqref{app:eq:hard-basic-conditions}, let
$\RankTailExponent>1$, use the same public vocabulary of size $\VocabSize$,
set
$\ResponseExponent=\ResponseConstant=\FrequencyConstant=1$, and suppose
\begin{equation}
 \HardVocabSize\TriggerMass>1,
 \qquad
 \TriggerMass\le\HardBankSize^{-\RankTailExponent},
 \qquad
 \EdgeSignal\le\EdgeCoupling,
 \qquad
 \EdgeCoupling^2\TriggerMass\le\frac{(\OracleRadius)^2}{2N}.
 \label{app:eq:finite-hard-inclusion}
\end{equation}
Then every target in
$\HardFamily[\HardVocabSize,\HardBankSize,\TriggerMass,\EdgeCoupling]$
satisfies the structural and response assumptions in
Appendix~\ref{app:setup}, as well as the finite rank--frequency inequality in
(RF\textcolor{red}{; \cref{app:eq:RF}}), and its selected oracle \eqref{app:eq:hard-oracle} belongs to
$\OracleClass(\HardTargetLaw{I};\OracleRadius)$.  If the public vocabulary
sequence also obeys the vocabulary-growth condition in (RF\textcolor{red}{; \cref{app:ass:rate-regime}}), then the full
(RF\textcolor{red}{; \cref{app:ass:rate-regime}}) assumption holds.  Under that additional asymptotic condition, this
matching family is a contained subfamily of the same public problem class
used by the upper theorem.
\end{proposition}

\begin{proof}[Proof of \cref{app:prop:hard-family-inclusion}]

\emph{1. Forest structure and the edge-signal condition (UEN\textcolor{red}{; \cref{app:eq:UEN}}).}
Apply \cref{app:lem:hard-local-facts}.  Its first conclusion proves
strict positivity and Assumption~\ref{app:ass:forest-structure}.  Complete
boundary assignments outside a matching edge do not alter that edge's law,
so \eqref{app:eq:hard-directed-response} makes both directed intrinsic
responses exactly \(\EdgeCoupling\).  The condition
\(\EdgeSignal\le\EdgeCoupling\) therefore proves (UEN\textcolor{red}{; \cref{app:eq:UEN}}).

\emph{2. Response-tail regularity (RT\textcolor{red}{; \cref{app:eq:RT}}).}
For (RT\textcolor{red}{; \cref{app:eq:RT}}), first consider distinct positions \(i,j\) that are not matched.
Conditioning on a complete boundary fixes the mates of \(i\) and \(j\), and
the product over matching components makes
\(\BoundaryRow{i}{j}{a}{z}\) independent of \(a\).  The response diameter is
therefore zero.

It remains to consider a matching edge.  Fix a threshold
\(t\).
\begin{itemize}
\item If \(t<\TriggerMass\), the marginal tail
contains no candidate token and hence no trigger.  Every source value in the
tail has the same centered feature \(-\TriggerMass\), so its response
diameter is zero.
\item If \(t\ge\TriggerMass\), the response diameter is at most
\(\EdgeCoupling\le\TriggerMass\le t\).
\end{itemize}
Thus (RT\textcolor{red}{; \cref{app:eq:RT}}) holds with
\(\ResponseExponent=\ResponseConstant=1\).

\emph{3. Rank--frequency regularity (RF\textcolor{red}{; \cref{app:eq:RF}}).}
The inequality \(\HardVocabSize\TriggerMass>1\) is equivalent to the
background mass in \eqref{app:eq:hard-common-marginal} being strictly less
than $\HardVocabSize^{-1}$, since
\[
 \frac{1-\HardBankSize\TriggerMass}{\HardVocabSize-\HardBankSize}
 <\frac1{\HardVocabSize}
 \quad\Longleftrightarrow\quad
 \HardVocabSize-\HardVocabSize\HardBankSize\TriggerMass
 <\HardVocabSize-\HardBankSize
 \quad\Longleftrightarrow\quad
 \HardVocabSize\TriggerMass>1.
\]  Hence the \(\HardBankSize\) candidates occupy
the first \(\HardBankSize\) marginal ranks.  For
\(k\le\HardBankSize\),
\[
 \HardMarginal(\HardCandidate{k})=\TriggerMass
 \le\HardBankSize^{-\RankTailExponent}
 \le k^{-\RankTailExponent};
\]
for \(k>\HardBankSize\), the vocabulary-floor term
\(\HardVocabSize^{-1}\) suffices.  This proves the rank--frequency inequality
in (RF\textcolor{red}{; \cref{app:eq:RF}}) with \(\FrequencyConstant=1\).  The separate vocabulary-growth clause
in (RF\textcolor{red}{; \cref{app:ass:rate-regime}}) holds whenever the public sequence \(\HardVocabSize\) has the
declared asymptotic scaling.

\emph{4. Uniform oracle accuracy (A2).}
The only nonzero oracle error is bounded by
\eqref{app:eq:hard-local-oracle-error}.  The final condition in
\eqref{app:eq:finite-hard-inclusion} therefore places it below
\((\OracleRadius)^2/(2N)\), proving (A2).  All required finite conditions hold
for every hidden matching and trigger vector.  Together with the stated
vocabulary-growth condition, this proves the asserted inclusion in the
shared asymptotic class.
\end{proof}

\begin{samepage}
\subsection{Adaptive midpoint comparison and instance prior}
\label{app:lower-proof-midpoint}

We relate the target law to the sampler's output law through the discrepancies
at individual commits. Three points guide the comparison:
\begin{itemize}
\item \emph{Goal.} Express the output affinity, or equivalently one minus the
squared-Hellinger error, in terms of local commit affinities.
\item \emph{Obstacle.} A different committed value can change subsequent
probes, batches, and the stopping round. Thus local affinities depend on the
realized history; the fixed-product affinity formula does not directly
combine them.
\item \emph{Bridge.} Normalize the geometric mean of the two commit kernels
at each history. These normalized kernels define a common path law under
which the product of local affinities equals the output affinity in
expectation (\cref{app:lem:adaptive-midpoint-lower}).
\end{itemize}
\end{samepage}

\emph{Fixed objects.}
Fix an admissible algorithm with finite pathwise budgets
\((\QueryBudget{},\RoundBudget{})\), an instance
\(I=(\HiddenMatching,\TriggerVector)\) of
\cref{app:def:hard-matching-family}, and one value \(w\) of the
decision-rule seed. Put
\begin{equation}
 \HardOutputLaw{I}{w}
 :=\SamplerOutputLaw{w}{\HardOracle{I}}.
 \label{app:eq:hard-seeded-output-law}
\end{equation}
The following three executions differ in their commit kernels, defined
below. The exact comparison process has terminal law \(\HardTargetLaw{I}\);
Step~2 of the proof verifies this even for adaptive batches.

\begin{center}
\small
\begin{tabularx}{\linewidth}{@{}>{\raggedright\arraybackslash}p{0.23\linewidth}
  >{\raggedright\arraybackslash}p{0.32\linewidth}
  >{\raggedright\arraybackslash}X@{}}
\toprule
Execution & Commit kernel at history \(h\) & Law and role \\
\midrule
Exact comparison
& \(\HardExactBatchLaw{t}{I}(\cdot\mid h)\): true joint conditional
& \(\HardTargetLaw{I}\): target output law. \\
\addlinespace
Product-commit sampler
& \(\HardProductBatchLaw{t}{I}(\cdot\mid h)\): product of oracle rows
& \(\HardOutputLaw{I}{w}\): output law being evaluated. \\
\addlinespace
Adaptive midpoint
& \(\HardMidpointKernel{t}{I}(\cdot\mid h)\): normalized geometric mean
& \(\HardMidpointLaw{I}{w}\): analytic path law for averaging local affinities. \\
\bottomrule
\end{tabularx}
\end{center}

\begin{itemize}
\item \emph{Shared operations.} All three use the same seed-fixed decision
rule and the same frozen oracle \(\HardOracle{I}\) for probes. At a common
history they select the same batch \(B_t(h)\).
\item \emph{Changed operation.} Only the law of the committed values changes.
The resulting histories, and hence later batches, can differ between executions.
\end{itemize}

\begin{definition*}[Commit kernels and adaptive midpoint law]
\emph{History and batch.}
Immediately before commit \(t\), after the intervening probes, let \(h\)
be the complete observable history: committed positions and values,
the earlier oracle transcript, and the decision-rule state.
Thus $h=\OpVisibleHistory{s_t}$ in the operation notation of
\cref{app:eq:operation-action,app:eq:operation-visible-history}.
Write \(G(h)\) for the committed set, \(x_{G(h)}\) for its values, and
\(y_t(h)\) for the commit state in \eqref{app:eq:commit-state}.
For fixed \((I,w)\), the probes and batch selection are deterministic
functions of the observed past; write \(B_t(h)\) for the selected batch.
The seed \(w\) is suppressed from local kernels because \(h\) includes
the seed-fixed decision-rule state, but retained in full-law and
terminal-round notation.

\emph{Commit kernels.}
For \(z\in\Vocab^{B_t(h)}\), define
\begin{align}
 \HardExactBatchLaw{t}{I}(z\mid h)
 &:=\HardTargetLaw{I}\bigl(
   X_{B_t(h)}=z\mid X_{G(h)}=x_{G(h)}
  \bigr),
 \label{app:eq:hard-exact-batch-law}\\
 \HardProductBatchLaw{t}{I}(z\mid h)
 &:=\prod_{j\in B_t(h)}\HardOracleRowMass{I}{y_t(h)}{j}{z_j}.
 \label{app:eq:hard-product-batch-law}
\end{align}

\emph{Local affinity and midpoint law.}
Define
\begin{align}
 \HardLocalAffinity{t}{I,h}
 &:=\sum_{z\in\Vocab^{B_t(h)}}
    \sqrt{
     \HardExactBatchLaw{t}{I}(z\mid h)
     \HardProductBatchLaw{t}{I}(z\mid h)},
 \label{app:eq:hard-local-affinity}\\
 \HardMidpointKernel{t}{I}(z\mid h)
 &:=\frac{
    \sqrt{
     \HardExactBatchLaw{t}{I}(z\mid h)
     \HardProductBatchLaw{t}{I}(z\mid h)}}
   {\HardLocalAffinity{t}{I,h}}.
 \label{app:eq:hard-midpoint-kernel}
\end{align}
Let \(\HardMidpointLaw{I}{w}\) be the analytic adaptive path law that runs
the same decision rule and frozen-oracle probes, drawing every commit
batch from \(\HardMidpointKernel{t}{I}(\cdot\mid h)\).
\end{definition*}

\paragraph{Why the construction is well defined.}
\begin{itemize}
\item \emph{Conditioning on the history.}
For fixed \((I,w)\), every reachable \(h\) can be reconstructed from
\(x_{G(h)}\): replay the seed-fixed decision rule from the initial state and,
at each earlier commit, supply the coordinates of \(x_{G(h)}\) selected
by that commit. The frozen oracle reproduces the intervening replies,
batches, and decision-rule states. Conversely, \(h\) records \(x_{G(h)}\).
Thus conditioning on this \(h\) adds no random event beyond
\(X_{G(h)}=x_{G(h)}\), justifying
\eqref{app:eq:hard-exact-batch-law}.
\item \emph{Normalization.}
Strict positivity makes the denominator nonzero. Cauchy--Schwarz gives
\(0<\HardLocalAffinity{t}{I,h}\le1\). In particular,
\[
 \sum_{z\in\Vocab^{B_t(h)}}\HardMidpointKernel{t}{I}(z\mid h)
 =\frac{\displaystyle
   \sum_{z\in\Vocab^{B_t(h)}}
    \sqrt{\HardExactBatchLaw{t}{I}(z\mid h)
           \HardProductBatchLaw{t}{I}(z\mid h)}}
  {\HardLocalAffinity{t}{I,h}}
 =\frac{\HardLocalAffinity{t}{I,h}}{\HardLocalAffinity{t}{I,h}}
 =1.
\]
\item \emph{Pathwise budgets.}
Both commit kernels are strictly positive, so the midpoint kernel has
the same support at every reachable history. It therefore traverses the
same decision tree as the implemented sampler. Every midpoint path has
at most \(\QueryBudget{}\) counterfactual submissions and at most
\(\RoundBudget{}\) nonempty commits.
\end{itemize}

\paragraph{From local affinity to output error.}
The defining factorization is
\[
 \sqrt{\HardExactBatchLaw{t}{I}(z\mid h)
        \HardProductBatchLaw{t}{I}(z\mid h)}
 =\HardLocalAffinity{t}{I,h}\HardMidpointKernel{t}{I}(z\mid h).
\]
Along a complete path, the normalized \(\HardMidpointKernel{t}{I}\) factors
supply its midpoint probability, while the \(\HardLocalAffinity{t}{I,h}\)
factors retain the local affinity losses.
The next identity averages their product over the adaptive paths.

\begin{lemma}[Adaptive midpoint identity]
\label{app:lem:adaptive-midpoint-lower}
For every fixed instance \(I\) and decision-rule seed \(w\),
\begin{equation}
 1-\sqHellinger(
      \HardTargetLaw{I},\HardOutputLaw{I}{w})
 =\mathbb E_{H\sim\HardMidpointLaw{I}{w}}
   \prod_{t=1}^{\HardTerminalRoundCount{I}{w}{H}}
    \HardLocalAffinity{t}{I,H_{t-1}},
 \label{app:eq:adaptive-midpoint-identity}
\end{equation}
where \(\HardTerminalRoundCount{I}{w}{H}\le \RoundBudget{}\) is the terminal number of
nonempty commit batches on the midpoint path.
Here \(H_{t-1}\) denotes the complete observable history immediately
before commit \(t\), after all intervening counterfactual operations.
The subscript counts earlier commits, not individual query operations.
\end{lemma}

The remaining lower-bound argument uses this identity at two levels:
\begin{itemize}
\item \emph{Local penalty.} \Cref{app:lem:collision-affinity} bounds the
factors \(\HardLocalAffinity{t}{I,h}\) using unresolved edges whose
endpoints are committed together.
\item \emph{Global error.} The survivor and collision arguments count
these events under the instance--midpoint law
\eqref{app:eq:joint-instance-midpoint-law}. The identity converts that
count into an output-error bound.
\end{itemize}

\begin{proof}

\emph{1. Reconstruct the unique adaptive path.}
Fix a complete output \(x\in\Vocab^N\).  Starting with the empty history,
replay the deterministic seed-fixed decision rule.  Whenever it chooses
\(B_t(h_{t-1}(x))\), feed it
\(z_t(x):=x_{B_t(h_{t-1}(x))}\).  This recursively determines the histories,
batches, and terminal round \(\HardTerminalRoundCount{I}{w}{x}\).  No
remasking makes the realized batches disjoint and exhaustive.  Put
\(G_t(x):=G_{t-1}(x)\cup B_t(h_{t-1}(x))\), with \(G_0(x)=\varnothing\).

\emph{2. Factor the exact and product output masses.}
The exact conditional factors telescope:
\begin{align*}
 &\prod_{t=1}^{\HardTerminalRoundCount{I}{w}{x}}
   \HardExactBatchLaw{t}{I}(z_t(x)\mid h_{t-1}(x))\\
 &\quad=\prod_{t=1}^{\HardTerminalRoundCount{I}{w}{x}}
   \frac{\HardTargetLaw{I}(X_{G_t(x)}=x_{G_t(x)})}
        {\HardTargetLaw{I}(X_{G_{t-1}(x)}=x_{G_{t-1}(x)})}
 =\frac{\HardTargetLaw{I}(X_{\PositionSet}=x)}{1}
 =\HardTargetLaw{I}(x).
\end{align*}
Here the numerator at round \(t\) is the denominator at round \(t+1\);
the initial event has probability one and the terminal set is \(\PositionSet\).
This identity and the
product-commit rule therefore give
\begin{align}
 \HardTargetLaw{I}(x)
 &=\prod_{t=1}^{\HardTerminalRoundCount{I}{w}{x}}
   \HardExactBatchLaw{t}{I}
    (z_t(x)\mid h_{t-1}(x)),
 \label{app:eq:hard-exact-path-factorization}\\
 \HardOutputLaw{I}{w}(x)
 &=\prod_{t=1}^{\HardTerminalRoundCount{I}{w}{x}}
   \HardProductBatchLaw{t}{I}
    (z_t(x)\mid h_{t-1}(x)).
 \label{app:eq:hard-product-path-factorization}
\end{align}
The batch-selection rule has no likelihood factor because it is deterministic
given the displayed history.

\emph{3. Normalize the midpoint path mass.}
For the same terminal output, the midpoint path mass is
\begin{equation}
 \HardMidpointLaw{I}{w}(x)
 :=\prod_{t=1}^{\HardTerminalRoundCount{I}{w}{x}}
   \HardMidpointKernel{t}{I}
    (z_t(x)\mid h_{t-1}(x)).
 \label{app:eq:hard-midpoint-path-mass}
\end{equation}
To verify that this path mass sums to one, first account for paths that
finish at different rounds.  Nonempty batches are disjoint, so every path
has at most \(N\) rounds.  For this calculation only, pad each finished path
to \(N\) rounds by setting \(B_t(h)=\varnothing\),
\(\Vocab^\varnothing=\{\varnothing\}\), and
\(\HardMidpointKernel{t}{I}(\varnothing\mid h)=1\) after termination.
These extra factors leave \eqref{app:eq:hard-midpoint-path-mass} unchanged.
Each complete output corresponds to exactly one padded path, and conversely,
because the seed-fixed decision rule is deterministic and the genuine batches
are disjoint and exhaustive.

Write \(h_{t-1}\) for the history determined by the preceding assignments
\(z_1,\ldots,z_{t-1}\).  For any \(1\le r\le N\), hold
\(z_1,\ldots,z_{r-1}\) fixed.  The earlier factors do not depend on \(z_r\),
so the local normalization above gives
\begin{align*}
 &\sum_{z_r\in\Vocab^{B_r(h_{r-1})}}
   \prod_{t=1}^{r}\HardMidpointKernel{t}{I}(z_t\mid h_{t-1})\\
 &\quad=
   \left[\prod_{t=1}^{r-1}
     \HardMidpointKernel{t}{I}(z_t\mid h_{t-1})\right]
   \underbrace{\sum_{z_r\in\Vocab^{B_r(h_{r-1})}}
     \HardMidpointKernel{r}{I}(z_r\mid h_{r-1})}_{=1}\\
 &\quad=\prod_{t=1}^{r-1}
     \HardMidpointKernel{t}{I}(z_t\mid h_{t-1}).
\end{align*}
Applying this equality for \(r=N,N-1,\ldots,1\) removes the nested sums
from the innermost one outwards and ends with the empty product \(1\):
\[
 \sum_{x\in\Vocab^N}\HardMidpointLaw{I}{w}(x)
 =\sum_{z_1\in\Vocab^{B_1(h_0)}}\cdots
  \sum_{z_N\in\Vocab^{B_N(h_{N-1})}}
   \prod_{t=1}^{N}\HardMidpointKernel{t}{I}(z_t\mid h_{t-1})
 =1.
\]

\emph{4. Convert the affinity sum into a midpoint expectation.}
Substituting
\(\sqrt{\HardExactBatchLaw{t}{I}\HardProductBatchLaw{t}{I}}
 =\HardLocalAffinity{t}{I,h}\HardMidpointKernel{t}{I}\) into the terminal
affinity sum yields
\begin{align*}
 &\HellingerAffinity(
    \HardTargetLaw{I},\HardOutputLaw{I}{w})\\
 &\quad=\sum_{x\in\Vocab^N}
  \sqrt{\HardTargetLaw{I}(x)\HardOutputLaw{I}{w}(x)}\\
 &\quad=\sum_{x\in\Vocab^N}
  \HardMidpointLaw{I}{w}(x)
  \prod_{t=1}^{\HardTerminalRoundCount{I}{w}{x}}
   \HardLocalAffinity{t}{I,h_{t-1}(x)}\\
 &\quad=\mathbb E_{H\sim\HardMidpointLaw{I}{w}}
  \prod_{t=1}^{\HardTerminalRoundCount{I}{w}{H}}
   \HardLocalAffinity{t}{I,H_{t-1}}.
\end{align*}
Since \(\HellingerAffinity(p,q)=1-\sqHellinger(p,q)\), this is
\eqref{app:eq:adaptive-midpoint-identity}.
\end{proof}

\paragraph{Instance prior for the remaining proof.}
Keep the seed \(w\) fixed. Draw an instance
\(I=(\HiddenMatching,\TriggerVector)\) from the prior
\(\HardInstancePrior\) as follows:
\begin{itemize}
\item \emph{Matching.} \(\HiddenMatching\) is uniform over all perfect
matchings of \(\PositionSet\).
\item \emph{Triggers.} \(\TriggerIndex{1},\ldots,\TriggerIndex{N}\) are
independent uniform elements of \([\HardBankSize]\), independent of
\(\HiddenMatching\).
\end{itemize}
All random variables used in the remaining proof are evaluated under the
joint analytic law
\begin{equation}
 I\sim\HardInstancePrior,
 \qquad
 H\mid(I,w)\sim\HardMidpointLaw{I}{w}.
 \label{app:eq:joint-instance-midpoint-law}
\end{equation}

\subsection{Unresolved edges, survivor invariant, and query retirement}
\label{app:lower-proof-invariant}

We need two facts under the instance--midpoint law
\eqref{app:eq:joint-instance-midpoint-law}:
\begin{itemize}
\item \emph{The distribution left unresolved.}
\Cref{app:lem:unresolved-invariant} shows that the remaining matching is
uniform and its triggers are independent uniform remaining candidates.
This permits the conditional collision calculation below.
\item \emph{The number of query retirements.}
\Cref{app:lem:query-disclosure} bounds
\(\mathbb E\GenieDeletions\le N\QueryBudget{}/\HardBankSize\).
A submission tests at most one candidate at each source, so probes can
retire only a limited number of edges before commits touch them.
\end{itemize}
We first define retirement and the information used to condition these
two statements. The decision-rule seed \(w\) remains fixed.
Use the operation clock, actions \(A_s\), and visible histories
\(\OpVisibleHistory{s}\) from
\cref{app:eq:operation-action,app:eq:operation-visible-history}.
For a query, write \(y=\OpSubmittedState{s}\) and \(J=J_s\).
At commit \(t\), \(s=s_t\) and the batch is \(B_t=B_{t_s}\).

\begin{definition*}[Unresolved edges and the analytic record]
\emph{The test at one source.}
For any submitted state \(y\) and position \(i\in\PositionSet\), put
\begin{equation}
 \HardTest{y}{i}:=
 \begin{cases}
  c,&y_i=\HardCandidate{c}\text{ for some }c\in[\HardBankSize],\\
  \bot,&y_i=\MASK\text{ or }y_i\in\HardBackground,
 \end{cases}
 \qquad
 \HardHit{y}{i}:=
 \mathbf 1\{\HardTest{y}{i}=\TriggerIndex{i}\}.
 \label{app:eq:hard-test-and-hit}
\end{equation}
Thus \(y_i=\HardCandidate{c}\) tests the equality
\(\TriggerIndex{i}=c\). A mask or background value tests no candidate.

\emph{Unresolved edges and vertices.}
Let \(\OpUnresolvedMatching{s}\) be the edges unresolved immediately
before operation \(s\). Initialize \(\OpUnresolvedMatching{1}:=\HiddenMatching\)
and update by
\begin{equation}
 \OpUnresolvedMatching{s+1}:=
 \begin{cases}
  \{e\in\OpUnresolvedMatching{s}:
       \sum_{i\in e}\HardHit{y}{i}=0\},
       &\text{query of }y,\\
  \{e\in\OpUnresolvedMatching{s}:e\cap B_t=\varnothing\},
       &\text{commit }t.
 \end{cases}
 \label{app:eq:unresolved-matching-update}
\end{equation}
Define the unresolved vertex set by
\begin{equation}
 \OpUnresolvedVertices{s}:=\bigcup_{e\in\OpUnresolvedMatching{s}}e.
 \label{app:eq:unresolved-vertices}
\end{equation}
Thus a query retires an edge when at least one endpoint hits; a commit
retires it when \(e\cap B_t\ne\varnothing\). In either case both endpoints
leave \(\OpUnresolvedVertices{s}\). A query commits neither endpoint,
and a commit fixes only those in \(e\cap B_t\).
Because each update removes whole edges, \(\OpUnresolvedMatching{s}\)
is a perfect matching on \(\OpUnresolvedVertices{s}\).
The target law remains
\(\HardTargetLaw{I}(x)=\prod_{e\in\HiddenMatching}\HardEdgeLaw{e}{I}(x_e)\),
and the frozen oracle \(\HardOracle{I}\) remains unchanged.
In particular, an uncommitted vertex may lie on a retired edge.

\emph{Retirement record.}
For each newly retired edge, record its identity and both trigger indices:
\begin{equation}
 \OpRetiredRecord{s}:=
 \left\{\bigl(e,(\TriggerIndex{i})_{i\in e}\bigr):
 e\in\OpUnresolvedMatching{s}\setminus\OpUnresolvedMatching{s+1}\right\}.
 \label{app:eq:retirement-record}
\end{equation}
Endpoint tuples use the public position order. Let
\(\GenieDeletions:=\sum_{s:\,A_s\text{ is a query}}
 |\OpUnresolvedMatching{s}\setminus\OpUnresolvedMatching{s+1}|\).
Simultaneous hits on both endpoints count as one retired edge.

\emph{Remaining candidates.}
Initially \(\OpRemainingTriggers{i}{1}:=[\HardBankSize]\).
For each surviving \(i\in\OpUnresolvedVertices{s+1}\), put
\begin{equation}
 \OpRemainingTriggers{i}{s+1}:=
 \begin{cases}
  \OpRemainingTriggers{i}{s}\setminus\{\HardTest{y}{i}\},
       &\text{query of }y,\\
  \OpRemainingTriggers{i}{s},&\text{commit}.
 \end{cases}
 \label{app:eq:operation-candidates}
\end{equation}
A surviving endpoint has hit bit zero, so any candidate tested there
has failed. Subtracting \(\bot\) or a previously excluded candidate
changes nothing. Candidate sets are needed only while the vertex survives.

\emph{Analytic observations and history.}
Use the ordinary observation \(\OpVisibleObservation{s}\) from
\cref{app:eq:operation-visible-observation}, with
\(\FrozenOracle=\HardOracle{I}\). At a commit,
\(z_s=X_{B_t}\) is drawn from the midpoint kernel
\(\HardMidpointKernel{t}{I}(\cdot\mid\OpVisibleHistory{s})\)
of \cref{app:eq:hard-midpoint-kernel}.
The additional information is:
\begin{itemize}
\item At a query, the entire hit-bit vector
\(\bigl(\HardHit{y}{i}\bigr)_{i\in\OpUnresolvedVertices{s}}\)
from \cref{app:eq:hard-test-and-hit}, including bits whose mate rows
were not requested. Evaluate these bits on the same pre-operation
set \(\OpUnresolvedVertices{s}\), then apply
\cref{app:eq:unresolved-matching-update}.
\item At each retirement, the tuples in \(\OpRetiredRecord{s}\) from
\cref{app:eq:retirement-record}. Commit values \(X_{B_t}\) are already
part of the ordinary observation, not additional hidden information.
\end{itemize}
Formally, initialize \(\OpAnalyticHistory{1}:=\OpVisibleHistory{1}\)
and define
\begin{equation}
 \OpObservation{s}:=
 \begin{cases}
  \bigl(\OpVisibleObservation{s},
       (\HardHit{y}{i})_{i\in\OpUnresolvedVertices{s}}\bigr),
       &\text{query of }y,\\
  \OpVisibleObservation{s},&\text{commit},
 \end{cases}
 \label{app:eq:analytic-observation}
\end{equation}
\begin{equation}
 \OpAnalyticHistory{(s+1)}
 :=(\OpAnalyticHistory{s},A_s,\OpObservation{s},\OpRetiredRecord{s}).
 \label{app:eq:analytic-record-update}
\end{equation}
The decision rule uses only the ordinary observations in
\cref{app:eq:operation-visible-history}; it does not receive hit bits
or retirement disclosures. Its internal state is recoverable from that
visible history. Define the analytic sigma-field and hidden survivor state by
\begin{equation}
 \begin{aligned}
 \OpAnalyticField{s}
   &:=\sigma(\OpAnalyticHistory{s},\OpUnresolvedVertices{s}),\\
 \OpSurvivorState{s}
   &:=\bigl(\OpUnresolvedMatching{s},
          (\TriggerIndex{i})_{i\in\OpUnresolvedVertices{s}}\bigr).
 \end{aligned}
 \label{app:eq:analytic-field-and-survivor}
\end{equation}
The visible history is a projection of the analytic record, so
\(\sigma(\OpVisibleHistory{s})\subseteq\OpAnalyticField{s}\).
The unresolved vertices are recoverable by removing all recorded retired
endpoints from \(\PositionSet\); their remaining matching and triggers
are not explicitly recorded. The survivor state \(\OpSurvivorState{s}\)
is not itself an entry of the record.
Here \emph{unresolved} means survival under
\cref{app:eq:unresolved-matching-update}, not a claim about the
algorithm's knowledge.
\end{definition*}

One state supplies at most one test per source, even if many returned rows
are inspected. Its submission contributes one to
\(\CounterfactualQueries\le\QueryBudget{}\), as in
\textcolor{red}{\cref{app:def:resources}}. A state inducing no tests still counts,
except for the separately accounted distinguished preprocessing readout.
Retired edges can still affect ordinary oracle replies; they are excluded
only from subsequent unresolved-edge collision counts.

\paragraph{State immediately before a commit.}
At operation \(s_t\), use the shorter pre-commit notation for the
objects in \cref{app:eq:unresolved-vertices,app:eq:operation-candidates,app:eq:analytic-record-update,app:eq:analytic-field-and-survivor}:
\begin{equation}
 \begin{aligned}
 \UnresolvedVertices{t}&:=\OpUnresolvedVertices{s_t},&
 \RemainingTriggers{i}{t}&:=\OpRemainingTriggers{i}{s_t},\\
 \AnalyticHistory{t}&:=\OpAnalyticHistory{s_t},&
 \AnalyticField{t}&:=\OpAnalyticField{s_t},\\
 \SurvivorState{t}&:=\OpSurvivorState{s_t},&
 H_{t-1}&=\OpVisibleHistory{s_t}.
 \end{aligned}
 \label{app:eq:remaining-trigger-set}
\end{equation}
Here \(H_{t-1}\) is the complete visible history used in
\cref{app:lem:adaptive-midpoint-lower}, after any queries preceding commit
\(t\). There may be queries before commit \(1\), so
\(\UnresolvedVertices{1}\) need not equal \(\PositionSet\).
The next operation after commit \(t\) is \(s_t+1\); intervening queries
are processed before the next commit boundary \(s_{t+1}\).

\paragraph{Why uniformity can survive an observation.}
A failed test removes one candidate. A hit or a commit instead retires
the affected edge and records both its triggers.
We show that, once this retired record is fixed, the probability of the
complete new observation is the same for every compatible remaining
matching and trigger assignment. Equal prior weights therefore give equal
posterior weights.

\Needspace{10\baselineskip}
\begin{lemma}[Unresolved-matching invariant]
\label{app:lem:unresolved-invariant}
Under the joint law \eqref{app:eq:joint-instance-midpoint-law},
\begin{equation}
 \SurvivorState{t}\mid\AnalyticField{t}
 \sim
 \operatorname{Unif}\bigl(\operatorname{PM}
   (\UnresolvedVertices{t})\bigr)
 \otimes
 \bigotimes_{i\in\UnresolvedVertices{t}}
 \operatorname{Unif}(\RemainingTriggers{i}{t}),
 \label{app:eq:strong-survivor-invariant}
\end{equation}
where \(\operatorname{PM}(U)\) denotes the perfect matchings of \(U\).
In particular, the unresolved matching is conditionally uniform, and every
surviving trigger is conditionally uniform on its untested candidates.
The same product law holds at every analytic operation boundary, with the
current unresolved set, remaining-candidate sets, and analytic sigma-field in
place of their pre-commit notation above.
\end{lemma}

{\color{blue}
\begin{proof}[Simplified proof via conditional independence]
We prove the operation-boundary statement and then specialize it to the
pre-commit boundaries.  We use the operation-time notation already defined
in \cref{app:eq:unresolved-vertices,app:eq:operation-candidates,app:eq:analytic-record-update,app:eq:analytic-field-and-survivor},
without introducing shorter names for these objects.  The desired induction
step has two parts:
\begin{enumerate}
\item after conditioning on the newly retired edges and their triggers,
the new survivor state still has the required product-uniform law; and
\item the new observation is conditionally independent of that survivor
state once the retirement record is fixed.
\end{enumerate}
Indeed, these two statements give
\begin{align*}
 &\mathcal L\!\left(
   \OpSurvivorState{s+1}\,\middle|\,
   \OpAnalyticField{s},\OpRetiredRecord{s},\OpObservation{s}\right)\\
 &\qquad=
 \mathcal L\!\left(
   \OpSurvivorState{s+1}\,\middle|\,
   \OpAnalyticField{s},\OpRetiredRecord{s}\right),
\end{align*}
and the conditioning field on the left is precisely
\(\OpAnalyticField{(s+1)}\), by
\cref{app:eq:analytic-record-update,app:eq:analytic-field-and-survivor}.

\emph{1. Initialize the induction.}
Before operation \(1\), the unresolved matching is uniform on
\(\operatorname{PM}(\PositionSet)\), and the trigger indices are mutually
independent and uniform on \([\HardBankSize]\), independently of the
matching.  Thus the required product law holds at the initial boundary.

Now assume that, conditionally on \(\OpAnalyticField{s}\),
\[
 \OpSurvivorState{s}
 \sim
 \operatorname{Unif}\bigl(
   \operatorname{PM}(\OpUnresolvedVertices{s})\bigr)
 \otimes
 \bigotimes_{i\in\OpUnresolvedVertices{s}}
 \operatorname{Unif}(\OpRemainingTriggers{i}{s}).
\]
The action \(A_s\) is already determined by \(\OpAnalyticField{s}\), because
the seed-fixed decision rule uses the visible history contained in this
field.

\emph{2. Condition first on the retirement record.}
Fix a retirement record of positive conditional probability,
\[
 \RetiredRecordValue
 =\{(e,(\theta_i^{\RetiredRecordValue})_{i\in e}):e\in E_0\},
 \qquad
 D_0:=\bigcup_{e\in E_0}e,
\]
where \(E_0\) is the set of edges retired in operation \(s\) and \(D_0\)
is its endpoint set.  On
\(\{\OpRetiredRecord{s}=\RetiredRecordValue\}\),
\[
 \OpUnresolvedVertices{s+1}
 =\OpUnresolvedVertices{s}\setminus D_0.
\]
For every \(i\in\OpUnresolvedVertices{s+1}\), the update already defined in
\eqref{app:eq:operation-candidates} reads
\[
 \OpRemainingTriggers{i}{s+1}=
 \begin{cases}
  \OpRemainingTriggers{i}{s}\setminus
       \{\HardTest{y}{i}\},&A_s\text{ is a query of }y,\\
  \OpRemainingTriggers{i}{s},&A_s\text{ is a commit}.
 \end{cases}
\]
For an edge \(e=\{i,j\}\in\OpUnresolvedMatching{s}\), the update
\eqref{app:eq:unresolved-matching-update} is equivalently
\[
 e\in\OpUnresolvedMatching{s+1}
 \quad\Longleftrightarrow\quad
 \begin{cases}
  \HardHit{y}{i}=\HardHit{y}{j}=0,
       &A_s\text{ is a query of }y,\\
  e\cap B_t=\varnothing,
       &A_s\text{ is commit }t.
 \end{cases}
\]
In the query case, a candidate \(\theta_i\) at a surviving endpoint obeys
\[
 \theta_i\in\OpRemainingTriggers{i}{s+1}
 \quad\Longleftrightarrow\quad
 \theta_i\in\OpRemainingTriggers{i}{s}
 \ \text{and}\
 \theta_i\ne\HardTest{y}{i}.
\]
This also covers \(\HardTest{y}{i}=\bot\), because trigger indices lie in
\([\HardBankSize]\).  Since the fixed record has positive conditional
probability, each \(e\in E_0\) satisfies the complementary retirement
condition: at least one endpoint hits in the query case, and
\(e\cap B_t\ne\varnothing\) in the commit case.
Hence the complete support of the new survivor state, after fixing
\(\RetiredRecordValue\), is
\[
 \operatorname{PM}(\OpUnresolvedVertices{s+1})
 \times
 \prod_{i\in\OpUnresolvedVertices{s+1}}
       \OpRemainingTriggers{i}{s+1}.
\]

We now verify explicitly that every element of this Cartesian support has
the same conditional weight.  Fix
\(
 \SurvivorConfigA
 =\left(M',(\theta_i)_{i\in\OpUnresolvedVertices{s+1}}\right)
 \in
 \operatorname{PM}(\OpUnresolvedVertices{s+1})
 \times
 \prod_{i\in\OpUnresolvedVertices{s+1}}
       \OpRemainingTriggers{i}{s+1}
\).
Together, \(\SurvivorConfigA\) and \(\RetiredRecordValue\) reconstruct exactly one
pre-operation state: its unresolved matching is
\(M'\mathbin{\dot\cup}E_0\), its trigger is \(\theta_i\) on
\(\OpUnresolvedVertices{s+1}\), and its trigger is
\(\theta_i^{\RetiredRecordValue}\) on \(D_0\).  Conversely, membership in the
displayed support ensures that this reconstructed state produces the fixed
retirement record.  Since
\[
 \OpUnresolvedVertices{s}
 =\OpUnresolvedVertices{s+1}\mathbin{\dot\cup}D_0,
\]
these two constructions are inverse to each other.  The induction hypothesis
therefore gives the same joint mass to every compatible
\(\SurvivorConfigA\):
\[
\begin{aligned}
 &\mathbb P\!\left(
   \OpSurvivorState{s+1}=\SurvivorConfigA,\,
   \OpRetiredRecord{s}=\RetiredRecordValue
   \,\middle|\,\OpAnalyticField{s}\right)\\
 &\quad=
 \mathbb P\!\left(
 \begin{gathered}
  \OpUnresolvedMatching{s}=M'\mathbin{\dot\cup}E_0,\\
  \TriggerIndex{i}=\theta_i
       \quad(i\in\OpUnresolvedVertices{s+1}),\\
  \TriggerIndex{i}=\theta_i^{\RetiredRecordValue}\quad(i\in D_0)
 \end{gathered}
 \middle|\OpAnalyticField{s}\right)\\
 &\quad=
 \underbrace{
 \frac{1}{|\operatorname{PM}(\OpUnresolvedVertices{s})|}
 \prod_{i\in\OpUnresolvedVertices{s}}
       \frac{1}{|\OpRemainingTriggers{i}{s}|}
 }_{=:c_s}.
\end{aligned}
\]
The common weight \(c_s\) is independent of \(\SurvivorConfigA\).

\emph{Normalize over the full support.}
The number of compatible survivor configurations is
\[
 K_{\RetiredRecordValue}
 :=|\operatorname{PM}(\OpUnresolvedVertices{s+1})|
 \prod_{i\in\OpUnresolvedVertices{s+1}}
   |\OpRemainingTriggers{i}{s+1}|.
\]
These configurations partition
\(\{\OpRetiredRecord{s}=\RetiredRecordValue\}\), and each has joint
weight \(c_s\).  Therefore
\[
 \mathbb P\!\left(
   \OpRetiredRecord{s}=\RetiredRecordValue
   \,\middle|\,\OpAnalyticField{s}\right)
 =K_{\RetiredRecordValue}c_s.
\]
For every compatible \(\SurvivorConfigA\), it follows that
\begin{align*}
 &\mathbb P\!\left(
   \OpSurvivorState{s+1}=\SurvivorConfigA
   \,\middle|\,
   \OpAnalyticField{s},\OpRetiredRecord{s}=\RetiredRecordValue\right)\\
 &\quad=
 \frac{
  \mathbb P(
   \OpSurvivorState{s+1}=\SurvivorConfigA,\,
   \OpRetiredRecord{s}=\RetiredRecordValue
   \mid\OpAnalyticField{s})}
 {\mathbb P(
   \OpRetiredRecord{s}=\RetiredRecordValue
   \mid\OpAnalyticField{s})}\\
 &\quad=
 \frac{c_s}{K_{\RetiredRecordValue}c_s}
 =\frac{1}{K_{\RetiredRecordValue}}\\
 &\quad=
 \frac{1}{|\operatorname{PM}(\OpUnresolvedVertices{s+1})|}
 \prod_{i\in\OpUnresolvedVertices{s+1}}
       \frac{1}{|\OpRemainingTriggers{i}{s+1}|}.
\end{align*}
The last line is exactly the mass function of
\[
 \operatorname{Unif}\bigl(
   \operatorname{PM}(\OpUnresolvedVertices{s+1})\bigr)
 \otimes
 \bigotimes_{i\in\OpUnresolvedVertices{s+1}}
   \operatorname{Unif}(\OpRemainingTriggers{i}{s+1}).
\]
Thus conditioning on retirement has not coupled the surviving matching and
triggers.

\emph{3. A query observation contains no further survivor information.}
Suppose \(A_s=(\mathsf{query},y,J)\).  We check the complete observation
edge by edge.
\begin{itemize}
\item On an edge retired before operation \(s\), its identity and triggers
are already in \(\OpAnalyticField{s}\), so every requested frozen-oracle row
on that edge is fixed.
\item On an edge retired by this query, its identity and both triggers are
in \(\RetiredRecordValue\).  Together with the submitted state \(y\), these values
fix its hit bits and every requested row.
\item Let \(e=\{i,j\}\) survive, and suppose the row at \(j\in J\) is
requested.  Then neither endpoint hits:
\(
 \HardHit{y}{i}=\HardHit{y}{j}=0
\).
Moreover, the edge has not met an earlier commit, so both endpoints are
uncommitted.  If \(i\) is masked, the selected frozen-oracle rule returns
\(\HardMarginal\).  If \(i\) is temporarily revealed, survival implies
\(y_i\ne\HardCandidate{\TriggerIndex{i}}\), and the same rule again returns
\(\HardMarginal\).  In formulas,
\[
 \HardOracleRow{I}{y}{j}
 =\begin{cases}
   \HardMarginal,&i\in\MaskSet{y},\\
   \HardMarginal,&
       i\in\ObservedSet{y}\setminus G(\OpVisibleHistory{s}),
       \ y_i\ne\HardCandidate{\TriggerIndex{i}}.
  \end{cases}
\]
Thus the requested row is independent of the identity of \(i\) and of the
surviving trigger values.
\end{itemize}
The hit bits on surviving endpoints are all zero, while every other hit bit
and row is fixed by
\((\OpAnalyticField{s},\OpRetiredRecord{s}=\RetiredRecordValue)\).
Hence the whole query observation \(\OpObservation{s}\) is determined by
\((\OpAnalyticField{s},\OpRetiredRecord{s})\) on the compatible support.
Equivalently, for every compatible \(\SurvivorConfigA\) and observation
value \(o\),
\begin{align*}
 &\mathbb P\!\left(
   \OpObservation{s}=o
   \,\middle|\,
   \OpAnalyticField{s},
   \OpRetiredRecord{s}=\RetiredRecordValue,
   \OpSurvivorState{s+1}=\SurvivorConfigA\right)\\
 &\qquad=
 \mathbb P\!\left(
   \OpObservation{s}=o
   \,\middle|\,
   \OpAnalyticField{s},
   \OpRetiredRecord{s}=\RetiredRecordValue\right)
 \in\{0,1\}.
\end{align*}
In particular,
\[
 \OpObservation{s}\ \perp\!\!\!\perp\ \OpSurvivorState{s+1}
 \ \mid\
 (\OpAnalyticField{s},\OpRetiredRecord{s})
 \qquad\text{for a query operation.}
\]

\emph{4. A commit observation depends only on retired edge components.}
Suppose \(A_s=(\mathsf{commit},y_t,B_t)\), so \(s=s_t\).  By
\cref{app:eq:unresolved-matching-update}, every unresolved edge that meets
\(B_t\) is included in \(\OpRetiredRecord{s}\).  Consequently every
surviving edge satisfies
\[
 e\in\OpUnresolvedMatching{s+1}
 \quad\Longrightarrow\quad e\cap B_t=\varnothing.
\]
Equivalently,
\[
 \OpUnresolvedVertices{s+1}\cap B_t=\varnothing.
\]
We now verify that the normalized midpoint law preserves this separation.

Fix a compatible instance \(I\).  For every edge
\(e\in\HiddenMatching\), define probability masses on
\(\Vocab^{e\cap B_t}\) by
\[
\begin{aligned}
 p_e(z_{e\cap B_t})
 &:=\HardEdgeLaw{e}{I}\!\left(
   X_{e\cap B_t}=z_{e\cap B_t}
   \middle|
   X_{e\cap G(\OpVisibleHistory{s})}
     =x_{e\cap G(\OpVisibleHistory{s})}\right),\\
 \widehat p_e(z_{e\cap B_t})
 &:=\prod_{j\in e\cap B_t}
       \HardOracleRowMass{I}
        {y_t(\OpVisibleHistory{s})}{j}{z_j}.
\end{aligned}
\]
Both are normalized masses.  For \(p_e\), this is the defining
normalization of a conditional law; for \(\widehat p_e\), it follows
coordinate by coordinate:
\[
 \sum_{z_{e\cap B_t}}p_e(z_{e\cap B_t})=1,
 \qquad
 \sum_{z_{e\cap B_t}}\widehat p_e(z_{e\cap B_t})
 =\prod_{j\in e\cap B_t}
   \left[\sum_{z_j\in\Vocab}
        \HardOracleRowMass{I}
         {y_t(\OpVisibleHistory{s})}{j}{z_j}\right]=1.
\]
When \(e\cap B_t=\varnothing\), there is one empty assignment and both
masses equal \(1\) on it.

Because the target factorizes over matching edges and the commit rule uses
the product of its requested oracle rows,
\begin{align*}
 \HardExactBatchLaw{t}{I}(z\mid\OpVisibleHistory{s})
 &=
 \frac{\displaystyle
  \prod_{e\in\HiddenMatching}
   \HardEdgeLaw{e}{I}\!\left(
    X_{e\cap B_t}=z_{e\cap B_t},\,
    X_{e\cap G(\OpVisibleHistory{s})}
      =x_{e\cap G(\OpVisibleHistory{s})}\right)}
 {\displaystyle
  \prod_{e\in\HiddenMatching}
   \HardEdgeLaw{e}{I}\!\left(
    X_{e\cap G(\OpVisibleHistory{s})}
      =x_{e\cap G(\OpVisibleHistory{s})}\right)}\\
 &=\prod_{e\in\HiddenMatching}p_e(z_{e\cap B_t}),\\
 \HardProductBatchLaw{t}{I}(z\mid\OpVisibleHistory{s})
 &=\prod_{e\in\HiddenMatching}\widehat p_e(z_{e\cap B_t}).
\end{align*}
The first equality conditions the edge-product target on the already
committed values; the denominators are positive by strict positivity of the
hard family.
The blocks \(e\cap B_t\), including empty ones, are disjoint and have union
\(B_t\).  Thus restriction to these blocks and the distributive law give,
for arbitrary functions \(f_e\), the finite-product identity
\[
 \sum_{z\in\Vocab^{B_t}}
   \prod_{e\in\HiddenMatching}f_e(z_{e\cap B_t})
 =
 \prod_{e\in\HiddenMatching}
   \left[
    \sum_{u\in\Vocab^{e\cap B_t}}f_e(u)
   \right].
\]
Applying it with \(f_e(u)=\sqrt{p_e(u)\widehat p_e(u)}\) gives the
midpoint normalizer.  Strict positivity of the hard family and frozen oracle
gives
\(\sum_u\sqrt{p_e(u)\widehat p_e(u)}>0\) for every edge, so all following
componentwise divisions are valid.

Substituting this normalizer into
\cref{app:eq:hard-midpoint-kernel} now gives
\[
\begin{aligned}
 \HardMidpointKernel{t}{I}
   (z\mid\OpVisibleHistory{s})
 &=\frac{\displaystyle
    \prod_{e\in\HiddenMatching}
       \sqrt{p_e(z_{e\cap B_t})\widehat p_e(z_{e\cap B_t})}}
   {\displaystyle
    \prod_{e\in\HiddenMatching}
     \left[
      \sum_{u\in\Vocab^{e\cap B_t}}
       \sqrt{p_e(u)\widehat p_e(u)}
     \right]}\\
 &=\prod_{e\in\HiddenMatching}
   \frac{\sqrt{p_e(z_{e\cap B_t})\widehat p_e(z_{e\cap B_t})}}
        {\displaystyle
         \sum_{u\in\Vocab^{e\cap B_t}}
           \sqrt{p_e(u)\widehat p_e(u)}}.
\end{aligned}
\]
For a surviving edge, \(e\cap B_t=\varnothing\), so its normalized factor
is \(1\).  Every nontrivial factor belongs either to an edge retired earlier,
whose identity and triggers are in \(\OpAnalyticField{s}\), or to an edge retired
now, whose identity and triggers are in \(\RetiredRecordValue\).  Thus the midpoint
probability of the observed commit value does not depend on
\(\OpSurvivorState{s+1}\).  The accompanying oracle rows have the same
property: rows on previously retired edges are fixed by
\(\OpAnalyticField{s}\), and rows on newly retired edges are fixed after
conditioning on \(\RetiredRecordValue\).

More explicitly, fix a compatible
\(\SurvivorConfigA\).  Together with
\((\OpAnalyticField{s},\OpRetiredRecord{s}=\RetiredRecordValue)\), it
specifies the compatible instance \(I\).  For a commit observation
\((\mathbf q,z)\),
\begin{align*}
 &\mathbb P\!\left(
   \OpObservation{s}=(\mathbf q,z)
   \,\middle|\,
   \OpAnalyticField{s},
   \OpRetiredRecord{s}=\RetiredRecordValue,
   \OpSurvivorState{s+1}=\SurvivorConfigA\right)\\
 &\quad=
 \mathbf 1\!\left\{
  \mathbf q=
  \bigl(
   \HardOracleRow{I}{y_t(\OpVisibleHistory{s})}{j}
  \bigr)_{j\in B_t}
 \right\}
 \HardMidpointKernel{t}{I}
   (z\mid\OpVisibleHistory{s}).
\end{align*}
Every row and every nonunit edge factor on the right is determined by
\((\OpAnalyticField{s},\OpRetiredRecord{s})\); changing
\(\SurvivorConfigA\) changes only survivor-edge factors, which equal \(1\).
Thus the displayed likelihood is the same for all compatible survivor
configurations.  Hence
\[
 \OpObservation{s}\ \perp\!\!\!\perp\ \OpSurvivorState{s+1}
 \ \mid\
 (\OpAnalyticField{s},\OpRetiredRecord{s})
 \qquad\text{for a commit operation.}
\]

\emph{5. Close the induction.}
Step~3 proves for a query, and Step~4 proves for a commit, that the
conditional law of the new observation does not change when the new
survivor state is added to the conditioning information:
\begin{align*}
 &\mathcal L\!\left(
   \OpObservation{s}\,\middle|\,
   \OpAnalyticField{s},\OpRetiredRecord{s},
   \OpSurvivorState{s+1}\right)\\
 &\qquad=
 \mathcal L\!\left(
   \OpObservation{s}\,\middle|\,
   \OpAnalyticField{s},\OpRetiredRecord{s}\right).
\end{align*}
Thus, for every pair
\((\RetiredRecordValue,o)\) of positive conditional probability and every
compatible survivor configuration \(\SurvivorConfigA\),
\begin{align*}
 &\mathbb P\!\left(
   \OpSurvivorState{s+1}=\SurvivorConfigA
   \,\middle|\,
   \OpAnalyticField{s},
   \OpRetiredRecord{s}=\RetiredRecordValue,
   \OpObservation{s}=o\right)\\
 &\qquad=
 \mathbb P\!\left(
   \OpSurvivorState{s+1}=\SurvivorConfigA
   \,\middle|\,
   \OpAnalyticField{s},
   \OpRetiredRecord{s}=\RetiredRecordValue\right)\\
 &\qquad=
 \frac{1}{|\operatorname{PM}(\OpUnresolvedVertices{s+1})|}
 \prod_{i\in\OpUnresolvedVertices{s+1}}
   \frac{1}{|\OpRemainingTriggers{i}{s+1}|}.
\end{align*}
The last equality is the normalization computed in Step~2.

It remains only to identify the conditioning field.  The selected action
\(A_s\) is \(\OpAnalyticField{s}\)-measurable, and the next unresolved set is
recovered from the current set and retirement record by
\[
 \OpUnresolvedVertices{s+1}
 =
 \OpUnresolvedVertices{s}
 \setminus
 \bigcup_{(e,\cdot)\in\OpRetiredRecord{s}}e.
\]
The current unresolved set is itself recoverable from
\(\OpAnalyticHistory{s}\) by the previously recorded retirements.  Expanding
\cref{app:eq:analytic-record-update,app:eq:analytic-field-and-survivor} and
using these two measurability facts gives
\begin{align*}
 \OpAnalyticField{(s+1)}
 &=\sigma\!\left(
   \OpAnalyticHistory{(s+1)},\OpUnresolvedVertices{s+1}\right)\\
 &=\sigma\!\left(
   \OpAnalyticHistory{s},A_s,\OpObservation{s},
   \OpRetiredRecord{s},\OpUnresolvedVertices{s+1}\right)\\
 &=\sigma\!\left(
   \OpAnalyticField{s},\OpRetiredRecord{s},\OpObservation{s}\right).
\end{align*}
Consequently, the preceding conditional mass is also the conditional mass
given \(\OpAnalyticField{(s+1)}\).  Therefore
\[
 \OpSurvivorState{s+1}\mid\OpAnalyticField{(s+1)}
 \sim
 \operatorname{Unif}\bigl(
   \operatorname{PM}(\OpUnresolvedVertices{s+1})\bigr)
 \otimes
 \bigotimes_{i\in\OpUnresolvedVertices{s+1}}
   \operatorname{Unif}(\OpRemainingTriggers{i}{s+1}).
\]
This proves the product law at every operation boundary.

Finally, the event that operation \(s\) is commit \(t\) is known at the
starting boundary:
\[
 \{s_t=s\}
 =\left\{\sum_{r<s}\mathbf 1\{A_r\text{ is a commit}\}=t-1,
             \ A_s\text{ is a commit}\right\}
 \in\OpAnalyticField{s}.
\]
Applying the operation-boundary law on each event \(\{s_t=s\}\), and using
\(\OpSurvivorState{s_t}=\SurvivorState{t}\) and
\(\OpAnalyticField{s_t}=\AnalyticField{t}\), proves
\eqref{app:eq:strong-survivor-invariant}.
\end{proof}
}

\begin{figure}[!htbp]
 \centering
 \includegraphics[width=\linewidth]{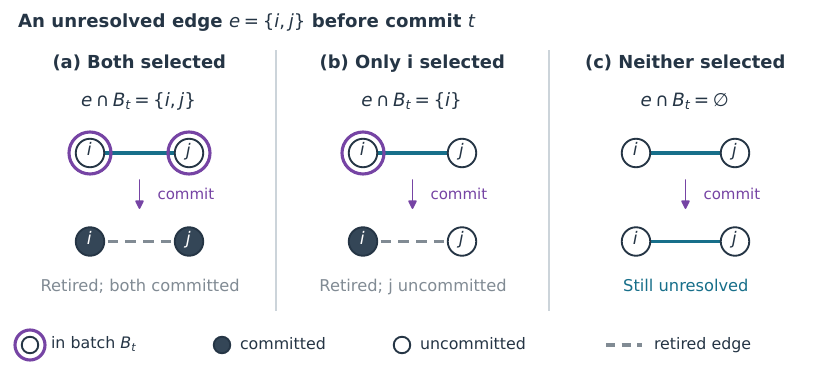}
 \caption{Three possible batch selections on an initially unresolved edge.
 Each panel shows the edge before (top) and after (bottom) one commit.
 Purple rings mark membership in $B_t$; filled nodes are committed.
 Selecting both endpoints gives a collision; selecting one gives a crossing
 edge. Both events retire the whole edge, but in (b) $j$ remains uncommitted.
 Selecting neither endpoint leaves the edge unresolved. The proof records
 the edge identity and both triggers in (a)--(b); the target matching itself
 is unchanged.}
 \label{fig:c4-matching-states}
\end{figure}

\paragraph{One source, at most one candidate per submission.}
Recall from \cref{app:eq:unresolved-matching-update,app:eq:retirement-record}
that the number of query-retired edges is
\[
 \GenieDeletions
 =\sum_{s:\,A_s\text{ is a query}}
   |\OpUnresolvedMatching{s}\setminus\OpUnresolvedMatching{s+1}|.
\]
Commit retirements do not contribute, and simultaneous hits at both endpoints
count as one edge.
A state \(y\) tests at most one candidate \(\HardTest{y}{i}\) at each source
\(i\), regardless of how many returned rows are inspected.
Together with the conditional uniformity in
\cref{app:lem:unresolved-invariant}, this gives the following bound directly.

\begin{lemma}[Equality-test disclosure]
\label{app:lem:query-disclosure}
For any algorithm using at most \(\QueryBudget{}\) counterfactual submissions,
\begin{equation}
 \mathbb E\GenieDeletions\le\frac{N\QueryBudget{}}{\HardBankSize}.
 \label{app:eq:expected-query-deletions}
\end{equation}
\end{lemma}

\begin{proof}
All probabilities and expectations use the joint instance--midpoint law
\eqref{app:eq:joint-instance-midpoint-law}, with the decision-rule seed
\(w\) fixed. The execution below is the midpoint execution.
It suffices to consider integer \(\QueryBudget{}\): for a noninteger
budget, apply the integer result to \(\lfloor\QueryBudget{}\rfloor\),
whose bound is no larger.
If \(\QueryBudget{}=0\), there are no query hits and
\(\GenieDeletions=0\). If \(\QueryBudget{}\ge\HardBankSize\), the bound
follows from \(\GenieDeletions\le N/2\le N\QueryBudget{}/\HardBankSize\).
Hence assume \(1\le\QueryBudget{}<\HardBankSize\).

\emph{1. Bound the hit probability at one query.}
Fix a vertex \(i\). Number counted counterfactual submissions by
\(k=1,2,\ldots\); commits may occur between them.
The distinguished all-mask preprocessing readout tests no candidate
and is not included in this count.
At the \(k\)-th such query, let \(s\) be its operation index and write
\(y=\OpSubmittedState{s}\), as in \cref{app:eq:operation-action}.
Recall that
\(\OpUnresolvedVertices{s}=\bigcup_{e\in\OpUnresolvedMatching{s}}e\)
is the unresolved vertex set just before this operation
(\cref{app:eq:unresolved-vertices}).
The past and selected action determine whether operation \(s\) is
the \(k\)-th counted query. Thus the operation-boundary conclusion of
\cref{app:lem:unresolved-invariant} applies here even when the query
time and state are chosen adaptively.

If \(i\in\OpUnresolvedVertices{s}\), each of the preceding \(k-1\)
counted submissions has removed at most one candidate at \(i\).
Commits leave the candidate set of a surviving vertex unchanged, by
\cref{app:eq:operation-candidates}. Consequently,
\[
 |\OpRemainingTriggers{i}{s}|
 \ge \HardBankSize-(k-1)=\HardBankSize-k+1.
\]
Conditional on \(\OpAnalyticField{s}\), the submitted state is fixed
and \(\TriggerIndex{i}\) is uniform on
\(\OpRemainingTriggers{i}{s}\). Therefore
\begin{align}
 \mathbb P(\HardHit{y}{i}=1\mid\OpAnalyticField{s})
 &=\mathbb P(\TriggerIndex{i}=\HardTest{y}{i}
                  \mid\OpAnalyticField{s})\notag\\
 &=\frac{\mathbf 1\{\HardTest{y}{i}\in\OpRemainingTriggers{i}{s}\}}
         {|\OpRemainingTriggers{i}{s}|}
 \le\frac1{\HardBankSize-k+1}.
 \label{app:eq:query-conditional-hit-bound}
\end{align}
A mask, background token, or previously excluded candidate gives numerator
zero. If the edge has already retired, there can be no further hit
\emph{while unresolved}, regardless of later oracle replies.

\emph{2. Accumulate the probabilities without assuming independence.}
For \(0\le k\le\QueryBudget{}\), define
\[
 p_{i,k}:=
 \mathbb P\!\left(
 \begin{gathered}
 i\text{ has a query hit while its edge is unresolved}\\
 \text{among the first }k\text{ counted submissions}
 \end{gathered}
 \right),
 \qquad p_{i,0}=0.
\]
If the execution contains fewer than \(k\) counted submissions, this event
uses all those that occurred; no extra queries are executed.
A hit retires the edge, so \(i\) can have at most one hit while unresolved.
The increment \(p_{i,k}-p_{i,k-1}\) is therefore the probability of such
a hit at query \(k\).

To bound this increment, condition on each possible analytic history
just before query \(k\), apply
\eqref{app:eq:query-conditional-hit-bound} when \(i\) is unresolved,
and average over these histories. This gives
\begin{align}
 p_{i,k}-p_{i,k-1}
 &\le
 \frac{\mathbb P(\text{query }k\text{ occurs and }i
                  \text{ is unresolved just before it})}
      {\HardBankSize-k+1}\notag\\
 &\le\frac{1-p_{i,k-1}}{\HardBankSize-k+1}.
 \label{app:eq:query-hit-recurrence}
\end{align}
For the last inequality, an unresolved vertex has not hit earlier.
Earlier retirement by a mate hit or a commit, as well as early termination,
can only remove paths from the event in the numerator.
Simultaneous hits at \(i\) and its mate still count as a hit at \(i\):
the edge was unresolved immediately before that query.

Rearranging \eqref{app:eq:query-hit-recurrence} step by step yields
\begin{align*}
 1-p_{i,k}
 &=1-p_{i,k-1}-(p_{i,k}-p_{i,k-1})\\
 &\ge(1-p_{i,k-1})
       \left(1-\frac1{\HardBankSize-k+1}\right)\\
 &=(1-p_{i,k-1})
       \frac{\HardBankSize-k}{\HardBankSize-k+1}.
\end{align*}
Starting from \(1-p_{i,0}=1\) and applying this inequality successively,
\begin{align*}
 1-p_{i,\QueryBudget{}}
 &\ge
 \prod_{k=1}^{\QueryBudget{}}
       \frac{\HardBankSize-k}{\HardBankSize-k+1}\\
 &=\frac{\HardBankSize-1}{\HardBankSize}
   \frac{\HardBankSize-2}{\HardBankSize-1}\cdots
   \frac{\HardBankSize-\QueryBudget{}}
        {\HardBankSize-\QueryBudget{}+1}\\
 &=\frac{\HardBankSize-\QueryBudget{}}{\HardBankSize}.
\end{align*}
For \(\QueryBudget{}=1\) the product has just its first factor.
The cancellation therefore gives
\begin{equation}
 p_{i,\QueryBudget{}}
 \le 1-\frac{\HardBankSize-\QueryBudget{}}{\HardBankSize}
 =\frac{\QueryBudget{}}{\HardBankSize}.
 \label{app:eq:query-source-hit-bound}
\end{equation}
This multiplication iterates conditional probability bounds; it does not
assume independent queries.

\emph{3. Count retired edges by their successful endpoints.}
At a query, an unresolved edge retires exactly when at least one of its
endpoints hits. For binary bits \(b_u,b_v\),
\(\mathbf 1\{b_u+b_v\ge1\}\le b_u+b_v\); two simultaneous hits still retire
only one edge. Hence, pathwise,
\[
 \begin{aligned}
 \GenieDeletions
 &=\sum_{s:\,A_s\text{ is a query}}
   \sum_{\{u,v\}\in\OpUnresolvedMatching{s}}
   \mathbf 1\{\HardHit{\OpSubmittedState{s}}{u}
              +\HardHit{\OpSubmittedState{s}}{v}\ge1\}\\
 &\le
   \sum_{i=1}^N
   \sum_{\substack{s:\,A_s\text{ is a query}\\
                    i\in\OpUnresolvedVertices{s}}}
             \HardHit{\OpSubmittedState{s}}{i}.
 \end{aligned}
\]
The change from edges to vertices uses that every unresolved vertex lies
on exactly one matching edge.
For each \(i\), the inner sum is either zero or one, because its first
hit retires its edge. It is thus the indicator of the event defining
\(p_{i,\QueryBudget{}}\). Taking expectations and using
\(\mathbb E\mathbf 1_E=\mathbb P(E)\) gives
\[
 \begin{aligned}
 \mathbb E\GenieDeletions
 &\le
 \sum_{i=1}^N\mathbb E\!\left[
   \sum_{\substack{s:\,A_s\text{ is a query}\\
                    i\in\OpUnresolvedVertices{s}}}
             \HardHit{\OpSubmittedState{s}}{i}\right]\\
 &=\sum_{i=1}^N p_{i,\QueryBudget{}}
 \le\sum_{i=1}^N\frac{\QueryBudget{}}{\HardBankSize}
 =\frac{N\QueryBudget{}}{\HardBankSize}.
 \end{aligned}
\]
No independence between vertices is used.
\end{proof}
\FloatBarrier

\subsection{Common collision bookkeeping}
\label{app:lower-proof-collisions}

We count how matching edges leave the unresolved-edge set; this pathwise
identity supplies the batch-size constraint in the collision lower-tail
proof. Recall that \(\UnresolvedVertices{t}\) contains the vertices on
edges still unresolved immediately before commit \(t\), after any preceding
queries. At this boundary, operation \(s_t\), the unresolved edge set is
\(\OpUnresolvedMatching{s_t}\), as in
\cref{app:eq:unresolved-matching-update,app:eq:remaining-trigger-set}.
Put \(\UnresolvedVertexCount{t}:=|\UnresolvedVertices{t}|\).
Removing whole edges makes this number even.

For the selected batch \(B_t\), define its active part and collision count:
\begin{equation}
 \ActiveBatch{t}:=B_t\cap\UnresolvedVertices{t},
 \qquad
 \ActiveBatchSize{t}:=|\ActiveBatch{t}|,
 \qquad
 \CollisionCount{t}:=
 \bigl|\{e\in\OpUnresolvedMatching{s_t}:
               e\subseteq\ActiveBatch{t}\}\bigr|,
 \qquad
 \TotalCollisions:=\sum_t\CollisionCount{t}.
 \label{app:eq:unresolved-batch-and-collisions}
\end{equation}
An unresolved edge has no committed endpoint, so
\(\UnresolvedVertices{t}\subseteq\ResidualPositionsAt{G_{t-1}}\).
The inclusion can be strict: query-retired endpoints and the uncommitted
mate of a commit-retired endpoint remain uncommitted but are no longer
active. Figure~\ref{fig:c4-matching-states}(b) shows the latter case.
A later commit of that mate creates no new unresolved-edge collision.

Let \(\ActiveRoundCount:=|\{t:\ActiveBatchSize{t}\ge1\}|\le\RoundBudget{}\)
be the number of rounds with a nonempty active batch.

\begin{lemma}[Pathwise edge-counting identity]
\label{app:lem:pathwise-edge-counting}
Under the retirement rules in \cref{app:eq:unresolved-matching-update},
every complete execution satisfies the pathwise identity
\begin{equation}
 \frac N2
 =\GenieDeletions+
  \sum_t(\ActiveBatchSize{t}-\CollisionCount{t}).
 \label{app:eq:lower-edge-ledger}
\end{equation}

Consequently,
\begin{equation}
 \sum_{t:\ActiveBatchSize{t}\ge1}
   (\ActiveBatchSize{t}-1)
 =\frac N2-\GenieDeletions+\TotalCollisions-\ActiveRoundCount
 \ge\frac N2-\GenieDeletions-\RoundBudget{}.
 \label{app:eq:active-excess-ledger}
\end{equation}
\end{lemma}

\begin{proof}
An active batch touches exactly
\(\ActiveBatchSize{t}-\CollisionCount{t}\) previously unresolved edges.
Indeed, its \(\CollisionCount{t}\) internal edges use
\(2\CollisionCount{t}\) vertices, while every remaining active vertex lies
on a distinct crossing edge. Thus the number touched is
\[
 \underbrace{\CollisionCount{t}}_{\text{internal edges}}
 +\underbrace{(\ActiveBatchSize{t}-2\CollisionCount{t})}_{\text{crossing edges}}
 =\ActiveBatchSize{t}-\CollisionCount{t}.
\]
Since every position is eventually committed, every original matching
edge is retired exactly once: either at a successful counterfactual
trigger test (query retirement), or at its first active commit batch
(commit retirement). Summing over the \(N/2\) original edges proves
\eqref{app:eq:lower-edge-ledger}.

Finally,
\(\sum_{t:\ActiveBatchSize{t}\ge1}(\ActiveBatchSize{t}-1)
=\sum_t\ActiveBatchSize{t}-\ActiveRoundCount\).
Substituting
\(\sum_t\ActiveBatchSize{t}=N/2-\GenieDeletions+\TotalCollisions\)
from \eqref{app:eq:lower-edge-ledger}, then using
\(\TotalCollisions\ge0\) and \(\ActiveRoundCount\le\RoundBudget{}\),
gives \eqref{app:eq:active-excess-ledger}.
\end{proof}

\subsection{A Laplace bound for uniform-matching collisions}
\label{app:lower-proof-matching-laplace}

The next lemma is a finite combinatorial statement independent of the hard
oracle and the adaptive decision rule. It gives a conditional
Laplace bound used below to derive an adaptive lower-tail estimate.

\paragraph{Intuition.}
For a fixed set of $k$ vertices in a uniform matching on $n$ vertices,
each unordered pair is an edge with probability $1/(n-1)$.  Hence the
expected number of internal edges is
\[
 \binom{k}{2}\frac1{n-1}
 =\frac{k(k-1)}{2(n-1)}
 \simeq\frac{k^2}{n}\qquad(k\ge2).
\]
To obtain a lower-tail estimate rather than only a mean, the proof exposes
order $k$ edges, each with conditional success probability at least order
$k/n$.  Their Laplace factors multiply by successive conditioning; the
exposures need not be independent.

\begin{lemma}[Internal-edge Laplace transform]
\label{app:lem:matching-collision-laplace}
Let \(M\) be a uniformly random perfect matching on a labeled set \(U\) of
even cardinality \(n\ge 2\). Fix \(S\subseteq U\), put \(k:=|S|\), and let
\[
 \InternalEdgeCount{S}:=|\{e\in M:e\subseteq S\}|.
\]
Then, for every \(t>0\),
\begin{equation}
 \mathbb E e^{-t\InternalEdgeCount{S}}
 \le
 \exp\left\{
  -\frac{1-e^{-t}}{32}\frac{\PositivePart{k-1}^2}{n}
 \right\}.
 \label{app:eq:matching-collision-laplace}
\end{equation}
\end{lemma}

\begin{proof}

\emph{1. Handle sets of size at most three.}
For \(k\in\{0,1\}\), both sides of
\eqref{app:eq:matching-collision-laplace} equal one. For
\(k\in\{2,3\}\), two internal matching edges would require four distinct
vertices of \(S\). Thus \(\InternalEdgeCount{S}\in\{0,1\}\), and
\[
 \mathbb E\InternalEdgeCount{S}
 =0\cdot\mathbb P(\InternalEdgeCount{S}=0)
  +1\cdot\mathbb P(\InternalEdgeCount{S}=1)
 =\mathbb P(\InternalEdgeCount{S}=1).
\]
Also, writing the edge count as a sum of indicators and taking expectations,
\begin{equation}
 \mathbb E \InternalEdgeCount{S}
 =\sum_{\{u,v\}\subseteq S}\mathbb P(\{u,v\}\in M)
 =\frac{\binom{k}{2}}{n-1}
 =\frac{k(k-1)}{2(n-1)}.
 \label{app:eq:small-set-internal-edge-mean}
\end{equation}
Writing \(a_t:=1-e^{-t}\in(0,1)\), the same two possible values give
\begin{align*}
 \mathbb E e^{-t\InternalEdgeCount{S}}
 &=\mathbb P(\InternalEdgeCount{S}=0)
   +e^{-t}\mathbb P(\InternalEdgeCount{S}=1)\\
 &=1-(1-e^{-t})\mathbb P(\InternalEdgeCount{S}=1)
 =1-a_t\mathbb E\InternalEdgeCount{S}.
\end{align*}
Using \(1-v\le e^{-v}\) for \(v\ge0\), we obtain
\[
 \mathbb E e^{-t\InternalEdgeCount{S}}
 =1-a_t\mathbb E \InternalEdgeCount{S}
 \le e^{-a_t\mathbb E \InternalEdgeCount{S}}
 \le
 \exp\left\{-\frac{a_t}{32}
                   \frac{(k-1)^2}{n}\right\},
\]
where the last inequality follows directly from
\eqref{app:eq:small-set-internal-edge-mean} for \(k=2,3\).

\emph{2. Expose edges and bound each conditional success probability.}
It remains to consider \(k\ge 4\). Partition \(S=A\sqcup B\) with
\(|A|=\lfloor k/2\rfloor\) and \(|B|=\lceil k/2\rceil\), and put
\(m:=\lfloor k/4\rfloor\). We reveal edges of the original matching
\(M\) on the full set \(U\), not of a matching restricted to \(S\).
Only the starting endpoint is required to lie in \(A\); its mate can lie
in \(A\), \(B\), or \(U\setminus S\).

Fix an ordering of \(U\). Let \(U_0:=U\), and for \(j=1,\ldots,m\) define
\begin{align*}
 a_j&:=\min(A\cap U_{j-1}),
 &b_j&:=\operatorname{mate}_{M}(a_j),\\
 e_j&:=\{a_j,b_j\}\in M,
 &U_j&:=U_{j-1}\setminus\{a_j,b_j\}.
\end{align*}
Thus \(U_{j-1}\) consists of the vertices whose matching edges have not yet
been revealed; these vertices are already matched in \(M\).
The minimum uses the fixed ordering, so the choice of \(a_j\) depends only
on the earlier revealed edges. It is well-defined because each previous
edge removes at most two vertices from \(A\), and
\[
 |A\cap U_{j-1}|
 \ge\lfloor k/2\rfloor-2(j-1)
 \ge2m-2(j-1)
 \ge2>0
 \qquad(1\le j\le m).
\]

Define the success indicator by
\[
 X_j:=\mathbf1_{\{b_j\in B\cap U_{j-1}\}},
 \qquad
 \{X_j=1\}=\{\operatorname{mate}_{M}(a_j)\in B\cap U_{j-1}\}.
\]
Since every starting endpoint \(a_\ell\) lies in \(A\), each previous
edge \(e_\ell\) removes at most one vertex from \(B\). Consequently,
\[
 |B\cap U_{j-1}|
 \ge\lceil k/2\rceil-(j-1)
 \ge k/2-(m-1)
 \ge k/4,
 \qquad
 |U_{j-1}|=n-2(j-1).
\]
Conditional on \(e_1,\ldots,e_{j-1}\), the remaining matching is uniform
on \(U_{j-1}\): every perfect matching there has exactly one extension
by the already revealed edges. Each possible mate of \(a_j\) has the same
number of matching completions, so
\[
 \mathbb P(b_j=v\mid e_1,\ldots,e_{j-1})
 =\frac1{|U_{j-1}|-1}
 \quad\text{for }v\in U_{j-1}\setminus\{a_j\}.
\]
Summing over the allowable mates in \(B\), which never include
\(a_j\in A\), gives
\begin{align*}
 \mathbb P(X_j=1\mid e_1,\ldots,e_{j-1})
 &=\sum_{v\in B\cap U_{j-1}}
   \mathbb P(b_j=v\mid e_1,\ldots,e_{j-1})\\
 &=\frac{|B\cap U_{j-1}|}{|U_{j-1}|-1}
 \ge\frac{k/4}{n}=\frac{k}{4n}.
\end{align*}
The earlier indicators \(X_1,\ldots,X_{j-1}\) are determined by these
revealed edges. Adding them to the conditioning therefore gives the same
probability:
\begin{equation}
 \mathbb P(X_j=1\mid X_1,\ldots,X_{j-1},
               \text{all preceding exposed edges})
 \ge u:=\frac{k}{4n}.
 \label{app:eq:matching-cross-success}
\end{equation}
Removing both endpoints after each exposure makes \(e_1,\ldots,e_m\)
distinct. If \(X_j=1\), then \(a_j\in A\) and \(b_j\in B\), so
\(e_j\subseteq A\cup B=S\). Therefore
\[
 \sum_{j=1}^mX_j
 =|\{e_j:1\le j\le m,\ X_j=1\}|
 \le|\{e\in M:e\subseteq S\}|=\InternalEdgeCount{S}.
\]

\emph{3. Iterate conditional Laplace factors without independence.}
For one exposure, the Bernoulli identity and the preceding probability
bound give
\begin{align*}
 &\mathbb E[e^{-tX_j}\mid\text{preceding exposed edges}]\\
 &\qquad=1-(1-e^{-t})\Pr(X_j=1\mid\text{preceding exposed edges})\\
 &\qquad\le1-u(1-e^{-t}).
\end{align*}
For \(1\le r\le m\), condition on the first \(r-1\) exposed edges.
The earlier factors are measurable with respect to this record, so
\begin{align*}
 \mathbb E e^{-t\sum_{j=1}^{r}X_j}
 &=\mathbb E\!\left[
   e^{-t\sum_{j=1}^{r-1}X_j}
   \mathbb E(e^{-tX_r}\mid\text{first \(r-1\) exposed edges})\right]\\
 &\le[1-u(1-e^{-t})]\,
      \mathbb E e^{-t\sum_{j=1}^{r-1}X_j}.
\end{align*}
Starting with the empty sum at \(r=0\), whose exponential is one, and
iterating to \(r=m\) gives the power below without an independence
assumption:
\begin{align}
 \mathbb E e^{-t\InternalEdgeCount{S}}
 &\le \mathbb E\exp\left\{-t\sum_{j=1}^mX_j\right\}\notag\\
 &\le\{1-u(1-e^{-t})\}^{m}
 \le\exp\{-mu(1-e^{-t})\}.
 \label{app:eq:matching-exposure-laplace}
\end{align}
Since \(m\ge k/8\) for \(k\ge4\),
\(mu\ge k^2/(32n)\ge(k-1)^2/(32n)\). Substitution in
\eqref{app:eq:matching-exposure-laplace} proves
\eqref{app:eq:matching-collision-laplace}.
\end{proof}

\subsection{An adaptive lower tail for batch collisions}
\label{app:lower-proof-collision-tail}

The conditional Laplace bound and the edge-counting identity now force many collisions
with constant probability under the joint instance--midpoint law. This is
the input needed to accumulate edge-level discrepancy nonlinearly; no
independence between commit rounds is assumed. Recall that
\(\TotalCollisions=\sum_t\CollisionCount{t}\) counts the unresolved edges
whose two endpoints are committed in the same active batch, as defined in
\eqref{app:eq:unresolved-batch-and-collisions}.

\paragraph{Intuition.}
If order $N$ unresolved vertices were spread evenly across $\RoundBudget$
batches while the unresolved pool had size order $N$, the collision scale
would be
\[
 \RoundBudget\frac{(N/\RoundBudget)^2}{N}
 =\frac{N}{\RoundBudget}.
\]
The proof does not assume even batches or a fixed pool: the edge-counting identity
forces enough total active batch mass, Cauchy--Schwarz gives the same
$N/\RoundBudget$ scale, and the conditional Laplace bound turns that
pathwise mass into a constant-probability collision event.

\begin{proposition}[Adaptive collision lower tail]
\label{app:prop:collision-lower-tail}
Fix a decision-rule seed under the joint law
\eqref{app:eq:joint-instance-midpoint-law}. If
\(\QueryBudget{}\le \HardBankSize/8\) and \(1\le \RoundBudget{}\le N/8\), then
\begin{equation}
 \mathbb P\left(
   \TotalCollisions\ge\frac{N}{8192\RoundBudget{}}
 \right)
 \ge
 \frac12-\exp\left\{-\frac{N}{8192\RoundBudget{}}\right\}.
 \label{app:eq:collision-lower-tail}
\end{equation}
In particular, if \(\RoundBudget{}\le N/16384\), then
\begin{equation}
 \mathbb P\left(
   \TotalCollisions\ge\frac{N}{8192\RoundBudget{}}
 \right)\ge\frac14.
 \label{app:eq:collision-quarter-probability}
\end{equation}
\end{proposition}

\begin{proof}

\emph{1. Fix the pre-exposure information and conditional Laplace bound.}
For an actual commit round \(t\), let \(B_t\) be the batch selected by the
seed-fixed decision rule. Recall the existing notation:
\[
 \ActiveBatch{t}=B_t\cap\UnresolvedVertices{t},
 \qquad
 \UnresolvedVertexCount{t}=|\UnresolvedVertices{t}|,
 \qquad
 \ActiveBatchSize{t}=|\ActiveBatch{t}|,
\]
\[
 \CollisionCount{t}
 =|\{e\in\OpUnresolvedMatching{s_t}:
                    e\subseteq\ActiveBatch{t}\}|,
 \qquad
 \TotalCollisions=\sum_t\CollisionCount{t}.
\]
Thus \(B_t\) is the full commit batch, whereas \(\ActiveBatch{t}\)
contains only its currently unresolved vertices.

After termination, append empty rounds until the total number of indexed
rounds is \(\RoundBudget{}\). In every appended round set
\(B_t=\ActiveBatch{t}=\UnresolvedVertices{t}=\varnothing\), so
\(\UnresolvedVertexCount{t}=\ActiveBatchSize{t}=\CollisionCount{t}=0\).
These rounds add no collision, and hence
\(\TotalCollisions=\sum_{t=1}^{\RoundBudget{}}\CollisionCount{t}\).
Introduce only the following additional abbreviation:
\[
 A_t:=
 \begin{cases}
  \PositivePart{\ActiveBatchSize{t}-1}^2/\UnresolvedVertexCount{t},
      &\UnresolvedVertexCount{t}>0,\\
  0,&\UnresolvedVertexCount{t}=0.
 \end{cases}
\]
The pre-commit analytic field \(\AnalyticField{t}\) determines both the
selected batch \(B_t\) and the unresolved set \(\UnresolvedVertices{t}\),
hence also \(\ActiveBatch{t}=B_t\cap\UnresolvedVertices{t}\).
Selecting \(B_t\) adds no observation: the matching edges incident to
\(\ActiveBatch{t}\) and the midpoint values \(X_{B_t}\) have not yet been
exposed. Thus \(\UnresolvedVertexCount{t}\), \(\ActiveBatchSize{t}\),
\(A_t\), and all quantities from earlier rounds are
\(\AnalyticField{t}\)-measurable. For appended empty rounds, use the terminal
field.

By \cref{app:lem:unresolved-invariant}, conditional on \(\AnalyticField{t}\),
the matching on \(\UnresolvedVertices{t}\) is uniform. For
\(\UnresolvedVertexCount{t}\ge2\), apply
\cref{app:lem:matching-collision-laplace} with
\(U=\UnresolvedVertices{t}\), \(S=\ActiveBatch{t}\),
\(n=\UnresolvedVertexCount{t}\), \(k=\ActiveBatchSize{t}\), and Laplace
parameter one. Since \(1-e^{-1}\ge1/2\), this gives
\begin{equation}
 \mathbb E\!\left[e^{-\CollisionCount{t}}\mid\AnalyticField{t}\right]
 \le
 \exp\left\{-\frac{1-e^{-1}}{32}A_t\right\}
 \le e^{-A_t/64}.
 \label{app:eq:conditional-collision-laplace}
\end{equation}
The unresolved set has even cardinality, so
\(\UnresolvedVertexCount{t}=1\) cannot occur. If
\(\UnresolvedVertexCount{t}=0\), then
\(\ActiveBatchSize{t}=\CollisionCount{t}=A_t=0\), and the same display holds
with equality.

\emph{2. Iterate the exponential expectation bound.}
Let \(r=0,\ldots,\RoundBudget{}\) count completed commit rounds, not
individual operations.  Define
\begin{equation}
 \CollisionExponential{r}:=
 \exp\left\{
   -\sum_{t=1}^{r}\CollisionCount{t}
   +\frac1{64}\sum_{t=1}^{r}A_t
 \right\},
 \qquad \CollisionExponential{0}:=1.
 \label{app:eq:collision-supermartingale}
\end{equation}
Both \(\CollisionExponential{r-1}\) and \(A_r\) are measurable with respect
to \(\AnalyticField{r}\), the field immediately before commit \(r\)
after any intervening counterfactual operations. Hence
\eqref{app:eq:conditional-collision-laplace} and the tower property give
\[
 \mathbb E \CollisionExponential{r}
 =
 \mathbb E\!\left[
   \CollisionExponential{r-1}e^{A_r/64}
   \mathbb E(e^{-\CollisionCount{r}}\mid\AnalyticField{r})
 \right]
 \le \mathbb E \CollisionExponential{r-1}.
\]
Iteration yields
\begin{equation}
 \mathbb E \CollisionExponential{\RoundBudget}\le1.
 \label{app:eq:terminal-collision-supermartingale}
\end{equation}
This is the only concentration step; adaptivity is absorbed into the
successive conditional expectations.

\emph{3. Turn edge counting into a pathwise Laplace budget.}
By \cref{app:lem:query-disclosure} and
\(\QueryBudget{}\le\HardBankSize/8\),
\[
 \mathbb E\GenieDeletions
 \le\frac{N\QueryBudget{}}{\HardBankSize}
 \le\frac{N(\HardBankSize/8)}{\HardBankSize}
 =\frac N8.
\]
Because \(\GenieDeletions\ge0\), Markov's inequality gives
\[
 \mathbb P(\GenieDeletions>N/4)
 \le\frac{\mathbb E\GenieDeletions}{N/4}
 \le\frac{N/8}{N/4}=\frac12.
\]
Taking the complement,
\(\mathbb P(\GenieDeletions\le N/4)=1-\mathbb P(\GenieDeletions>N/4)\),
therefore yields
\begin{equation}
 \mathbb P(\GenieDeletions\le N/4)\ge\frac12.
 \label{app:eq:few-genie-deletions-event}
\end{equation}
On this event, the pathwise counting identity
\eqref{app:eq:active-excess-ledger} in \cref{app:lem:pathwise-edge-counting}, the bounds
\(\ActiveRoundCount\le \RoundBudget{}\le N/8\), and
\(\TotalCollisions\ge0\) imply
\begin{equation}
 \sum_{t=1}^{\RoundBudget{}}\PositivePart{\ActiveBatchSize{t}-1}
 \ge \frac N2-\GenieDeletions-\RoundBudget{}
 \ge\frac N8.
 \label{app:eq:pathwise-active-excess}
\end{equation}
Since \(\UnresolvedVertexCount{t}\le N\), the definition of \(A_t\)
implies \(A_t\ge\PositivePart{\ActiveBatchSize{t}-1}^2/N\).
This also holds when \(\UnresolvedVertexCount{t}=0\), because then both
sides are zero. Cauchy--Schwarz over the \(\RoundBudget{}\) indexed rounds
(including appended empty rounds) now gives
\begin{align}
 \sum_{t=1}^{\RoundBudget{}}A_t
 &\ge\frac1N\sum_{t=1}^{\RoundBudget{}}\PositivePart{\ActiveBatchSize{t}-1}^2\notag\\
 &\ge\frac1{N\RoundBudget{}}
       \left(\sum_{t=1}^{\RoundBudget{}}\PositivePart{\ActiveBatchSize{t}-1}\right)^2
 \ge\frac{N}{64\RoundBudget{}}.
 \label{app:eq:pathwise-laplace-budget}
\end{align}

\emph{4. Subtract the small-collision event.}
Set \(x:=N/(8192\RoundBudget{})\). On the intersection
\(\{\GenieDeletions\le N/4\}\cap\{\TotalCollisions<x\}\),
\eqref{app:eq:pathwise-laplace-budget} yields
\[
 \log \CollisionExponential{\RoundBudget}
 =-\TotalCollisions+\frac1{64}\sum_{t=1}^{\RoundBudget{}}A_t
 >-x+\frac{N}{4096\RoundBudget{}}
 =x.
\]
Using \eqref{app:eq:terminal-collision-supermartingale} and Markov's
inequality,
\begin{align*}
 \mathbb P(\GenieDeletions\le N/4,\ \TotalCollisions<x)
 &\le \mathbb P(\CollisionExponential{\RoundBudget}>e^x)\\
 &\le\frac{\mathbb E\CollisionExponential{\RoundBudget}}{e^x}
 \le\frac1{e^x}=e^{-x}.
\end{align*}
Combining with \eqref{app:eq:few-genie-deletions-event},
\begin{align*}
 \mathbb P(\TotalCollisions\ge x)
 &\ge \mathbb P(\GenieDeletions\le N/4,\ \TotalCollisions\ge x)\\
 &=\mathbb P(\GenieDeletions\le N/4)
   -\mathbb P(\GenieDeletions\le N/4,\ \TotalCollisions<x)\\
 &\ge \frac12-e^{-x},
\end{align*}
which proves \eqref{app:eq:collision-lower-tail}.
If \(\RoundBudget{}\le N/16384\), then \(x\ge2\)
and \(1/2-e^{-2}>1/4\), proving
\eqref{app:eq:collision-quarter-probability}.
\end{proof}

\subsection{Collision-affinity bound}
\label{app:lower-proof-affinity}

We combine the midpoint identity with the one-edge Hellinger loss to
convert unresolved collisions into a multiplicative affinity penalty.
Recall our normalization: for probability masses \(p,q\) on the same
finite set \(\mathcal X\),
\(\HellingerAffinity(p,q):=\sum_{x\in\mathcal X}\sqrt{p(x)q(x)}\).
The squared Hellinger distance is
\(\sqHellinger(p,q):=\tfrac12\sum_{x\in\mathcal X}
(\sqrt{p(x)}-\sqrt{q(x)})^2=1-\HellingerAffinity(p,q)\).

\paragraph{Intuition.}
Each unresolved edge whose endpoints are committed together contributes
the factor $1-\HardEdgeDistance$.  Since
$\HardEdgeDistance\simeq\EdgeCoupling^2$,
\[
 1-(1-\HardEdgeDistance)^{\TotalCollisions}
 \simeq
 \min\{1,\EdgeCoupling^2\TotalCollisions\}.
\]
For justification, $1-(1-d)^z\le\min\{1,zd\}$ for integer $z\ge0$, while
$1-(1-d)^z\ge1-e^{-zd}\ge(1-e^{-1})\min\{1,zd\}$.
This is an edge-penalty scale; the next lemma places its expectation below
the actual squared-Hellinger risk.

\begin{lemma}[Affinity loss from unresolved collisions]
\label{app:lem:collision-affinity}
For every fixed decision-rule seed \(w\),
\begin{equation}
 \mathbb E_{I\sim\HardInstancePrior}
 \sqHellinger(\HardTargetLaw{I},\HardOutputLaw{I}{w})
 \ge
 \mathbb E_{I,H}
 \left[1-(1-\HardEdgeDistance)^{\TotalCollisions}\right],
 \label{app:eq:global-collision-affinity}
\end{equation}
where the expectation on the right is under
\eqref{app:eq:joint-instance-midpoint-law}.
\end{lemma}

\begin{proof}
\emph{1. Specify the edge components at a fixed history.}
Fix \(I,w\) and a reachable pre-commit history \(h\), and write
\(B_t=B_t(h)\). As in Step~4 of \cref{app:lem:unresolved-invariant}, the
two component laws on \(\Vocab^{e\cap B_t}\), for
\(e\in\HiddenMatching\), are
\begin{align*}
 p_e(z_{e\cap B_t})
 &:=\HardEdgeLaw{e}{I}
        (z_{e\cap B_t}\mid x_{e\cap G(h)}),\\
 \widehat p_e(z_{e\cap B_t})
 &:=\prod_{j\in e\cap B_t}
       \HardOracleRowMass{I}{y_t(h)}{j}{z_j}.
\end{align*}
The first line is the conditional probability of the current edge
coordinates given the already committed coordinates of that edge.
Both masses sum to one; on an empty coordinate set both equal one on the
empty assignment. The conditional product calculation in that step gives
\[
 \HardExactBatchLaw{t}{I}(z\mid h)
   =\prod_{e\in\HiddenMatching}p_e(z_{e\cap B_t}),
 \qquad
 \HardProductBatchLaw{t}{I}(z\mid h)
   =\prod_{e\in\HiddenMatching}\widehat p_e(z_{e\cap B_t}).
\]
These products are at fixed \(I,h\); each component is a block of batch
coordinates from one edge, not necessarily a single vertex.

\emph{2. Compute the factor of an unresolved collision.}
Partition the unresolved edges according to the number of selected endpoints:
\begin{equation}
 \mathcal E_t^{(r)}:=
 \{e\in\OpUnresolvedMatching{s_t}:
       |e\cap\ActiveBatch{t}|=r\},
 \qquad r\in\{0,1,2\}.
 \label{app:eq:commit-edge-partition}
\end{equation}
Thus the collision-edge set and its size are
\[
 \mathcal E_t^{(2)}
 =\{e\in\OpUnresolvedMatching{s_t}:
                       e\subseteq\ActiveBatch{t}\},
 \qquad
 |\mathcal E_t^{(2)}|=\CollisionCount{t},
 \qquad
 \ActiveBatch{t}=B_t\cap\UnresolvedVertices{t}.
\]
For \(e=\{i,j\}\in\mathcal E_t^{(2)}\), both endpoints belong to \(B_t\)
and neither belongs to \(G(h)\). Thus the exact component has no
conditioning within \(e\), while both oracle rows see their mate masked:
\[
 p_e(z_i,z_j)=\HardEdgeLaw{e}{I}(z_i,z_j),
 \qquad
 \widehat p_e(z_i,z_j)
 =\HardMarginal(z_i)\HardMarginal(z_j)
 =\HardIndependentEdgeLaw{e}(z_i,z_j).
\]
Here the oracle equality uses \eqref{app:eq:hard-oracle}, and the absence
of conditioning from other edges uses the target's edgewise product law.
Consequently,
\(\HellingerAffinity(p_e,\widehat p_e)
 =1-\sqHellinger(\HardEdgeLaw{e}{I},\HardIndependentEdgeLaw{e})
 =1-\HardEdgeDistance\), where \(\HardEdgeDistance\) is defined in
\eqref{app:eq:hard-edge-distance}. In terms of the actual edge masses,
this is
\begin{equation}
 \sum_{z_i,z_j}
 \sqrt{\HardEdgeLaw{e}{I}(z_i,z_j)
       \HardIndependentEdgeLaw{e}(z_i,z_j)}
 =1-\HardEdgeDistance.
 \label{app:eq:one-collision-affinity}
\end{equation}
\emph{3. Bound every remaining component and apply product affinity.}
Since \(p_e\) and \(\widehat p_e\) are probability masses,
\(0\le\HellingerAffinity(p_e,\widehat p_e)\le1\) by Cauchy--Schwarz.
This covers all noncollision components, including previously retired
edges where the oracle need not be exact. More specifically, an unresolved
crossing edge has \(p_e=\widehat p_e=\HardMarginal\), and an edge with
\(e\cap B_t=\varnothing\) has both masses equal to one on the empty
assignment; these components have affinity exactly one.

Apply \cref{app:lem:finite-hellinger-calculus} with index set
\(J=\HiddenMatching\), spaces \(\mathcal X_e=\Vocab^{e\cap B_t}\),
and component laws \(p_e,\widehat p_e\).
The two product laws to which it applies are exactly
\(\HardExactBatchLaw{t}{I}(\cdot\mid h)\) and
\(\HardProductBatchLaw{t}{I}(\cdot\mid h)\) from Step~1.
By the definition \eqref{app:eq:hard-local-affinity} and that lemma,
\begin{align*}
 \HardLocalAffinity{t}{I,h}
 &=\HellingerAffinity\!\left(
      \HardExactBatchLaw{t}{I}(\cdot\mid h),
      \HardProductBatchLaw{t}{I}(\cdot\mid h)\right)\\
 &=\prod_{e\in\HiddenMatching}\HellingerAffinity(p_e,\widehat p_e)\\
 &=\left(\prod_{e\in\mathcal E_t^{(2)}}(1-\HardEdgeDistance)\right)
   \left(\prod_{e\in\HiddenMatching\setminus\mathcal E_t^{(2)}}
                \HellingerAffinity(p_e,\widehat p_e)\right)\\
 &=(1-\HardEdgeDistance)^{\CollisionCount{t}}
   \prod_{e\in\HiddenMatching\setminus\mathcal E_t^{(2)}}
                \HellingerAffinity(p_e,\widehat p_e).
\end{align*}
Every factor in the remaining product lies in \([0,1]\), so
\begin{equation}
 \HardLocalAffinity{t}{I,h}
 \le(1-\HardEdgeDistance)^{\CollisionCount{t}}.
 \label{app:eq:local-collision-affinity}
\end{equation}
\emph{4. Multiply along the adaptive path and average.}
The preceding bound holds at each reached history, including histories
under the midpoint law. Multiplying over rounds and using
\(\TotalCollisions=\sum_t\CollisionCount{t}\) yields
\[
 \prod_t\HardLocalAffinity{t}{I,H_{t-1}}
 \le(1-\HardEdgeDistance)^{\TotalCollisions}.
\]
Using the adaptive identity \eqref{app:eq:adaptive-midpoint-identity}
first at fixed \(I,w\) and then averaging over \(I\) gives
\begin{align*}
 &\mathbb E_I\sqHellinger(\HardTargetLaw{I},\HardOutputLaw{I}{w})\\
 &\quad=1-\mathbb E_I
    \mathbb E_{H\sim\HardMidpointLaw{I}{w}}
       \prod_t\HardLocalAffinity{t}{I,H_{t-1}}\\
 &\quad\ge1-\mathbb E_{I,H}(1-\HardEdgeDistance)^{\TotalCollisions}\\
 &\quad=\mathbb E_{I,H}
       [1-(1-\HardEdgeDistance)^{\TotalCollisions}],
\end{align*}
where the joint expectation uses the conditional midpoint law given \(I,w\).
This proves \eqref{app:eq:global-collision-affinity}.
\end{proof}

\subsection{Nonlinear finite minimax bound}
\label{app:lower-proof-nonlinear}

\paragraph{Intuition.}
When the query budget is a small enough fraction of the bank size,
the preceding two results give
\[
 \text{risk}\ \gtrsim\
 1-\exp\{-cN\EdgeCoupling^2/\RoundBudget\}
 \simeq\min\{1,N\EdgeCoupling^2/\RoundBudget\},
\]
where $c>0$ is a fixed numerical constant.
The constant-probability collision event supplies order
$N/\RoundBudget$ missed edges; each contributes order $\EdgeCoupling^2$.
For a small target tolerance $\varepsilon$, this forces
$\RoundBudget\gtrsim N\EdgeCoupling^2/\varepsilon$.
The proof retains the finite constants and checks the prior-to-minimax
step explicitly.

\begin{proof}[Proof of \cref{app:thm:nonlinear-finite-lower}]

\emph{1. Obtain the nonlinear bound at a fixed seed.}
Fix an arbitrary algorithm in \(\HardAlgorithmClass{\QueryBudget{}}{\RoundBudget{}}\) and a decision-rule
seed \(w\). Suppose \(\QueryBudget{}\le\HardBankSize/8\) and
\(1\le \RoundBudget{}\le N/16384\). By
\cref{app:lem:collision-affinity,app:prop:collision-lower-tail},
\begin{align}
 &\mathbb E_{I\sim\HardInstancePrior}
 \sqHellinger(\HardTargetLaw{I},\HardOutputLaw{I}{w})\notag\\
 &\quad\ge
 \mathbb E_{I,H}
 \left[1-(1-\HardEdgeDistance)^{\TotalCollisions}\right]\notag\\
 &\quad\ge
 \frac14\left[
  1-(1-\HardEdgeDistance)^{N/(8192\RoundBudget{})}
 \right].
 \label{app:eq:nonlinear-fixed-seed-step}
\end{align}
The exponent in the last display need not be an integer: on the event
\(\TotalCollisions\ge N/(8192\RoundBudget{})\), monotonicity of
\(c\mapsto1-(1-\HardEdgeDistance)^c\) gives the displayed lower bound.
Using \(1-d\le e^{-d}\) for \(d\in[0,1]\) and the one-edge estimate
\(\HardEdgeDistance\ge\EdgeCoupling^2/12\) from
\eqref{app:eq:hard-edge-distance}, we obtain
\begin{equation}
 \mathbb E_{I\sim\HardInstancePrior}
 \sqHellinger(\HardTargetLaw{I},\HardOutputLaw{I}{w})
 \ge
 \frac14\left[
  1-\exp\left\{-\frac{N\EdgeCoupling^2}{98304\RoundBudget{}}\right\}
 \right].
 \label{app:eq:nonlinear-fixed-seed-lower}
\end{equation}

\emph{2. Average over seeds and pass from the prior to minimax risk.}
Finally, \(\dTV\ge\sqHellinger\) by
\cref{app:lem:finite-hellinger-calculus}. The decision-rule seed \(W\) is
independent of \(I\sim\HardInstancePrior\), so Tonelli's theorem gives
\begin{align}
 \mathbb E_I\SeedRisk(\mathcal A;\HardTargetLaw{I},\HardOracle{I})
 &=\mathbb E_W\mathbb E_I
    \dTV(\HardTargetLaw{I},\HardOutputLaw{I}{W})\notag\\
 &\ge\mathbb E_W\mathbb E_I
    \sqHellinger(\HardTargetLaw{I},\HardOutputLaw{I}{W})\notag\\
 &\ge\frac14\left[
  1-\exp\left\{-\frac{N\EdgeCoupling^2}{98304\RoundBudget{}}\right\}
 \right].
 \label{app:eq:prior-to-seed-risk}
\end{align}
At least one fixed instance has seed-averaged TV risk at least this
prior average. Taking the supremum over the hard family and then the
infimum over \(\HardAlgorithmClass{\QueryBudget{}}{\RoundBudget{}}\)
proves \eqref{app:eq:nonlinear-finite-lower}.
\end{proof}

\begin{proof}[Proof of
\cref{app:cor:nonlinear-finite-target-tradeoff}]
Suppose the first two alternatives in
\eqref{app:eq:nonlinear-finite-target-alternative} fail. Then
\cref{app:thm:nonlinear-finite-lower} and
\eqref{app:eq:finite-target-risk} imply
\[
 \varepsilon
 \ge
 \frac14\left[
  1-\exp\left\{-\frac{N\EdgeCoupling^2}{98304\RoundBudget{}}\right\}
 \right].
\]
Since \(\varepsilon\le1/8\), rearrangement gives
\begin{align*}
 \frac{N\EdgeCoupling^2}{98304\RoundBudget{}}
 &\le-\log(1-4\varepsilon)\\
 &\le\frac{4\varepsilon}{1-4\varepsilon}
 \le8\varepsilon.
\end{align*}
Here \(-\log(1-u)\le u/(1-u)\) for \(u\in[0,1)\), obtained by integrating
\((1-x)^{-1}\le(1-u)^{-1}\) over \(x\in[0,u]\). Thus
\[
 \RoundBudget{}\ge\frac{N\EdgeCoupling^2}{786432\varepsilon},
\]
which is the third alternative in
\eqref{app:eq:nonlinear-finite-target-alternative}.
\end{proof}

\section{Algorithmic construction}
\label{app:algorithm}

We specify the randomized sampler $\RandomizedPackedSampler$ of
Theorem~\ref{app:thm:finite-upper}, also denoted
$\PackedSampler:=\RandomizedPackedSampler$.
All operations use only public parameters, the realized commit history, and
replies of the single frozen oracle.  The hidden forest, exact marginals,
endpoint kernels, ranks, and response witnesses are never read.

The sampler flow in \cref{fig:algorithm-overview-20260918} is illustrated
concretely by a $15$-vertex construction in
\cref{ex:algorithm-banks-20260918,ex:algorithm-submissions-20260918,ex:algorithm-peel-centroids-20260918}.

\subsection{Complete sampler and public parameters}
\label{app:algorithm-overview}
\label{app:algorithm-complete-sampler}

We assemble the sampler from read-only screens and commits, with public
caps that bound its execution on every path.

Fix $N\ge10$.  The two resource controls are the degree cutoff
$\DegreeCutoff\in\{9,\ldots,N-1\}$ and readout-chunk parameter
$\ChunkCount\in\PositionSet$: the former sets the peeling threshold, while
the latter sets the maximum readout size
$\ReadoutChunkSize:=\lceil N/\ChunkCount\rceil$.
The accuracy inputs are a tail tolerance $\TailTolerance\in(0,1]$ and a
screen-failure budget $\ScreenFailureBudget\in(0,1)$.
The threshold floor $\ThresholdFloor$ uses the standard choice
\eqref{app:eq:rate-threshold-floor}.
The draft generator is the zero-query choice
\eqref{app:eq:standard-draft} unless otherwise stated.
The standard row selector keeps the highest-vote candidates on overflow.
More generally, fix any admissible public selector $\RowSelector$ from
Definition~\ref{app:def:row-selector}; denote the resulting decision rule by
$\SelectedPackedSampler$.  The unadorned sampler uses the standard selector.
The row-error radius and resulting screen resolution are derived quantities:
\begin{align}
 \RowTVError&:=\frac{\OracleRadius}{\sqrt N},
 \label{app:eq:oracle-tv-radii}\\
 \ScreenResolution&:=4\RowTVError+\TailTolerance.
 \label{app:eq:diameter-screen-resolution}
\end{align}
The oracle condition \eqref{app:eq:A2} and
\cref{app:lem:finite-hellinger-calculus} imply, for every valid $(y,j)$,
\begin{equation}
 \dTV\bigl(\ExactRow{y}{j},\OracleRow{y}{j}\bigr)
 \le\RowTVError.
 \label{app:eq:oracle-tv-radius}
\end{equation}
\textcolor{red}{The preprocessing and screen certificates use the row-TV bound
\eqref{app:eq:oracle-tv-radius} in both cases.  In \KLCase,
\eqref{app:eq:kl-oracle-inclusion} supplies this bound from the uniform
forward row-KL condition.  The forward-KL output guarantee is proved in
\cref{app:thm:finite-kl-upper} using the joint-reference identity of
\cref{app:kl-joint-reference}.}
The theorem requires feasible preprocessing and
$\ScreenResolution<\EdgeSignal$; neither condition asks the sampler to inspect
the hidden forest.

Define the public peel-phase cap
\begin{equation}
 \PeelPhaseCap
 :=\left\lceil
 \frac{\PositivePart{\log(2N/(\DegreeCutoff+1))}}
      {\log(\DegreeCutoff/8)}
 \right\rceil.
 \label{app:eq:phase-cap}
\end{equation}
Set
\begin{equation}
 \ScreenCallCap:=\PeelPhaseCap+1,
 \qquad
 \HardRoundCap
 :=\left\lceil\frac{4N\PeelPhaseCap}{\DegreeCutoff}\right\rceil
   +\left\lceil\log_2(N+1)\right\rceil+2.
 \label{app:eq:screen-call-and-round-caps}
\end{equation}
The caps are fixed before the first commit; they are not additional tuning
parameters.  Algorithm~\ref{app:alg:complete-sampler} is the full decision rule.
Its one-time preprocessing, per-screen draft, and read-only screen are
specified in the next three subsections.
The named steps agree with Section~\ref{sec:main-algorithm} and
Figure~\ref{fig:algorithm-overview-20260918}.
Their identifiers name recurring operations, not commit rounds or oracle stages.
\AlgoBlockRef{screen} is expanded into
\AlgoBlockRef{chunks}, \AlgoBlockRef{probes}, and \AlgoBlockRef{aggregate}
in Algorithm~\ref{app:alg:packed-screen}.
Proofs refer to these named steps as well as the algorithm's line numbers.
In Algorithm~\ref{app:alg:complete-sampler}, $r$ counts completed nonempty
commit rounds and $k$ counts completed peel phases.

\begin{algorithm}[p]
\caption{Randomized complete packed-screen sampler}
\label{app:alg:complete-sampler}
\begin{algorithmic}[1]
\Require the public setup and parameters
  $\TailTolerance,\ThresholdFloor,\DegreeCutoff,\ChunkCount$,
  $\ScreenFailureBudget$, an admissible $\DraftGenerator$ (standard by default),
  an admissible public row selector $\RowSelector$ (top votes by default),
  and access to $\FrozenOracle$
\Ensure $\SamplerOutput\in\Vocab^N$ when preprocessing is feasible; otherwise
  \InfeasibleOutcome
\Statex \AlgoBlockTarget{preprocess}
\State Construct $\PreprocessData$ by Definition~\ref{app:def:preprocessing}
\If{preprocessing reports infeasibility}
  \State \Return \InfeasibleOutcome
\EndIf
\State Set $\ScreenCallCap,\HardRoundCap$ by
  \ref{app:eq:screen-call-and-round-caps} and set
  $G\gets\emptyset$, $x_G\gets\emptyset$, $k\gets0$, and $r\gets0$
\Loop
  \If{$\ResidualPositions=\emptyset$}
    \State \Return the resulting assignment $\SamplerOutput$
  \EndIf
  \Statex \AlgoBlockTarget{draft}
  \State Fix $\DraftFiller\gets\DraftGenerator$ using the current history and past replies
  \State Draw the coloring family by \eqref{app:eq:random-color-draw},
    with parameters from \eqref{app:eq:random-color-parameters}
  \Statex \AlgoBlockTarget{screen}
  \State $\ScreenRows\gets
    \Call{\PackedScreenName}{\History;\PreprocessData,\DraftFiller,
      \ColorSet,(\ColorHash{\tau}{\cdot})_{\tau\in\ColoringIndexSet},
      \DegreeCutoff,\ChunkCount,\FrozenOracle,\RowSelector,\ScreenTranscript}$
    \Comment{\cref{app:alg:packed-screen}}
  \Statex \AlgoBlockTarget{peel}
  \State Form $\ClaimCount{\cdot}$ and $\PeelSet$ by
    \ref{app:eq:claim-and-peel-set}
  \If{$\PeelSet=\emptyset$ or $k=\PeelPhaseCap$}
    \State \textbf{break}
  \EndIf
  \State Freeze $\FrozenPeelSet\gets\PeelSet$
  \For{$v\in\FrozenPeelSet$ in increasing position order}
    \If{$r=\HardRoundCap-1$} \Comment{\AlgoGuardRef}
      \State $\Call{\CommitName}{\ResidualPositions}$ and
        $r\gets r+1$
      \Return the resulting assignment $\SamplerOutput$
    \EndIf
    \State $\Call{\CommitName}{\{v\}}$
    \State $r\gets r+1$
  \EndFor
  \State $k\gets k+1$
\EndLoop
\Statex \AlgoBlockTarget{terminal}
\State $\TerminalHistory\gets\History$,
  $\TerminalPositions\gets\ResidualPositions$, and
\Statex \hspace{\algorithmicindent}$\displaystyle
 \CertifiedEdges\gets
 \{\{u,v\}\subseteq\TerminalPositions:
 v\in\ScreenRowAt{\TerminalHistory}{u}
 \text{ or }u\in\ScreenRowAt{\TerminalHistory}{v}\}$
\State $\CertifiedForest\gets(\TerminalPositions,\CertifiedEdges)$
\While{$\CertifiedForest$ contains a cycle}
  \If{$r=\HardRoundCap-1$} \Comment{\AlgoGuardRef}
    \State $\Call{\CommitName}{\ResidualPositions}$ and $r\gets r+1$; \Return $\SamplerOutput$
  \EndIf
  \State Choose a maximum-degree cycle vertex $v$; break ties by smallest index
  \State $\Call{\CommitName}{\{v\}}$; $r\gets r+1$; delete $v$ and its incident edges from $\CertifiedForest$
\EndWhile
\Statex \AlgoBlockTarget{centroid}
\While{$\CertifiedForest$ has a vertex}
  \State In each component choose its smallest-label centroid; call their set $B$
  \If{$r=\HardRoundCap-1$} \Comment{\AlgoGuardRef}
    \State $\Call{\CommitName}{\ResidualPositions}$ and
      $r\gets r+1$
    \Return the resulting assignment $\SamplerOutput$
  \EndIf
  \State $\Call{\CommitName}{B}$
  \State $r\gets r+1$
  \State Delete $B$ and its incident edges from $\CertifiedForest$
\EndWhile
\State \Return the resulting assignment $\SamplerOutput$
\end{algorithmic}
\end{algorithm}

Given the screen output at $\History$, recall the incoming claim count and
peel set of \textcolor{red}{\cref{main:eq:claim-and-peel-set}}:
\begin{equation}
 \ClaimCount{v}
 :=\bigl|\{u\in\ResidualPositions:v\in\ScreenRow{u}\}\bigr|,
 \qquad
 \PeelSet:=\{v\in\ResidualPositions:
             \ClaimCount{v}>\DegreeCutoff/2\}.
 \label{app:eq:claim-and-peel-set}
\end{equation}
\textcolor{red}{The terminal row graph on $\TerminalPositions$ is constructed in
\AlgoBlockRef{terminal} of \cref{app:alg:complete-sampler}.}

At any history, $\CommitOperation(B)$ means the product-commit operation
of Definition~\ref{app:def:admissible-algorithm}: submit the current
history-compatible commit state \ref{app:eq:commit-state}, read
$\OracleRow{y}{j}$ for $j\in B$, draw independently from the corresponding
product law, and append the realized values to the history.  This is the
sampler's only history-changing operation. The batch $B$ is fixed before
the reply; this submission adds one to $\CommitRounds$ and nothing to
$\CounterfactualQueries$.

\begin{enumerate}
 \item \emph{Screen, then test whether to stop.}
 Each pass first computes rows at the current history
 (\AlgoBlockRef{screen}) and only then tests whether the peel set is empty
 or the phase cap has been reached (\AlgoBlockRef{peel}).
 In particular, the last permitted peel phase is followed by another
 screen: the terminal graph never uses rows from before those commits.
 A nonempty peel set is frozen before its singleton commits; its members
 are committed in position order, without an intervening screen or
 recomputation of that set.

 \item \emph{Repair cycles in the terminal estimate.}
 In \AlgoBlockRef{terminal}, the working graph $\CertifiedForest$ initially
 equals $(\TerminalPositions,\CertifiedEdges)$. A cycle vertex belongs to
 at least one simple cycle of this graph. The selection rule is
 \[
 v=\min\operatorname*{arg\,max}_{
       \substack{u\in V(\CertifiedForest)\\
                 u\text{ lies on a cycle of }\CertifiedForest}}
       d_{\CertifiedForest}(u).
 \]
 The degree is measured in the whole current working graph.
 Commit $v$ using the current history, delete $v$ and its incident edges,
 and recompute cycle membership and degrees in the remaining graph.
 No new screen is run: the graph changes only by vertex deletion, whereas
 each commit queries fresh rows at the updated history.
 Once it is a forest, proceed to \AlgoBlockRef{centroid}.

 \item \AlgoGuardTarget.
 Before each peel, repair, or centroid commit, the check
 $r=\HardRoundCap-1$ reserves the last permitted round for one product
 commit of all remaining positions. This ends the execution even if
 repair is incomplete. The cap bounds the round count on every path;
 this final product law need not approximate the true joint conditional.

 \item \emph{Separate the observable past from the random sources.}
 All ties use the public position or token order.
 The public transcript $\ScreenTranscript$ starts with preprocessing
 and is updated after each query, coloring draw, selector output, and
 commit. It records only information already available to the decision
 rule, not unrevealed future randomness.
 The decision-rule seed $\ControllerSeed$ contains the coloring blocks
 and auxiliary draft randomness. Product-commit draws use a separate
 random source: even after conditioning on $\ControllerSeed=w$, those
 draws remain random and are integrated into
 $\SamplerOutputLaw{w}{\FrozenOracle}$.
\end{enumerate}

\subsection{One-time all-mask preprocessing}
\label{app:algorithm-preprocessing}

In \AlgoBlockNamedRef{preprocess}, one all-mask submission supplies the
local vocabulary banks $\TokenBank{i}$ and tail representatives
$\TailRepresentative{i}$ reused by every later screen.

\begin{definition}[All-mask preprocessing]
\label{app:def:preprocessing}
Use the standard floor
\begin{equation}
 \ThresholdFloor:=\StandardThresholdFloor{\FrequencyConstant}.
 \label{app:eq:rate-threshold-floor}
\end{equation}
Here $\FrequencyConstant\ge1$ is the public RF\textcolor{red}{~(\cref{app:eq:RF})} constant when an envelope is
specified; otherwise use $\FrequencyConstant=1$.
This keeps the threshold above both the row-noise and
frequency-envelope scales.
If $\ThresholdFloor>1$, return $\InfeasibleOutcome$.
Otherwise put
$\ThresholdGridDepth:=\lceil\log_2(1/\ThresholdFloor)\rceil$ and set
\begin{equation}
 \ThresholdGrid
 :=\{\max\{\ThresholdFloor,2^{-m}\}:0\le m\le\ThresholdGridDepth\}.
 \label{app:eq:threshold-grid}
\end{equation}
Form the feasible set
\begin{equation}
 \FeasibleThresholdGrid
 :=\{t\in\ThresholdGrid:
       \min\{1,\ResponseConstant t^{\ResponseExponent}\}\le\TailTolerance\}.
 \label{app:eq:feasible-threshold-grid}
\end{equation}
If it is empty, return $\InfeasibleOutcome$ before any submission or commit.
Otherwise set $\TailThreshold:=\max\FeasibleThresholdGrid$.

Submit $\AllMaskState=(\MASK,\ldots,\MASK)$ once and, for each position $i$,
form its noisy marginal and \emph{local vocabulary bank}:
\begin{align}
 \NoisyMarginal{i}&:=\OracleRow{\AllMaskState}{i},
 \label{app:eq:noisy-marginal}\\
 \TokenBank{i}
 &:=\{a\in\Vocab:\NoisyMarginal{i}(a)\ge\TailThreshold-\RowTVError\}.
 \label{app:eq:local-bank}
\end{align}
Here $\TokenBank{i}\subseteq\Vocab$ is the subset of vocabulary tokens
selected by the buffered all-mask marginal threshold
$\TailThreshold-\RowTVError$ for explicit source-value probing.
Tokens outside this set are represented in the response test by
$\TailRepresentative{i}$ below; the set does not restrict commit outputs.
Enumerate the vocabulary bank $\TokenBank{i}$ as
$\TokenBank{i}=\{\BankToken{i}{1},\ldots,\BankToken{i}{\BankSize{i}}\}$
in the public token order, with $\BankSize{i}:=|\TokenBank{i}|$ and
$\MaxBankSize:=\max_i\BankSize{i}$, and choose
\begin{equation}
 \TailRepresentative{i}:=
 \min_{\TokenOrder}\bigl(\Vocab\setminus\TokenBank{i}\bigr).
 \label{app:eq:positionwise-baseline}
\end{equation}
The floor ensures that this complement is nonempty;
\cref{app:lem:preprocessing-certificates} proves both existence and its
exact-tail certificate.  Return the reusable tuple
\begin{equation}
 \PreprocessData:=
 \left(\TailThreshold,
 (\NoisyMarginal{i},\TokenBank{i},\TailRepresentative{i},\BankSize{i})_{i\in\PositionSet},
 \MaxBankSize\right).
 \label{app:eq:preprocessing-output}
\end{equation}
Apart from the all-mask submission, preprocessing is deterministic and uses
no commit draw.
\end{definition}

\subsection{Per-screen draft and source columns}
\label{app:algorithm-draft}

We first fix the common background for one screen, then define the source
assignments whose oracle replies will be compared.

\begin{example*}[Standard zero-query draft]
Unless stated otherwise, \AlgoBlockRef{draft} uses the standard draft
\begin{equation}
 \DraftToken{i}
 :=\mathop{\arg\max}_{a\in\Vocab}\NoisyMarginal{i}(a),
 \qquad i\in\ResidualPositions,
 \label{app:eq:standard-draft}
\end{equation}
with ties resolved by $\TokenOrder$.  It reuses the initial all-mask rows,
works even for empty vocabulary banks, and requires no additional submission or adaptive
stage.
\end{example*}

\paragraph{Source columns.}
For the source assignments in \AlgoBlockRef{probes}, recall that
$\TokenBank{i}=\{\BankToken{i}{1},\ldots,\BankToken{i}{\BankSize{i}}\}$
lists the local vocabulary bank in the public token order, not in
marginal-probability order.
Here $\BankSize{i}=|\TokenBank{i}|$ and
$\MaxBankSize=\max_{i\in\PositionSet}\BankSize{i}$ were fixed by preprocessing.
For a fixed screen draft, set
\begin{equation}
 \ScreenColumnSet:=[\MaxBankSize]\cup\{\TailColumn\},
 \qquad
 \PackedColumn{\kappa}{i}
 :=\begin{cases}
 \BankToken{i}{\kappa},&\kappa\in[\BankSize{i}],\\
 \DraftToken{i},&\kappa\in[\MaxBankSize]\setminus[\BankSize{i}],\\
 \TailRepresentative{i},&\kappa=\TailColumn.
 \end{cases}
 \label{app:eq:packed-column}
\end{equation}
The column order is $1,\ldots,\MaxBankSize,\TailColumn$.
Thus a bank-column index denotes a position-dependent assignment, not one
common vocabulary token.  All screen quantities below depend on the fixed
$\DraftFiller$; this dependence is suppressed in row-family, vote, and graph
notation.  When $\MaxBankSize=0$, the screen returns empty rows without
submitting even the tail column.

For resource accounting, the number of submitted columns is
\begin{equation}
 \ScreenColumnCount:=
 \begin{cases}
 \MaxBankSize+1,&\MaxBankSize\ge1,\\
 0,&\MaxBankSize=0.
 \end{cases}
 \label{app:eq:screen-column-count}
\end{equation}

\begin{example}[A column indexes local vocabulary banks, not a common token]
\label{ex:algorithm-banks-20260918}
Take $N=15$, $G=\emptyset$, and a vocabulary whose first three tokens in
the public order are $a\prec b\prec c$ (further tokens may follow).
Suppose the local vocabulary banks $\TokenBank{i}$ from
\AlgoBlockRef{preprocess} (\cref{app:def:preprocessing}) and the fixed draft are
\[
 \TokenBank{1}=\{a,c\},\quad \DraftToken{1}=a,
 \qquad
 \TokenBank{i}=\{b\},\quad \DraftToken{i}=b
 \quad (i\in[15]\setminus\{1\}).
\]
This specifies the screen inputs, not a target distribution or an oracle.
In particular, $\BankSize{1}=2$, $\BankSize{i}=1$ for $i\ne1$, and
$\MaxBankSize=2$.  The first token outside each bank is
$\TailRepresentative{1}=b$ and $\TailRepresentative{i}=a$ for $i\ne1$.
Substituting into \eqref{app:eq:packed-column} gives
\[
\begin{array}{c|ccc}
 \text{position}&\kappa=1&\kappa=2&\kappa=\TailColumn\\ \hline
 1& a&c&b\\
 i\ne1&b&\DraftToken{i}=b&a
\end{array}
 \qquad
 \ScreenColumnSet=\{1,2,\TailColumn\},\quad \ScreenColumnCount=3.
\]
Thus column $1$ assigns $a$ at position $1$ but $b$ at position $4$.
Column $2$ uses the second bank token at position $1$ and draft padding
at position $4$, whose bank has no second token.  The tail column uses
the representative outside each bank.  The same three columns are
reused for every source color and readout chunk.
\end{example}

\paragraph{Three different indices.}
The screen uses the following roles; their domains and sampling law are
specified in \cref{app:algorithm-screen}.
\par\smallskip
\noindent
\begin{tabularx}{\linewidth}{@{}lXX@{}}
\toprule
Symbol & What it selects & What it does not select\\
\midrule
$\tau$ & One entire coloring $h_\tau:\PositionSet\to\ColorSet$
       & A color within that coloring\\
$c^{\mathrm{src}},c^{\mathrm{rd}}$
       & Source and readout colors under the chosen $h_\tau$
       & Vocabulary tokens\\
$\kappa\in\ScreenColumnSet$
       & One slot in the vocabulary bank $\TokenBank{i}$ or its tail representative,
         separately at each source $i$
       & A color or one common token\\
\bottomrule
\end{tabularx}
\par\smallskip
Thus fixing $\tau$ fixes all position colors; fixing $c^{\mathrm{src}}$
then fixes which residual positions are sources. Varying $\kappa$
changes the tokens assigned to those sources, without recoloring them.
The concrete substitution into the submitted-state definition is in
\cref{ex:algorithm-submissions-20260918}.

\begin{definition}[Zero-query screen draft]
\label{app:def:draft-generator}
Before each screen, $\DraftGenerator$ fixes a token vector
$\DraftFiller\in\Vocab^{\ResidualPositions}$ using only public data,
the realized history, cached oracle replies, and auxiliary seed randomness
independent of the coloring and product-commit random sources. It terminates
without a submission or commit and without observing current or future
colors. The draft is fixed across all colorings, chunks, source colors,
and columns of this screen; it may be reused or changed at the next screen.
A draft token $\DraftToken{i}$ need not belong to the vocabulary bank
$\TokenBank{i}$ and is never committed directly.
\end{definition}

\subsection{Packed randomized-color screen}
\label{app:algorithm-screen}

We generate one color family, pack its readouts into chunks, and compare
source-column replies to form the screen's neighbor claims.

\paragraph{Random color generation.}
For \AlgoBlockRef{draft}, set
\begin{equation}
 \RandomColorCount:=8(\DegreeCutoff+1),
 \qquad
 \RandomColoringCount
 :=\left\lceil
 8\log\frac{\ScreenCallCap N^2}{\ScreenFailureBudget}
 \right\rceil.
 \label{app:eq:random-color-parameters}
\end{equation}
Write $\ColorCount=\RandomColorCount$,
$\ColoringCount=\RandomColoringCount$,
$\ColorSet=[\ColorCount]$, and $\ColoringIndexSet=[\ColoringCount]$.
These are deterministic public parameters.
At each reached screen $s$, conditional on the realized past including its
completed draft, the decision rule reveals a fresh block $W_s$ with law
\begin{equation}
 \ColorHash{\tau}{i}
 \overset{\mathrm{i.i.d.}}{\sim}\operatorname{Unif}([\ColorCount]),
 \qquad
 (\tau,i)\in[\ColoringCount]\times\PositionSet.
 \label{app:eq:random-color-draw}
\end{equation}
The blocks $(W_s)_{s\le\ScreenCallCap}$ are mutually independent and independent
of the auxiliary draft random source and primitive commit randomness;
no current or future block is revealed during draft generation.
The same family is used for all chunks, source colors, and columns of this
screen: \cref{app:alg:packed-screen} does not draw or refresh colors.

\paragraph{Accuracy calibration for the main benchmarks.}
The choices $\ScreenFailureBudget=\OracleRadius/2$ in \HellingerCase\
and $\ScreenFailureBudget=\KLTargetTolerance/(2N\log\VocabSize)$ in
\KLCase\ specialize \eqref{app:eq:random-color-parameters} to
\begin{equation}
 \RandomColorCount=8(\DegreeCutoff+1),
 \qquad
 \RandomColoringCount
 :=\begin{cases}
 \displaystyle\left\lceil
 8\log\frac{2\ScreenCallCap N^2}{\OracleRadius}
 \right\rceil,&\text{\HellingerCase},\\[6pt]
 \displaystyle\left\lceil
 8\log\frac{2\ScreenCallCap N^3\log\VocabSize}{\KLTargetTolerance}
 \right\rceil,&\text{\KLCase}
 \end{cases}
 \label{app:eq:accuracy-color-parameters}
\end{equation}
In \KLCase, all radius-dependent preprocessing and screen quantities use
$\KLOracleRadius$ in place of $\OracleRadius$,
as in \cref{app:cor:kl-epsilon}. The decision rule and commit rule
are unchanged.

\paragraph{Readout chunks.}
Given the colors, \AlgoBlockRef{chunks} partitions each readout class
deterministically so that
each probe masks at most $\ReadoutChunkSize$ positions.
For every $\tau\in\ColoringIndexSet$ and
$\ReadoutColorValue\in\ColorSet$, let
\begin{equation}
 \ColorClass{\tau}{\ReadoutColorValue}
 :=\{j\in\ResidualPositions:
      \ColorHash{\tau}{j}=\ReadoutColorValue\},
 \qquad
 \ColorClassSize{\tau}{\ReadoutColorValue}
 :=|\ColorClass{\tau}{\ReadoutColorValue}|.
 \label{app:eq:color-class}
\end{equation}

For a nonempty class, put
$n:=\ColorClassSize{\tau}{\ReadoutColorValue}$, list its elements in increasing
position order as
$j_{\tau,\ReadoutColorValue,1}<\cdots<j_{\tau,\ReadoutColorValue,n}$, and
partition it into
\begin{equation}
 \ReadoutChunk{\tau}{\ReadoutColorValue}{m}
 :=\left\{
 j_{\tau,\ReadoutColorValue,\ell}:
 (m-1)\ReadoutChunkSize<\ell\le
 \min\{m\ReadoutChunkSize,n\}
 \right\},
 \quad
 1\le m\le\left\lceil
 \frac{n}{\ReadoutChunkSize}
 \right\rceil.
 \label{app:eq:readout-chunks}
\end{equation}
Empty color classes create no chunk.  The chunks depend on the current
residual set $\ResidualPositions$, form a disjoint cover of each color class,
and have size at most $\ReadoutChunkSize$.  Write $\ReadoutChunkOf{\tau}{j}$ for the unique chunk containing $j$.

\paragraph{Submitted states.}
In \AlgoBlockRef{probes}, each probe masks one readout chunk and changes
one source color's
assignment, while keeping the remaining background fixed.
For a nonempty readout chunk
$\mathcal C=\ReadoutChunk{\tau}{\ReadoutColorValue}{m}$,
$\kappa\in\ScreenColumnSet$, and a source color
$\SourceColorValue\in
\ColorSet\setminus\{\ReadoutColorValue\}$, define the probe state
\begin{equation}
 (\ProbeState{\kappa}{\tau}{\SourceColorValue}{\mathcal C})_k
 :=\begin{cases}
 x_k,&k\in G,\\
 \MASK,&k\in \mathcal C,\\
 \PackedColumn{\kappa}{k},
   &k\in\ResidualPositions\setminus \mathcal C\text{ and }
    \ColorHash{\tau}{k}=\SourceColorValue,\\
 \DraftToken{k},&\text{otherwise}.
 \end{cases}
 \label{app:eq:screen-probe-state}
\end{equation}
Here $\ReadoutColorValue$ is the readout color,
$\SourceColorValue$ is the source color, and $\tau$ selects the entire
coloring.  Since $\SourceColorValue\ne\ReadoutColorValue$, the source class
and readout chunk are
disjoint.  A missing bank entry leaves its source coordinate at the draft
value.  Column $\TailColumn$ instead assigns $\TailRepresentative{k}$ to the
source class only; all other unmasked residual coordinates stay at
$\DraftToken{k}$.  There is no separate all-draft baseline submission, and
the tail-column state depends on the source color.
\begin{example}[From one coloring to three submitted states]
\label{ex:algorithm-submissions-20260918}
Continue \cref{ex:algorithm-banks-20260918}.  Set
$\DegreeCutoff=\ChunkCount=9$, so
$\ColorCount=8(\DegreeCutoff+1)=80$ and
$\ReadoutChunkSize=\lceil15/9\rceil=2$.
Fix one index $\tau\in[\ColoringCount]$ and suppose \AlgoBlockRef{draft}
draws the coloring
\[
 \ColorHash{\tau}{i}=
 \begin{cases}
 1,&i\in\{1,4\},\\
 2,&i\in\{10,13\},\\
 3,&i\in[15]\setminus\{1,4,10,13\}.
 \end{cases}
\]
These are three of the $80$ available colors, not a change to the
color-count parameter.  With
$\SourceColorValue=1$ and $\ReadoutColorValue=2$, \AlgoBlockRef{chunks}
applies \eqref{app:eq:color-class}--\eqref{app:eq:readout-chunks} to obtain
\[
 \ColorClass{\tau}{1}=\{1,4\},\qquad
 \ColorClass{\tau}{2}=\{10,13\},\qquad
 \mathcal C=\ReadoutChunk{\tau}{2}{1}=\{10,13\}.
\]
For this example only, write
$y^\kappa:=\ProbeState{\kappa}{\tau}{1}{\{10,13\}}$.
Because $G=\emptyset$, substituting the bank columns into
\eqref{app:eq:screen-probe-state} yields
\[
\begin{array}{c|cc|cc}
 &\multicolumn{2}{c|}{\text{src (source)}}
 &\multicolumn{2}{c}{\text{target (readout)}}\\
 \kappa&y^\kappa_1&y^\kappa_4&y^\kappa_{10}&y^\kappa_{13}\\ \hline
 1&a&b&\MASK&\MASK\\
 2&c&b&\MASK&\MASK\\
 \TailColumn&b&a&\MASK&\MASK
\end{array}
 \qquad
 y^\kappa_k=\DraftToken{k}=b
 \quad(k\in[15]\setminus\{1,4,10,13\}).
\]
These remaining positions are neither sources nor masked readouts.
Since $G=\emptyset$, they take the \emph{otherwise} branch of
\eqref{app:eq:screen-probe-state} and retain their draft values
$\DraftToken{k}$, which equal $b$ by \cref{ex:algorithm-banks-20260918}.
In \AlgoBlockRef{probes}, each of the three submissions returns both
$\OracleRow{y^\kappa}{10}$ and $\OracleRow{y^\kappa}{13}$.
Their role is to test dependence, not to commit the source tokens:
$G$ remains empty throughout.  In this particular example, column $1$
happens to agree with the draft at all unmasked positions; it is still
one of the three columns, not an extra baseline submission.
\Cref{fig:algorithm-submissions-20260918} shows these exact assignments.
\end{example}

\begin{figure}[tbp]
 \centering
 \includegraphics[width=\linewidth]{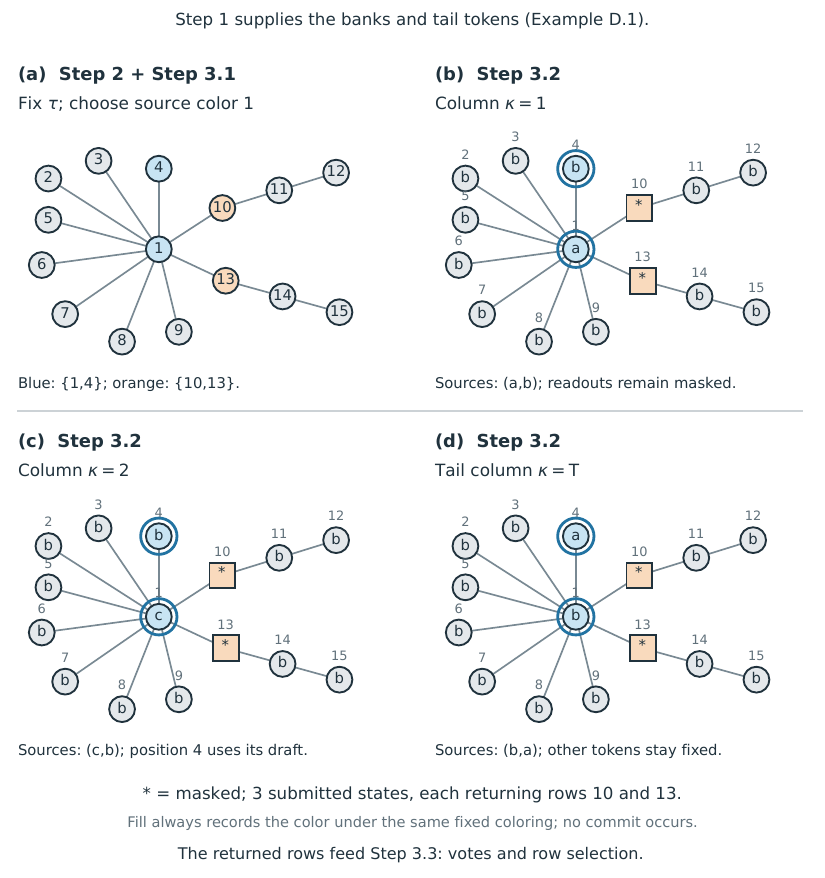}
 \caption{The coloring and three submitted states of
 \cref{ex:algorithm-submissions-20260918}, labeled by the steps of
 Algorithms~\ref{app:alg:complete-sampler}--\ref{app:alg:packed-screen}.
 All $15$ positions are drawn
 separately at fixed locations. Fill denotes the color under $\tau$;
 square outlines mark masked readouts and double outlines mark sources.
 Small outer numbers are position indices, and inner symbols in
 (b)--(d) are submitted tokens. No vertex is committed by these probes.
 The hidden edges are shown only to explain the construction.}
 \label{fig:algorithm-submissions-20260918}
\end{figure}

\paragraph{TV tests and aggregation.}
Within each coloring, \AlgoBlockRef{aggregate} compares all pairs of
source-column replies at the
same readout; a strict majority across colorings determines each candidate row.
For distinct $i,j\in\ResidualPositions$, when
$\ColorHash{\tau}{i}\ne\ColorHash{\tau}{j}$, define the observable rows
\begin{equation}
 \ProbeObservedRow{\tau}{\kappa}{i}{j}
 :=\OracleRow{\ProbeState{\kappa}{\tau}{\ColorHash{\tau}{i}}
            {\ReadoutChunkOf{\tau}{j}}}{j},
 \quad \kappa\in\ScreenColumnSet,
 \label{app:eq:observed-screen-rows}
\end{equation}
and the observable row family
\begin{equation}
 \ObservedRowFamily{\tau}{i}{j}
 :=\{\ProbeObservedRow{\tau}{\kappa}{i}{j}:
          \kappa\in\ScreenColumnSet\}.
 \label{app:eq:observed-row-family}
\end{equation}
Its TV diameter and the corresponding vote are
\begin{align}
 \ScreenDiameter{\tau}{i}{j}
 &:=\max_{q,q'\in\ObservedRowFamily{\tau}{i}{j}}
      \dTV(q,q'),
 \label{app:eq:observed-screen-diameter}\\
 \ScreenVote{\tau}{i}{j}
 &:=\begin{cases}
 \mathbf 1\{\ScreenDiameter{\tau}{i}{j}>2\RowTVError\},
   &\ColorHash{\tau}{i}\ne\ColorHash{\tau}{j},\\
 0,&\ColorHash{\tau}{i}=\ColorHash{\tau}{j}.
 \end{cases}
 \label{app:eq:screen-vote}
\end{align}
When the colors coincide, the row family and diameter need not be formed.
The vote is set to zero and never refers to an unsubmitted probe.  If
$\MaxBankSize=0$, the screen returns empty rows before forming these objects.

Computing the diameter requires no additional submissions.  It compares
every pair of bank-plus-tail rows, not just each bank row against the tail
column.

For each readout $j$, define the candidate row by
\begin{align}
 \CandidateRow{j}
 &:=\left\{
 i\in\ResidualPositions\setminus\{j\}:
 \sum_{\tau\in\ColoringIndexSet}
 \ScreenVote{\tau}{i}{j}>\frac{\ColoringCount}{2}
 \right\},
 \label{app:eq:provisional-row}
\end{align}

\begin{definition}[Admissible public row selector]
\label{app:def:row-selector}
Let $\ScreenTranscript$ contain the public transcript available when the
screen has finished, including its colors and votes and any retained public
decision-rule state.  A selector $\RowSelector$ is a fixed, terminating,
deterministic measurable rule.  Given $j$, a candidate set
$C\subseteq\ResidualPositions\setminus\{j\}$, the cutoff $\DegreeCutoff$,
and $\ScreenTranscript$, its output has the form
\[
 \RowSelector(j,C,\DegreeCutoff;\ScreenTranscript)
 =
 \begin{cases}
 C,&|C|\le\DegreeCutoff,\\
 C\setminus C_{\mathrm{drop}},&|C|>\DegreeCutoff.
 \end{cases}
\]
On overflow, $C_{\mathrm{drop}}\subseteq C$ is the set of positions that
the rule chooses to remove from these inputs, with
$|C_{\mathrm{drop}}|\ge |C|-\DegreeCutoff$. Thus
$|C\setminus C_{\mathrm{drop}}|=|C|-|C_{\mathrm{drop}}|\le\DegreeCutoff$;
the rule may return fewer than $\DegreeCutoff$ positions, including none.
It makes no oracle submission or commit and reads neither hidden target
information nor unobserved future randomness or oracle replies.
\end{definition}

The returned row $\ScreenRow{j}$ is the output of this selector applied
to $C=\CandidateRow{j}$:
\begin{equation}
 \ScreenRow{j}
 :=\RowSelector\bigl(j,\CandidateRow{j},\DegreeCutoff;
                     \ScreenTranscript\bigr).
 \label{app:eq:returned-row}
\end{equation}
\begin{example*}[Standard top-vote row selector]
The standard choice of $\RowSelector$ for returning $\ScreenRow{j}$ in
\eqref{app:eq:returned-row} orders the positions
$i_1,\ldots,i_{|C|}$ of $C$ by decreasing total vote
$\sum_{\tau\in\ColoringIndexSet}\ScreenVote{\tau}{i}{j}$, breaking ties
by increasing position index. It returns
\[
 \RowSelector(j,C,\DegreeCutoff;\ScreenTranscript)
 =\{i_\ell:1\le\ell\le\min\{\DegreeCutoff,|C|\}\}.
\]
Thus it keeps all of $C$ when $|C|\le\DegreeCutoff$ and otherwise removes
the positions ranked after $\DegreeCutoff$. The empty input gives the
empty set. Taking $C=\CandidateRow{j}$ gives the standard returned row
$\ScreenRow{j}$.
\end{example*}
The size cap alone would not suffice: deleting a candidate from a
nonoverflowing row can destroy low-degree exactness.
When $\MaxBankSize=0$, every returned row is defined to be empty.

\begin{proposalcontext}{Alternative overflow rule}
Returning $\emptyset$ on overflow is another admissible selector.
The standard sampler uses top-vote selection; the invariance proposition
compares the guarantees of these rules, not their realized performance.
\end{proposalcontext}

Given its inputs, public transcript, and the frozen oracle,
\cref{app:alg:packed-screen} is deterministic; the full decision rule in \cref{app:alg:complete-sampler}
supplies its color family in \AlgoBlockRef{draft}.
It expands \MainAlgoBlockRef{screen} of
Algorithm~\ref{main:alg:upper-bound-sampler}.
Within \AlgoBlockRef{screen}, \AlgoBlockRef{aggregate} computes votes
after each chunk's probes, then forms and selects rows after all colorings.

\begin{algorithm}[H]
\caption{Packed color screen with a fixed draft}
\label{app:alg:packed-screen}
\label{main:alg:packed-screen}
\begin{algorithmic}[1]
\Require $\History$, $\PreprocessData$, $\DraftFiller$, a color set $\ColorSet$, and a family
  $(\ColorHash{\tau}{\cdot})_{\tau\in\ColoringIndexSet}$,
  $\DegreeCutoff,\ChunkCount$, read access to $\FrozenOracle$,
  $\RowSelector$, and the current public transcript $\ScreenTranscript$
\Ensure $(\ScreenRow{j})_{j\in\ResidualPositions}$; the history is unchanged
\If{$\MaxBankSize=0$ or $\ResidualPositions=\emptyset$}
  \State \Return $(\emptyset)_{j\in\ResidualPositions}$ without a counterfactual submission
\EndIf
\State Initialize all votes in \ref{app:eq:screen-vote} to zero
\For{$\tau\in\ColoringIndexSet$}
  \Statex \AlgoBlockTarget{chunks}
  \State Form the color classes and current readout chunks by
    \ref{app:eq:color-class}--\ref{app:eq:readout-chunks}
  \For{each nonempty readout chunk
      $\mathcal C=\ReadoutChunk{\tau}{\ReadoutColorValue}{m}$}
    \Statex \AlgoBlockTarget{probes}
    \For{$\kappa\in\ScreenColumnSet$}
      \For{$\SourceColorValue\in
             \ColorSet\setminus\{\ReadoutColorValue\}$}
        \State Submit $\ProbeState{\kappa}{\tau}{\SourceColorValue}{\mathcal C}$
          and read all rows $j\in \mathcal C$
      \EndFor
    \EndFor
    \Statex \AlgoBlockTarget{aggregate}
    \State Form the row-family diameters and votes for all
      $i\in\ResidualPositions\setminus\{j\}$ and $j\in \mathcal C$ using
      \ref{app:eq:observed-row-family}--\ref{app:eq:screen-vote}
  \EndFor
\EndFor
\Statex \AlgoBlockHeading{aggregate} (continued)
\State Form the candidate rows by \ref{app:eq:provisional-row}
\State Append the current screen's replies, colors, and votes to $\ScreenTranscript$
\State Apply the public selector by \ref{app:eq:returned-row}
\State \Return $(\ScreenRow{j})_{j\in\ResidualPositions}$
\end{algorithmic}
\end{algorithm}

Every bank- or tail-column state is a noncommit counterfactual submission, and the
submitted states agree with $\History$ on $G$.  One submitted state returns all
rows indexed by its readout chunk and contributes one to
$\CounterfactualQueries$, regardless of $|\mathcal C|$
(\textcolor{red}{\cref{app:def:resources}}). The
tail-column reply for a fixed $(\tau,\mathcal C,\SourceColorValue)$ is reused in the
diameter computation for that source color, not across different source colors.

\subsection{Roles of the construction and proof route}
\label{app:algorithm-proof-route}

\textcolor{red}{The overview in \cref{app:upper-proof-overview} maps the steps above
to their structural, resource, and error guarantees. It also gives the proof
routes for both accuracy cases and their specialization to
\cref{main:thm:upper-curve}.}

\begin{proposalcontext}{Scope of cycle repair}
Cycle repair avoids an immediate all-residual product commit merely because
the terminal estimate has a cycle. A repaired forest need not contain every
true residual edge. The guarantees still charge failed-screen paths to the
screen-failure budget; they do not assert an accuracy improvement over the
one-batch alternative.
\end{proposalcontext}

\clearpage
\section{\HellingerCaseTitle: Hellinger-to-TV upper bound}
\label{app:upper-proof}

\textcolor{red}{We prove the finite TV guarantee for the sampler in
Appendix~\ref{app:algorithm}; \cref{app:upper-proof-overview} gives the proof map.
Screen certification, peeling, and resource counts support both accuracy
cases via \eqref{app:eq:kl-oracle-inclusion}. The error analysis here proves
\HellingerCase; Appendix~\ref{app:kl-extension} proves \KLCase.
Appendix~\ref{app:benchmark-specializations} specializes both to
\cref{main:thm:upper-curve}.}
\ifincludeexponents
Its subsequent corollaries give joint-limit rates.
\fi
\ifincludeexponents
Independent-signal rates are collected in
Appendix~\ref{app:independent-signal}.
\fi

\subsection{Finite upper bound}
\label{app:results-upper}

The finite theorem uses the standard preprocessing floor of
Appendix~\ref{app:algorithm-preprocessing}; its query bound uses the realized
vocabulary-bank sizes $\BankSize{i}=|\TokenBank{i}|$.
Together with \cref{app:lem:benchmark-feasibility}, it yields
\HellingerCase\ of \cref{main:thm:upper-curve} through
\cref{app:cor:fixed-accuracy-main}.

\begin{theorem}[\HellingerCaseTitle: finite randomized Hellinger-to-TV upper bound]
\label{app:thm:finite-upper}
We work in \HellingerCase, with the oracle condition of
Assumption~\ref{app:ass:oracle-accuracy}.
Fix $N\ge10$ and the objects of Appendix~\ref{app:setup}.  Let
$\RandomizedPackedSampler$ be Algorithm~\ref{app:alg:complete-sampler} with
the standard zero-query draft of \eqref{app:eq:standard-draft} and
$\ScreenFailureBudget\in(0,1)$.  Suppose preprocessing is feasible and
\begin{equation}
 \ScreenResolution<\EdgeSignal.
 \label{app:eq:finite-screen-separation}
\end{equation}
Here $\RandomColorCount,\RandomColoringCount$ are defined in
\eqref{app:eq:random-color-parameters}, and
$\ScreenCallCap,\HardRoundCap$ in
\eqref{app:eq:screen-call-and-round-caps}; $\ScreenColumnCount$ is
$\MaxBankSize+1$ when $\MaxBankSize\ge1$ and zero when $\MaxBankSize=0$, as in
\eqref{app:eq:screen-column-count}.
Then, for every admissible frozen oracle, pathwise over the random colors
and product-commit draws,
\begin{align}
 \CounterfactualQueries
 &\le \ScreenColumnCount\ScreenCallCap\RandomColoringCount
       \RandomColorCount(\RandomColorCount+\ChunkCount),
 \label{app:eq:finite-upper-Q}\\
 \CommitRounds
 &\le\HardRoundCap.
 \label{app:eq:finite-upper-R}
\end{align}
Moreover,
\begin{equation}
 \SeedRisk(\RandomizedPackedSampler;\TargetLaw,\FrozenOracle)
 \le
 \OracleRadius+\ScreenFailureBudget
 \quad\text{for every }
 \FrozenOracle\in\OracleClass(\TargetLaw;\OracleRadius).
 \label{app:eq:finite-upper-seed-risk}
\end{equation}
Consequently,
\begin{equation}
 \RobustRisk(\RandomizedPackedSampler;\TargetLaw)
 \le
 \OracleRadius+\ScreenFailureBudget.
 \label{app:eq:finite-upper-risk}
\end{equation}
In particular, if $\OracleRadius>0$ and
\begin{equation}
 \ScreenFailureBudget:=\frac{\OracleRadius}{2},
 \label{app:eq:finite-failure-budget-choice}
\end{equation}
then
$\RobustRisk(\RandomizedPackedSampler;\TargetLaw)\le3\OracleRadius/2$.
If the distinguished all-mask readout is charged, then
$\ReportedQueries=\CounterfactualQueries+1$ by
\eqref{app:eq:reported-query-count}.  Thus the bound for
$\ReportedQueries$ is the right-hand side of
\eqref{app:eq:finite-upper-Q} plus one; $\CounterfactualQueries$ still
excludes that readout.
Including that initial stage, the adaptive oracle depth is at most
$\OracleDepth\le1+\CommitRounds+\ScreenCallCap$.
\end{theorem}

\subsection{\texorpdfstring{\textcolor{red}{Proof overview}}{Proof overview}}
\label{app:upper-proof-overview}

{\color{red}
The proof turns local screen certificates into safe parallel commits:
accurate low-degree rows make peeling contract the high-degree core, after
which the remaining forest is known and its components can be sampled in
parallel.  We then bound resources and output error separately.
\Cref{app:tab:upper-proof-map} follows this chain under the finite hypotheses
above.  The structural certificates hold when every reached screen succeeds;
the resource caps hold on every path.  The same construction supports both
accuracy cases, with different output-error arguments.

\begin{table}[H]
\centering
\small
\color{red}
\setlength{\tabcolsep}{4pt}
\begin{tabularx}{\linewidth}{@{}>{\raggedright\arraybackslash}p{0.17\linewidth}
  >{\raggedright\arraybackslash}X
  >{\raggedright\arraybackslash}p{0.14\linewidth}
  >{\raggedright\arraybackslash}p{0.19\linewidth}@{}}
\toprule
Proof task & Key argument and conclusion & Algorithm steps & References \\
\midrule
\textbf{Certify the vocabulary banks}
& Preprocessing gives a nonempty complement of $\TokenBank{i}$, whose
tokens have true marginal mass below $\TailThreshold$, certifying the
tail representative.  Under the standard floor, (RF) bounds
$\BankSize{i}\le(4\FrequencyConstant/\TailThreshold)^{1/\RankTailExponent}$.
& \AlgoBlockRef{preprocess}
& \cref{app:lem:preprocessing-certificates} \\
\addlinespace
\textbf{Recover low-degree neighborhoods}
& Separating colors reduce packed probes to one-source comparisons at a
fixed boundary.  RT--UEN and the screening margin make their votes correct;
fresh-color majorities give
$\ScreenRow{j}=\mathcal N_{\ResidualForest}(j)$ for degree at most
$\DegreeCutoff$, except on reached-screen failures of total probability
at most $\ScreenFailureBudget$.
& \AlgoBlockRef{draft}--\AlgoBlockRef{screen};
\AlgoBlockRef{chunks}--\AlgoBlockRef{aggregate}
& \cref{app:lem:residual-markov,app:lem:random-screen-success,app:lem:fixed-draft-rows};
\AppLowDegreeScreenRef \\
\addlinespace
\textbf{Peel the high-degree core}
& Let $\HighDegreeCount{k}$ count vertices of degree above $\DegreeCutoff$
before phase $k$.  On successful paths, the forest degree sum gives
$\HighDegreeCount{k+1}\le(4/\DegreeCutoff)\HighDegreeCount{k}$.
Either stopping test leaves maximum degree at most $\DegreeCutoff$,
certifying the true residual forest.
& \AlgoBlockRef{peel}--\AlgoBlockRef{terminal}
& \cref{app:lem:successful-path-structure} \\
\addlinespace
\textbf{Complete safe parallel sampling}
& Singleton peels and one centroid per true residual component are safe.
Exact-row products then equal joint batch conditionals.  Centroid deletion
halves component sizes, giving logarithmically many terminal rounds on
successful paths.
& \AlgoBlockRef{peel}, \AlgoBlockRef{centroid}
& \cref{app:lem:residual-markov,app:lem:upper-resources,app:lem:exact-hybrid} \\
\addlinespace
\textbf{Bound all-path resources}
& Count screen loops for $\CounterfactualQueries$, and peel and centroid
commits for rounds.  \AlgoGuardRef\ caps rounds even on failed-screen paths.
Each screen is one parallel stage:
$\OracleDepth\le1+\CommitRounds+\ScreenCallCap$.
$\CounterfactualQueries$ excludes preprocessing and commits.
& All: \AlgoBlockRef{preprocess}--\AlgoBlockRef{centroid}, plus \AlgoGuardRef
& \cref{app:lem:upper-resources} \\
\addlinespace
\textbf{\HellingerCaseTitle: bound output TV}
& Each coordinate is committed once, so adaptive Hellinger composition
bounds an analysis-only safe rule's oracle error.  Coupling to that rule adds at most
$\ScreenFailureBudget$, giving seed-averaged TV at most
$\OracleRadius+\ScreenFailureBudget$.
& All: \AlgoBlockRef{preprocess}--\AlgoBlockRef{centroid}; output-law analysis
& \cref{app:lem:exact-hybrid,app:lem:adaptive-hellinger-upper,app:lem:safe-controller-comparison};
\cref{app:thm:finite-upper} \\
\addlinespace
\textbf{\KLCaseTitle: bound output KL}
& With forward row KL at most $\KLRowBudget$, the joint-reference identity
separates row error from batch total correlation.  The latter vanishes
on successful reference paths, giving seed-averaged KL at most
$N\KLRowBudget+N\log\VocabSize\,\ScreenFailureBudget$.
& All: \AlgoBlockRef{preprocess}--\AlgoBlockRef{centroid}; output-law analysis
& \cref{app:lem:adaptive-joint-kl,app:thm:finite-kl-upper} \\
\addlinespace
\textbf{Recover the main theorem}
& Verify preprocessing feasibility and screen separation for the main-text
scaling, then specialize the finite bounds and choose the degree cutoff.
Both cases yield \cref{main:thm:upper-curve}.
& All: \AlgoBlockRef{preprocess}--\AlgoBlockRef{centroid}; parameter choices
& \cref{app:lem:benchmark-feasibility,app:cor:fixed-accuracy-main} \\
\bottomrule
\end{tabularx}
\caption{Proof steps for the upper bound.  Step labels link to
Algorithms~\ref{app:alg:complete-sampler}--\ref{app:alg:packed-screen}.}
\label{app:tab:upper-proof-map}
\end{table}

Two distinctions guide the proof:
\begin{itemize}
\item High-degree rows need only satisfy the size cap
$|\ScreenRow{j}|\le\DegreeCutoff$.  Accurate low-degree rows suffice to
force contraction, even when high-degree rows are incorrect.
\item For \HellingerCase, an analysis-only safe rule switches to singleton
completion at the first failed screen (and before any cap fallback).
We compare its full output law with the implemented sampler's law;
no conditioning of the output on screen success is needed.
\end{itemize}
Combining these bounds gives \cref{app:eq:randomized-upper-TV-proof}:
\[
 \mathbb E_{\ControllerSeed}
 \dTV\bigl(\TargetLaw,
            \SamplerOutputLaw{\ControllerSeed}{\FrozenOracle}\bigr)
 \le
 \underbrace{\OracleRadius}_{\text{oracle error at commits}}
 +\underbrace{\ScreenFailureBudget}_{\text{failed-screen coupling}}.
\]
\par}

\subsection{\texorpdfstring{\textcolor{red}{Steps 1--3: }}{Steps 1--3: }Preprocessing and exact low-degree certification}
\label{app:upper-proof-screen}

We first establish the residual Markov and tail certificates, then prove the
fresh-color success bound and the low-degree screen contract.
\textcolor{red}{These certify the vocabulary banks from
\AlgoBlockNamedRef{preprocess}, the fresh colors from
\AlgoBlockNamedRef{draft}, and the rows returned by
\AlgoBlockNamedRef{screen}.}

The next lemma derives the residual Markov property from
\cref{main:eq:forest-factorization}.
Recall from \eqref{app:eq:history} that
$\ResidualPositions=\PositionSet\setminus G$ and
$\ResidualForest=\TargetForest[\ResidualPositions]$.
\textcolor{red}{Its row identity is used in \AlgoBlockNamedRef{probes},
and its component independence justifies \AlgoBlockNamedRef{centroid}.}

\begin{lemma}[Residual Markov property]
\label{app:lem:residual-markov}
Fix a history $\History=(G,x_G)$.  Conditional on $X_G=x_G$, the variables in
distinct connected components of $\ResidualForest$ are independent.  Moreover,
fix $j\in\ResidualPositions$ and the values of all residual neighbors of $j$.
The exact conditional row of $X_j$ is unchanged if any residual nonneighbor is
revealed, masked, changed, or marginalized.
\end{lemma}

\begin{proof}

\emph{1. Factor the conditional law over residual components.}
Recall the forest factorization \eqref{app:eq:S0}:
$\TargetLaw(x)=\prod_{u\in\TargetRoots}\PosMarginal{u}(x_u)
 \prod_{j\notin\TargetRoots}\EndpointKernel{\ParentOf{j}}{j}
 (x_j\mid x_{\ParentOf{j}})$.
Here $\TargetRoots$ and $\ParentOf{\cdot}$ are the original reference roots
and parent map.  Substitute the fixed committed values $x_G$.
For each connected component $T$ of $\ResidualForest$, identified with
its vertex set, define
\[
\begin{aligned}
 g_T(x_T):={}&
 \prod_{u\in\TargetRoots\cap T}\PosMarginal{u}(x_u)
 &&\text{roots in }T\\
 &{}\times
 \prod_{\substack{j\in T\setminus\TargetRoots\\\ParentOf{j}\in T}}
 \EndpointKernel{\ParentOf{j}}{j}(x_j\mid x_{\ParentOf{j}})
 &&\text{edges within }T\\
 &{}\times
 \prod_{\substack{j\in T\setminus\TargetRoots\\\ParentOf{j}\in G}}
 \EndpointKernel{\ParentOf{j}}{j}(x_j\mid x_{\ParentOf{j}})
 &&\text{edges }G\to T\\
 &{}\times
 \prod_{\substack{j\in G\setminus\TargetRoots\\\ParentOf{j}\in T}}
 \EndpointKernel{\ParentOf{j}}{j}(x_j\mid x_{\ParentOf{j}})
 &&\text{edges }T\to G.
\end{aligned}
\]
The fixed $x_G$ dependence is suppressed in $g_T$.
In particular, the last product retains factors for committed children of
vertices in $T$: their parent values still vary with $x_T$.
Every edge with a residual endpoint is either internal to one such $T$ or
crosses between $T$ and $G$; no edge joins distinct residual components.
Thus, with products over $T$ ranging over these connected components,
\[
 \TargetLaw(x_G,x_{\ResidualPositions})
 =
 \left[
 \prod_{u\in\TargetRoots\cap G}\PosMarginal{u}(x_u)
 \prod_{\substack{j\in G\setminus\TargetRoots\\\ParentOf{j}\in G}}
 \EndpointKernel{\ParentOf{j}}{j}(x_j\mid x_{\ParentOf{j}})
 \right]\prod_T g_T(x_T).
\]
The bracket is positive and depends only on $x_G$.
It therefore cancels when dividing this joint probability by
$\TargetLaw(X_G=x_G)
 =\sum_{z_{\ResidualPositions}}\TargetLaw(x_G,z_{\ResidualPositions})$.
Since the connected components partition $\ResidualPositions$,
finite distributivity gives
\[
 \sum_{x_{\ResidualPositions}}\prod_T g_T(x_T)
 =\prod_T\left(\sum_{x_T}g_T(x_T)\right).
\]
Consequently
\[
\begin{aligned}
 \TargetLaw(X_{\ResidualPositions}=x_{\ResidualPositions}\mid X_G=x_G)
 &=\frac{\prod_T g_T(x_T)}
         {\sum_{z_{\ResidualPositions}}\prod_T g_T(z_T)}\\
 &=\prod_T\frac{g_T(x_T)}{\sum_{z_T}g_T(z_T)}.
\end{aligned}
\]
All sums range over the vocabulary assignments on their indicated vertex
sets; their denominators are positive by strict positivity of the factors.
This proves conditional independence of the residual components.

\emph{2. Cancel all factors not incident to the readout.}
For the second claim, root the original connected component at $j$,
using (C0).
Every original neighbor $v$ of $j$ is observed: either $v\in G$, or its
value is fixed in the lemma.  Call this value $x_v$.  The factors involving
$X_j=a$ are exactly
\[
 \PosMarginal{j}(a)
 \prod_{v\in\mathcal N_{\TargetForest}(j)}
       \EndpointKernel{j}{v}(x_v\mid a).
\]
All other factors are independent of $a$ once these neighbor values are
fixed.  Revealing, changing, or summing out nonneighbors only changes a
multiplicative factor independent of $a$.  It cancels in the normalized row
\[
 \frac{\PosMarginal{j}(a)
       \prod_{v\in\mathcal N_{\TargetForest}(j)}
          \EndpointKernel{j}{v}(x_v\mid a)}
      {\sum_{b\in\Vocab}\PosMarginal{j}(b)
       \prod_{v\in\mathcal N_{\TargetForest}(j)}
          \EndpointKernel{j}{v}(x_v\mid b)}.
\]
Thus the history and fixed neighbor values determine the row.
\end{proof}

The next lemma certifies the omitted tokens for screening and, under (RF\textcolor{red}{; \cref{app:eq:RF}})
with the standard floor, bounds the number of source columns.

\paragraph{Intuition.}
The local vocabulary bank $\TokenBank{i}$ from \eqref{app:eq:local-bank}
includes every token whose true marginal mass exceeds $\TailThreshold$.  Conversely, under the standard floor, every included
token has true mass at least $\TailThreshold/2$.
The rank envelope then bounds how many such tokens exist:
\[
 \frac{\TailThreshold}{2}
 \le\FrequencyConstant(\BankSize{i}^{-\RankTailExponent}+\VocabSize^{-1}),
 \quad
 \FrequencyConstant\VocabSize^{-1}\le\frac{\TailThreshold}{4}
 \quad\Longrightarrow\quad
 \BankSize{i}\le(4\FrequencyConstant/\TailThreshold)^{1/\RankTailExponent}.
\]
This is an upper bound on bank size; (RF\textcolor{red}{; \cref{app:eq:RF}}) does not require the envelope to
be attained.  The complementary tokens can all be represented by one
certified tail token in the response test.

\begin{lemma}[Preprocessing certificates\textcolor{red}{; \AlgoBlockRef{preprocess}}]
\label{app:lem:preprocessing-certificates}
Suppose preprocessing in \cref{app:def:preprocessing} is feasible, and let
$\TailThreshold$ be the threshold it selects.  Every local vocabulary bank
$\TokenBank{i}$ has a nonempty complement in $\Vocab$, so each
$\TailRepresentative{i}$ is well-defined.  For every
$i\in\PositionSet$ and $a\in\Vocab$,
\begin{equation}
 \bigl|\PosMarginal{i}(a)-\NoisyMarginal{i}(a)\bigr|
 \le\RowTVError,
 \qquad
 a\notin\TokenBank{i}
 \Longrightarrow a\in\TailSet{i}{\TailThreshold}.
 \label{app:eq:proof-preprocessing-certificates}
\end{equation}
Under (RF\textcolor{red}{; \cref{app:eq:RF}}), the standard floor \eqref{app:eq:rate-threshold-floor}
additionally gives
\begin{equation}
 \BankSize{i}
 \le\left(\frac{4\FrequencyConstant}{\TailThreshold}\right)^{1/\RankTailExponent}
 \quad\text{for every }i.
 \label{app:eq:finite-bank-size-proof}
\end{equation}
\end{lemma}

\begin{proof}

\emph{1. Establish a nonempty vocabulary-bank complement.}
In \AlgoBlockNamedRef{preprocess}, preprocessing fixes
$\TailThreshold=\max\FeasibleThresholdGrid$ once from the
public parameters, before the all-mask submission; this threshold is reused
by every screen.
Since $\FrequencyConstant\ge1$, the standard floor satisfies
$\ThresholdFloor=4\max\{\RowTVError,\FrequencyConstant/\VocabSize\}
\ge2(\RowTVError+1/\VocabSize)$.
Together with $\TailThreshold\ge\ThresholdFloor$, this gives
\begin{equation}
 \TailThreshold-\RowTVError
 \ge\RowTVError+\frac2{\VocabSize}
 >\frac1{\VocabSize}
 \ge\min_{a\in\Vocab}\NoisyMarginal{i}(a).
 \label{app:eq:baseline-existence}
\end{equation}
Thus $\TokenBank{i}\ne\Vocab$, proving tail-representative existence.

\emph{2. Certify every omitted token using the row-TV bound.}
At the all-mask state, $\ExactRow{\AllMaskState}{i}=\PosMarginal{i}$ and
$\OracleRow{\AllMaskState}{i}=\NoisyMarginal{i}$.
Thus \eqref{app:eq:oracle-tv-radius} and the singleton-event bound for TV give,
for every $a\in\Vocab$,
\[
 \bigl|\PosMarginal{i}(a)-\NoisyMarginal{i}(a)\bigr|
 \le \sup_{A\subseteq\Vocab}
       \bigl|\PosMarginal{i}(A)-\NoisyMarginal{i}(A)\bigr|
 =\dTV(\PosMarginal{i},\NoisyMarginal{i})
 \le\RowTVError.
\]
This bounds both signs of the difference, whether or not $a$ belongs to
$\TokenBank{i}$.  For a token outside the vocabulary bank $\TokenBank{i}$,
definition \eqref{app:eq:local-bank} then gives
\begin{equation}
 a\notin\TokenBank{i}
 \quad\Longrightarrow\quad
 \PosMarginal{i}(a)
 \le\NoisyMarginal{i}(a)+\RowTVError<\TailThreshold.
 \label{app:eq:bank-complement-tail}
\end{equation}
This proves \eqref{app:eq:proof-preprocessing-certificates}, including the
tail certificate for $\TailRepresentative{i}$.

\emph{3. Apply (RF\textcolor{red}{; \cref{app:eq:RF}}) only for the bank-size estimate.}
For the vocabulary-bank size $\BankSize{i}=|\TokenBank{i}|$,
the standard floor gives
$\TailThreshold\ge\ThresholdFloor=\StandardThresholdFloor{\FrequencyConstant}$,
hence both $\TailThreshold\ge4\RowTVError$ and
$\TailThreshold\ge4\FrequencyConstant/\VocabSize$.
Every $a\in\TokenBank{i}$ therefore satisfies
\begin{equation}
 \PosMarginal{i}(a)
 \ge\NoisyMarginal{i}(a)-\RowTVError
 \ge\TailThreshold-2\RowTVError
 \ge\frac{\TailThreshold}{2}.
 \label{app:eq:bank-member-mass}
\end{equation}
If
$\BankSize{i}\ge1$, at least $\BankSize{i}$ tokens have exact mass at least
$\TailThreshold/2$; hence the exact rank-$\BankSize{i}$ mass has the same
lower bound.  Assumption (RF\textcolor{red}{; \cref{app:eq:RF}}) and
$\TailThreshold\ge4\FrequencyConstant/\VocabSize$ yield
\[
 \frac{\TailThreshold}{2}
 \le\FrequencyConstant
 \left(\BankSize{i}^{-\RankTailExponent}+\VocabSize^{-1}\right)
 \le\FrequencyConstant\BankSize{i}^{-\RankTailExponent}
    +\frac{\TailThreshold}{4}.
\]
Rearranging proves \eqref{app:eq:finite-bank-size-proof}.  The inequality is
trivial when $\BankSize{i}=0$.
\end{proof}

\paragraph{Screen-success event (analysis only).}
\textcolor{red}{The colors drawn in \AlgoBlockNamedRef{draft} determine
the following event for \AlgoBlockNamedRef{screen}.}
Fix a realized history $\History=(G,x_G)$.  For distinct
$i,j\in\ResidualPositions$ with
$d_{\ResidualForest}(j)\le\DegreeCutoff$, call $\tau$
\emph{$(i,j)$-separating} when
\begin{equation}
\begin{aligned}
 \ColorHash{\tau}{v}&\ne\ColorHash{\tau}{j}
   &&\text{for every }v\in\mathcal N_{\ResidualForest}(j),\\
 \ColorHash{\tau}{v}&\ne\ColorHash{\tau}{i}
   &&\text{for every }
   v\in\mathcal N_{\ResidualForest}(j)\setminus\{i\}.
\end{aligned}
\label{app:eq:pair-separating-coloring}
\end{equation}
The screen family is \emph{successful at $\History$}, denoted
$\ScreenSuccess{\History}$, if
\begin{equation}
 \left|
 \left\{\tau\in\ColoringIndexSet:
       \tau\text{ is $(i,j)$-separating}\right\}
 \right|>\frac{\ColoringCount}{2}
 \label{app:eq:screen-success-event}
\end{equation}
for every such ordered pair $(i,j)$.  This event is defined using the hidden
residual forest only for analysis; the implemented screen neither observes
nor tests it.  For every separating $\tau$, the two properties used by the
screen proof are
\begin{equation}
 \mathcal N_{\ResidualForest}(j)
 \cap\ColorClass{\tau}{\ColorHash{\tau}{j}}=\emptyset,
 \qquad
 \mathcal N_{\ResidualForest}(j)
 \cap\ColorClass{\tau}{\ColorHash{\tau}{i}}\subseteq\{i\}.
 \label{app:eq:pairwise-local-isolation}
\end{equation}
\begin{example}[Checking which colors separate a pair]
\label{ex:algorithm-separation-20260918}
Use the coloring drawn in \AlgoBlockRef{draft} from
\cref{ex:algorithm-submissions-20260918} and the forest in
\cref{ex:algorithm-peel-centroids-20260918}.  For the tested pair
$(i,j)=(1,10)$, the readout has neighbors
$\mathcal N_{\ResidualForest}(10)=\{1,11\}$.  The two tests in
\eqref{app:eq:pairwise-local-isolation} are therefore
\[
 \{1,11\}\cap\ColorClass{\tau}{2}=\emptyset,\qquad
 \{1,11\}\cap\ColorClass{\tau}{1}=\{1\}\subseteq\{i\}.
\]
Both hold: $h_\tau(1)=1$, $h_\tau(10)=2$, and $h_\tau(11)=3$.
If instead $h_\tau(11)=2$, then neighbor $11$ shares the readout color
and the first intersection contains $11$.  If instead $h_\tau(11)=1$,
then an additional neighbor shares the source color and the second
intersection is $\{1,11\}\not\subseteq\{1\}$.
This checks one pair and one coloring.  It does not assert the
simultaneous strict-majority event \eqref{app:eq:screen-success-event},
which requires the full family of $\ColoringCount$ colorings.
\end{example}

\paragraph{Call index, draft, and information before fresh colors.}
Fix the target, its forest, and the frozen oracle throughout.
The index $s\in[\ScreenCallCap]$ counts screen calls
(\AlgoBlockRef{screen}), not individual colorings or commit rounds.
Define the \emph{reach event}
\[
 R_s:=\{\text{Algorithm~\ref{app:alg:complete-sampler} reaches its $s$th screen call}\}.
\]
Thus $R_s$ is an event, not the round count $\CommitRounds$.
Recall that $\DraftFiller=(\DraftToken{i})_{i\in\ResidualPositions}
\in\Vocab^{\ResidualPositions}$ is the background token vector fixed by
$\DraftGenerator$ before fresh colors are drawn
(Definition~\ref{app:def:draft-generator}).  Its call-indexed version is
\[
 f^{(s)}:=
 \begin{cases}
 \DraftFiller\text{ fixed for call }s\text{ in \AlgoBlockRef{draft}},
     &\text{on }R_s,\\
 \dagger,&\text{on }R_s^{\mathsf c}.
 \end{cases}
\]
The symbol $\dagger$ means \emph{no $s$th call}: it is a formal marker,
not a vocabulary token, $\MASK$, or a draft vector.
Different executions can stop after different numbers of screens.
This marker defines $f^{(s)}$ on every execution, so the proof can sum
over the fixed range $s=1,\ldots,\ScreenCallCap$ without adding any call.

On $R_s$, let $\mathcal T_{s-}$ record the past colors, replies, commits,
and decision-rule state just after $f^{(s)}$ is fixed and before the
current color block is drawn.  On $R_s^{\mathsf c}$, use the terminal
transcript instead.  Define
\[
 \mathcal F_{s-1}:=\sigma(\mathcal T_{s-},f^{(s)}),
 \qquad
 \mathcal F:=\mathcal F_{s-1}\quad\text{at the call under consideration}.
\]
On a reached call, this is precisely the information available between
draft generation and fresh color generation: the history and completed
draft are fixed under this conditioning, but the current colors are not.
It excludes current or future colors and the full seed $W$; fixing $W$
would fix those colors.  No union over possible histories or drafts is needed.

\paragraph{Intuition.}
A low-degree readout creates at most $2\DegreeCutoff$ forbidden color
equalities.  Thus
$\RandomColorCount=8(\DegreeCutoff+1)\simeq\DegreeCutoff$
makes one coloring good with probability greater than $3/4$.
Repeating independently $\RandomColoringCount$ times makes a wrong
majority exponentially unlikely:
\[
 \Pr(\hbox{wrong majority for one pair}\mid\mathcal F)
 \le e^{-\RandomColoringCount/8},\qquad
 \ScreenCallCap N^2 e^{-\RandomColoringCount/8}\le\ScreenFailureBudget.
\]
A union bound over pairs yields \cref{main:eq:conditional-screen-success}.
The chosen $\RandomColoringCount$ is the ceiling of the logarithm
obtained by solving the second inequality.  Conditioning is on the past before the fresh colors.

\begin{lemma}[Fresh randomized screens succeed adaptively\textcolor{red}{; \AlgoBlockRef{draft}}]
\label{app:lem:random-screen-success}
Suppose Algorithm~\ref{app:alg:complete-sampler} reaches a screen call with
realized past $\mathcal F$, including the completed generation of its fixed
draft but preceding the current coloring block.  Conditional on
$\mathcal F$, the fresh family in
\eqref{app:eq:random-color-parameters} satisfies
\begin{equation}
 \Pr\bigl(\ScreenSuccess{\History}^{\mathsf c}\mid\mathcal F\bigr)
 \le \frac{\ScreenFailureBudget}{\ScreenCallCap}.
 \label{app:eq:conditional-screen-failure}
\end{equation}
Consequently, over all decision-rule randomness and commit draws,
\begin{equation}
 \Pr\{\text{some reached screen is unsuccessful}\}
 \le\ScreenFailureBudget.
 \label{app:eq:adaptive-screen-failure}
\end{equation}
\end{lemma}

\begin{proof}

\emph{1. Separate one ordered pair with one fresh coloring.}
Conditional on $\mathcal F$, the history, residual forest, and draft are fixed, while
the current colors drawn in \AlgoBlockRef{draft} are mutually independent
and uniform.  Fix an ordered pair
$(i,j)$ with $d_{\ResidualForest}(j)\le\DegreeCutoff$.  The two lines of
\eqref{app:eq:pair-separating-coloring} contain at most
$2d_{\ResidualForest}(j)\le2\DegreeCutoff$ forbidden color equalities.  Each
has probability $1/\RandomColorCount$, so one coloring is nonseparating with
conditional probability at most
\begin{equation}
 \frac{2\DegreeCutoff}{\RandomColorCount}
 =\frac{\DegreeCutoff}{4(\DegreeCutoff+1)}<\frac14.
 \label{app:eq:one-random-coloring-bad}
\end{equation}

\emph{2. Amplify to a strict majority for that pair.}
Let $X_\tau$ indicate that coloring $\tau$ separates this pair.  Conditional
on $\mathcal F$, the $X_\tau$ are independent and have mean at least $3/4$.
Hoeffding's inequality therefore gives
\begin{equation}
 \Pr\left\{\sum_{\tau=1}^{\RandomColoringCount}X_\tau
       \le\frac{\RandomColoringCount}{2}\,\middle|\,\mathcal F\right\}
 \le \exp\{-2\RandomColoringCount(3/4-1/2)^2\}
 =\exp\{-\RandomColoringCount/8\}.
 \label{app:eq:random-majority-hoeffding}
\end{equation}

\emph{3. Take a union bound over ordered pairs at the current history.}
There are fewer than $N^2$ relevant ordered pairs.  A union bound and
\eqref{app:eq:random-color-parameters} give
\[
 N^2e^{-\RandomColoringCount/8}
 \le\ScreenFailureBudget/\ScreenCallCap,
\]
which proves \eqref{app:eq:conditional-screen-failure}.

\emph{4. Sum over adaptively reached calls by conditional expectation.}
Recall that $R_s$ means the algorithm reaches its $s$th screen call,
and $f^{(s)}$ is that call's draft; $\dagger$ marks an unreached call.
On $R_s$, let $H^{(s)}$ be the committed history at that call, and define
\hypertarget{app-reached-screen-failure}{}
\[
 E_s:=\{R_s\text{ occurs and }\ScreenSuccess{H^{(s)}}\text{ fails}\}.
\]
Thus $E_s$ means \emph{call $s$ is reached and fails}; it is false on
$R_s^{\mathsf c}$, with no history or screen evaluated there.
The draft rule terminates before the fresh colors
(Definition~\ref{app:def:draft-generator}), so reaching the call is
decided before those colors.  In the notation above,
$R_s=\{f^{(s)}\ne\dagger\}\in
\mathcal F_{s-1}=\sigma(\mathcal T_{s-},f^{(s)})$.
On $R_s$ apply \eqref{app:eq:conditional-screen-failure};
on $R_s^{\mathsf c}$ the conditional failure probability is zero.  Hence
\[
 \Pr(E_s\mid\mathcal F_{s-1})
 \le
 \begin{cases}
 \ScreenFailureBudget/\ScreenCallCap,&\text{on }R_s,\\
 0,&\text{on }R_s^{\mathsf c}
 \end{cases}
 =\mathbf 1_{R_s}\frac{\ScreenFailureBudget}{\ScreenCallCap}.
\]
Taking expectations gives
\[
 \Pr(E_s)
 =\mathbb E\!\left[
   \Pr(E_s\mid\mathcal F_{s-1})
  \right]
 \le\mathbb E\!\left[
   \mathbf 1_{R_s}\frac{\ScreenFailureBudget}{\ScreenCallCap}
  \right]
 =\Pr(R_s)\frac{\ScreenFailureBudget}{\ScreenCallCap}
 \le\frac{\ScreenFailureBudget}{\ScreenCallCap}.
\]
The failed-screen event is the union of these reached-call failures, so
\[
 \Pr\!\left(\bigcup_{s=1}^{\ScreenCallCap}E_s\right)
 \le\sum_{s=1}^{\ScreenCallCap}\Pr(E_s)
 \le\ScreenCallCap\,
      \frac{\ScreenFailureBudget}{\ScreenCallCap}
 =\ScreenFailureBudget.
\]
This proves \eqref{app:eq:adaptive-screen-failure}.  No independence between
different screen-success events is asserted or needed.  The same separating coloring
isolates every packed column for a fixed ordered pair, so neither the event
nor this union bound requires an additional union over bank values, columns,
or possible drafts.  Draft generation is allowed to depend on the past
oracle transcript; independence of that draft from oracle error is not used.
\end{proof}

\paragraph{Intuition\textcolor{red}{~(\AlgoBlockRef{chunks}, \AlgoBlockRef{probes})}.}
\textcolor{red}{By \cref{app:lem:random-screen-success}, with high probability
every reached screen has a strict separating majority for each pair with
readout degree at most $\DegreeCutoff$.
Fix one such coloring with different source and readout colors.
The readout chunk contains no neighbor of $j$, and the probed source class
contains no neighbor of $j$ other than possibly $i$
(\cref{app:eq:pairwise-local-isolation}).
Thus every neighbor is observed, and only $i$ can change among them.
By \cref{app:lem:residual-markov}, the exact row equals the row obtained by
varying only $i$ at the common draft boundary in
\cref{app:eq:fixed-draft-boundary-row}; this is the boundary at which
(RT; \cref{app:eq:RT}) and (UEN; \cref{app:eq:UEN}) apply.}

\begin{lemma}[Exact rows at a fixed draft boundary\textcolor{red}{; \AlgoBlockRef{chunks}, \AlgoBlockRef{probes}}]
\label{app:lem:fixed-draft-rows}
Suppose preprocessing is feasible and $\MaxBankSize\ge1$.
Fix $\History=(G,x_G)$, $\DraftFiller\in\Vocab^{\ResidualPositions}$, and
distinct residual positions $i,j$.  Define the complete boundary $z$ on
$\PositionSet\setminus\{i,j\}$ by $z_k=x_k$ for $k\in G$ and
$z_k=\DraftToken{k}$ otherwise.  For an $(i,j)$-separating coloring with
different source and readout colors, every $\kappa\in\ScreenColumnSet$
satisfies
\begin{equation}
 \ExactRow{
   \ProbeState{\kappa}{\tau}{\ColorHash{\tau}{i}}
              {\ReadoutChunkOf{\tau}{j}}}{j}
 =\BoundaryRow{i}{j}{\PackedColumn{\kappa}{i}}{z}.
 \label{app:eq:fixed-draft-boundary-row}
\end{equation}
In particular, the exact row family contains all source values in the
local vocabulary bank $\TokenBank{i}$ and its tail representative
$\TailRepresentative{i}$, at this same boundary.
Padding may additionally contribute the row for $\DraftToken{i}$.
\end{lemma}

\begin{proof}
By \AlgoBlockNamedRef{chunks}, the readout chunk is contained in the
readout color class.
Thus the two isolation clauses in \eqref{app:eq:pairwise-local-isolation}
give, respectively,
\[
 \mathcal N_{\ResidualForest}(j)\cap\ReadoutChunkOf{\tau}{j}
 =\varnothing,
 \qquad
 \mathcal N_{\ResidualForest}(j)\cap
       \ColorClass{\tau}{\ColorHash{\tau}{i}}
 \subseteq\{i\}.
\]
All residual neighbors of $j$ are therefore observed.  Since the source
and readout colors differ, $i$ is also outside the masked readout chunk.
In the state submitted by \AlgoBlockNamedRef{probes}, the relevant
observed coordinates are exactly
\[
 \left(\ProbeState{\kappa}{\tau}{\ColorHash{\tau}{i}}
                    {\ReadoutChunkOf{\tau}{j}}\right)_v
 =
 \begin{cases}
 x_v,&v\in G,\\
 \PackedColumn{\kappa}{i},&v=i,\\
 \DraftToken{v},&v\in\mathcal N_{\ResidualForest}(j)\setminus\{i\}.
 \end{cases}
\]
These are the same values as in the complete boundary with source value
$\PackedColumn{\kappa}{i}$ and remaining values $z$.  Conditional on
$X_G=x_G$, Lemma~\ref{app:lem:residual-markov} says that the row for $j$
depends only on its residual neighbors.  Completing or changing all other
coordinates to $z$ therefore leaves it unchanged, proving the identity.
The bank columns and column $\TailColumn$ give the asserted subfamily by
\eqref{app:eq:packed-column}.  The boundary is independent of the coloring
because one draft is fixed across the entire screen.
\end{proof}

\textcolor{red}{\Cref{app:lem:random-screen-success} bounds the probability
that any reached screen fails.  Under the screen-success event,
\cref{app:lem:low-degree-screen} formalizes \eqref{main:eq:exact-low-degree-row}
and identifies the output of
\AlgoBlockNamedRef{aggregate} with the true neighbor set for every readout
of residual degree at most $\DegreeCutoff$.}
The target obeys the forest and RT--UEN\textcolor{red}{~%
(\cref{app:eq:RT,app:eq:UEN})} assumptions
(Assumptions~\ref{app:ass:forest-structure}
and~\ref{app:ass:response-regularity}), and oracle rows satisfy the TV bound
\eqref{app:eq:oracle-tv-radius}.
Recall the general resolution
$\ScreenResolution=4\RowTVError+\TailTolerance$ from
\eqref{app:eq:diameter-screen-resolution}; the main text uses
$\RowTVError=\TargetTolerance/(2\sqrt N)$.
The lemma's screening-margin hypothesis $\ScreenResolution<\EdgeSignal$
means that the target signal exceeds the row-noise and omitted-response allowances.
This is a sufficient screening condition, not a consequence of RT--UEN\textcolor{red}{~%
(\cref{app:eq:RT,app:eq:UEN})} alone.
Preprocessing feasibility is specified in
Definition~\ref{app:def:preprocessing}, and
$\ScreenSuccess{\History}$ is defined in
\eqref{app:eq:screen-success-event}.

\begin{lemma}[Low-degree screen contract\textcolor{red}{; \AlgoBlockRef{aggregate}}]
\label{app:lem:low-degree-screen}
\label{main:lem:low-degree-screen}
Suppose preprocessing is feasible and
$\ScreenResolution<\EdgeSignal$ holds. At any history $\History$,
with any draft fixed throughout the screen, $\ScreenSuccess{\History}$
implies
\begin{equation*}
 \ScreenRow{j}=\mathcal N_{\ResidualForest}(j)
 \quad\text{for every }j\in\ResidualPositions
 \text{ with }d_{\ResidualForest}(j)\le\DegreeCutoff.
\end{equation*}
\end{lemma}
High-degree rows have only the construction's size cap; no subset or
correctness guarantee is asserted for them.

\paragraph{Intuition for the screen contract.}
At a separating coloring, all exact rows for a nonedge coincide, so its
observed diameter is at most $2\RowTVError$.  For an edge, restricting source
values to the vocabulary bank plus tail representative,
$\TokenBank{i}\cup\{\TailRepresentative{i}\}$, loses at most
$\TailTolerance$ of the full response
diameter, and row noise loses at most another $2\RowTVError$.  Thus the
edge's observed diameter is at least
\[
 \EdgeSignal-\TailTolerance-2\RowTVError>2\RowTVError.
\]
The strict inequality is exactly
$4\RowTVError+\TailTolerance<\EdgeSignal$.  The proof below verifies the
bank/tail comparison, all empty-bank cases, and the strict-majority step.

\begin{proof}[Proof of \AppLowDegreeScreenRef]

\emph{1. Fix a successful screen and the hidden exact rows.}
We analyze the votes and selected rows of \AlgoBlockNamedRef{aggregate}.
If $\MaxBankSize=0$, use the empty-bank argument in the final paragraph
below.  Otherwise assume $\MaxBankSize\ge1$.
Fix a history $\History$ and a screen draft $\DraftFiller$ for which
$\ScreenSuccess{\History}$ holds, a readout
$j\in\ResidualPositions$ with
$d_{\ResidualForest}(j)\le\DegreeCutoff$, and a candidate source
$i\in\ResidualPositions\setminus\{j\}$.  By
\eqref{app:eq:screen-success-event}, more than $\ColoringCount/2$
evaluations are $(i,j)$-separating.

Recall the binary dependence-test vote $\ScreenVote{\tau}{i}{j}$ in
\eqref{app:eq:screen-vote}: it is zero if
$\ColorHash{\tau}{i}=\ColorHash{\tau}{j}$, and otherwise
$\ScreenVote{\tau}{i}{j}
 =\mathbf 1\{\ScreenDiameter{\tau}{i}{j}>2\RowTVError\}$.
Here $\ScreenDiameter{\tau}{i}{j}$ is the observed TV diameter in
\eqref{app:eq:observed-screen-diameter}.
By \eqref{app:eq:provisional-row}, $i\in\CandidateRow{j}$ exactly when
$\sum_{\tau\in\ColoringIndexSet}\ScreenVote{\tau}{i}{j}
 >\ColoringCount/2$.

For a separating evaluation whose source and readout colors differ, define
the hidden exact rows
\begin{align}
 \ProbeExactRow{\tau}{\kappa}{i}{j}
 &:=\ExactRow{
   \ProbeState{\kappa}{\tau}{\ColorHash{\tau}{i}}
              {\ReadoutChunkOf{\tau}{j}}}{j},
 \qquad \kappa\in\ScreenColumnSet,
 \label{app:eq:hidden-screen-rows}\\
 \ExactRowFamily{\tau}{i}{j}
 &:=\{\ProbeExactRow{\tau}{\kappa}{i}{j}:
             \kappa\in\ScreenColumnSet\}.
 \label{app:eq:hidden-screen-row-family}
\end{align}

\emph{2. No false positives.}
Assume $\{i,j\}\notin E(\ResidualForest)$ and fix a separating $\tau$.
If $\ColorHash{\tau}{i}=\ColorHash{\tau}{j}$, then
$\ScreenVote{\tau}{i}{j}=0$ by \eqref{app:eq:screen-vote}.
Otherwise, the probe state in \eqref{app:eq:screen-probe-state} masks
$\ReadoutChunkOf{\tau}{j}$, so the other masked positions are exactly
$\ReadoutChunkOf{\tau}{j}\setminus\{j\}$.
Since $\ReadoutChunkOf{\tau}{j}
\subseteq\ColorClass{\tau}{\ColorHash{\tau}{j}}$, the first relation in
\eqref{app:eq:pairwise-local-isolation} gives
$\mathcal N_{\ResidualForest}(j)\cap\ReadoutChunkOf{\tau}{j}=\varnothing$.
The positions that can change across columns lie in the source class
$\ColorClass{\tau}{\ColorHash{\tau}{i}}$.
The second relation in \eqref{app:eq:pairwise-local-isolation} bounds its
intersection with $\mathcal N_{\ResidualForest}(j)$ by $\{i\}$;
this intersection is empty because $i\notin\mathcal N_{\ResidualForest}(j)$.
Thus all bank- and tail-column states agree on the committed history and
reveal every residual neighbor of $j$ at its fixed draft value.
Lemma~\ref{app:lem:residual-markov} gives
\begin{equation}
 \ProbeExactRow{\tau}{\TailColumn}{i}{j}
 =\ProbeExactRow{\tau}{\kappa}{i}{j}
 \quad\text{for every }\kappa\in\ScreenColumnSet.
 \label{app:eq:nonedge-exact-row-equality}
\end{equation}
Every observed row is within $\RowTVError$ of this common exact row by
\eqref{app:eq:oracle-tv-radius}; hence every pair of observed rows is within
$2\RowTVError$, and $\ScreenVote{\tau}{i}{j}=0$.
A strict majority of all evaluations are separating, so
$\sum_{\tau\in\ColoringIndexSet}\ScreenVote{\tau}{i}{j}<\ColoringCount/2$.
Thus $i\notin\CandidateRow{j}$ by \eqref{app:eq:provisional-row}, proving
\begin{equation}
 \CandidateRow{j}\subseteq\mathcal N_{\ResidualForest}(j).
 \label{app:eq:no-false-positive-row}
\end{equation}

\emph{3. No false negatives when the local vocabulary bank is nonempty.}
Now assume $\{i,j\}\in E(\ResidualForest)$ and $\BankSize{i}\ge1$.
For every separating $\tau$, the first relation in
\eqref{app:eq:pairwise-local-isolation} implies
$\ColorHash{\tau}{i}\ne\ColorHash{\tau}{j}$.
The second relation in \eqref{app:eq:pairwise-local-isolation} makes $i$
the only residual neighbor of $j$ in the probed source class.  Fix a complete
boundary $z\in\Vocab^{\PositionSet\setminus\{i,j\}}$ that agrees with the
history and assigns $\DraftToken{v}$ to every uncommitted coordinate
$v\ne i,j$.  Lemma~\ref{app:lem:fixed-draft-rows} gives
\begin{align}
 \ProbeExactRow{\tau}{\TailColumn}{i}{j}
 &=\BoundaryRow{i}{j}{\TailRepresentative{i}}{z},\notag\\
 \ProbeExactRow{\tau}{\kappa}{i}{j}
 &=\begin{cases}
 \BoundaryRow{i}{j}{\BankToken{i}{\kappa}}{z},
     &\kappa\le\BankSize{i},\\
 \BoundaryRow{i}{j}{\DraftToken{i}}{z},
     &\BankSize{i}<\kappa\le\MaxBankSize.
 \end{cases}
 \label{app:eq:isolated-boundary-row-family}
\end{align}
Every separating $\tau$ therefore contains the same bank-plus-tail exact
subfamily.  Its diameter is
\begin{equation}
 D^{\mathrm{ex}}_{\History}(i,j)
 :=\max_{a,b\in\TokenBank{i}\cup\{\TailRepresentative{i}\}}
 \dTV\bigl(\BoundaryRow{i}{j}{a}{z},
            \BoundaryRow{i}{j}{b}{z}\bigr).
 \label{app:eq:queried-exact-diameter}
\end{equation}
Choose columns $\kappa,\kappa'$ whose exact rows attain this finite
bank-plus-tail maximum.  Padding can add rows but cannot remove these
columns.  Applying \eqref{app:eq:oracle-tv-radius} to both rows gives
\begin{align*}
 \ScreenDiameter{\tau}{i}{j}
 &\ge\dTV\bigl(\ProbeObservedRow{\tau}{\kappa}{i}{j},
                  \ProbeObservedRow{\tau}{\kappa'}{i}{j}\bigr)\\
 &\ge\dTV\bigl(\ProbeExactRow{\tau}{\kappa}{i}{j},
                  \ProbeExactRow{\tau}{\kappa'}{i}{j}\bigr)\\
 &\quad-\dTV\bigl(\ProbeObservedRow{\tau}{\kappa}{i}{j},
                    \ProbeExactRow{\tau}{\kappa}{i}{j}\bigr)\\
 &\quad-\dTV\bigl(\ProbeObservedRow{\tau}{\kappa'}{i}{j},
                    \ProbeExactRow{\tau}{\kappa'}{i}{j}\bigr)\\
 &\ge D^{\mathrm{ex}}_{\History}(i,j)-2\RowTVError.
\end{align*}
If $D^{\mathrm{ex}}_{\History}(i,j)>4\RowTVError$, this is larger than
$2\RowTVError$, so $\ScreenVote{\tau}{i}{j}=1$ for every separating $\tau$.
The strict majority then puts $i$ in $\CandidateRow{j}$ by
\eqref{app:eq:provisional-row}.
Taking the contrapositive,
\begin{equation}
 i\notin\CandidateRow{j}
 \quad\Longrightarrow\quad
 D^{\mathrm{ex}}_{\History}(i,j)\le4\RowTVError.
 \label{app:eq:missed-queried-diameter}
\end{equation}
To rule out this missed edge, it suffices to show
$\DirectedResponse{i}{j}{z}\le\ScreenResolution<\EdgeSignal$,
contradicting (UEN\textcolor{red}{; \cref{app:eq:UEN}}).  We establish this full-response bound next.

\emph{Recovering the full response diameter.}
The tail representative $\TailRepresentative{i}$ and every token outside
the vocabulary bank $\TokenBank{i}$ lie in
$\TailSet{i}{\TailThreshold}$ by
Lemma~\ref{app:lem:preprocessing-certificates}.  Feasibility of
$\TailThreshold$ and (RT\textcolor{red}{; \cref{app:eq:RT}}) therefore give
\begin{equation}
 \max_{a,b\notin\TokenBank{i}}
 \dTV\bigl(\BoundaryRow{i}{j}{a}{z},
            \BoundaryRow{i}{j}{b}{z}\bigr)
 \le\min\{1,\ResponseConstant\TailThreshold^{\ResponseExponent}\}
 \le\TailTolerance.
 \label{app:eq:tail-diameter-at-screen-boundary}
\end{equation}
Suppose $i\notin\CandidateRow{j}$.  In the mixed case
$a\in\TokenBank{i}$ and $b\notin\TokenBank{i}$, the triangle inequality
through the tail representative gives
\begin{align*}
 &\dTV\bigl(\BoundaryRow{i}{j}{a}{z},\BoundaryRow{i}{j}{b}{z}\bigr)\\
 &\quad\le
 \dTV\bigl(\BoundaryRow{i}{j}{a}{z},
           \BoundaryRow{i}{j}{\TailRepresentative{i}}{z}\bigr)
 +\dTV\bigl(\BoundaryRow{i}{j}{\TailRepresentative{i}}{z},
             \BoundaryRow{i}{j}{b}{z}\bigr)\\
 &\quad\le 4\RowTVError+\TailTolerance.
\end{align*}
The first distance uses \eqref{app:eq:missed-queried-diameter}; the second
uses \eqref{app:eq:tail-diameter-at-screen-boundary}.  Interchanging $a,b$
handles the other mixed case.  Together with the bank--bank and tail--tail
bounds, this gives, for arbitrary $a,b\in\Vocab$,
\begin{equation}
 \dTV\bigl(\BoundaryRow{i}{j}{a}{z},
            \BoundaryRow{i}{j}{b}{z}\bigr)
 \le
 \begin{cases}
  4\RowTVError,&a,b\in\TokenBank{i},\\
  4\RowTVError+\TailTolerance,
    &\text{exactly one of $a,b$ lies in }\TokenBank{i},\\
  \TailTolerance,&a,b\notin\TokenBank{i}.
 \end{cases}
 \label{app:eq:screen-three-case-diameter}
\end{equation}
The maximum of the three upper bounds in
\eqref{app:eq:screen-three-case-diameter} is $\ScreenResolution$ by
\eqref{app:eq:diameter-screen-resolution}.  Hence a missed true edge would
have directed response at most $\ScreenResolution$, contradicting (UEN\textcolor{red}{; \cref{app:eq:UEN}}) and
$\ScreenResolution<\EdgeSignal$.

The same three comparisons also give the quantitative link used in the
intuition, without assuming that the edge was missed:
\begin{align*}
 \DirectedResponse{i}{j}{z}
 &\le\max\{D^{\mathrm{ex}}_{\History}(i,j),
           D^{\mathrm{ex}}_{\History}(i,j)+\TailTolerance,
           \TailTolerance\}\\
 &=D^{\mathrm{ex}}_{\History}(i,j)+\TailTolerance,\\
 \ScreenDiameter{\tau}{i}{j}
 &\ge D^{\mathrm{ex}}_{\History}(i,j)-2\RowTVError\\
 &\ge\DirectedResponse{i}{j}{z}-\TailTolerance-2\RowTVError\\
 &\ge\EdgeSignal-\TailTolerance-2\RowTVError>2\RowTVError.
\end{align*}
Here the three entries are the bank--bank, mixed, and tail--tail cases;
the last line uses (UEN\textcolor{red}{; \cref{app:eq:UEN}}) for this true edge.

\emph{4. Empty vocabulary banks and row selection.}
If $\BankSize{i}=0$, every source token lies in the certified tail, so
\eqref{app:eq:tail-diameter-at-screen-boundary} bounds the entire directed
response by
$\TailTolerance\le\ScreenResolution<\EdgeSignal$.  Such an $i$ cannot be
adjacent to $j$ by (UEN\textcolor{red}{; \cref{app:eq:UEN}}).  This also covers $\MaxBankSize=0$: then every local
vocabulary bank $\TokenBank{i}$ is empty, every returned row is defined to
be empty, and no residual edge
incident to a low-degree readout can exist.

\emph{5. Preserve the exact candidate row under public selection.}
We have proved that $\CandidateRow{j}$ equals the true neighbor set.  Its size
is at most $\DegreeCutoff$, so every admissible selector in
Definition~\ref{app:def:row-selector}, as applied in \AlgoBlockRef{aggregate},
preserves it by the nonoverflow identity condition.  Therefore
$\ScreenRow{j}=\mathcal N_{\ResidualForest}(j)$.
\end{proof}

\subsection{\texorpdfstring{\textcolor{red}{Steps 4--6: }}{Steps 4--6: }Peeling, terminal certification, and resources}
\label{app:upper-proof-structure}

The screen contract now yields the contraction in
\cref{main:eq:peeling-contraction} and sound stopping \textcolor{red}{in
\AlgoBlockNamedRef{peel}, certifying the graph in \AlgoBlockNamedRef{terminal}}.
We then count resources \textcolor{red}{across the complete sampler, including
\AlgoBlockNamedRef{centroid} and \AlgoGuardRef}, distinguishing
successful-path estimates from unconditional implementation caps.
\begin{example}[One peel commit, then two centroid batches]
\label{ex:algorithm-peel-centroids-20260918}
Let $\ResidualPositions=[15]$, $\DegreeCutoff=9$, and let the current
residual forest have edge set
\[
 E=\{\{1,i\}:2\le i\le9\}
   \cup\{\{1,10\},\{10,11\},\{11,12\},
          \{1,13\},\{13,14\},\{14,15\}\}.
\]
Vertex $1$ has degree $10$; every other vertex has degree at most $2$.
Suppose the reached screens are successful.  By
\AppLowDegreeScreenRef, $\ScreenRow{u}=\mathcal N_{\ResidualForest}(u)$
for $u\ne1$, whereas the row $\ScreenRow{1}$ need not be exact.
Nevertheless, \eqref{app:eq:claim-and-peel-set} gives
\begin{align*}
 \ClaimCount{1}
 &=\sum_{u\ne1}\mathbf1\{1\in\mathcal N_{\ResidualForest}(u)\}=10>9/2,\\
 \ClaimCount{v}
 &=|\mathcal N_{\ResidualForest}(v)\setminus\{1\}|
     +\mathbf1\{v\in\ScreenRow{1}\}
 \le2+1=3<9/2 \quad(v\ne1).
\end{align*}
Hence $\PeelSet=\{1\}$, regardless of the high-degree row's contents.
In \AlgoBlockRef{peel}, the sampler freezes this set and commits
vertex $1$ alone.  The residual
components are then
\[
 \{2\},\ldots,\{9\},\qquad
 10\text{--}11\text{--}12,\qquad 13\text{--}14\text{--}15.
\]
The next pass through \AlgoBlockRef{screen} uses the new history, before
the stopping test in \AlgoBlockRef{peel}.
Every residual degree is now at most $2$, so all its rows are exact.
Every incoming count is at most $2<9/2$, giving an empty peel set and
the exact terminal forest in \AlgoBlockRef{terminal}.
The successive batches of \AlgoBlockRef{centroid} are
\[
 B_1=\{2,\ldots,9,11,14\},\quad |B_1|=10,
 \qquad
 B_2=\{10,12,13,15\},\quad |B_2|=4.
\]
The eight isolates and the two path centers form $B_1$; deleting them
leaves the four singleton components in $B_2$.
Thus the example uses $1+2=3$ commit rounds.
This count excludes screen submissions and is not the total oracle depth.
\Cref{fig:algorithm-workflow-20260918} keeps the vertex positions fixed
throughout these steps.
\end{example}

\begin{figure}[H]
 \centering
 \includegraphics[width=\linewidth]{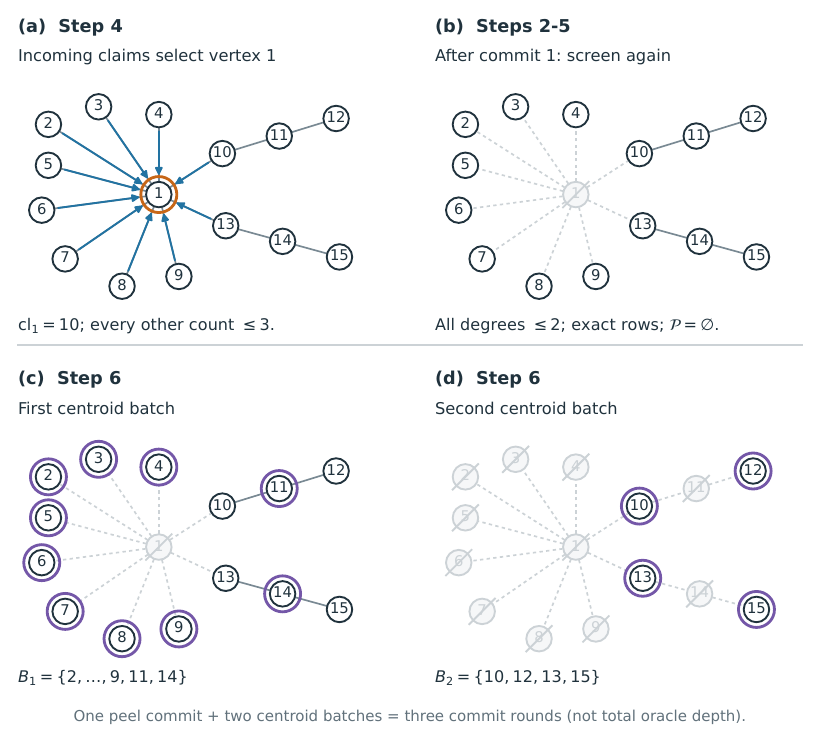}
 \caption{The successful-screen path of
 \cref{ex:algorithm-peel-centroids-20260918}, with $N=15$ and $d=9$.
 Orange and violet rings denote the next peel or centroid commit;
 pale crossed vertices are already committed. The fills here are
 neutral, not random colors. The step labels refer to
 Algorithm~\ref{app:alg:complete-sampler}. A fresh screen separates (a) from (b);
 the terminal forest then needs no further discovery. All vertices
 remain individually visible.}
 \label{fig:algorithm-workflow-20260918}
\end{figure}

\paragraph{Intuition.}
A high-degree vertex that receives at most $\DegreeCutoff/2$ claims has
at most $\DegreeCutoff/2$ low-degree neighbors, because those neighbors
report it exactly.  It must therefore have more than $\DegreeCutoff/2$
neighbors inside the high-degree subforest.  A forest has fewer than
twice as many total degrees as vertices, so only a fraction at most
$4/\DegreeCutoff$ of high-degree vertices can survive one phase.
When $\DegreeCutoff=N^{\DegreeExponent+o(1)}$ with fixed
$\DegreeExponent>0$, the required number of phases is
$O_{\DegreeExponent}(1)$.

\Needspace{12\baselineskip}
\begin{lemma}[Successful-path contraction and sound stopping\textcolor{red}{; \AlgoBlockRef{peel}, \AlgoBlockRef{terminal}}]
\label{app:lem:successful-path-structure}
Under the setup of Appendix~\ref{app:setup}, suppose preprocessing is feasible
and $\ScreenResolution<\EdgeSignal$.  Consider an execution path on which
every reached screen is successful in the sense of
\eqref{app:eq:screen-success-event}.  Let
$\HighDegreeCount{k}$ be the number of vertices whose residual degree exceeds
$\DegreeCutoff$ at the beginning of peel phase $k$.  After a nonempty peel
phase,
\begin{equation}
 \HighDegreeCount{k+1}
 \le\frac{4}{\DegreeCutoff}\HighDegreeCount{k},
 \label{app:eq:proof-peeling-contraction}
\end{equation}
with strict inequality whenever $\HighDegreeCount{k}>0$.
Moreover, either peel-loop stopping test in \AlgoBlockRef{peel} of
Algorithm~\ref{app:alg:complete-sampler}---an empty peel set or the
peel-phase cap---implies that the terminal residual
maximum degree is at most $\DegreeCutoff$.  The terminal certified graph is
therefore the true residual forest, so no cycle repair is performed.
\end{lemma}

\begin{proof}
\emph{1. High-degree neighbors of an unpeeled vertex.}
Fix a phase and define its high-degree set by
\begin{equation}
 \HighDegreeSet:=\{v\in\ResidualPositions:
                   d_{\ResidualForest}(v)>\DegreeCutoff\}.
 \label{app:eq:current-high-degree-set}
\end{equation}
At the beginning of phase $k$, $\HighDegreeCount{k}=|\HighDegreeSet|$.
Recall from \AlgoBlockNamedRef{peel} that the incoming claim count is
$\ClaimCount{v}=\sum_{u\in\ResidualPositions}\mathbf 1\{v\in\ScreenRow{u}\}$
and the peel threshold is $\ClaimCount{v}>\DegreeCutoff/2$.
For $v\in\HighDegreeSet\setminus\PeelSet$, every low-degree neighbor
reports $v$ exactly by \AppLowDegreeScreenRef.  Hence
\begin{align*}
 |\mathcal N_{\ResidualForest}(v)\setminus\HighDegreeSet|
 &=\sum_{u\in\ResidualPositions\setminus\HighDegreeSet}
       \mathbf 1\{v\in\mathcal N_{\ResidualForest}(u)\}\\
 &=\sum_{u\in\ResidualPositions\setminus\HighDegreeSet}
       \mathbf 1\{v\in\ScreenRow{u}\}
 \le\ClaimCount{v}\le\DegreeCutoff/2,\\
 d_{\ResidualForest[\HighDegreeSet]}(v)
 &=d_{\ResidualForest}(v)
       -|\mathcal N_{\ResidualForest}(v)\setminus\HighDegreeSet|\\
 &\ge d_{\ResidualForest}(v)-\ClaimCount{v}
 >\DegreeCutoff/2.
\end{align*}
The first equality uses the undirected nature of the forest.  Claims from
high-degree rows can only increase the incoming count, so they cannot
invalidate the bound for an unpeeled vertex.  We have shown that
\begin{equation}
 d_{\ResidualForest[\HighDegreeSet]}(v)
 >\frac{\DegreeCutoff}{2}.
 \label{app:eq:survivor-high-neighbors}
\end{equation}

\emph{2. Contraction after one peel phase.}
\emph{Empty survivor set.}
If $\HighDegreeSet\setminus\PeelSet=\emptyset$, no current high-degree vertex
survives the phase.  Deleting vertices cannot increase residual degrees, so
$\HighDegreeCount{k+1}=0$; this proves
\eqref{app:eq:proof-peeling-contraction}, strictly when
$\HighDegreeSet\ne\emptyset$.

\emph{Nonempty survivor set.}
Suppose instead that $\HighDegreeSet\setminus\PeelSet\ne\emptyset$.
Summing \eqref{app:eq:survivor-high-neighbors} over these unpeeled high-degree
vertices and using the forest degree-sum identity in
\cref{app:lem:forest-basics} gives
\begin{equation}
 \frac{\DegreeCutoff}{2}
 |\HighDegreeSet\setminus\PeelSet|
 <\sum_{v\in\HighDegreeSet\setminus\PeelSet}
   d_{\ResidualForest[\HighDegreeSet]}(v)
 \le2|E(\ResidualForest[\HighDegreeSet])|
 <2|\HighDegreeSet|.
 \label{app:eq:proof-high-degree-sum}
\end{equation}
Let $H'=(G',x_{G'})$ be the history after all singleton commits in the frozen
peel set of this phase, and recall the induced residual forest
$\ResidualForestAt{G'}=\TargetForest[\ResidualPositionsAt{G'}]$ from
\eqref{app:eq:history}.  Define
$\HighDegreeSetAt{H'}:=\{v\in\ResidualPositionsAt{G'}:
 d_{\ResidualForestAt{G'}}(v)>\DegreeCutoff\}$.
Deleting vertices cannot increase residual degrees:
\[
 \HighDegreeSetAt{H'}\subseteq\HighDegreeSet\setminus\PeelSet,
 \qquad
 \HighDegreeCount{k+1}=|\HighDegreeSetAt{H'}|
 \le|\HighDegreeSet\setminus\PeelSet|.
\]
Dividing \eqref{app:eq:proof-high-degree-sum} by $\DegreeCutoff/2$ now gives
$|\HighDegreeSet\setminus\PeelSet|<(4/\DegreeCutoff)|\HighDegreeSet|$,
and hence the strict form of \eqref{app:eq:proof-peeling-contraction}.  If
$\HighDegreeSet=\emptyset$, then the first case above gives equality
$\HighDegreeCount{k+1}=\HighDegreeCount{k}=0$.

\emph{3. Sound stopping at the phase cap.}
Initially, \cref{app:lem:forest-basics} makes the forest degree sum smaller
than $2N$, while each high-degree vertex has degree at least
$\DegreeCutoff+1$.  Hence
\begin{equation}
 \HighDegreeCount{0}<\frac{2N}{\DegreeCutoff+1}.
 \label{app:eq:initial-high-count}
\end{equation}
If the right-hand side is at most one, integrality already gives
$\HighDegreeCount{0}=0$.
Otherwise, once $\HighDegreeCount{k}=0$, deleting further vertices keeps every
later high-degree count equal to zero.  If this never occurs before the phase
cap, then $\HighDegreeCount{k}>0$ in every preceding phase, so the strict part
of \eqref{app:eq:proof-peeling-contraction} may be iterated.  In either case,
after $\PeelPhaseCap$ completed phases,
\begin{equation}
 \HighDegreeCount{\PeelPhaseCap}
 <\left(\frac4{\DegreeCutoff}\right)^{\PeelPhaseCap}
   \frac{2N}{\DegreeCutoff+1}
 \le
 \left(\frac8{\DegreeCutoff}\right)^{\PeelPhaseCap}
   \frac{2N}{\DegreeCutoff+1}
 \le1,
 \label{app:eq:phase-cap-eliminates-high-core}
\end{equation}
To check the last inequality, $\DegreeCutoff\ge9$ gives
$\log(\DegreeCutoff/8)>0$, and the ceiling in
\eqref{app:eq:phase-cap} gives
\begin{align*}
 \PeelPhaseCap\log(\DegreeCutoff/8)
 &\ge \PositivePart{\log(2N/(\DegreeCutoff+1))},\\
 \left(\frac8{\DegreeCutoff}\right)^{\PeelPhaseCap}
       \frac{2N}{\DegreeCutoff+1}
 &=\exp\!\left\{\log\frac{2N}{\DegreeCutoff+1}
                -\PeelPhaseCap\log(\DegreeCutoff/8)\right\}\le1.
\end{align*}
The integer in \eqref{app:eq:phase-cap-eliminates-high-core} is therefore zero.

\emph{4. Sound stopping with an empty peel set.}
If instead $\PeelSet=\emptyset$ while $\HighDegreeSet$ is nonempty, then
\eqref{app:eq:proof-high-degree-sum} holds with
$\HighDegreeSet\setminus\PeelSet=\HighDegreeSet$.  It would imply
$\DegreeCutoff/2<2$, contradicting $\DegreeCutoff\ge9$.  Hence an empty peel
set also certifies the absence of high-degree vertices.

\emph{5. \textcolor{red}{Bound peel-set size and certify the terminal graph.}}
We first check that the hard round cap cannot interrupt peeling.
\textcolor{red}{Fix a phase at history $\History$ with
$n:=|\ResidualPositions|\ge1$; if no vertex remains, the algorithm has
already returned.}
Low-degree exactness\textcolor{red}{~(\AppLowDegreeScreenRef)} and the row-size cap\textcolor{red}{~%
(\cref{app:def:row-selector})} imply
$|\ScreenRow{u}|\le d_{\ResidualForest}(u)$ for every $u$:
there is equality at low degree, while
$|\ScreenRow{u}|\le\DegreeCutoff<d_{\ResidualForest}(u)$ at high degree.
Counting each reported ordered pair once\textcolor{red}{, first at its
recipient and then in its reporting row,} gives
{\color{red}
\begin{align}
 \sum_{v\in\ResidualPositions}\ClaimCount{v}
 &=\sum_{v\in\ResidualPositions}\sum_{u\in\ResidualPositions}
       \mathbf 1\{v\in\ScreenRow{u}\}\notag\\
 &=\sum_{u\in\ResidualPositions}\sum_{v\in\ResidualPositions}
       \mathbf 1\{v\in\ScreenRow{u}\}\notag\\
 &=\sum_{u\in\ResidualPositions}|\ScreenRow{u}|\notag\\
 &\le\sum_{u\in\ResidualPositions}d_{\ResidualForest}(u)
 =2|E(\ResidualForest)|<2n.
 \label{app:eq:total-claim-bound}
\end{align}}
\textcolor{red}{The second equality interchanges finite sums; the last line
uses the forest degree sum in \cref{app:lem:forest-basics}.
Recall $\PeelSet=\{v\in\ResidualPositions:\ClaimCount{v}>\DegreeCutoff/2\}$
from \cref{app:eq:claim-and-peel-set}.
If $\PeelSet=\emptyset$, then $|\PeelSet|=0<4n/\DegreeCutoff$.
Otherwise its strict threshold gives}
{\color{red}
\begin{equation}
 \begin{aligned}
 \frac{\DegreeCutoff}{2}|\PeelSet|
 &<\sum_{v\in\PeelSet}\ClaimCount{v}
 \le\sum_{v\in\ResidualPositions}\ClaimCount{v}<2n,\\
 |\PeelSet|&<\frac{4n}{\DegreeCutoff}\le\frac{4N}{\DegreeCutoff}.
 \end{aligned}
 \label{app:eq:peel-set-size-proof}
\end{equation}}
\textcolor{red}{\AlgoBlockNamedRef{peel} freezes this set before its singleton
commits, so the phase uses exactly $|\PeelSet|$ rounds if completed and no
more if interrupted.}
With at most $\PeelPhaseCap$ phases, immediately before any planned peel
commit the completed-round counter satisfies
{\color{red}
\begin{equation}
 r<\frac{4N\PeelPhaseCap}{\DegreeCutoff}<\HardRoundCap-1,
 \label{app:eq:peel-before-hard-cap}
\end{equation}}
by \eqref{app:eq:screen-call-and-round-caps}.
\textcolor{red}{Thus \AlgoGuardRef\ cannot interrupt peeling.}
If all positions have already been committed, the residual graph is empty.
Otherwise peeling reaches one of the two tests identified in the statement.
Each pass executes \AlgoBlockRef{screen} before either stopping test
in \AlgoBlockRef{peel}.  At the
terminal history every residual vertex is consequently covered by
\AppLowDegreeScreenRef; the terminal screen is successful by
hypothesis, so all terminal rows are exact.  The
symmetrized edge set in \AlgoBlockNamedRef{terminal} is therefore
exactly $E(\ResidualForestAt{\star})$.  In particular, the cycle-repair loop
is skipped on the successful path.
\end{proof}

\begin{example}[A high-degree vertex can survive a peel phase]
\label{ex:algorithm-survivor-20260918}
This separate example illustrates the row contract used in
\cref{app:lem:successful-path-structure}, rather than specifying a
target--oracle pair.  Let $\DegreeCutoff=9$ and take the $65$ vertices
\[
 U=\{z\}\cup\{h_i:1\le i\le6\}
   \cup\{u_{i,j}:1\le i\le6,\ 1\le j\le9\}
   \cup\{s_1,s_2,s_3,s_4\},
\]
with edges
\[
 E=\{\{z,h_i\}:1\le i\le6\}
 \cup\{\{h_i,u_{i,j}\}:1\le i\le6,\ 1\le j\le9\}
 \cup\{\{z,s_j\}:1\le j\le4\}.
\]
At the history under consideration, take $\ResidualPositions=U$ and
$\ResidualForest=F=(U,E)$.
The center $z$ and all six hubs have degree $10$; all other vertices
are leaves.  Take the following row outputs, which obey the contract
\emph{exact at degree at most $d$, size at most $d$ elsewhere}:
\[
 \ScreenRow{z}=\ScreenRow{h_i}=\emptyset,\qquad
 \ScreenRow{u_{i,j}}=\{h_i\},\qquad
 \ScreenRow{s_j}=\{z\}.
\]
Computing the incoming counts, rather than reading a high-degree row's
size as its degree, gives
\begin{align*}
 \ClaimCount{z}&=4\le9/2, & \ClaimCount{h_i}&=9>9/2,\\
 \ClaimCount{u_{i,j}}&=\ClaimCount{s_j}=0,
 & \PeelSet&=\{h_1,\ldots,h_6\}.
\end{align*}
Thus $z$ is not peeled.  Before the commits its high-degree neighbors
are exactly $\{h_1,\ldots,h_6\}$, so their number is $6>9/2$, as required by
\eqref{app:eq:survivor-high-neighbors}.  After the six frozen singleton
commits in \AlgoBlockRef{peel}, its remaining neighbors are $\{s_1,\ldots,s_4\}$:
\begin{align*}
 d_{F[U\setminus\PeelSet]}(z)&=4,\\
 |\{v\in U:d_F(v)>9\}|&=7,\\
 |\{v\in U\setminus\PeelSet:d_{F[U\setminus\PeelSet]}(v)>9\}|&=0.
\end{align*}
The center survives as a vertex but ceases to be high-degree.
The $54$ hub leaves become isolated; none was committed in this phase.
\Cref{fig:algorithm-survivor-20260918} draws each of them separately.
\end{example}

\begin{figure}[H]
 \centering
 \includegraphics[width=\linewidth]{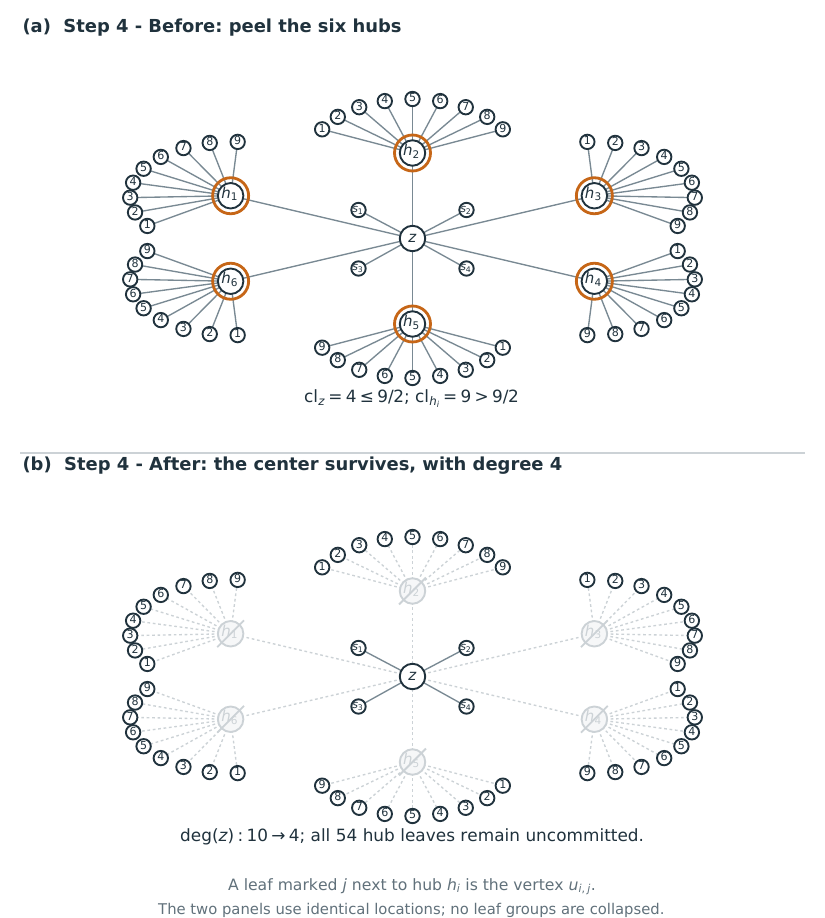}
 \caption{The separate contraction example of
 \cref{ex:algorithm-survivor-20260918}. The six orange-ringed hubs are
 committed one at a time in \AlgoBlockRef{peel}, with no intervening screen or peel-set
 recomputation. Their $54$ leaves are drawn individually, not bundled.
 The unpeeled center $z$ changes from degree $10$ to degree $4$.
 The lower panel keeps the same locations and crosses out only the
 committed hubs.}
 \label{fig:algorithm-survivor-20260918}
\end{figure}

\paragraph{Intuition.}
\textcolor{red}{The total counterfactual-submission count is bounded by the
product of the screen-loop bounds:}
\[
 \textcolor{red}{\CounterfactualQueries\le}
 \underbrace{\ScreenColumnCount}_{\text{\textcolor{red}{source columns}}}\qquad
 \underbrace{\ScreenCallCap}_{\text{\textcolor{red}{screen calls}}}\qquad
 \underbrace{\RandomColoringCount}_{\text{\textcolor{red}{colorings/screen}}}\qquad
 \underbrace{\RandomColorCount}_{\text{source colors}}\qquad
 \underbrace{(\RandomColorCount+\ChunkCount)}_{\text{\textcolor{red}{chunks/coloring}}}.
\]
For rounds, \textcolor{red}{\cref{app:eq:peel-set-size-proof} bounds each peel set by}
$4N/\DegreeCutoff$; after the peel phases, centroid recursion needs only
$\lceil\log_2(N+1)\rceil$ layers.  The public cap extends the round bound to
failed-screen paths.  Total masked-state submissions, including
preprocessing and commits, are $\CounterfactualQueries+\CommitRounds+1$.

\Needspace{14\baselineskip}
\begin{lemma}[Finite query and round counts\textcolor{red}{; \AlgoBlockRef{preprocess}--\AlgoBlockRef{centroid}}]
\label{app:lem:upper-resources}
Suppose preprocessing is feasible and
$\ScreenResolution<\EdgeSignal$.  On every execution path of
$\SelectedPackedSampler$ with any admissible row selector and the standard draft generator,
\begin{align}
 \CounterfactualQueries
 &\le\ScreenColumnCount\ScreenCallCap\RandomColoringCount
       \RandomColorCount(\RandomColorCount+\ChunkCount),
 \label{app:eq:proof-finite-Q}\\
 \CommitRounds
 &\le\HardRoundCap.
 \label{app:eq:proof-finite-R}
\end{align}
where
$\HardRoundCap=\lceil4N\PeelPhaseCap/\DegreeCutoff\rceil
+\lceil\log_2(N+1)\rceil+2$ is the public commit-round cap and
$\ScreenCallCap=\PeelPhaseCap+1$ is the screen-call cap from
\eqref{app:eq:screen-call-and-round-caps};
$\PeelPhaseCap$ is the peel-phase cap in \eqref{app:eq:phase-cap}.
In \eqref{app:eq:proof-finite-Q}, $\RandomColorCount$ and
$\RandomColoringCount$ are the fixed numbers of colors and independent
colorings per screen from \eqref{app:eq:random-color-parameters}.
The subscript $\mathrm{rnd}$ labels the randomized-color construction,
not a round index.
On a path on which every reached screen is successful, the sharper bound
\begin{equation}
 \CommitRounds
 <\frac{4N\PeelPhaseCap}{\DegreeCutoff}
   +\left\lceil\log_2(N+1)\right\rceil
 \label{app:eq:successful-path-rounds}
\end{equation}
holds, with no cycle repair or round-cap fallback.
The adaptive oracle depth, including the initial all-mask preprocessing
stage, satisfies
\begin{equation}
 \OracleDepth\le1+\CommitRounds+\ScreenCallCap.
 \label{app:eq:finite-oracle-depth}
\end{equation}
\end{lemma}

\begin{proof}

\emph{1. Count submissions over chunks, columns, and screens.}
Fix one screen execution at history $\History$.  Recall the quantities
counted in its loops:
\begin{itemize}
\item $\ScreenQueries$ denotes the number of masked-state submissions in
this one call to \AlgoBlockNamedRef{screen}, all made in
\AlgoBlockRef{probes}; preprocessing and commit submissions are excluded.
\item $\MaxBankSize=\max_{i\in\PositionSet}|\TokenBank{i}|$ is the largest
local vocabulary-bank size fixed in \cref{app:def:preprocessing}.
The number of submitted source columns is
$\ScreenColumnCount=\MaxBankSize+1$ when $\MaxBankSize\ge1$, and
$\ScreenColumnCount=0$ otherwise, as in \eqref{app:eq:screen-column-count}.
\item $\ColorCount=\RandomColorCount=8(\DegreeCutoff+1)$ is the number of
colors, and $\ColoringCount=\RandomColoringCount$ is the number of
independent colorings supplied by \AlgoBlockNamedRef{draft};
their choices are \eqref{app:eq:random-color-parameters}.
\item $\ChunkCount\in\PositionSet$ is the readout-chunk parameter from
\cref{app:algorithm-overview}: it sets
$\ReadoutChunkSize=\lceil N/\ChunkCount\rceil$, the maximum chunk size.
It is not the actual number of chunks.
\end{itemize}
For one coloring $\tau$, write $n_c:=|\ColorClass{\tau}{c}|$ and let
$\ChunksPerColoring{\tau}$ be the total number of readout chunks across
its color classes.  The deterministic partition
\eqref{app:eq:readout-chunks} in \AlgoBlockRef{chunks} creates exactly
$\lceil n_c/\ReadoutChunkSize\rceil$ chunks for every nonempty color class.
Therefore
\begin{align}
 \ChunksPerColoring{\tau}
 &=\sum_{c:n_c>0}
   \left\lceil\frac{n_c}{\ReadoutChunkSize}\right\rceil\notag\\
 &\le\sum_{c:n_c>0}
   \left(1+\frac{n_c}{\ReadoutChunkSize}\right)\notag\\
 &\le\ColorCount+\frac{|\ResidualPositions|}{\ReadoutChunkSize}
 \le\ColorCount+\ChunkCount.
 \label{app:eq:proof-chunk-count}
\end{align}
For each chunk, the bank columns and source-specific tail column in
\AlgoBlockRef{probes} use
$(\MaxBankSize+1)(\ColorCount-1)$ submissions.  When
$\MaxBankSize\ge1$, this is at most
$\ScreenColumnCount\ColorCount$.  Thus \AlgoBlockRef{chunks} supplies at
most $\ColorCount+\ChunkCount$ chunks per coloring, and
\AlgoBlockRef{probes} submits at most $\ScreenColumnCount\ColorCount$
states per chunk.  Summing over the $\ColoringCount$ colorings from
\AlgoBlockRef{draft} gives
\begin{equation}
 \ScreenQueries
 \le
 \underbrace{\ScreenColumnCount}_{\text{\textcolor{red}{source columns}}}\qquad
 \underbrace{\ColoringCount}_{\text{\textcolor{red}{colorings/screen}}}\qquad
 \underbrace{\ColorCount}_{\text{source colors}}\qquad
 \underbrace{(\ColorCount+\ChunkCount)}_{\text{chunks per coloring}}.
 \label{app:eq:proof-one-screen-count}
\end{equation}
When $\MaxBankSize=0$, the algorithm submits no screen state and both sides
are zero.  The screen runs initially and once after each of at most
$\PeelPhaseCap$ completed peel phases, hence at most $\ScreenCallCap$ times.
Substituting the randomized values of $\ColoringCount$ and $\ColorCount$
proves \eqref{app:eq:proof-finite-Q}.  The standard draft reuses preprocessing
and adds no query. Cycle repair runs no further screen: each singleton
repair is one commit submission, charged to $\CommitRounds$, not to
$\CounterfactualQueries$.

\emph{Returned oracle rows.}
For the standard zero-query draft, each counterfactual submission reads
only its chunk, of size at most $\ReadoutChunkSize=\lceil N/\ChunkCount\rceil$.
Preprocessing reads $N$ rows, and the disjoint commit batches
$B_1,\ldots,B_{\CommitRounds}$ partition $\PositionSet$
(\cref{app:def:admissible-algorithm}), so $\sum_{t=1}^{\CommitRounds}|B_t|=N$.
The main construction takes $\ChunkCount=\DegreeCutoff$.
Consequently the total number of returned oracle rows is at most
\[
 \underbrace{\lceil N/\ChunkCount\rceil\CounterfactualQueries}_{\text{screen submissions}}
 +\underbrace{N}_{\text{preprocessing}}
 +\underbrace{\sum_{t=1}^{\CommitRounds}|B_t|}_{\text{commits}}
 =\lceil N/\ChunkCount\rceil\CounterfactualQueries+2N.
\]

\emph{2. Count adaptive oracle stages.}
Once the history, draft, and colors are fixed, all submissions in a screen
can be issued in one parallel oracle stage \textcolor{red}{in
\AlgoBlockRef{probes}, because their states do not depend on replies from
that screen}.  Each commit requires one stage,
and preprocessing \textcolor{red}{in \AlgoBlockRef{preprocess}} requires one. This proves
\eqref{app:eq:finite-oracle-depth}.

\emph{3. Bound peeling rounds on a successful path.}
\textcolor{red}{On a path where every reached screen succeeds,
\eqref{app:eq:peel-set-size-proof} in
\cref{app:lem:successful-path-structure} shows that every nonempty peel phase in
\AlgoBlockNamedRef{peel} commits fewer than $4N/\DegreeCutoff$ vertices,
one per round.  At most $\PeelPhaseCap$ such phases therefore use fewer than
$4N\PeelPhaseCap/\DegreeCutoff$ rounds.
The bound \eqref{app:eq:peel-before-hard-cap} also shows that
\AlgoGuardRef\ cannot interrupt these commits.}

\emph{4. Bound terminal centroid rounds and exclude repair and fallback.}
By Lemma~\ref{app:lem:successful-path-structure}, the terminal certified graph
is the true residual forest.  Lemma~\ref{app:lem:residual-markov} makes its
connected components conditionally independent.  At any layer of
\AlgoBlockNamedRef{centroid}, let
the current components be $T_1,\ldots,T_m$, with selected centroids
$\ComponentCentroid{1},\ldots,\ComponentCentroid{m}$, and set
$B=\{\ComponentCentroid{1},\ldots,\ComponentCentroid{m}\}$.  Then
\begin{equation}
 \TargetLaw(X_B=x_B\mid X_G=x_G)
 =\prod_{\ell=1}^m
   \TargetLaw\bigl(X_{\ComponentCentroid{\ell}}
   =x_{\ComponentCentroid{\ell}}\mid X_G=x_G\bigr).
 \label{app:eq:centroid-batch-factorization}
\end{equation}
Thus their exact singleton rows define the exact joint batch conditional.
By
\cref{app:lem:forest-basics}, removing a centroid leaves components of size at
most half the parent size, so after
$\ell$ terminal batches every remaining component has size at most $N/2^\ell$.
For $\ell=\lceil\log_2(N+1)\rceil$ this is smaller than one.  Adding the peeling
and terminal rounds proves \eqref{app:eq:successful-path-rounds}.  Its
right-hand side is smaller than $\HardRoundCap-1$ by
\eqref{app:eq:screen-call-and-round-caps}; hence the hard-cap check is never
triggered during the terminal stage either. The terminal graph is already
a forest, so the cycle-repair loop adds no rounds.

\emph{5. Enforce the round cap on every path.}
It remains to verify the unconditional randomized round bound enforced by
\AlgoGuardRef. Initially $r=0\le\HardRoundCap-1$. Before every peel, repair,
or centroid commit, the algorithm checks whether $r=\HardRoundCap-1$.
If equality holds, it commits all remaining positions in the reserved
final round and returns. Otherwise the integer counter satisfies
$r\le\HardRoundCap-2$, and the next commit leaves
$r+1\le\HardRoundCap-1$.
Every such commit is nonempty and irrevocable, so it removes at least
one remaining position; cycle repair cannot continue indefinitely.
If it ends without the cap fallback, the remaining working graph is a
forest and centroid completion applies. Induction over commit operations
therefore gives $\CommitRounds\le\HardRoundCap$ on every path, proving
\eqref{app:eq:proof-finite-R}.
\end{proof}

\subsection{Exact-commit hybrid and adaptive error composition}
\label{app:upper-proof-error}

This subsection completes the statistical part of the upper bound.
\Cref{app:lem:exact-hybrid} identifies a structurally safe exact-row hybrid
with the target law.  Separately,
\cref{app:lem:adaptive-hellinger-upper} bounds the distance between any
product-commit decision rule and its exact-row hybrid, whether or not its
batches are structurally safe.  A safe-decision-rule comparison then accounts
for failed screens in the implemented randomized sampler.

\paragraph{Structural safety.}
At a history $\History=(G,x_G)$, let $\ResidualForest$ be the true residual
forest on $\ResidualPositions=\PositionSet\setminus G$.
A complete decision rule is \emph{structurally safe} if, at every realized
history, each batch $B\subseteq\ResidualPositions$ it commits satisfies
\[
 |B\cap V(T)|\le1
 \qquad\text{for every connected component }T\text{ of }\ResidualForest.
\]
In particular, every singleton batch is safe.
This is a property used in the analysis, not permission for the implemented
sampler to inspect the hidden forest.

\paragraph{The two executions and their output laws.}
Fix one complete decision rule, the frozen oracle $\FrozenOracle$, and a
decision-rule seed $\ControllerSeed=w$, including the colors and all
auxiliary draft randomness.  The product-commit draws remain random after
this conditioning.  For a chosen batch $B$ and the history-compatible
commit state $y$ in \eqref{app:eq:commit-state}, compare:
\begin{itemize}
\item \emph{Oracle execution.}
A batch assignment $z\in\Vocab^B$ has mass
$\prod_{j\in B}\OracleRowMass{y}{j}{z_j}$.
Recall from \textcolor{red}{\cref{app:def:output-risk}} that its full output law is
$\SamplerOutputLaw{w}{\FrozenOracle}(x)
 =\Pr(\SamplerOutput=x\mid\ControllerSeed=w)$, with $\FrozenOracle$ fixed
and all product-commit draws integrated out.
\item \emph{Exact-commit hybrid.}
Replace only this commit mass by
$\prod_{j\in B}\ExactRowMass{y}{j}{z_j}$, using the exact rows in
\eqref{app:eq:exact-row}.  Keep the decision rule, draft generator, query
rules, and all frozen-oracle replies unchanged.
Denote this execution's full output law by $\ExactCommitOutputLaw{w}$.
\end{itemize}
Both are laws on $\Vocab^N$; $\otimes$ refers to the product within each
commit, not to independence of all output coordinates.
The hybrid can still depend on $\FrozenOracle$ through its decisions;
that dependence is suppressed in $\ExactCommitOutputLaw{w}$.
The shared rules do not imply shared realized transcripts: different commit
outcomes can lead to different later queries and batches.

Structural safety identifies the hybrid's product at each commit with the
target's joint batch conditional; for an unsafe batch they need not agree.
Thus \cref{app:lem:exact-hybrid} uses safety to prove
$\ExactCommitOutputLaw{w}=\TargetLaw$, whereas
\cref{app:lem:adaptive-hellinger-upper} compares the two output laws without
requiring safety.  Both lemmas also apply to the analysis-only decision rule
introduced below.

\paragraph{Intuition.}
A safe batch takes at most one vertex from each independent residual
component.  Its product of exact rows is therefore the exact joint
conditional.  Multiplying these joint conditionals along the adaptive
partition is the ordinary chain rule, so the output law is exactly
$\TargetLaw$.

\begin{lemma}[Exact hybrid equals the target law]
\label{app:lem:exact-hybrid}
For every structurally safe complete decision rule, the exact-commit hybrid
satisfies
\begin{equation}
 \ExactCommitOutputLaw{w}=\TargetLaw
 \qquad\text{for every decision-rule seed }w.
 \label{app:eq:exact-hybrid-target}
\end{equation}
\end{lemma}

\begin{proof}

\emph{1. Identify the exact joint law of a safe batch.}
For a singleton batch the product of exact singleton rows is trivially its
exact joint conditional.  For a nonsingleton safe batch, its vertices lie in
distinct true residual components.  Lemma~\ref{app:lem:residual-markov}
therefore identifies their exact joint conditional with the product of their
exact singleton conditionals, as in
\eqref{app:eq:centroid-batch-factorization}.

\emph{2. Apply the chain rule along each adaptive path.}
For fixed $w$, draft generation is deterministic as a function of the observed
past, even when its rule is randomized before conditioning on $w$.
For a full assignment $x\in\Vocab^N$, let $B_t(x)$ be the batch selected along
the unique hybrid path consistent with $x$, and let $G_{t-1}(x)$ be the set
committed before that batch, with $G_0(x)=\varnothing$ and
$G_t(x)=G_{t-1}(x)\cup B_t(x)$.  The batches form an adaptive ordered partition
of $\PositionSet$.  The preceding paragraph identifies every hybrid batch law
with the corresponding exact joint conditional, so the conditional chain
rule gives
\begin{align}
 \ExactCommitOutputLaw{w}(x)
 &=\prod_t
 \TargetLaw\!\left(
  X_{B_t(x)}=x_{B_t(x)}
 \mid X_{G_{t-1}(x)}=x_{G_{t-1}(x)}
 \right)\notag\\
 &=\prod_t
 \frac{\TargetLaw(X_{G_t(x)}=x_{G_t(x)})}
      {\TargetLaw(X_{G_{t-1}(x)}=x_{G_{t-1}(x)})}\notag\\
 &=\TargetLaw(x).
 \label{app:eq:adaptive-exact-chain-rule}
\end{align}
The batch partition may depend on earlier committed values, but for each fixed
$x$ the displayed factors are precisely the chain-rule factors along that
path.  Consecutive numerators and denominators cancel; the first denominator
is the probability of the empty constraint, hence one, and the last numerator
is $\TargetLaw(x)$ because every coordinate has been committed.
This proves the equality pointwise.
\end{proof}

We next compare the full output laws $\ExactCommitOutputLaw{w}$ (exact-row
commits) and $\SamplerOutputLaw{w}{\FrozenOracle}$ (oracle-row commits) for
the same decision rule and fixed seed.  For probability masses $P,Q$ on a
finite set, write
\begin{equation}
 \HellingerAffinity(P,Q)
 :=\sum_x\sqrt{P(x)Q(x)}
 =1-\sqHellinger(P,Q).
 \label{app:eq:affinity-definition}
\end{equation}

\paragraph{Intuition.}
The error budget is attached to committed coordinates, not to probes or
rounds.  Each coordinate is committed once, so the total budget is
\[
 N\cdot\frac{(\OracleRadius)^2}{2N}=\frac{(\OracleRadius)^2}{2}.
\]
For adaptive batches, backward induction on the execution tree justifies
this accounting.  At each node, charge the current batch and then use the
budget of the remaining coordinates.

\begin{lemma}[Adaptive squared-Hellinger composition]
\label{app:lem:adaptive-hellinger-upper}
{\color{red}Assume \cref{app:ass:oracle-accuracy} (A2): every valid row pair obeys
$\sqHellinger(\ExactRow{y}{j},\OracleRow{y}{j})
 \le(\OracleRadius)^2/(2N)$.  For every complete decision rule using the
product-commit rule and every fixed decision-rule seed $w$, regardless of
structural safety,}
\begin{equation}
 \sqHellinger\bigl(
  \ExactCommitOutputLaw{w},
  \SamplerOutputLaw{w}{\FrozenOracle}
 \bigr)
 \le\frac{(\OracleRadius)^2}{2}.
 \label{app:eq:global-upper-hellinger}
\end{equation}
\end{lemma}

\begin{proof}

\emph{1. Compare the two kernels at a common node.}
Fix $\FrozenOracle$ and $w$ throughout the proof.
\hypertarget{app-execution-tree-state}{}
An \emph{execution-tree node} is a stored algorithm state immediately before
a commit, or a terminal state after completion; it is not a vertex of the
hidden forest.  A pre-commit state records
\[
 \mathfrak h=(\History,\ScreenTranscript,\text{decision-rule state}),
 \qquad
 \History=(G,x_G),\quad
 \ResidualPositions=\PositionSet\setminus G.
\]
Here $\ScreenTranscript$ contains the preceding query/reply and commit
record, and the internal state includes any current draft and counters.
We use this complete stored state to define continuation kernels directly,
without replaying the decision rule from its initial state.
With $\FrozenOracle$ and $w$ fixed, all operations between commits are
deterministic.  Thus, from the \emph{same stored state} $\mathfrak h$, both
executions receive the same frozen-oracle replies and select the same
batch $B(\mathfrak h)\subseteq\ResidualPositions$ and commit state
$y(\mathfrak h)$ from \eqref{app:eq:commit-state}.
Comparing kernels at this common state does not assert that two separately
sampled executions have identical paths.
For a batch assignment $z\in\Vocab^{B(\mathfrak h)}$, their kernels are
\begin{align}
 \ExactBatchMass{\mathfrak h}{z}
 &:=\prod_{j\in B(\mathfrak h)}
       \ExactRowMass{y(\mathfrak h)}{j}{z_j},\notag\\
 \OracleBatchMass{\mathfrak h}{z}
 &:=\prod_{j\in B(\mathfrak h)}
       \OracleRowMass{y(\mathfrak h)}{j}{z_j}.
 \label{app:eq:exact-oracle-batch-laws}
\end{align}
Both kernels are probability masses on $\Vocab^{B(\mathfrak h)}$:
\begin{align*}
 \sum_{z\in\Vocab^{B(\mathfrak h)}}\ExactBatchMass{\mathfrak h}{z}
 &=\prod_{j\in B(\mathfrak h)}
       \sum_{a\in\Vocab}\ExactRowMass{y(\mathfrak h)}{j}{a}=1,\\
 \sum_{z\in\Vocab^{B(\mathfrak h)}}\OracleBatchMass{\mathfrak h}{z}
 &=\prod_{j\in B(\mathfrak h)}
       \sum_{a\in\Vocab}\OracleRowMass{y(\mathfrak h)}{j}{a}=1.
\end{align*}
The finite sums factor by distributivity over the Cartesian product.
The first kernel is a product of exact rows, whether or not it equals the
target's joint batch conditional.  \textcolor{red}{Hence this comparison needs
no structural safety condition.}

\emph{2. Charge the current batch to its coordinates.}
Apply the product identity in \cref{app:lem:finite-hellinger-calculus}
with index set $B(\mathfrak h)$, each coordinate space equal to $\Vocab$,
and component laws $\ExactRow{y(\mathfrak h)}{j}$ and
$\OracleRow{y(\mathfrak h)}{j}$.  This gives
\begin{equation}
 \HellingerAffinity(\ExactBatchLaw{\mathfrak h},
                    \OracleBatchLaw{\mathfrak h})
 =\prod_{j\in B(\mathfrak h)}
  \left[1-\sqHellinger\bigl(
        \ExactRow{y(\mathfrak h)}{j},
        \OracleRow{y(\mathfrak h)}{j}\bigr)\right].
 \label{app:eq:upper-batch-affinity}
\end{equation}
Using $1-\prod_j(1-u_j)\le\sum_j u_j$ for $u_j\in[0,1]$ and
\textcolor{red}{\eqref{app:eq:A2}},
\begin{align}
 1-\HellingerAffinity(\ExactBatchLaw{\mathfrak h},
                      \OracleBatchLaw{\mathfrak h})
 &=\sqHellinger(\ExactBatchLaw{\mathfrak h},
                \OracleBatchLaw{\mathfrak h})
 \label{app:eq:local-batch-cost-definition}\\
 &\le\sum_{j\in B(\mathfrak h)}
     \sqHellinger\bigl(\ExactRow{y(\mathfrak h)}{j},
                      \OracleRow{y(\mathfrak h)}{j}\bigr)\notag\\
 &\le |B(\mathfrak h)|\frac{(\OracleRadius)^2}{2N}.
 \label{app:eq:local-batch-cost-bound}
\end{align}

\emph{3. Derive the suffix-affinity recursion.}
Take a stored execution state $\mathfrak h$ from part~1 of this proof, with committed
history $\History=(G,x_G)$ and remaining set
$\ResidualPositions=\PositionSet\setminus G$.
Here a \emph{suffix} is the assignment of all still-uncommitted coordinates,
not necessarily a contiguous block of positions.  Define:
\begin{itemize}
\item $\ExactSuffixLaw{\mathfrak h}$: the law of those remaining values
when execution resumes from this entire stored state using exact-row commits.
\item $\OracleSuffixLaw{\mathfrak h}$: the corresponding law when it
resumes from the same state using oracle-row commits.
\end{itemize}
Both are probability laws on $\Vocab^{\ResidualPositions}$, with
$\FrozenOracle$ and $w$ still fixed.  The recursion below defines these
continuation laws from the batch kernels in
\eqref{app:eq:exact-oracle-batch-laws}; at a terminal state
$\ResidualPositions=\varnothing$, each puts mass one on the empty assignment.
Set
\begin{equation}
 \SuffixAffinity{\mathfrak h}
 :=\HellingerAffinity(\ExactSuffixLaw{\mathfrak h},
                      \OracleSuffixLaw{\mathfrak h}).
 \label{app:eq:suffix-affinity}
\end{equation}
After a batch value $z\in\Vocab^{B(\mathfrak h)}$, write
$\mathfrak h\oplus z$ for the next pre-commit node (or terminal node),
including all intervening deterministic operations.  It has remaining
coordinates $\ResidualPositions\setminus B(\mathfrak h)$.
For $x'\in\Vocab^{\ResidualPositions\setminus B(\mathfrak h)}$,
the suffix assignment $(z,x')$ has masses
\begin{align*}
 \ExactSuffixMass{\mathfrak h}{(z,x')}
 &=\ExactBatchMass{\mathfrak h}{z}\,
   \ExactSuffixMass{\mathfrak h\oplus z}{x'},\\
 \OracleSuffixMass{\mathfrak h}{(z,x')}
 &=\OracleBatchMass{\mathfrak h}{z}\,
   \OracleSuffixMass{\mathfrak h\oplus z}{x'}.
\end{align*}
These are normalized: summing first over $x'$ gives one for each child
law, and summing over $z$ then gives one by the batch normalization above.
There are at most $N$ nonempty commits, so backward induction starts at
the empty suffix and defines every law.
Summing the geometric means first over $x'$ now yields
\begin{align}
 \SuffixAffinity{\mathfrak h}
 &=\sum_z\sqrt{\ExactBatchMass{\mathfrak h}{z}
               \OracleBatchMass{\mathfrak h}{z}}
       \sum_{x'}\sqrt{
        \ExactSuffixMass{\mathfrak h\oplus z}{x'}
        \OracleSuffixMass{\mathfrak h\oplus z}{x'}}\notag\\
 &=\sum_z\sqrt{\ExactBatchMass{\mathfrak h}{z}
               \OracleBatchMass{\mathfrak h}{z}}
       \SuffixAffinity{\mathfrak h\oplus z}.
 \label{app:eq:suffix-affinity-recursion}
\end{align}
At a terminal node the suffix is empty and its affinity is one.

\emph{4. Bound the remaining loss by backward induction.}
At a node with committed set $G$, every complete continuation partitions
$\ResidualPositions$ into disjoint future batches.  Thus its total
coordinate budget is
\begin{equation}
 \sum_{\text{future }t}|B(\mathfrak h_t)|
       \frac{(\OracleRadius)^2}{2N}
 =|\ResidualPositions|\frac{(\OracleRadius)^2}{2N}.
 \label{app:eq:remaining-path-budget}
\end{equation}
We prove by backward induction that this budget bounds the suffix loss:
\begin{equation}
 1-\SuffixAffinity{\mathfrak h}
 \le|\ResidualPositions|\frac{(\OracleRadius)^2}{2N}.
 \label{app:eq:adaptive-affinity-induction}
\end{equation}
The terminal case has both sides zero.  At a nonterminal node, subtract
the recursion from one and add and subtract the batch affinity:
\begin{align*}
 1-\SuffixAffinity{\mathfrak h}
 &=1-\HellingerAffinity(\ExactBatchLaw{\mathfrak h},
                        \OracleBatchLaw{\mathfrak h})\\
 &\quad+\sum_z
   \sqrt{\ExactBatchMass{\mathfrak h}{z}\OracleBatchMass{\mathfrak h}{z}}
        \bigl(1-\SuffixAffinity{\mathfrak h\oplus z}\bigr)\\
 &\le |B(\mathfrak h)|\frac{(\OracleRadius)^2}{2N}
   +\bigl(|\ResidualPositions|-|B(\mathfrak h)|\bigr)
      \frac{(\OracleRadius)^2}{2N}
      \sum_z\sqrt{\ExactBatchMass{\mathfrak h}{z}
                  \OracleBatchMass{\mathfrak h}{z}}\\
 &\le|\ResidualPositions|\frac{(\OracleRadius)^2}{2N}.
\end{align*}
The first inequality uses the batch bound and the induction hypothesis:
every child has exactly $|\ResidualPositions|-|B(\mathfrak h)|$
uncommitted coordinates.  \textcolor{red}{In the last inequality, the
remaining sum is the batch affinity, at most $\sqrt{1\cdot1}=1$ by
Cauchy--Schwarz and the two kernel normalizations in part~1.  The resulting
coordinate count is $|B(\mathfrak h)|+(|\ResidualPositions|-|B(\mathfrak h)|)
=|\ResidualPositions|$.}

\emph{5. Evaluate the budget at the root.}
The root $\mathfrak h_{\emptyset}$ is the common state before the first
commit, after the initial deterministic operations.  It has $G=\varnothing$
and $\ResidualPositions=\PositionSet$, so its continuation laws are the
full output laws recalled above:
\[
 \ExactSuffixLaw{\mathfrak h_{\emptyset}}=\ExactCommitOutputLaw{w},
 \qquad
 \OracleSuffixLaw{\mathfrak h_{\emptyset}}
 =\SamplerOutputLaw{w}{\FrozenOracle}.
\]
In particular, $|\ResidualPositions|=N$.  Equivalently, along every complete
path the disjoint commit batches satisfy
\begin{equation}
 \sum_t|B(\mathfrak h_t)|=N.
 \label{app:eq:path-batches-partition-N}
\end{equation}
Therefore
\begin{equation}
 \sqHellinger(\ExactCommitOutputLaw{w},
              \SamplerOutputLaw{w}{\FrozenOracle})
 =1-\SuffixAffinity{\mathfrak h_{\emptyset}}
 \le N\frac{(\OracleRadius)^2}{2N}
 =\frac{(\OracleRadius)^2}{2}.
 \label{app:eq:root-path-budget}
\end{equation}
Counterfactual operations affect which common node is reached; they do not
generate committed coordinates and add no term to this budget.
\end{proof}

We now separate oracle error from screen failures by comparing with an
analysis-only safe decision rule, whose seed-averaged discrepancy from the
sampler is charged to the failed-screen probability.

Recall the \hyperlink{app-reached-screen-failure}{reached-call event}
$E_s=\{R_s\text{ occurs and }\ScreenSuccess{H^{(s)}}\text{ fails}\}$:
$R_s$ means the implemented sampler reaches call $s$, and $H^{(s)}$ is
its history.  Failure means that the separating-color majority condition
\eqref{app:eq:screen-success-event} is violated.
Thus $\ScreenFailureEvent=\bigcup_{s=1}^{\ScreenCallCap}E_s$.

\paragraph{Intuition.}
The safe decision rule follows the sampler until a screen fails and then
finishes by singleton commits.  Its batches are always safe, so the previous
two lemmas bound its TV error by $\OracleRadius$.  A common-path coupling
charges the discrepancy between the two outputs only to screen failure:
\[
 \mathbb E_W\mathrm{TV}(\hbox{target},\hbox{sampler})
 \le\OracleRadius+\ScreenFailureBudget.
\]
The first contribution is controlled for every fixed seed; the failure
contribution is bounded only after averaging over the colors and commit
draws, as verified below.
\textcolor{red}{We compare the full output laws: conditioning the implemented
output on screen success would also condition its reached histories.}

\begin{lemma}[Safe decision rule and failed-screen coupling]
\label{app:lem:safe-controller-comparison}
Fix a target--oracle pair satisfying the finite hypotheses of
Theorem~\ref{app:thm:finite-upper}.  There is an analysis-only structurally
safe decision rule whose conditional output law
$\SafeSamplerOutputLaw{w}{\FrozenOracle}$ satisfies
\begin{align}
 \dTV\bigl(\TargetLaw,
       \SafeSamplerOutputLaw{w}{\FrozenOracle}\bigr)
 &\le \OracleRadius
 &&\text{for every }w,
 \label{app:eq:safe-seedwise-TV}\\
 \mathbb E_{\ControllerSeed}
 \dTV\bigl(
   \SamplerOutputLaw{\ControllerSeed}{\FrozenOracle},
   \SafeSamplerOutputLaw{\ControllerSeed}{\FrozenOracle}
 \bigr)
 &\le\ScreenFailureBudget.
 \label{app:eq:implemented-safe-TV}
\end{align}
\end{lemma}

\begin{proof}

\emph{1. Define the analysis-only safe decision rule.}
The safe decision rule simulates Algorithm~\ref{app:alg:complete-sampler} with
the same fixed public row selector, draft generator, auxiliary seed, and fresh colorings while every
reached screen satisfies
\eqref{app:eq:screen-success-event}.  It is additionally allowed to inspect
the hidden residual forest.  Immediately after the first unsuccessful
execution of \AlgoBlockRef{screen}, and also before any \AlgoGuardRef\
fallback, it stops using the screen
output and commits all remaining coordinates as singletons in increasing
position order.  More formally, number the successive simulation checkpoints
(screen returns and pre-fallback checks) by $n$, and define
\[
 \tau_{\mathrm{sw}}
 :=\inf\left\{n:
 \begin{array}{l}
 \text{an unsuccessful screen has just returned, or}\\
 \text{the simulated sampler is about to activate a fallback}
 \end{array}\right\},
 \qquad \inf\varnothing:=\infty.
\]
Before this checkpoint the two rules agree; from this checkpoint onward
the safe rule uses only the singleton completion just specified.
The checkpoint index is not the peel-phase index.
Hidden-forest inspection defines only this comparison law; it is not
available to the implemented sampler.
Each singleton commit uses the same frozen oracle $\FrozenOracle$ at the
updated committed history.  Writing $\widehat X^{\mathrm{safe}}$ for its
output, define
\[
 \SafeSamplerOutputLaw{w}{\FrozenOracle}(x)
 :=\Pr(\widehat X^{\mathrm{safe}}=x\mid\ControllerSeed=w),
\]
with the target and oracle fixed and all product-commit draws integrated out.

\emph{2. Check structural safety in all three stages.}
\begin{itemize}
\item Before a switch, the peel commits in \AlgoBlockRef{peel} are singletons.
\item If the terminal stage is
reached without a switch, \cref{app:lem:successful-path-structure} identifies
the graph constructed in \AlgoBlockRef{terminal} with the true residual
forest, so every centroid batch in \AlgoBlockRef{centroid} contains one
vertex from each true residual component.
\item After a switch, all
batches are singletons.
\end{itemize}
Hence the safe decision rule is structurally safe on
every path.

\emph{3. Bound oracle error for each fixed seed.}
Apply \cref{app:lem:exact-hybrid,app:lem:adaptive-hellinger-upper}
to this safe rule and its own product-of-exact-rows hybrid.
The first lemma identifies that hybrid with $\TargetLaw$; the second
compares it with the safe rule's oracle output.  Thus, for each fixed $w$,
\begin{equation}
 \sqHellinger\bigl(
  \TargetLaw,\SafeSamplerOutputLaw{w}{\FrozenOracle}
 \bigr)
 \le\frac{(\OracleRadius)^2}{2}.
 \label{app:eq:safe-global-hellinger}
\end{equation}
By \eqref{app:eq:aux-tv-hellinger},
\begin{equation}
 \dTV\bigl(\TargetLaw,
       \SafeSamplerOutputLaw{w}{\FrozenOracle}\bigr)
 \le\sqrt{2\sqHellinger\bigl(
  \TargetLaw,\SafeSamplerOutputLaw{w}{\FrozenOracle}
 \bigr)}
 \le\OracleRadius,
\end{equation}
which proves \eqref{app:eq:safe-seedwise-TV}.

\emph{4. Couple the sampler to its own safe decision rule.}
For \eqref{app:eq:implemented-safe-TV}, fix the full decision-rule seed $w$ and couple the
two conditional executions as follows.  Before the safe decision rule switches,
whenever their histories agree, use
the same draw from their common product-commit law.  Thus their histories and
all drafts and deterministic frozen-oracle replies remain identical until the safe
decision rule switches.  By \cref{app:lem:upper-resources}, in the
implemented execution,
\[
 \{\text{cycle repair or a hard round-cap fallback occurs}\}
 \subseteq\ScreenFailureEvent.
\]
An empty peel set or the peel-phase cap triggers the normal loop exit in
\AlgoBlockRef{peel}.  On $\ScreenFailureEvent^{\mathsf c}$ no switch occurs
and the outputs coincide; repair can start only after the safe rule has
switched to singleton completion.  For the coupled outputs
$\SamplerOutput$ and $\widehat X^{\mathrm{safe}}$, the common prefix gives
the pathwise inclusions
\[
 \{\SamplerOutput\ne\widehat X^{\mathrm{safe}}\}
 \subseteq\{\tau_{\mathrm{sw}}<\infty\}
 \subseteq\ScreenFailureEvent.
\]
Indeed, a first switch caused by a failed screen occurs in that common
prefix.  A first switch caused by a fallback with no preceding failure is
excluded by \cref{app:lem:upper-resources}.  After a switch the two
executions may be completed with any coupling of their respective laws.
The first common-path failure is exactly the first failure of the
implemented execution, since the coupled histories agree up to that point.
\textcolor{red}{For any coupling $(X,Y)$ of laws $P,Q$, the coupling inequality
is $\dTV(P,Q)\le\Pr(X\ne Y)$.  Applying it to these two outputs gives}
{\color{red}
\begin{equation}
 \dTV\bigl(
   \SamplerOutputLaw{w}{\FrozenOracle},
   \SafeSamplerOutputLaw{w}{\FrozenOracle}
 \bigr)
 \le
 \Pr(\SamplerOutput\ne\widehat X^{\mathrm{safe}}
       \mid\ControllerSeed=w)
 \le
 \Pr(\ScreenFailureEvent\mid\ControllerSeed=w),
 \label{app:eq:conditional-safe-coupling}
\end{equation}
}
where the probability integrates the coupled commit draws.

\emph{5. Average the failed-screen probability over seeds.}
Averaging \eqref{app:eq:conditional-safe-coupling} over the seed gives
\begin{align*}
 &\mathbb E_W\dTV\bigl(\SamplerOutputLaw{W}{\FrozenOracle},
                        \SafeSamplerOutputLaw{W}{\FrozenOracle}\bigr)\\
 &\quad\le\mathbb E_W\Pr(\ScreenFailureEvent\mid W)
 =\Pr(\ScreenFailureEvent)
 =\Pr\!\left(\bigcup_{s=1}^{\ScreenCallCap}E_s\right)
 \le\ScreenFailureBudget.
\end{align*}
The equality averaging conditional probabilities is the law of total
expectation over colors, draft randomness, and commit draws; the final
inequality is \eqref{app:eq:adaptive-screen-failure}.  This proves
\eqref{app:eq:implemented-safe-TV}, without a seedwise failure bound.
\end{proof}

\begin{proof}[Proof of Theorem~\ref{app:thm:finite-upper}]

\emph{1. Account for submissions, commits, and depth.}
With the standard zero-query generator, the pathwise query and round bounds are
\eqref{app:eq:proof-finite-Q}--\eqref{app:eq:proof-finite-R} from
\cref{app:lem:upper-resources}.  The all-mask preprocessing state is excluded
from $\CounterfactualQueries$ by
\textcolor{red}{\cref{app:def:resources}}. Charging it gives
$\ReportedQueries=\CounterfactualQueries+1$, as in
\eqref{app:eq:reported-query-count}.
The depth bound is \eqref{app:eq:finite-oracle-depth}.

\emph{2. Combine oracle error and failed-screen error.}
For every target--oracle pair, the TV triangle inequality followed by
\cref{app:lem:safe-controller-comparison} gives
\begin{align}
 &\mathbb E_{\ControllerSeed}
 \dTV\bigl(
  \TargetLaw,
  \SamplerOutputLaw{\ControllerSeed}{\FrozenOracle}
 \bigr)\notag\\
 &\quad\le
 \mathbb E_{\ControllerSeed}
 \dTV\bigl(
  \TargetLaw,
  \SafeSamplerOutputLaw{\ControllerSeed}{\FrozenOracle}
 \bigr)
 +\mathbb E_{\ControllerSeed}
 \dTV\bigl(
  \SafeSamplerOutputLaw{\ControllerSeed}{\FrozenOracle},
  \SamplerOutputLaw{\ControllerSeed}{\FrozenOracle}
 \bigr)\notag\\
 &\quad\le
 \OracleRadius+\ScreenFailureBudget.
 \label{app:eq:randomized-upper-TV-proof}
\end{align}
This proves \eqref{app:eq:finite-upper-seed-risk} for each fixed oracle.
Taking the supremum over
$\FrozenOracle\in\OracleClass(\TargetLaw;\OracleRadius)$ gives
\eqref{app:eq:finite-upper-risk}.  With the choice
\eqref{app:eq:finite-failure-budget-choice}, its
right-hand side is $3\OracleRadius/2$.
\end{proof}


\section{\KLCaseTitle: forward-KL upper bound}
\label{app:kl-extension}

We prove \KLCase\ of \cref{main:thm:upper-curve}, with the finite
guarantee in \cref{app:thm:finite-kl-upper}.
The same sampler uses the row-KL and screen-failure budgets defined below.
Its analysis uses a reference process that draws each batch from the true
\emph{joint} conditional, even for unsafe batches.
This differs from the product-of-exact-rows hybrid used for TV.

\begin{proposalcontext}{Relation to the TV formulation}
The Hellinger/TV and forward-KL guarantees use the same product-commit rule
with different public accuracy calibrations.  The Hellinger condition (A2)
applies to the TV case.
\end{proposalcontext}

\subsection{Oracle condition and output objective}
\label{app:kl-setup}

We use the same frozen row family, probe/commit operations, and resource
counts as in
\textcolor{red}{\cref{app:def:frozen-oracle,app:def:admissible-algorithm,app:def:resources}},
with the following forward-KL accuracy condition and output risk, corresponding
to \cref{main:ass:oracle-accuracy,main:def:algorithm-risk-resources}.

\begin{definition}[Uniform forward-KL oracle and seed-averaged KL risk]
\label{app:def:kl-oracle-risk}
For a strictly positive target law and $\KLRowBudget\ge0$, let
$\KLOracleClass(\TargetLaw;\KLRowBudget)$ consist of the strictly positive
deterministic frozen maps of Definition~\ref{app:def:frozen-oracle} such that
\begin{equation}
 \sup_{(y,j):\,j\notin\ObservedSet{y}}
 \KLDivergence{\ExactRow{y}{j}}{\OracleRow{y}{j}}
 \le\KLRowBudget.
 \label{app:eq:kl-row-condition}
\end{equation}
For a complete admissible decision rule,
\textcolor{red}{use the output law of \cref{app:def:output-risk} and define}
\begin{equation}
 \KLSeedRisk(\mathcal A;\TargetLaw,\FrozenOracle)
 :=\mathbb E_{\ControllerSeed}
   \KLDivergence{\TargetLaw}
     {\SamplerOutputLaw{\ControllerSeed}{\FrozenOracle}}.
 \label{app:eq:kl-seed-risk}
\end{equation}
The decision-rule seed excludes the primitive randomness used to draw commit
values.  All KL divergences use natural logarithms and the displayed forward
direction: target or exact conditional first, generated law or oracle second.
\end{definition}

The bound in \eqref{app:eq:kl-row-condition} is uniform over \emph{all} masked
states, including counterfactual and draft-filled inputs, not merely an
average under a data--mask or generated-trajectory law.  Lemma
\ref{app:lem:kl-hellinger} implies
\begin{equation}
 \KLRowBudget\le\frac{(\OracleRadius)^2}{N}
 \quad\Longrightarrow\quad
 \KLOracleClass(\TargetLaw;\KLRowBudget)
 \subseteq\OracleClass(\TargetLaw;\OracleRadius).
 \label{app:eq:kl-oracle-inclusion}
\end{equation}
Thus all existing preprocessing and screen certificates apply with this
public Hellinger radius.  Neither the converse inclusion nor an output-KL
bound from (A2) alone is asserted.  Reverse row KL is sufficient for (A2),
but is not the assumption used for the forward-KL result here.

\subsection{Adaptive joint-reference identity}
\label{app:kl-joint-reference}

Learning error and conditional total correlation also appear in the adaptive
planner decomposition of \citet[Proposition~2]{lavenant2025error}
(\href{https://arxiv.org/html/2510.25544v2}{arXiv:2510.25544v2}).
\textcolor{red}{For completeness, we prove the fixed-seed equality for our
frozen-oracle interface, keeping counterfactual operations in the adaptive
execution path.}

\emph{The joint-reference process.}
Fix the frozen oracle and a decision-rule seed $w$.  Keep the decision rule,
counterfactual queries, oracle replies, public row selector, and draft-generation rule fixed, but
replace each commit kernel by the \emph{true joint conditional} of that
batch given the committed history.  Denote this analysis-only process by
$\JointReferenceLaw{w}$; its dependence on the target, oracle, and decision rule
is suppressed.  Unlike $\ExactCommitOutputLaw{w}$, it draws a dependent
batch jointly when the decision rule selects an unsafe batch.  It need not be
implementable by the available oracle.

\emph{Its conditional batch law and total correlation.}
Here $\mathfrak h$ is the
\hyperlink{app-execution-tree-state}{stored pre-commit execution-tree state}
from the proof of Lemma~\ref{app:lem:adaptive-hellinger-upper}: it includes
the committed history, transcript, and internal decision-rule state.
At this node, let $G(\mathfrak h)$ and
$x_{G(\mathfrak h)}$ be the committed set and values, $B(\mathfrak h)$
the next batch, and $y(\mathfrak h)$ the corresponding commit state.  Put
\[
 \JointBatchLaw{\mathfrak h}
 :=\mathcal L_{\TargetLaw}
   (X_{B(\mathfrak h)}\mid
    X_{G(\mathfrak h)}=x_{G(\mathfrak h)}).
\]
Its singleton marginals are $\ExactRow{y(\mathfrak h)}{j}$.
The conditional batch's total correlation is
\[
 \mathrm{TC}(\JointBatchLaw{\mathfrak h})
 :=\KLDivergence{\JointBatchLaw{\mathfrak h}}
  {\bigotimes_{j\in B(\mathfrak h)}\ExactRow{y(\mathfrak h)}{j}}.
\]

\paragraph{Intuition.}
For one batch, insert the product of its true singleton marginals into
the likelihood ratio:
\[
 \log\frac{p(X_B)}{\prod_jq_j(X_j)}
 =\log\frac{p(X_B)}{\prod_jp_j(X_j)}
  +\sum_j\log\frac{p_j(X_j)}{q_j(X_j)}.
\]
Expectation under the true joint law gives the batch's total correlation
plus its row-KL errors.  The adaptive chain rule then adds these quantities
along the true-joint reference path.

\begin{lemma}[Joint-reference chain rule and KL decomposition
\textcolor{red}{(cf.\ {\citealp[Eq.~(7) and Appendix~A.1]{benhamu2025accelerated}})}]
\label{app:lem:adaptive-joint-kl}
For every complete admissible decision rule and every fixed seed $w$, the output
law under $\JointReferenceLaw{w}$ equals $\TargetLaw$, even if its batches
are not structurally safe.  Moreover,
\begin{equation}
\begin{aligned}
 &\KLDivergence{\TargetLaw}{\SamplerOutputLaw{w}{\FrozenOracle}}\\
 &\quad=\mathbb E_{\JointReferenceLaw{w}}\sum_t
 \left[
  \mathrm{TC}(\JointBatchLaw{\mathfrak h_t})
  +\sum_{j\in B(\mathfrak h_t)}
    \KLDivergence{\ExactRow{y(\mathfrak h_t)}{j}}{\OracleRow{y(\mathfrak h_t)}{j}}
 \right].
\end{aligned}
 \label{app:eq:adaptive-joint-kl}
\end{equation}
Here $t$ runs over the nonempty commits on the reference path.  Under
\eqref{app:eq:kl-row-condition}, the sum of row-KL terms is at most
$N\KLRowBudget$ on every such path.  The sum of total-correlation terms
is at most $N\log\VocabSize$ on every path and is zero whenever all batches
are structurally safe for the target forest.
\end{lemma}

\begin{proof}
{\color{red}

\emph{1. Identify the reference output law along each adaptive path.}
With $w$ and the frozen map fixed, all operations between commits are
deterministic functions of the observed past.  Every full assignment $x$
therefore determines a unique execution path.  For a reachable node
$\mathfrak h$, the full assignments that visit it form the cylinder
\[
 \{u\in\Vocab^N:\mathfrak h\text{ lies on the fixed-seed path of }u\}
 =
 \{u\in\Vocab^N:u_{G(\mathfrak h)}=x_{G(\mathfrak h)}\}.
\]
For the forward inclusion, visiting the node requires its committed
values.  For the reverse inclusion, induct along the path to
$\mathfrak h$: every earlier committed set is a subset of
$G(\mathfrak h)$, so these assignments produce the same earlier commit
values, states, frozen-oracle replies, and next actions.
Counterfactual replies are functions of the submitted states, not of the
uncommitted realization of $u$.  Thus the transcript imposes no further
restriction on that realization.

For the ordered partition $(B_t(x))_t$ chosen along this path, the adaptive
chain rule gives $\prod_t \TargetLaw\!\left(
  X_{B_t(x)}=x_{B_t(x)}\mid
  X_{G_{t-1}(x)}=x_{G_{t-1}(x)}\right)=\TargetLaw(x)$.
The product telescopes through successive committed sets exactly as in
\eqref{app:eq:adaptive-exact-chain-rule}.
This is the reference output mass, so the reference output has law
$\TargetLaw$.  The cylinder identity then gives
$\JointReferenceLaw{w}(\text{reach }\mathfrak h)
 =\TargetLaw(X_{G(\mathfrak h)}=x_{G(\mathfrak h)})>0$ and the node
conditional
$\JointReferenceLaw{w}(X_{B(\mathfrak h)}=z\mid\text{reach }\mathfrak h)
 =\TargetLaw(X_{B(\mathfrak h)}=z\mid
                     X_{G(\mathfrak h)}=x_{G(\mathfrak h)})
 =p(z\mid\mathfrak h)$,
where $p(z\mid\mathfrak h)$ is the mass of
$\JointBatchLaw{\mathfrak h}$.
The actual output mass along the same path is
$\SamplerOutputLaw{w}{\FrozenOracle}(x)
 =\prod_t\prod_{j\in B_t(x)}
  \OracleRowMass{y(\mathfrak h_t)}{j}{x_j}$.

\emph{2. Group the pathwise log ratios by reached nodes.}
Strict positivity makes both masses positive.  Since the reference output
has law $\TargetLaw$ and its conditional at each reached node $\mathfrak h$
is $p(\cdot\mid\mathfrak h)$, grouping paths by these nodes gives
\begin{align*}
 &\KLDivergence{\TargetLaw}{\SamplerOutputLaw{w}{\FrozenOracle}}\\
 &=\sum_x\TargetLaw(x)
   \sum_{\mathfrak h\text{ on the path of }x}
   \log\frac{p(x_{B(\mathfrak h)}\mid\mathfrak h)}
             {\prod_{j\in B(\mathfrak h)}
               \OracleRowMass{y(\mathfrak h)}{j}{x_j}}\\
 &=\sum_{\mathfrak h}
   \Pr_{\JointReferenceLaw{w}}\{\text{reach }\mathfrak h\}
   \sum_z p(z\mid\mathfrak h)
   \log\frac{p(z\mid\mathfrak h)}
             {\prod_{j\in B(\mathfrak h)}
               \OracleRowMass{y(\mathfrak h)}{j}{z_j}}\\
 &=\mathbb E_{\JointReferenceLaw{w}}\sum_t
   \KLDivergence{\JointBatchLaw{\mathfrak h_t}}
     {\bigotimes_{j\in B(\mathfrak h_t)}\OracleRow{y(\mathfrak h_t)}{j}}.
\end{align*}
Here $\mathfrak h$ ranges over nonterminal nodes of the fixed-seed execution
tree.  Every path has at most $N$ nonempty commits and each batch has finitely
many outcomes, so the sums are finite.  Each node contributes exactly when
the path visits it; no common path length is required.

\emph{3. Decompose the batch KL.}
For a positive batch law $p$ with singleton marginals $p_j$, recall
$\mathrm{TC}(p)
 =\mathbb E_p\log\frac{p(X_B)}{\prod_jp_j(X_j)}$.
Inserting the product of these marginals into the likelihood ratio gives
$\KLDivergence{p}{\bigotimes_jq_j}
 =\mathrm{TC}(p)
  +\sum_j\mathbb E_p\log\frac{p_j(X_j)}{q_j(X_j)}$.
For each $j$, expanding the expectation gives
$\mathbb E_p\log\frac{p_j(X_j)}{q_j(X_j)}
 =\sum_{a\in\Vocab}\bigl(\sum_{x_B:x_j=a}p(x_B)\bigr)
     \log\frac{p_j(a)}{q_j(a)}$.
Since $\sum_{x_B:x_j=a}p(x_B)=p_j(a)$, this equals
$\sum_{a\in\Vocab}p_j(a)\log\frac{p_j(a)}{q_j(a)}
 =\KLDivergence{p_j}{q_j}$.
Thus $\KLDivergence{p}{\bigotimes_jq_j}
 =\mathrm{TC}(p)+\sum_j\KLDivergence{p_j}{q_j}$, which proves
\eqref{app:eq:adaptive-joint-kl} after substitution into Step~2.

\emph{4. Bound the pathwise sums.}
Uniform row accuracy and
$\sum_t|B(\mathfrak h_t)|=N$ give the $N\KLRowBudget$ bound.
\hypertarget{app-finite-entropy}{}
For a finite probability mass $p$, write
$H(p):=-\sum_zp(z)\log p(z)$, with $0\log0:=0$.
For a batch law on $\Vocab^B$ with marginals $p_j$, expanding gives
$\mathrm{TC}(p)
 =\sum_zp(z)\log p(z)
   -\sum_{j\in B}\sum_{a\in\Vocab}
      \bigl(\sum_{z:z_j=a}p(z)\bigr)\log p_j(a)$.
Substituting $\sum_{z:z_j=a}p(z)=p_j(a)$ yields
$\mathrm{TC}(p)=\sum_{j\in B}H(p_j)-H(p)$.
Here $H(p)\ge0$ because $p(z)\le1$, and
$H(p_j)\le\log\VocabSize$ follows from
$\KLDivergence{p_j}{\operatorname{Unif}(\Vocab)}
 =\log\VocabSize-H(p_j)\ge0$.
Apply this calculation with $p=\JointBatchLaw{\mathfrak h}$ and
$p_j=\ExactRow{y(\mathfrak h)}{j}$ to obtain
$0\le\mathrm{TC}(\JointBatchLaw{\mathfrak h})
 \le|B(\mathfrak h)|\log\VocabSize$.
Summing over the disjoint batches gives $N\log\VocabSize$.
Finally, a structurally safe batch
contains at most one vertex per true residual component (or is a singleton).
Lemma~\ref{app:lem:residual-markov} makes its true conditional a product, so
its total correlation is zero.
}
\end{proof}

\subsection{Finite bound and the epsilon calibration}
\label{app:kl-finite}

As in the finite \HellingerCase\ bound of \cref{app:thm:finite-upper},
preprocessing uses the standard floor;
feasibility and screen separation remain explicit hypotheses.
With the accuracy calibration in \cref{app:cor:kl-epsilon}, the theorem
below yields \KLCase\ of \cref{main:thm:upper-curve} through
\cref{app:cor:fixed-accuracy-main}.

\paragraph{Intuition.}
Every path commits $N$ coordinates, so the row-error sum is at most
$N\KLRowBudget$.  Successful reference paths use only safe batches and
have zero batch total correlation.  On all other paths the total
correlation is at most $N\log\VocabSize$.  Thus
\[
 \hbox{expected output KL}
 \le N\KLRowBudget
     +N\log\VocabSize\,
        \Pr_{\text{joint reference and seed}}(\ScreenFailureEvent).
\]
The fresh-color argument bounds this last probability by
$\ScreenFailureBudget$ under the reference process as well.

\begin{theorem}[\KLCaseTitle: finite randomized forward-KL upper bound]
\label{app:thm:finite-kl-upper}
We work in \KLCase, with the oracle class and output risk of
\cref{app:def:kl-oracle-risk}.
Fix $N\ge10$, $\VocabSize\ge2$, and the target, response, and algorithmic
setup of Appendices~\ref{app:setup} and~\ref{app:algorithm}.  Let
$\OracleRadius\in[0,1]$ and $0\le\KLRowBudget\le(\OracleRadius)^2/N$.
Use Algorithm~\ref{app:alg:complete-sampler} with the standard zero-query
draft, the public radius $\OracleRadius$, and
$\ScreenFailureBudget\in(0,1)$.  Suppose preprocessing is feasible and
$\ScreenResolution<\EdgeSignal$.  For every
$\FrozenOracle\in\KLOracleClass(\TargetLaw;\KLRowBudget)$,
\begin{equation}
 \KLSeedRisk(\RandomizedPackedSampler;\TargetLaw,\FrozenOracle)
 \le N\KLRowBudget+N\log\VocabSize\,\ScreenFailureBudget.
 \label{app:eq:finite-kl-risk}
\end{equation}
On every path, including failed-screen paths, the resources satisfy
\begin{align}
 \CounterfactualQueries
 &\le\ScreenColumnCount\ScreenCallCap\RandomColoringCount
       \RandomColorCount(\RandomColorCount+\ChunkCount),\notag\\
 \CommitRounds&\le\HardRoundCap,
 \qquad
 \OracleDepth\le1+\CommitRounds+\ScreenCallCap.
 \label{app:eq:finite-kl-resources}
\end{align}
The parameters on the right are the original finite parameters in
\eqref{app:eq:screen-column-count}, \eqref{app:eq:random-color-parameters},
and \eqref{app:eq:screen-call-and-round-caps}.  Charging the initial all-mask
state gives $\ReportedQueries=\CounterfactualQueries+1$ by
\eqref{app:eq:reported-query-count}, so its submission bound is the
displayed bound for $\CounterfactualQueries$ plus one.
If feasibility and separation hold
for every oracle in this KL class, the supremum of the risk in
\eqref{app:eq:finite-kl-risk} over that class has the same bound.
\end{theorem}

\begin{proof}

\emph{1. Transfer the oracle and resource guarantees.}
\textcolor{red}{By \cref{app:lem:kl-hellinger} and
\eqref{app:eq:kl-row-condition}, every valid $(y,j)$ satisfies}
{\color{red}
\[
 \sqHellinger(\ExactRow{y}{j},\OracleRow{y}{j})
 \le\tfrac12\KLDivergence{\ExactRow{y}{j}}{\OracleRow{y}{j}}
 \le\tfrac12\KLRowBudget
 \le\frac{(\OracleRadius)^2}{2N}.
\]
}
\textcolor{red}{Thus the oracle meets \cref{app:ass:oracle-accuracy} (A2),
as recorded in \eqref{app:eq:kl-oracle-inclusion}, and
\cref{app:lem:upper-resources} transfers the preprocessing, screen, and
resource guarantees.  The output-KL bound instead uses the preceding
\cref{app:lem:adaptive-joint-kl}.}

\emph{2. Bound the KL cost at a fixed seed.}
\textcolor{red}{Using the call notation of
\cref{app:lem:random-screen-success} on the joint-reference execution,
define the reached-screen failure event by}
{\color{red}
\[
 \ScreenFailureEvent
 :=\bigcup_{s=1}^{\ScreenCallCap}
   \{R_s\text{ occurs and }\ScreenSuccess{H^{(s)}}\text{ fails}\}.
\]
}
On $\ScreenFailureEvent^{\mathsf c}$, the deterministic successful-path
certificates in Lemmas~\ref{app:lem:successful-path-structure}
and~\ref{app:lem:upper-resources} imply that neither cycle repair nor
\AlgoGuardRef\ is activated, and every commit is a singleton peel in
\AlgoBlockRef{peel} or a safe
centroid batch in \AlgoBlockRef{centroid}.  These statements hold
for every sequence of realized values, so they also hold for reference
commits.  Thus the total-correlation cost vanishes on
$\ScreenFailureEvent^{\mathsf c}$, while the pathwise bound from
Lemma~\ref{app:lem:adaptive-joint-kl} applies on the failure event:
\[
 \sum_t\mathrm{TC}(\JointBatchLaw{\mathfrak h_t})
 =\mathbf 1_{\ScreenFailureEvent}
   \sum_t\mathrm{TC}(\JointBatchLaw{\mathfrak h_t})
 \le N\log\VocabSize\,\mathbf 1_{\ScreenFailureEvent}.
\]
Repair batches are singletons, so each has zero total correlation.
Their coordinates still enter the row-KL sum only once, as for every
other commit; \textcolor{red}{on every path,}
{\color{red}
\[
 \sum_t\sum_{j\in B(\mathfrak h_t)}
 \KLDivergence{\ExactRow{y(\mathfrak h_t)}{j}}
              {\OracleRow{y(\mathfrak h_t)}{j}}
 \le\KLRowBudget\sum_t|B(\mathfrak h_t)|
 =N\KLRowBudget.
\]
}
Taking the reference expectation replaces the indicator by
$\JointReferenceLaw{w}(\ScreenFailureEvent)$.  Adding the pathwise row-KL
budget $N\KLRowBudget$ in Lemma~\ref{app:lem:adaptive-joint-kl} therefore
gives, for each fixed $w$,
\begin{equation}
 \KLDivergence{\TargetLaw}{\SamplerOutputLaw{w}{\FrozenOracle}}
 \le N\KLRowBudget
   +N\log\VocabSize\,\JointReferenceLaw{w}(\ScreenFailureEvent).
 \label{app:eq:kl-fixed-seed-failure}
\end{equation}

\emph{3. Average the screen failure under the joint-reference law.}
It remains to average the failure term over the decision-rule seed.  Generate
reference commit values with a fresh primitive random source independent of
the color blocks and draft randomness.  Just before each reached color
draw in \AlgoBlockRef{draft}, condition on the reference history and
completed draft generation,
but not on the current or future colors.  The residual forest and draft are
then fixed and the new colors remain independent and uniform.  The proof of
Lemma~\ref{app:lem:random-screen-success} uses only these facts, not the law
of preceding commit values: for each ordered pair, at most $2\DegreeCutoff$
color equalities are forbidden, each with probability
$1/\RandomColorCount$, and independent repetition gives the failure bound
$\exp\{-\RandomColoringCount/8\}$.  Union bounds over fewer than $N^2$
pairs give the same conditional bound at a reached reference call.
To pass from that conditional estimate to the seed average, define the
joint probability of an event $A$ by
\[
 \Pr_{\mathrm{joint}}(A):=
       \mathbb E_W\JointReferenceLaw{W}(A).
\]
Form $\mathcal T_{s-},f^{(s)},\mathcal F_{s-1},R_s,E_s$ exactly as in
\cref{app:lem:random-screen-success}, now from the reference execution.
In particular, $\mathcal F_{s-1}=\sigma(\mathcal T_{s-},f^{(s)})$
excludes the current and future colors, and
$\ScreenFailureEvent=\bigcup_{s=1}^{\ScreenCallCap}E_s$.  Then
\[
 \Pr_{\mathrm{joint}}(E_s\mid\mathcal F_{s-1})
 \le\mathbf 1_{R_s}\frac{\ScreenFailureBudget}{\ScreenCallCap}.
\]
The union bound and conditional expectation give
\begin{align*}
 \mathbb E_W\JointReferenceLaw{W}(\ScreenFailureEvent)
 &=\Pr_{\mathrm{joint}}\!\left(
       \bigcup_{s=1}^{\ScreenCallCap}E_s\right)\\
 &\le\sum_{s=1}^{\ScreenCallCap}
       \mathbb E_{\mathrm{joint}}\!
        \left[\Pr_{\mathrm{joint}}(E_s\mid\mathcal F_{s-1})\right]\\
 &\le\frac{\ScreenFailureBudget}{\ScreenCallCap}
       \sum_{s=1}^{\ScreenCallCap}\Pr_{\mathrm{joint}}(R_s)
 \le\ScreenFailureBudget.
\end{align*}
\textcolor{red}{This sums conditional failure probabilities over adaptively
reached calls; it does not require the calls to be independent.}
We have therefore proved
\begin{equation}
 \mathbb E_{\ControllerSeed}
   \JointReferenceLaw{\ControllerSeed}(\ScreenFailureEvent)
 \le\ScreenFailureBudget.
 \label{app:eq:joint-reference-screen-failure}
\end{equation}
This is a joint probability over fresh colors and reference commit draws;
it is \emph{not} a failure-probability bound conditional on the full seed
$w$.  Averaging \eqref{app:eq:kl-fixed-seed-failure} proves
\eqref{app:eq:finite-kl-risk}.  In particular, failed-screen paths are charged
through their finite total-correlation cost, not by attempting to convert a
TV coupling bound into KL.  No minimum atom size or bounded likelihood ratio
is assumed beyond the existing strict positivity.
\end{proof}

\begin{corollary}[\KLCaseTitle\ with output KL at most epsilon]
\label{app:cor:kl-epsilon}
Fix $0<\KLTargetTolerance\le1$ and use the preceding finite target and
algorithmic setup.  Choose
\begin{equation}
 \KLRowBudget:=\frac{\KLTargetTolerance}{2N},
 \qquad
 \ScreenFailureBudget:=
       \frac{\KLTargetTolerance}{2N\log\VocabSize}.
 \label{app:eq:kl-epsilon-calibration}
\end{equation}
Run the existing algorithm with its public radius set to $\KLOracleRadius$.
If its preprocessing is feasible and $\ScreenResolution<\EdgeSignal$ with
this radius, then for every
$\FrozenOracle\in\KLOracleClass(\TargetLaw;\KLRowBudget)$,
\begin{equation}
 \KLSeedRisk(\RandomizedPackedSampler;\TargetLaw,\FrozenOracle)
 \le\KLTargetTolerance,
 \label{app:eq:kl-epsilon-output}
\end{equation}
with the pathwise bounds in \eqref{app:eq:finite-kl-resources}.  Thus the
per-row budget scales as $\KLTargetTolerance/N$, not
$(\KLTargetTolerance)^2/N$.
\end{corollary}

\begin{proof}
The auxiliary radius is in $[0,1]$,
$\KLRowBudget=(\KLOracleRadius)^2/N$, and the displayed failure budget is
in $(0,1)$ for $N\ge10$, $\VocabSize\ge2$.  Each term on the right of
\eqref{app:eq:finite-kl-risk} is $\KLTargetTolerance/2$.
\end{proof}

The auxiliary radius $\KLOracleRadius$ is an input calibration for \KLCase,
not a redefinition of the \HellingerCase\ parameter $\OracleRadius$.

\begin{proposalcontext}{Additional metric consequences}
Pinsker's
inequality and Jensen's inequality imply from \eqref{app:eq:kl-epsilon-output}
that $\mathbb E_W\dTV(\TargetLaw,\SamplerOutputLaw{W}{\FrozenOracle})^2
\le\KLTargetTolerance/2$ and the mean TV is at most
$\sqrt{\KLTargetTolerance/2}$.  This explains the KL versus squared-TV
scaling convention in the comparison; it is not a two-sided equivalence of
the metrics.  Convexity also bounds the forward KL of the seed-marginalized
output by \eqref{app:eq:kl-seed-risk}; the latter, stronger objective is the
one proved here.
\end{proposalcontext}

\ifincludeexponents
The independent-signal KL rate and its direct-parameter form are retained
in Appendix~\ref{app:independent-signal}.
\fi

\section{From finite bounds to the main theorems}
\label{app:benchmark-specializations}

We derive the main-text Theorems~\ref{main:thm:upper-curve}
and~\ref{main:thm:lower-envelope} directly from the finite results.
Under the main-text scaling in \textcolor{red}{\cref{main:ass:main-scaling}
(Scaling)}, the row-TV
allowance $\TargetTolerance/(2\sqrt N)$ and vocabulary floor
$\VocabSize^{-1}$ are $o(\EdgeSignal)$.
\Cref{app:lem:benchmark-feasibility} verifies the resulting finite
upper-bound design conditions.
Next, \cref{app:cor:signal-calibrated-finite-lower} calibrates the finite
lower bound with an independent Hellinger radius and checks the matching
witness's class inclusion at the actual finite public values.
With both ingredients in place, \cref{app:cor:fixed-accuracy-main} derives
both main theorems with accuracy fixed before $N\to\infty$, using
$\OracleRadius=\TargetTolerance/2$ in \HellingerCase\ and
$\KLTargetTolerance=\TargetTolerance^2/2$ in \KLCase.
\ifincludeexponents
The subsequent corollaries additionally give joint-limit rates.
\fi

\subsection{A finite design condition}

\paragraph{Intuition.}
With response constant and exponent equal to one, choosing
$\TailTolerance=\EdgeSignal/2$ leaves the other half of the edge signal
for row noise.  If both $x/\sqrt N$ and $\VocabSize^{-1}$ are sufficiently
below $\EdgeSignal$, the threshold grid has a feasible point with
$\TailThreshold\simeq\EdgeSignal$.  The bank bound then gives a column
upper bound of order $\EdgeSignal^{-1/\RankTailExponent}$.
The following finite inequalities specify the margins and the grid rounding.
Here $x$ denotes the public radius supplied to preprocessing:
use $x=\OracleRadius$ under (A2), or the Hellinger radius obtained from
the KL inclusion \eqref{app:eq:kl-oracle-inclusion}.  It is not a separate
output-error tolerance.

\begin{lemma}[Public feasibility with a half-signal tail tolerance]
\label{app:lem:benchmark-feasibility}
Let $N\ge10$, $\VocabSize\ge2$, $\RankTailExponent>1$, and
$\ResponseExponent=\ResponseConstant=\FrequencyConstant=1$.
Use a public preprocessing radius $x\in(0,1]$, so the marginal and screen
TV error bounds used by the construction are $x/\sqrt N$.
Suppose
\begin{equation}
 16\max\{x/\sqrt N,\VocabSize^{-1}\}
 \le\EdgeSignal\le\frac12.
 \label{app:eq:benchmark-finite-margin}
\end{equation}
Set $\TailTolerance=\EdgeSignal/2$ and use the standard threshold floor
\eqref{app:eq:rate-threshold-floor}.  Then preprocessing is feasible,
the selected threshold satisfies
\begin{equation}
 \frac{\EdgeSignal}{4}\le\TailThreshold\le\frac{\EdgeSignal}{2},
 \qquad
 4x/\sqrt N+\TailTolerance\le\frac34\EdgeSignal<\EdgeSignal,
 \label{app:eq:benchmark-threshold-and-separation}
\end{equation}
and, under (RF\textcolor{red}{; \cref{app:eq:RF}}), the number of source columns obeys
\begin{equation}
 \ScreenColumnCount
 \le \left(\frac{16}{\EdgeSignal}\right)^{1/\RankTailExponent}+1.
 \label{app:eq:benchmark-column-cap}
\end{equation}
\end{lemma}

\begin{proof}

\emph{1. Exhibit a feasible grid point.}
Under the stated constants, the public floor is exactly
\[
 \ThresholdFloor
 =4\max\{x/\sqrt N,\VocabSize^{-1}\}\le\EdgeSignal/4.
\]
Specifically, choose
\[
 t'=2^{-\lceil\log_2(2/\EdgeSignal)\rceil},
 \qquad \EdgeSignal/4<t'\le\EdgeSignal/2.
\]
The bracket follows by applying $a\le\lceil a\rceil<a+1$ to
$a=\log_2(2/\EdgeSignal)$.  Since $t'>\ThresholdFloor$ and
$\lceil\log_2(2/\EdgeSignal)\rceil
 \le\lceil\log_2(1/\ThresholdFloor)\rceil=\ThresholdGridDepth$,
this is an unclipped point of the grid in \eqref{app:eq:threshold-grid}.
Every such point is feasible because its response bound is the threshold
itself, at most $\TailTolerance=\EdgeSignal/2<1$.
The maximal feasible grid point lies in the same interval.

\emph{2. Verify separation and bound the bank.}
The screen error bound in \eqref{app:eq:benchmark-finite-margin} gives
$4x/\sqrt N\le\EdgeSignal/4$, proving strict separation.
Lemma~\ref{app:lem:preprocessing-certificates} now bounds every bank by
$(4/\TailThreshold)^{1/\RankTailExponent}
\le(16/\EdgeSignal)^{1/\RankTailExponent}$.
Adding the tail column proves \eqref{app:eq:benchmark-column-cap}; if every
bank is empty the column count is zero and the same bound holds.
\end{proof}

\subsection{Finite lower-bound calibration for the main text}

The next corollary supplies the finite lower bound used in
\cref{main:thm:lower-envelope}. It places the matching witness in the
normalized Hellinger/TV model of \cref{app:setup-oracle}, then
calibrates the signal to the independently specified radius $\OracleRadius$.
The main-text class takes $\OracleRadius=\TargetTolerance/2$.
The fixed-accuracy limit, including removal of the floor and the finite
round minimum, is carried out in \cref{app:cor:fixed-accuracy-main}.

\begin{corollary}[Signal-calibrated finite query--round lower bound]
\label{app:cor:signal-calibrated-finite-lower}
Let $N\ge4$ be even, let $\VocabSize\ge2$ be an integer, and fix
$\RankTailExponent>1$, $0<\OracleRadius\le1$, and
$\ResponseExponent=\ResponseConstant=\FrequencyConstant=1$.
Suppose
\begin{equation}
 0<\EdgeSignal\le\frac12,\qquad
 \VocabSize\EdgeSignal>1,\qquad
 \EdgeSignal^3\le\frac{(\OracleRadius)^2}{2N}.
 \label{app:eq:signal-calibrated-finite-domain}
\end{equation}
Fix $0<\TargetTolerance\le1/8$.  Every admissible algorithm with integer
worst-case pathwise budgets $\QueryBudget\ge0$, $\RoundBudget\ge1$ and
seed-averaged output-TV error at most $\TargetTolerance$ uniformly over the target--oracle
tuples satisfying the structural and response assumptions of
Appendix~\ref{app:setup}, the finite envelope \eqref{app:eq:RF}, and
\textcolor{red}{\cref{app:ass:oracle-accuracy} (A2)},
with these public parameters satisfies
\begin{equation}
 \QueryBudget>\frac{\lfloor\EdgeSignal^{-1/\RankTailExponent}\rfloor}{8}
 \quad\text{or}\quad
 \RoundBudget\ge
 \min\left\{\frac N{16384},
            \frac{N\EdgeSignal^2}{786432\TargetTolerance}\right\}.
 \label{app:eq:benchmark-finite-lower}
\end{equation}
For the benchmark choice $\EdgeSignal=\HellingerBenchmarkSignal$,
the conditions in \eqref{app:eq:signal-calibrated-finite-domain} reduce to
$\VocabSize>2(N/(\OracleRadius)^2)^{1/3}$, and the alternative becomes
\begin{equation}
 \begin{aligned}
 \QueryBudget
 &>\frac18\left\lfloor
   2^{1/\RankTailExponent}
   \left(\frac N{(\OracleRadius)^2}\right)^{1/(3\RankTailExponent)}
   \right\rfloor
 \quad\text{or}\\
 \RoundBudget
 &\ge\min\left\{\frac N{16384},
   \frac{N^{1/3}(\OracleRadius)^{4/3}}{3145728\TargetTolerance}\right\}.
 \end{aligned}
 \label{app:eq:signal-calibrated-direct-lower}
\end{equation}
\end{corollary}

\begin{proof}

\emph{1. Construct the witness using the same public parameters.}
Keep the public vocabulary of size \(\VocabSize\), its order,
\(\RankTailExponent\), \(\EdgeSignal\), and \(\OracleRadius\) from the
statement, with
\(\ResponseExponent=\ResponseConstant=\FrequencyConstant=1\).
In the matching family of \cref{app:def:hard-matching-family}, choose
\begin{equation}
 \EdgeCoupling=\TriggerMass=\EdgeSignal,\qquad
 \HardBankSize=\lfloor\EdgeSignal^{-1/\RankTailExponent}\rfloor.
 \label{app:eq:benchmark-hard-witness}
\end{equation}
The inequalities
\[
 \begin{aligned}
 1\le\HardBankSize
 &\le\EdgeSignal^{-1/\RankTailExponent}<\EdgeSignal^{-1}<\VocabSize,\\
 \HardBankSize\EdgeSignal
 &\le\EdgeSignal^{1-1/\RankTailExponent}<1,\qquad
 \EdgeSignal\le\HardBankSize^{-\RankTailExponent}
 \end{aligned}
\]
give \(1\le\HardBankSize<\HardVocabSize\) and
\(\HardBankSize\TriggerMass<1\). Together with
\(0<\EdgeCoupling=\TriggerMass=\EdgeSignal\le1/2\), these verify
\eqref{app:eq:hard-basic-conditions}.
For the last rank inequality, explicitly,
\(\HardBankSize\le\EdgeSignal^{-1/\RankTailExponent}\) implies
\(\HardBankSize^{\RankTailExponent}\le\EdgeSignal^{-1}\), hence
\(\TriggerMass=\EdgeSignal\le\HardBankSize^{-\RankTailExponent}\).
The remaining hypotheses of \cref{app:prop:hard-family-inclusion} are
\[
 \HardVocabSize\TriggerMass=\VocabSize\EdgeSignal>1,
 \qquad
 \EdgeSignal=\EdgeCoupling,
 \qquad
 \EdgeCoupling^2\TriggerMass
 =\EdgeSignal^3\le\frac{(\OracleRadius)^2}{2N}.
\]
Thus that proposition applies at the actual finite public values.
\textcolor{red}{It verifies \cref{main:ass:forest-structure} and the
RT, UEN, and RF items of \cref{main:ass:regularity-scaling}},
and the Hellinger condition (A2) in Assumption~\ref{app:ass:oracle-accuracy}.
At $\OracleRadius=\TargetTolerance/2$, the latter is precisely
Assumption~\ref{main:ass:oracle-accuracy}(i).
\textcolor{red}{The vocabulary-growth requirement is additionally supplied
when taking the public sequence in \cref{main:ass:main-scaling} (Scaling).}

\emph{2. Restrict the uniform guarantee and apply the finite tradeoff.}
Fix any admissible algorithm satisfying this corollary's risk and
pathwise-budget hypotheses. Every target--oracle pair in the chosen
matching family belongs to the stipulated finite class, so the same
algorithm satisfies
\[
 \sup_{I\in\HardFamily[\VocabSize,\HardBankSize,\EdgeSignal,\EdgeSignal]}
 \SeedRisk(\mathcal A;\HardTargetLaw{I},\HardOracle{I})
 \le\TargetTolerance.
\]
Its deterministic budgets \(\QueryBudget,\RoundBudget\) still hold on
every instance, seed, transcript, and path in this subfamily.
Apply \cref{app:cor:nonlinear-finite-target-tradeoff} at tolerance
\(\TargetTolerance\), with \(\EdgeCoupling=\EdgeSignal\);
the required tolerance range is precisely
\(0<\TargetTolerance\le1/8\).
If \(\QueryBudget>\HardBankSize/8\), the query alternative in
\eqref{app:eq:benchmark-finite-lower} already holds.
Otherwise that corollary implies
\textcolor{red}{either $\RoundBudget>N/16384$ or
$\RoundBudget\ge N\EdgeSignal^2/(786432\TargetTolerance)$; either
alternative implies $\RoundBudget\ge
\min\{N/16384,N\EdgeSignal^2/(786432\TargetTolerance)\}$.}
Substituting
\(\HardBankSize=\lfloor\EdgeSignal^{-1/\RankTailExponent}\rfloor\)
proves \eqref{app:eq:benchmark-finite-lower}.

\Needspace{6\baselineskip}
\emph{3. Substitute the signal calibration.}
\textcolor{red}{Use the benchmark signal
$\EdgeSignal=\HellingerBenchmarkSignal
=\tfrac12((\OracleRadius)^2/N)^{1/3}$.  It places the matching witness
inside the public oracle-accuracy class while retaining the
$\EdgeSignal^{-1/\RankTailExponent}$ query threshold.}
This is the signal choice in \textcolor{red}{\cref{main:ass:main-scaling}
(Scaling)} at
\(\OracleRadius=\TargetTolerance/2\).
Since \(0<\OracleRadius\le1\) and \(N\ge4\), this gives
\(0<\EdgeSignal\le1/2\), and
\[
 \EdgeSignal^3
 =\frac18\,\frac{(\OracleRadius)^2}{N}
 =\frac{(\OracleRadius)^2}{8N}
 \le\frac{(\OracleRadius)^2}{2N}.
\]
The only remaining condition in
\eqref{app:eq:signal-calibrated-finite-domain} is the vocabulary condition:
\[
 \VocabSize\EdgeSignal>1
 \quad\Longleftrightarrow\quad
 \frac{\VocabSize}{2}
      \left(\frac{(\OracleRadius)^2}{N}\right)^{1/3}>1
 \quad\Longleftrightarrow\quad
 \VocabSize>2\left(\frac N{(\OracleRadius)^2}\right)^{1/3}.
\]
The two resource thresholds become
\begin{align*}
 \EdgeSignal^{-1/\RankTailExponent}
 &=\left[\frac12
        \left(\frac{(\OracleRadius)^2}{N}\right)^{1/3}
   \right]^{-1/\RankTailExponent}
 =2^{1/\RankTailExponent}
       \left(\frac N{(\OracleRadius)^2}\right)^{1/(3\RankTailExponent)},\\
 \frac{N\EdgeSignal^2}{786432\TargetTolerance}
 &=\frac{N}{786432\TargetTolerance}\,
       \frac14\left(\frac{(\OracleRadius)^2}{N}\right)^{2/3}
 =\frac{N^{1/3}(\OracleRadius)^{4/3}}{3145728\TargetTolerance}.
\end{align*}
Substitution into \eqref{app:eq:benchmark-finite-lower} proves
\eqref{app:eq:signal-calibrated-direct-lower}, with its floor and minimum
still intact. The lower-bound part of
\cref{app:cor:fixed-accuracy-main} completes the passage from this finite
statement to \cref{main:thm:lower-envelope}.
\end{proof}

\subsection{Fixed-accuracy specialization for the main text}
\label{app:fixed-accuracy-specialization}

The main text fixes its accuracy parameter before taking $N\to\infty$.
We first retain independent oracle and output budgets, then specialize
them at the end of the proof. The cutoff $\DegreeCutoff$ remains a direct
parameter, and the logarithmic overhead is displayed explicitly.

The three main claims use the following proof routes; the final part of
Corollary~\ref{app:cor:fixed-accuracy-main} performs their common
main-text specialization to the classes in \cref{main:def:shared-rate-class}.
\begin{center}
\begin{tabularx}{\linewidth}{@{}lX@{}}
\toprule
Main claim & Finite result and calibration\\
\midrule
Theorem~\ref{main:thm:upper-curve}, \HellingerCase
 & Theorem~\ref{app:thm:finite-upper} and
   Lemma~\ref{app:lem:benchmark-feasibility}\\
Theorem~\ref{main:thm:upper-curve}, \KLCase
 & Theorem~\ref{app:thm:finite-kl-upper},
   Corollary~\ref{app:cor:kl-epsilon}, and
   Lemma~\ref{app:lem:benchmark-feasibility}\\
Theorem~\ref{main:thm:lower-envelope}
 & Theorem~\ref{app:thm:nonlinear-finite-lower},
   Corollaries~\ref{app:cor:nonlinear-finite-target-tradeoff}
   and~\ref{app:cor:signal-calibrated-finite-lower}\\
\bottomrule
\end{tabularx}
\end{center}

\begin{corollary}[Fixed-accuracy main-text bounds]
\label{app:cor:fixed-accuracy-main}
Fix $\RankTailExponent>1$, $\VocabGrowthExponent>1/3$, and a public
$\VocabSize=N^{\VocabGrowthExponent+o(1)}$.
For simplicity, set $\ResponseExponent=\ResponseConstant=\FrequencyConstant=1$ and
take all target--oracle tuples satisfying the forest, RT--UEN\textcolor{red}{~%
(\cref{app:eq:RT,app:eq:UEN})}, and RF\textcolor{red}{~%
(\cref{app:ass:rate-regime})}
assumptions of Appendix~\ref{app:setup}, with either calibration:
\begin{itemize}
 \item \emph{\HellingerCaseTitle:} fix $0<\OracleRadius\le1$,
 use (A2), and set $\EdgeSignal=\HellingerBenchmarkSignal$.
 \item \emph{\KLCaseTitle:} fix $0<\KLTargetTolerance\le1$,
 impose uniform forward row KL at most $\KLTargetTolerance/(2N)$,
 and set $\EdgeSignal=\KLBenchmarkSignal$.
\end{itemize}
The exponents $\RankTailExponent,\VocabGrowthExponent$ and the chosen
accuracy parameter ($\OracleRadius$ or $\KLTargetTolerance$) are fixed
independently of $N$; the vocabulary size, signal floor, and per-row
oracle-error allowance follow the displayed $N$-dependent calibrations.
Here $\SharedClass$ is the Hellinger class at these public values, and
$\SharedRisk(\mathcal A):=\sup_{I\in\SharedClass}
 \SeedRisk(\mathcal A;\InstanceTargetLaw{I},\InstanceOracle{I})$
abbreviates its worst-case seed-averaged TV risk.
For every public integer choice $9\le\DegreeCutoff\le N-1$ and
$\ChunkCount=\DegreeCutoff$, use the standard draft, the standard threshold
floor, and $\TailTolerance=\EdgeSignal/2$.
Choose $\ScreenFailureBudget=\OracleRadius/2$ in \HellingerCase, or
$\ScreenFailureBudget=\KLTargetTolerance/(2N\log\VocabSize)$ in \KLCase.
For all sufficiently large $N$, uniformly over this cutoff range and
the respective target--oracle class, the construction is feasible and
\begin{equation}
 \CounterfactualQueries
 =O\!\left(\EdgeSignal^{-1/\RankTailExponent}
             \DegreeCutoff^2(\log N)^2\right),
 \qquad
 \CommitRounds,\OracleDepth
 =O\!\left(\frac N{\DegreeCutoff}\log N\right).
 \label{app:eq:fixed-accuracy-resources}
\end{equation}
These bounds hold on every path. The output guarantees are
$\SharedRisk(\RandomizedPackedSampler)\le3\OracleRadius/2$ in
\HellingerCase\ and expected forward KL at most $\KLTargetTolerance$
in \KLCase.

For \HellingerCase, every admissible algorithm with uniform TV risk
at most the fixed $0<\TargetTolerance\le1/8$ satisfies, for all
sufficiently large even $N$,
\begin{equation}
 \QueryBudget=\Omega\!\left(
  (N/(\OracleRadius)^2)^{1/(3\RankTailExponent)}\right)
 \quad\text{or}\quad
 \RoundBudget=\Omega\!\left(
  N^{1/3}(\OracleRadius)^{4/3}/\TargetTolerance\right).
 \label{app:eq:fixed-accuracy-lower}
\end{equation}
Taking $\OracleRadius=\TargetTolerance/2$ in \HellingerCase\ and
$\KLTargetTolerance=\TargetTolerance^2/2$ in \KLCase\ gives
Theorems~\ref{main:thm:lower-envelope}
and~\ref{main:thm:upper-curve}, including their balanced bounds.
\end{corollary}

\begin{proof}
\emph{Upper bound: feasibility for \cref{main:thm:upper-curve}.}
Set $x=\OracleRadius$ in \HellingerCase\ and
$x=\sqrt{\KLTargetTolerance/2}$ in \KLCase.
Both calibrations have
$\EdgeSignal=\frac12(x^2/N)^{1/3}$, so
\[
 \frac{x/\sqrt N}{\EdgeSignal}
 =2x^{1/3}N^{-1/6}\longrightarrow0,
 \qquad
 \frac{\VocabSize^{-1}}{\EdgeSignal}
 =2x^{-2/3}N^{1/3-\VocabGrowthExponent-o(1)}
 \longrightarrow0.
\]
Also $\EdgeSignal\to0$.
Lemma~\ref{app:lem:benchmark-feasibility} therefore supplies preprocessing,
screen separation, and
$\ScreenColumnCount=O_{\RankTailExponent}
(\EdgeSignal^{-1/\RankTailExponent})$ for all sufficiently large $N$.
This step does not depend on $\DegreeCutoff$.

\emph{Phase and coloring counts.}
The finite cap \eqref{app:eq:phase-cap} satisfies
\[
 \PeelPhaseCap
 \le 1+\frac{\log(2N)}{\log(9/8)}=O(\log N),
 \qquad
 \ScreenCallCap=\PeelPhaseCap+1=O(\log N),
\]
uniformly for $9\le\DegreeCutoff\le N-1$.
The chosen failure budgets give respectively
\[
 \RandomColoringCount=
 \begin{cases}
 \left\lceil8\log
   \dfrac{2\ScreenCallCap N^2}{\OracleRadius}\right\rceil,
 &\text{\HellingerCase},\\[5pt]
 \left\lceil8\log
   \dfrac{2\ScreenCallCap N^3\log\VocabSize}{\KLTargetTolerance}
 \right\rceil,
 &\text{\KLCase}.
 \end{cases}
\]
Each is $O(\log N)$: $\log\ScreenCallCap=O(\log\log N)$,
$\log\log\VocabSize=O(\log\log N)$, and the logarithms of the inverse
accuracy parameters are fixed constants.
With $\RandomColorCount=8(\DegreeCutoff+1)\le16\DegreeCutoff$,
\[
 \RandomColorCount(\RandomColorCount+\ChunkCount)
 \le16\DegreeCutoff(17\DegreeCutoff)=272\DegreeCutoff^2.
\]
Substituting these factors into \eqref{app:eq:finite-upper-Q} gives
\begin{align*}
 \CounterfactualQueries
 &\le\ScreenColumnCount\,\ScreenCallCap\,\RandomColoringCount\,
          \RandomColorCount(\RandomColorCount+\ChunkCount)\\
 &\le O_{\RankTailExponent}(\EdgeSignal^{-1/\RankTailExponent})
          \,O(\log N)\,O(\log N)\,O(\DegreeCutoff^2)\\
 &=O\!\left(\EdgeSignal^{-1/\RankTailExponent}
                \DegreeCutoff^2(\log N)^2\right).
\end{align*}
This proves the query bound.  The round cap and the depth inequality give
\[
 \begin{aligned}
 \CommitRounds
 &\le\left\lceil\frac{4N\PeelPhaseCap}{\DegreeCutoff}\right\rceil
       +\lceil\log_2(N+1)\rceil+2\\
 &=O\!\left(\frac N{\DegreeCutoff}\log N+\log N\right)
  =O\!\left(\frac N{\DegreeCutoff}\log N\right),\\
 \OracleDepth&\le1+\CommitRounds+\ScreenCallCap
  =O\!\left(\frac N{\DegreeCutoff}\log N\right).
 \end{aligned}
\]
The last equalities use $N/\DegreeCutoff\ge1$.
All inputs to these caps are public, so the estimates also cover fallback
paths. The finite TV theorem gives
$\OracleRadius+\ScreenFailureBudget=3\OracleRadius/2$.
The finite KL calibration gives
\[
 N\frac{\KLTargetTolerance}{2N}
 +N\log\VocabSize\,
   \frac{\KLTargetTolerance}{2N\log\VocabSize}
 =\KLTargetTolerance.
\]

For either value of $x$ chosen above,
\[
 \EdgeSignal^{-1/\RankTailExponent}
 =\left[\frac12(x^2/N)^{1/3}\right]^{-1/\RankTailExponent}
 =2^{1/\RankTailExponent}
        (N/x^2)^{1/(3\RankTailExponent)}.
\]
Thus suppressing the displayed powers of $\log N$ gives the resource
bound before the final single-accuracy substitution below.
The TV tolerance is met when
$\TargetTolerance\ge3\OracleRadius/2$; the KL guarantee uses the separate
forward-KL hypothesis, not a TV-to-KL conversion.
The constants and the sufficiently-large-$N$ threshold depend only on the
fixed public parameters and sequences, not on the cutoff, target, or
frozen oracle.

\emph{Balancing and sublinearity.}
Put
$\DegreeCutoff=\ChunkCount=
 \lceil(N\EdgeSignal^{1/\RankTailExponent})^{1/3}\rceil$.
Using $\EdgeSignal=\frac12(x^2/N)^{1/3}$, the unrounded value is
\[
 (N\EdgeSignal^{1/\RankTailExponent})^{1/3}
 =2^{-1/(3\RankTailExponent)}
       x^{2/(9\RankTailExponent)}N^{1/3-1/(9\RankTailExponent)}.
\]
Since $\RankTailExponent>1$ and $x>0$ is fixed, this tends to infinity
and is $o(N)$; the cutoff eventually belongs to $\{9,\ldots,N-1\}$.
For all large $N$, the unrounded value is at least one, so
\[
 (N\EdgeSignal^{1/\RankTailExponent})^{1/3}
 \le\DegreeCutoff
 <(N\EdgeSignal^{1/\RankTailExponent})^{1/3}+1
 \le2(N\EdgeSignal^{1/\RankTailExponent})^{1/3}.
\]
Consequently rounding changes only a constant factor, and
\[
 \begin{aligned}
 \EdgeSignal^{-1/\RankTailExponent}\DegreeCutoff^2
 &=\Theta\!\left(
     N^{2/3}\EdgeSignal^{-1/(3\RankTailExponent)}\right),\\
 N/\DegreeCutoff
 &=\Theta\!\left(
     N^{2/3}\EdgeSignal^{-1/(3\RankTailExponent)}\right).
 \end{aligned}
\]
The two calibrations yield
\[
 \EdgeSignal^{-1/(3\RankTailExponent)}
 =
 \begin{cases}
 2^{1/(3\RankTailExponent)}
 (N/(\OracleRadius)^2)^{1/(9\RankTailExponent)},&\text{\HellingerCase},\\
 2^{4/(9\RankTailExponent)}
 (N/\KLTargetTolerance)^{1/(9\RankTailExponent)},&\text{\KLCase}.
 \end{cases}
\]
Thus the balanced resources are
$\widetilde O(N^{2/3+1/(9\RankTailExponent)}x^{-2/(9\RankTailExponent)})$.
After division by $N$, their upper bounds are a fixed constant
times at most
$N^{-1/3+1/(9\RankTailExponent)}(\log N)^2$, which tends to zero.
The additional all-mask submission contributes only $1/N$.

\emph{Lower bound: deduction of \cref{main:thm:lower-envelope}.}
Work in the corollary's independently parameterized \HellingerCase.
The public parameters \(\RankTailExponent>1\),
\(0<\OracleRadius\le1\), and \(0<\TargetTolerance\le1/8\) are fixed
before \(N\to\infty\), and
\(\VocabSize=N^{\VocabGrowthExponent+o(1)}\) with
\(\VocabGrowthExponent>1/3\).
The finite conditions of \cref{app:cor:signal-calibrated-finite-lower}
hold for all sufficiently large even \(N\), because
\begin{align*}
 \EdgeSignal
 &=\frac12(\OracleRadius)^{2/3}N^{-1/3}\longrightarrow0,\\
 \VocabSize\EdgeSignal
 &=\frac12(\OracleRadius)^{2/3}
      N^{\VocabGrowthExponent-1/3+o(1)}\longrightarrow\infty,\\
 \EdgeSignal^3
 &=\frac{(\OracleRadius)^2}{8N}
 \le\frac{(\OracleRadius)^2}{2N}.
\end{align*}
The matching witness in that corollary is therefore a subfamily of the
Hellinger class \(\SharedClass\) at these public values.
For any admissible \(\mathcal A\) with
\(\SharedRisk(\mathcal A)\le\TargetTolerance\) and the stated deterministic
budgets, \eqref{app:eq:signal-calibrated-direct-lower} applies.

To remove the query threshold's floor, use
\[
 u_N:=2^{1/\RankTailExponent}
       (N/(\OracleRadius)^2)^{1/(3\RankTailExponent)}
       \longrightarrow\infty,
 \qquad
 \lfloor u_N\rfloor\ge u_N-1\ge u_N/2
 \quad(u_N\ge2).
\]
If the finite query alternative holds, then
\[
 \QueryBudget>\frac{\lfloor u_N\rfloor}{8}
 \ge\frac{u_N}{16}
 =\frac{2^{1/\RankTailExponent}}{16}
       \left(\frac N{(\OracleRadius)^2}\right)^{1/(3\RankTailExponent)},
\]
which is the query alternative in \eqref{app:eq:fixed-accuracy-lower}.
Otherwise the finite corollary forces its round alternative.
To identify the smaller term in that minimum,
\[
 \frac{N^{1/3}(\OracleRadius)^{4/3}/(3145728\TargetTolerance)}
      {N/16384}
 =\frac{(\OracleRadius)^{4/3}}
        {192\TargetTolerance N^{2/3}}\longrightarrow0.
\]
Because both accuracy parameters are fixed and positive, this ratio is
at most one for all sufficiently large \(N\). Thus the finite round
alternative gives
\[
 \RoundBudget\ge
 \min\left\{\frac N{16384},
   \frac{N^{1/3}(\OracleRadius)^{4/3}}{3145728\TargetTolerance}\right\}
 =\frac{N^{1/3}(\OracleRadius)^{4/3}}{3145728\TargetTolerance}.
\]
Together with the query case, this proves
\eqref{app:eq:fixed-accuracy-lower}.
The query coefficient \(2^{1/\RankTailExponent}/16\) depends only on
\(\RankTailExponent\), and the round coefficient is \(1/3145728\).
The sufficiently-large-\(N\) threshold can depend on the fixed public
parameters and vocabulary sequence, but none of these choices depends on
the hidden instance or the algorithm.

\emph{Single-accuracy substitution for both main theorems.}
Fix $0<\TargetTolerance\le1/8$. Set
$\OracleRadius=\TargetTolerance/2$ in \HellingerCase\ and
$\KLTargetTolerance=\TargetTolerance^2/2$ in \KLCase.
The respective row assumptions become
\[
 \frac{(\TargetTolerance/2)^2}{2N}
 =\frac{\TargetTolerance^2}{8N},
 \qquad
 \frac{\TargetTolerance^2/2}{2N}
 =\frac{\TargetTolerance^2}{4N},
\]
as required in Assumption~\ref{main:ass:oracle-accuracy}.
Both cases use $x=\TargetTolerance/2$, so
\[
 \EdgeSignal=\frac12\left(\frac{(\TargetTolerance/2)^2}{N}\right)^{1/3}
 =\left(\frac{\TargetTolerance^2}{32N}\right)^{1/3},
 \qquad
 \EdgeSignal^{-1/\RankTailExponent}
 =\left(\frac{32N}{\TargetTolerance^2}\right)^{1/(3\RankTailExponent)}.
\]
This is exactly the target-class calibration in
\textcolor{red}{\cref{main:ass:main-scaling} (Scaling)}. The output bounds are
\[
 \SharedRisk(\RandomizedPackedSampler)
 \le\frac{3\OracleRadius}{2}
 =\frac{3\TargetTolerance}{4}\le\TargetTolerance,
 \qquad
 \mathbb E_W\KLDivergence{\TargetLaw}
 {\SamplerOutputLaw{W}{\FrozenOracle}}
 \le\KLTargetTolerance
 =\frac{\TargetTolerance^2}{2}\le\TargetTolerance^2,
\]
in cases (i) and (ii), respectively. Substituting $x=\TargetTolerance/2$
in the resource and balanced bounds proves
\cref{main:thm:upper-curve}.
For the two lower alternatives,
\[
 (N/(\OracleRadius)^2)^{1/(3\RankTailExponent)}
 =2^{2/(3\RankTailExponent)}
   (N/\TargetTolerance^2)^{1/(3\RankTailExponent)},
 \qquad
 \frac{N^{1/3}(\OracleRadius)^{4/3}}{\TargetTolerance}
 =2^{-4/3}N^{1/3}\TargetTolerance^{1/3}.
\]
Absorbing only these fixed factors into the constants proves
\cref{main:thm:lower-envelope}. The independent finite statements
remain available for other oracle and output budgets.
\ifincludeexponents
This fixed-accuracy deduction is separate from the joint-limit
corollaries below.
\fi
\end{proof}

\paragraph{Accuracy dependence in the abstract.}
Fix $0<\TargetTolerance\le1/8$ and take
$\OracleRadius=\TargetTolerance/2$ in \HellingerCase.
The output TV bound is then $3\TargetTolerance/4\le\TargetTolerance$.
The balanced bound \eqref{main:eq:balanced-upper} gives the abstract's
upper form with $C=2/3+1/(9\RankTailExponent)<1$ and
$a=2/(9\RankTailExponent)>0$.
The lower alternatives become
$\QueryBudget=\Omega(N^{1/(3\RankTailExponent)}
\TargetTolerance^{-2/(3\RankTailExponent)})$ or
$\RoundBudget=\Omega(N^{1/3}\TargetTolerance^{1/3})$.
Since $N\ge1$, $0<\TargetTolerance\le1$, and $\RankTailExponent>1$,
\[
 N^{1/(3\RankTailExponent)}\TargetTolerance^{-2/(3\RankTailExponent)}
 \ge N^{1/(3\RankTailExponent)}\TargetTolerance^{1/3},
 \qquad
 N^{1/3}\TargetTolerance^{1/3}
 \ge N^{1/(3\RankTailExponent)}\TargetTolerance^{1/3}.
\]
Thus either counterfactual submissions or commit rounds have the common
lower bound $\Omega(N^c\TargetTolerance^b)$ with
$c=1/(3\RankTailExponent)>0$ and $b=1/3$.
This is a coarser summary, not a matching lower bound for the upper rate.

\ifincludeexponents
\input{src/sections/supplements/benchmark_joint_limits}
\fi

\section{Auxiliary results}
\label{app:auxiliary}

This appendix collects the classical finite facts invoked by the construction
and the upper-bound proof, followed by the supplementary row-selector
invariance result (\cref{app:prop:row-selector-invariance}).
The elementary finite facts are deterministic algebraic or graph
statements and introduce no additional model assumptions.

\subsection{Finite Hellinger-affinity calculus}
\label{app:auxiliary-affinity}

\textcolor{red}{The textbook cited below uses $H^2(p,q)=2\sqHellinger(p,q)$
(\href{https://people.lids.mit.edu/yp/homepage/data/itbook-export.pdf}{author-hosted PDF}).
We include finite proofs for completeness.}

\begin{lemma}[Product affinity and TV comparison; \textcolor{red}{{\citealp[Eqs.~(7.22) and (7.26)]{polyanskiy2025information}}}]
\label{app:lem:finite-hellinger-calculus}
Let $J$ be finite.  For every $j\in J$, let $p_j,q_j$ be probability masses
on a finite set $\mathcal X_j$.  Then
\begin{equation}
 \HellingerAffinity\!\left(\bigotimes_{j\in J}p_j,
                           \bigotimes_{j\in J}q_j\right)
 =\prod_{j\in J}\HellingerAffinity(p_j,q_j),
 \label{app:eq:aux-product-affinity}
\end{equation}
and consequently
\begin{equation}
 \sqHellinger\!\left(\bigotimes_{j\in J}p_j,
                      \bigotimes_{j\in J}q_j\right)
 \le\sum_{j\in J}\sqHellinger(p_j,q_j).
 \label{app:eq:aux-product-hellinger}
\end{equation}
For any two probability masses $p,q$ on the same finite set,
\begin{equation}
 \sqHellinger(p,q)
 \le\dTV(p,q)
 \le\sqrt{2\sqHellinger(p,q)}.
 \label{app:eq:aux-tv-hellinger}
\end{equation}
\end{lemma}

\begin{proof}\color{red}
\emph{1. Products.}
For $f_j(x):=\sqrt{p_j(x)q_j(x)}$, distributivity gives
$\HellingerAffinity(\bigotimes_jp_j,\bigotimes_jq_j)
=\sum_{(x_j)_j\in\prod_j\mathcal X_j}\prod_j f_j(x_j)
=\prod_j(\sum_{x\in\mathcal X_j}f_j(x))
=\prod_j\HellingerAffinity(p_j,q_j)$.
Put $u_j:=\sqHellinger(p_j,q_j)\in[0,1]$, enumerate $J$, and telescope:
$1-\prod_j(1-u_j)
=\sum_j[\prod_{k<j}(1-u_k)-\prod_{k\le j}(1-u_k)]
=\sum_j u_j\prod_{k<j}(1-u_k)\le\sum_j u_j$.
Since $\sqHellinger=1-\HellingerAffinity$, this and
\eqref{app:eq:aux-product-affinity} give \eqref{app:eq:aux-product-hellinger}.

\emph{2. TV comparison.}
The identity $\min\{a,b\}=(a+b-|a-b|)/2$ gives
$\sum_x\min\{p(x),q(x)\}
=(2-\sum_x|p(x)-q(x)|)/2=1-\dTV(p,q)$.
Since \(\min\{p(x),q(x)\}\le\sqrt{p(x)q(x)}\), subtracting the summed
inequality from one gives
\(\sqHellinger(p,q)\le\dTV(p,q)\).  For the other direction,
Cauchy--Schwarz yields
\begin{align*}
 2\dTV(p,q)
 &=\sum_x|\sqrt{p(x)}-\sqrt{q(x)}|
          (\sqrt{p(x)}+\sqrt{q(x)})\\
 &\le
 \left(\sum_x(\sqrt{p(x)}-\sqrt{q(x)})^2\right)^{1/2}
 \left(\sum_x(\sqrt{p(x)}+\sqrt{q(x)})^2\right)^{1/2}\\
 &\le 2\sqrt{2\sqHellinger(p,q)}.
\end{align*}
Here the first squared sum is
$\sum_x(\sqrt{p(x)}-\sqrt{q(x)})^2
=2-2\HellingerAffinity(p,q)=2\sqHellinger(p,q)$, and the second is
$\sum_x(\sqrt{p(x)}+\sqrt{q(x)})^2
=\sum_xp(x)+\sum_xq(x)+2\sum_x\sqrt{p(x)q(x)}
=2+2\HellingerAffinity(p,q)\le4$.
The last inequality uses Cauchy--Schwarz:
$\HellingerAffinity(p,q)\le\sqrt{\sum_xp(x)\sum_xq(x)}=1$.
Dividing by two proves \eqref{app:eq:aux-tv-hellinger}.
\end{proof}

\begin{lemma}[KL-to-Hellinger comparison; \textcolor{red}{{\citealp[Eq.~(7.33)]{polyanskiy2025information}}}]
\label{app:lem:kl-hellinger}
For strictly positive probability masses $p,q$ on a finite set,
\begin{equation}
 \sqHellinger(p,q)
 \le 1-\exp\{-\KLDivergence{p}{q}/2\}
 \le\frac12\KLDivergence{p}{q}.
 \label{app:eq:aux-kl-hellinger}
\end{equation}
The same inequalities hold with the KL arguments reversed.
\end{lemma}

\begin{proof}\color{red}
This is the order-$1/2$ R\'enyi/KL comparison in our Hellinger
normalization (see also \citealp{vanerven2014renyi}).
Jensen's inequality for the convex exponential gives
$\HellingerAffinity(p,q)
=\mathbb E_{X\sim p}\exp\{-\tfrac12\log\frac{p(X)}{q(X)}\}
\ge\exp\{-\KLDivergence{p}{q}/2\}$.
Subtract from one and use $1-e^{-x}\le x$ for $x\ge0$.
The reversed statement follows by exchanging $p,q$ and using the symmetry
of Hellinger distance.
\end{proof}

\begin{remark}[No upper KL bound from Hellinger alone]
\label{app:rem:no-hellinger-kl-converse}
Strict positivity without a uniform lower bound on the masses does not
give a converse upper bound on KL.  For $0<u<1/4$, consider
\[
 p_u=(1-u,u),\qquad
 q_u=(1-u e^{-1/u^2},u e^{-1/u^2}).
\]
Both laws are strictly positive, and
\[
 \sqHellinger(p_u,q_u)\le\dTV(p_u,q_u)
 =u(1-e^{-1/u^2})\le u\longrightarrow0.
\]
Nevertheless,
\[
 \KLDivergence{p_u}{q_u}
 =\frac1u+(1-u)\log\frac{1-u}{1-u e^{-1/u^2}}
 \ge\frac1u-u\longrightarrow\infty,
\]
For the last inequality, $1-u e^{-1/u^2}\le1$ implies
$\log((1-u)/(1-u e^{-1/u^2}))\ge\log(1-u)$, while
\[
 -\log(1-u)=\int_0^u\frac{dt}{1-t}
 \le\frac{u}{1-u}
 \quad\Longrightarrow\quad (1-u)\log(1-u)\ge-u.
\]
Swapping the two laws proves the corresponding failure for reverse KL.
\end{remark}

\subsection{Elementary forest facts}
\label{app:auxiliary-forest}

\begin{lemma}[Forest degree sum and tree centroid]
\label{app:lem:forest-basics}
Let $H=(W,E_H)$ be a finite nonempty forest with $c(H)$ connected
components.  Then
\begin{equation}
 |E_H|=|W|-c(H),
 \qquad
 \sum_{v\in W}d_H(v)=2|E_H|<2|W|.
 \label{app:eq:aux-forest-degree-sum}
\end{equation}
Moreover, every finite nonempty tree $T$ has a vertex $v$ such that every
component of $T\setminus\{v\}$ has at most $|V(T)|/2$ vertices.
\end{lemma}

\begin{proof}
Each tree component with $n_\ell$ vertices has $n_\ell-1$ edges.  Summing
over the components gives the first identity; the degree-sum identity then
gives the remaining claims in \eqref{app:eq:aux-forest-degree-sum}.

For the centroid claim, put $n:=|V(T)|$ and choose $v$ minimizing the largest component size
after deleting $v$.  If a component $S$ of $T\setminus\{v\}$ had more than
$|V(T)|/2$ vertices, let $w$ be the neighbor of $v$ in $S$.  After deleting
$w$, the component containing $v$ has $|V(T)|-|S|<|V(T)|/2$ vertices, while
every other component is a proper subset of $S$ and hence has at most
$|S|-1$ vertices.  Thus the largest component after deleting $w$ has size
at most
\[
 \max\{n-|S|,\ |S|-1\}<|S|,
\]
whereas the largest component after deleting $v$ has size $|S|$
(all other components together have $n-1-|S|<|S|$ vertices).
This contradicts the choice of $v$, proving the centroid claim.
\end{proof}

\subsection{Public row-selector invariance}

\paragraph{Intuition.}
The structural proof uses two selector properties: exact candidate rows
of size at most $\DegreeCutoff$ are preserved, and every returned row
has size at most $\DegreeCutoff$.  Consequently the same successful-path and all-path arguments apply
to every public selector with this contract.  The statistical comparisons
are then constructed for that selector's own decision rule.

\begin{proposition}[Invariance under public row selection]
\label{app:prop:row-selector-invariance}
Replace the standard top-vote selector by any rule in
Definition~\ref{app:def:row-selector}, keeping all other screen parameters,
fresh-color rules, and decision-rule caps unchanged.
Under the finite hypotheses of Theorem~\ref{app:thm:finite-upper},
$\SelectedPackedSampler$ has the same low-degree screen contract,
successful-path contraction and stopping guarantees, all-path resource
bounds, and expected-TV bound $\OracleRadius+\ScreenFailureBudget$.
Under the additional oracle hypotheses of
Theorem~\ref{app:thm:finite-kl-upper}, its expected forward KL is at most
$N\KLRowBudget+N\log\VocabSize\,\ScreenFailureBudget$.
Consequently both main-text benchmark guarantees retain their original
calibrations and rates.  This statement compares guarantees, not the
realized paths or errors of two selectors.
\end{proposition}

\begin{proof}

Recall the selector contract of \cref{app:def:row-selector}:
\[
 \begin{gathered}
 \RowSelector(j,C,\DegreeCutoff;\ScreenTranscript)\subseteq C,
 \qquad
 |\RowSelector(j,C,\DegreeCutoff;\ScreenTranscript)|\le\DegreeCutoff,\\
 |C|\le\DegreeCutoff
 \quad\Longrightarrow\quad
 \RowSelector(j,C,\DegreeCutoff;\ScreenTranscript)=C.
 \end{gathered}
\]
The selector is a deterministic function of the available public transcript,
makes no submission or commit, and does not read future randomness.
The dependencies of the proof are as follows.

\begin{center}
\small
\begin{tabularx}{\linewidth}{@{}>{\raggedright\arraybackslash}p{.28\linewidth}
 >{\raggedright\arraybackslash}p{.30\linewidth}
 >{\raggedright\arraybackslash}X@{}}
\toprule
Property retained & Argument using it & Guarantee retained\\
\midrule
Identity on nonoverflowing rows
 & \AppLowDegreeScreenRef & Low-degree exactness\\
Row-size cap and low-degree exactness
 & \Cref{app:lem:successful-path-structure,app:lem:upper-resources}
 & Contraction, stopping, resource bounds\\
Fresh colors conditional on the past
 & \Cref{app:lem:random-screen-success} & Failure budget\\
This selector's own safe rule
 & \Cref{app:lem:exact-hybrid,app:lem:adaptive-hellinger-upper,app:lem:safe-controller-comparison}
 & Expected TV bound\\
This selector's own joint reference
 & \Cref{app:lem:adaptive-joint-kl,app:thm:finite-kl-upper}
 & Expected forward-KL bound\\
\bottomrule
\end{tabularx}
\end{center}

\emph{1. Preserve the low-degree screen contract.}
On a successful screen, the proof of
\AppLowDegreeScreenRef{} identifies the candidate row with
the true neighborhood before applying the selector in \AlgoBlockRef{aggregate}.
Its size is at most
$\DegreeCutoff$, so the nonoverflow identity preserves it.  Every other
returned row still has size at most $\DegreeCutoff$; no claim that it
contains only true neighbors is needed.

\emph{2. Transfer the structural and resource bounds.}
These two properties are exactly those used in
\eqref{app:eq:survivor-high-neighbors} and
\eqref{app:eq:total-claim-bound}.  Thus the contraction, terminal exactness,
and successful-path round count above apply to the decision rule with
the chosen selector.
In particular, neither cycle repair nor the round-cap fallback activates
on a successful path.
The selector adds no submission or commit, and the unchanged caps bound
resources on every path, including failed screens.

\emph{3. Retain the conditional screen-failure budget.}
Conditioning on this decision rule's realized past, including previous
selector outputs, still leaves the next colors drawn in \AlgoBlockRef{draft}
independent and uniform.  Lemma~\ref{app:lem:random-screen-success} therefore gives the
same failure budget without a union bound over selectors or histories.

\emph{4. Use the chosen selector's own safe decision rule for TV.}
For TV, construct the safe decision rule of
Lemma~\ref{app:lem:safe-controller-comparison} using this same selector.
The common-path coupling compares $\SelectedPackedSampler$ with its own
safe decision rule using the same selector; it does not compare different
selectors.  Structural safety and the generic adaptive Hellinger bound
give $\OracleRadius$, and a failed screen contributes at most
$\ScreenFailureBudget$.

\emph{5. Use the chosen selector's own joint-reference process for KL.}
For KL, instead construct the joint-reference process of
Appendix~\ref{app:kl-joint-reference} for the decision rule with the chosen
selector.  Write $\ScreenFailureEvent$ for the event that a reached
screen fails in this selector's joint-reference process.
On its reference paths, the costs in
\eqref{app:eq:adaptive-joint-kl} satisfy
\begin{align*}
 \sum_t\sum_{j\in B(\mathfrak h_t)}
  \KLDivergence{\ExactRow{y(\mathfrak h_t)}{j}}
               {\OracleRow{y(\mathfrak h_t)}{j}}
 &\le\KLRowBudget\sum_t|B(\mathfrak h_t)|=N\KLRowBudget,\\
 \sum_t\mathrm{TC}(\JointBatchLaw{\mathfrak h_t})
 &\le N\log\VocabSize\,\mathbf 1_{\ScreenFailureEvent}.
\end{align*}
The second line is zero on a successful reference path by the structural
argument; on any path, the sum is at most $N\log\VocabSize$.
Conditioning before each fresh color block in this same reference process
gives \eqref{app:eq:joint-reference-screen-failure}.  Therefore
\begin{align*}
 \mathbb E_W\KLDivergence{\TargetLaw}{\SamplerOutputLaw{W}{\FrozenOracle}}
 &\le N\KLRowBudget+
       N\log\VocabSize\,\mathbb E_W\JointReferenceLaw{W}(\ScreenFailureEvent)\\
 &\le N\KLRowBudget+N\log\VocabSize\,\ScreenFailureBudget.
\end{align*}
All output and reference laws here belong to the chosen selector.
No comparison with the standard top-vote selector, or conversion from the
TV coupling, is used.
\end{proof}

\begin{proposalcontext}{Interpreting selector invariance}
The resource guarantees do not bound a selector's local arithmetic or
memory costs.
\end{proposalcontext}

\ifincludeexponents
\input{src/sections/proofs/lower_oracle_floor}
\fi

\ifincludeexponents
\input{src/sections/appendix_independent_signal}
\fi
\section{Existing guarantees in the units of this paper}
\label{app:comparison}

We separate two tasks: accounting for oracle submissions, and evaluating
the dependence quantities in the representative schedule bounds of
Table~\ref{tab:resource-comparison} and Section~\ref{sec:related-work}.
The latter uses one explicit path, evaluated for TC, DTC, and effective TC.
The comparison substitutes into
published sufficient upper bounds; it is not a lower bound on another
sampler or a tight worst-case complexity over forests.
Appendix~\ref{app:oracle-comparison} records the source-specific oracle
conditions and output objectives.

\subsection{Submissions and parallel depth}
\label{app:comparison-states}

One submitted masked state returns all requested singleton rows; repeated
submissions and different inputs in a parallel batch are charged separately.
A query for one conditional row can therefore be implemented by one such
submission, so a published row-query upper bound is also a full-output
submission upper bound.  This conversion need not be tight.
A product-unmasking schedule with $K$ updates needs at most $K$ submissions
and depth $K$.  Our completed feasible run uses
$\CounterfactualQueries+\CommitRounds+1$ submissions, including preprocessing.
All these counts exclude local arithmetic and vocabulary readout costs.

Parallel verification does not require sequential prefix queries.
Write $\mu$ for the target law used by the external verification sampler,
$h$ for its accepted history, and
$z=(z_1,\ldots,z_b)$ for an ordered candidate block, with
$z_{<j}=(z_1,\ldots,z_{j-1})$.
Once $z$ and $h$ are available, every input in
the chain-rule product
\[
 \mu(z\mid h)=\prod_j\mu(z_j\mid h,z_{<j})
\]
is known and the conditional rows can be queried in parallel.
This is consistent with the sublinear-depth Anari results summarized in
the table; their exact resource statements are recorded only in
Appendix~\ref{app:comparison-exact-oracles}.

\subsection{One path with linear TC and DTC}
\label{app:comparison-path}

All entropies and mutual informations below are under $X\sim P$.
Using the \hyperlink{app-finite-entropy}{finite entropy}
$H(p)=-\sum_xp(x)\log p(x)$, write $H(A):=H(\mathcal L_P(A))$ and,
for finite random variables $A,B,C$,
\[
 \begin{aligned}
 H(A\mid C)&:=\sum_cP(C=c)H(P_{A\mid C=c}),\\
 I(A;B\mid C)&:=\sum_cP(C=c)
   \KLDivergence{P_{AB\mid C=c}}
                {P_{A\mid C=c}\otimes P_{B\mid C=c}}.
 \end{aligned}
\]
The sums run over positive-probability values of $C$;
omitting $C$ gives unconditional mutual information.
Here $P_{AB\mid C=c}$ denotes the conditional joint law and
$P_{A\mid C=c},P_{B\mid C=c}$ its marginals.
With $X_{-i}:=(X_j)_{j\ne i}$, write
\[
 \mathrm{TC}(P):=\sum_iH(X_i)-H(X),\qquad
 \mathrm{DTC}(P):=H(X)-\sum_iH(X_i\mid X_{-i}).
\]
For every forest law, choose a rooted ordering in which each parent
precedes its children.  In the factorization \eqref{app:eq:S0}, a nonroot
vertex $j$, conditional on earlier vertices, has law
$\EndpointKernel{\ParentOf{j}}{j}(\cdot\mid X_{\ParentOf{j}})$:
none of its descendants has appeared yet, and its parent separates it
from the other earlier vertices.  Roots belong to independent components
and retain their marginal laws.  The entropy chain rule therefore gives
\[
 H(X)=\sum_{r\in\TargetRoots}H(X_r)
       +\sum_{j\notin\TargetRoots}H(X_j\mid X_{\ParentOf{j}}).
\]
Subtracting from $\sum_iH(X_i)$ cancels the
root entropies and yields
\begin{align}
 \mathrm{TC}(P)
 &=\sum_{j\notin\TargetRoots}
      [H(X_j)-H(X_j\mid X_{\ParentOf{j}})]\notag\\
 &=\sum_{j\notin\TargetRoots}I(X_j;X_{\ParentOf{j}}).
 \label{app:eq:forest-tc-identity}
\end{align}
A forest need not have large TC: the matching lower-bound witness becomes
weakly coupled as accuracy changes.  The following different witness
has constant dependence per edge while retaining the response and tail
assumptions.  All background tokens share one type, making rare-token
exchanges invisible to the conditional rows.

\paragraph{Intuition.}
The vocabulary has only two dependence-relevant types.  The type chain
has a fixed nonzero dependence between adjacent positions, while token
variation within a type contributes only independent emissions.
Consequently each edge contributes a positive constant to TC, and each
interior position contributes a positive constant to DTC:
\[
 \mathrm{TC}(P)\simeq N,\qquad \mathrm{DTC}(P)\simeq N.
\]
The background vocabulary can grow without changing these dependence
scales.  The proof checks both class membership and the entropy identities.

\begin{proposition}[A shared-class path with linear TC and DTC]
\label{app:prop:path-in-class}
Fix $\RankTailExponent>1$ and put $p=2^{-\RankTailExponent}$.
For every $N,\VocabSize\ge3$, let
$\Vocab=\{a_1,a_2\}\sqcup\mathcal B$, $|\mathcal B|=\VocabSize-2$, with
\[
 \pi(a_1)=1-p,\qquad \pi(a_2)=p/2,\qquad
 \pi(b)=\frac{p}{2(\VocabSize-2)}\quad(b\in\mathcal B).
\]
Define $\tau(a_1)=\mathsf A$ and $\tau(a)=\mathsf B$ for $a\ne a_1$.
The type masses are $P_{\mathsf A}=1-p$, $P_{\mathsf B}=p$.
In the type order $(\mathsf A,\mathsf B)$, set
\begin{equation}
 J=\begin{pmatrix}1-3p/2&p/2\\p/2&p/2\end{pmatrix},
 \qquad
 g(t,t')=\frac{J(t,t')}{P_tP_{t'}},\qquad
 K(y\mid x)=\pi(y)g(\tau(x),\tau(y)).
 \label{app:eq:two-type-kernel}
\end{equation}
Then $P(x)=\pi(x_1)\prod_{j=2}^NK(x_j\mid x_{j-1})$ is a strictly
positive path law with all marginals $\pi$.
There is a constant $\omega_\star(p)>0$, independent of $N,\VocabSize$,
such that this law satisfies (RT\textcolor{red}{; \cref{app:eq:RT}})--(UEN\textcolor{red}{; \cref{app:eq:UEN}}) with
$\ResponseConstant=\ResponseExponent=1$ and every
$\EdgeSignal\le\omega_\star(p)$, and (RF\textcolor{red}{; \cref{app:eq:RF}}) with $\FrequencyConstant=1$.
Its exact oracle satisfies both oracle accuracy conditions.
Moreover,
\begin{equation}
 \mathrm{TC}(P)=\Theta(N),\qquad \mathrm{DTC}(P)=\Theta(N),
 \label{app:eq:path-tc-dtc}
\end{equation}
where the constants depend only on $p$.
Consequently this path belongs to each main-text benchmark class for all
sufficiently large $N$ along that class's vocabulary and accuracy regime.
\end{proposition}

\begin{proof}
\emph{1. Kernel normalization and stationary marginal.}
For a type $t$, the definition of its mass is
$P_t=\sum_{x:\tau(x)=t}\pi(x)$.
Here $0<p<1/2$, and the four entries of $J$ are positive.  Its row sums are
\[
 (1-3p/2)+p/2=1-p=P_{\mathsf A},\qquad
 p/2+p/2=p=P_{\mathsf B}.
\]
Since $J$ is symmetric, its column sums are the same.
For a fixed token $x$ of type $t=\tau(x)$, group the sum over $y$ by type:
\begin{align*}
 \sum_yK(y\mid x)
 &=\sum_{t'\in\{\mathsf A,\mathsf B\}}
       \sum_{y:\tau(y)=t'}\pi(y)\frac{J(t,t')}{P_tP_{t'}}\\
 &=\sum_{t'\in\{\mathsf A,\mathsf B\}}
       P_{t'}\frac{J(t,t')}{P_tP_{t'}}
 =\frac1{P_t}\sum_{t'}J(t,t')=1.
\end{align*}
Thus $K$ is a probability kernel.  For a fixed $y$ of type $t'$,
the column-sum identity similarly yields
\begin{align*}
 \sum_x\pi(x)K(y\mid x)
 &=\pi(y)\sum_{t\in\{\mathsf A,\mathsf B\}}
       \sum_{x:\tau(x)=t}\pi(x)\frac{J(t,t')}{P_tP_{t'}}\\
 &=\pi(y)\frac1{P_{t'}}\sum_tJ(t,t')=\pi(y).
\end{align*}
Starting the path from $\pi$ therefore gives marginal $\pi$ at every
position, by induction along the path.  Also,
\[
 \pi(x)K(y\mid x)
 =\pi(x)\pi(y)\frac{J(\tau(x),\tau(y))}
                      {P_{\tau(x)}P_{\tau(y)}}
 =\pi(y)K(x\mid y),
\]
which verifies endpoint compatibility.  Positivity of $\pi$ and $K$
gives strict positivity of the path law.

\emph{2. Boundary-conditioned rows and the edge signal.}
Substituting the kernel formula into the path law gives
\[
 P(x)=\left(\prod_{j=1}^N\pi(x_j)\right)
      \left(\prod_{j=2}^Ng(\tau(x_{j-1}),\tau(x_j))\right).
\]
For fixed values outside readout $j$, all factors not incident to $j$
cancel.  Write $\mathcal N_F(j)=\{j-1,j+1\}\cap[N]$ for its path
neighbors.  Its normalized conditional row is
\[
 P(X_j=y\mid X_{-j}=x_{-j})
 =\frac{\pi(y)\prod_{k\in\mathcal N_F(j)}
                   g(\tau(y),\tau(x_k))}
        {\sum_z\pi(z)\prod_{k\in\mathcal N_F(j)}
                   g(\tau(z),\tau(x_k))}.
\]
Within type $t$, every factor $g$ is constant, so the conditional token
law given the type is always $\pi(y)/P_t$ for $\tau(y)=t$.
Summing the numerator separately over the two types shows that the
conditional odds of type $\mathsf A$ against type $\mathsf B$ are
\begin{align*}
 &\frac{P(\tau(X_j)=\mathsf A\mid X_{-j}=x_{-j})}
        {P(\tau(X_j)=\mathsf B\mid X_{-j}=x_{-j})}\\
 &\qquad=
 \frac{P_{\mathsf A}\prod_{k\in\mathcal N_F(j)}g(\mathsf A,\tau(x_k))}
      {P_{\mathsf B}\prod_{k\in\mathcal N_F(j)}g(\mathsf B,\tau(x_k))}
 =\rho\prod_{k\in\mathcal N_F(j)}r_{\tau(x_k)},\\
 &\rho:=\frac{P_{\mathsf A}}{P_{\mathsf B}}=\frac{1-p}{p},
 \qquad r_t:=\frac{g(\mathsf A,t)}{g(\mathsf B,t)}
 \quad(t\in\{\mathsf A,\mathsf B\}).
\end{align*}
Thus $\rho$ is the marginal type-odds ratio, and $r_t$ is the
multiplicative contribution of a neighbor of type $t$.
Their values follow from the four entries
\[
 g(\mathsf A,\mathsf A)=\frac{1-3p/2}{(1-p)^2},\quad
 g(\mathsf A,\mathsf B)=g(\mathsf B,\mathsf A)=\frac1{2(1-p)},\quad
 g(\mathsf B,\mathsf B)=\frac1{2p}:
\]
\begin{align*}
 \PathOddsA&=
 \frac{g(\mathsf A,\mathsf A)}{g(\mathsf B,\mathsf A)}
 =
 \frac{(1-3p/2)/(1-p)^2}{1/[2(1-p)]}
 =\frac{2-3p}{1-p},\\
 \PathOddsB&=
 \frac{g(\mathsf A,\mathsf B)}{g(\mathsf B,\mathsf B)}
 =
 \frac{1/[2(1-p)]}{1/(2p)}
 =\frac{p}{1-p}.
\end{align*}
In particular, $\PathOddsA-\PathOddsB=2(1-2p)/(1-p)>0$.

Fix a directed edge $i\to j$ and a complete boundary
$z\in\Vocab^{[N]\setminus\{i,j\}}$.
The two rows to compare are precisely
\begin{align*}
 \BoundaryRow{i}{j}{a_1}{z}
 &=\mathcal L_P(X_j\mid X_i=a_1,X_{[N]\setminus\{i,j\}}=z),\\
 \BoundaryRow{i}{j}{a_2}{z}
 &=\mathcal L_P(X_j\mid X_i=a_2,X_{[N]\setminus\{i,j\}}=z).
\end{align*}
Here $\tau(a_1)=\mathsf A$ and $\tau(a_2)=\mathsf B$.
Put $f(o)=o/(1+o)$, the type-$\mathsf A$ probability at odds $o$, and
define the contribution of the other neighbors by
\[
 q:=\prod_{k\in\mathcal N_F(j)\setminus\{i\}}r_{\tau(z_k)}
 \in\{1,\PathOddsA,\PathOddsB\}.
\]
The empty product is one; a path readout has at most one other neighbor.
The type-$\mathsf A$ weights of these two rows are, respectively,
\begin{align*}
 u&:=P(\tau(X_j)=\mathsf A\mid X_i=a_1,X_{[N]\setminus\{i,j\}}=z)
       =f(\rho\PathOddsA q),\\
 u'&:=P(\tau(X_j)=\mathsf A\mid X_i=a_2,X_{[N]\setminus\{i,j\}}=z)
       =f(\rho\PathOddsB q).
\end{align*}
Within each type, both rows have the same token law $\pi(y)/P_t$.
Expanding TV over the disjoint type supports therefore gives
\begin{align*}
 &\dTV\bigl(\BoundaryRow{i}{j}{a_1}{z},
             \BoundaryRow{i}{j}{a_2}{z}\bigr)\\
 &\quad=\frac12\left(
    \sum_{\tau(y)=\mathsf A}|u-u'|\frac{\pi(y)}{P_{\mathsf A}}
   +\sum_{\tau(y)=\mathsf B}|(1-u)-(1-u')|
                            \frac{\pi(y)}{P_{\mathsf B}}\right)\\
 &\quad=\frac12\bigl(|u-u'|+|(1-u)-(1-u')|\bigr)
 =u-u'=:D(q),\\
 &D(q)=f(\rho\PathOddsA q)-f(\rho\PathOddsB q)
 =\frac{\rho q(\PathOddsA-\PathOddsB)}
        {(1+\rho\PathOddsA q)(1+\rho\PathOddsB q)}>0.
\end{align*}
The two sums of $\pi(y)/P_t$ equal one by the definition of $P_t$.
The conditional row depends on the source token only through its type,
so $a_1,a_2$ attain the full directed response:
\[
 \DirectedResponse{i}{j}{z}
 =\max_{c,c'\in\Vocab}
   \dTV\bigl(\BoundaryRow{i}{j}{c}{z},
              \BoundaryRow{i}{j}{c'}{z}\bigr)
 =D(q).
\]
Define
$\omega_\star(p):=\min\{D(1),D(\PathOddsA),D(\PathOddsB)\}>0$.
This minimum depends only on $p$, not on $N,\VocabSize$ or the boundary.
The same calculation applies to either orientation of every path edge,
so the required condition \eqref{app:eq:UEN} is
\[
 \min\bigl\{\DirectedResponse{i}{j}{z},
             \DirectedResponse{j}{i}{z}\bigr\}
 \ge\omega_\star(p)\ge\EdgeSignal
 \qquad\text{(UEN\textcolor{red}{; \cref{app:eq:UEN}})}.
\]

\emph{3. Tail response and the rank--frequency envelope.}
Recall that $\PosMarginal{i}=\pi$ at every position, so the marginal
tail in \eqref{app:eq:tail} is
$\TailSet{i}{t}=\{c\in\Vocab:\pi(c)\le t\}$.
For a nonedge $\{i,j\}\notin\TargetEdges$, the conditional row does
not involve the source value, and hence for every $c,c'\in\Vocab$,
\[
 \BoundaryRow{i}{j}{c}{z}=\BoundaryRow{i}{j}{c'}{z},
 \qquad
 \dTV\bigl(\BoundaryRow{i}{j}{c}{z},
            \BoundaryRow{i}{j}{c'}{z}\bigr)=0.
\]
For an edge, the two threshold ranges are as follows.
\begin{itemize}
 \item If $t<1-p=\pi(a_1)$, then
 $\TailSet{i}{t}\subseteq\{a_2\}\cup\mathcal B=\tau^{-1}(\mathsf B)$.
 Thus all permitted source tokens have the same type and give the same
 readout row; their maximum TV distance is zero.
 \item If $t\ge1-p$, then $\TailSet{i}{t}=\Vocab$, since $1-p$ is the
 largest marginal mass.  The full directed response is $D(q)$ from
 step~2, so it remains to prove $D(q)\le1-p$.
\end{itemize}

Here is the needed bound with its maximization made explicit.
Set $x=\sqrt{\PathOddsA/\PathOddsB}>1$ and
$\zeta=\rho\PathOddsB q>0$.  Then
\[
 D(q)=\frac{x^2\zeta}{1+x^2\zeta}-\frac\zeta{1+\zeta}
     =\frac{(x^2-1)\zeta}{(1+x^2\zeta)(1+\zeta)}.
\]
The denominator is $1+(1+x^2)\zeta+x^2\zeta^2$.
The quotient rule gives
\[
 \frac{d}{d\zeta}
 \frac{(x^2-1)\zeta}{(1+x^2\zeta)(1+\zeta)}
 =\frac{(x^2-1)(1-x^2\zeta^2)}
        {(1+x^2\zeta)^2(1+\zeta)^2}.
\]
The derivative is positive for $\zeta<1/x$ and negative for $\zeta>1/x$.
Substituting $\zeta=1/x$ therefore gives
\[
 \sup_{q>0}D(q)
 =\frac{(x^2-1)/x}{(1+x)(1+1/x)}
 =\frac{x-1}{x+1}.
\]
All quantities are positive, and
\[
 \frac{x-1}{x+1}\le1-p
 \quad\Longleftrightarrow\quad
 px\le2-p
 \quad\Longleftrightarrow\quad
 p^2x^2\le(2-p)^2.
\]
Finally,
\[
 p^2x^2=p(2-3p),\qquad
 (2-p)^2-p(2-3p)=4-6p+4p^2>0
 \quad(0<p<1/2).
\]
Combining the nonedge and edge cases gives, for every distinct $i,j$,
complete boundary $z$, and $t\in[0,1]$,
\[
 \begin{aligned}
 &\max_{c,c'\in\TailSet{i}{t}}
 \dTV\bigl(\BoundaryRow{i}{j}{c}{z},
            \BoundaryRow{i}{j}{c'}{z}\bigr)\\
 &\qquad\le
 \begin{cases}
  0,&0\le t<1-p,\\
  1-p,&1-p\le t\le1
 \end{cases}
 \le t=\ResponseConstant t^{\ResponseExponent}
 \qquad\text{(RT\textcolor{red}{; \cref{app:eq:RT}})},
 \end{aligned}
\]
where $\ResponseConstant=\ResponseExponent=1$, and the maximum is zero
for a tail with at most one token, as in \eqref{app:eq:RT}.

For (RF\textcolor{red}{; \cref{app:eq:RF}}), $a_1$ has mass $1-p>1/2$ and is the largest atom.
The mass of $a_2$ is $p/2=2^{-\RankTailExponent-1}\le2^{-\RankTailExponent}$.
Each background mass obeys, since $\VocabSize\ge3$,
\[
 \frac{p}{2(\VocabSize-2)}
 \le\frac1{4(\VocabSize-2)}
 \le\frac1{\VocabSize};
 \qquad
 \VocabSize\le4(\VocabSize-2)
 \ \Longleftrightarrow\ 3\VocabSize\ge8.
\]
The background masses are no larger than $p/2$.
Therefore, with the ranked tokens of \eqref{app:eq:hidden-ranks}, every
position $i$ and rank $k\in[\VocabSize]$ satisfy
\begin{align*}
 \PosMarginal{i}(\RankedToken{i}{k})
 &\le
 \begin{cases}
  1,&k=1,\\
  2^{-\RankTailExponent},&k=2,\\
  \VocabSize^{-1},&3\le k\le\VocabSize
 \end{cases}
 \\
 &\le k^{-\RankTailExponent}+\VocabSize^{-1}
 =\FrequencyConstant(k^{-\RankTailExponent}+\VocabSize^{-1})
 \qquad\text{(RF\textcolor{red}{; \cref{app:eq:RF}})},
\end{align*}
with $\FrequencyConstant=1$.
If $\VocabSize=3$, the possible tie at rank two has mass $p/2$ on
either token, so the same bounds apply under the public tie-breaking order.
This is the finite envelope \eqref{app:eq:RF}; the asymptotic part of
(RF\textcolor{red}{; \cref{app:ass:rate-regime}}) is supplied by the prescribed sequence
$\VocabSize=N^{\VocabGrowthExponent+o(1)}$.
The exact oracle has zero Hellinger and KL error.

\emph{4. Total correlation.}
Let $T_j=\tau(X_j)$.  Its transition probability follows by grouping $K$:
\[
 P(T_{j+1}=t'\mid T_j=t)
 =\sum_{\tau(y)=t'}\pi(y)g(t,t')
 =\frac{J(t,t')}{P_t}.
\]
Since its marginal is $P_t$, the adjacent type joint law is $J$.
The adjacent token joint law is
\[
 P(X_j=x,X_{j+1}=y)
 =J(t,t')\frac{\pi(x)}{P_t}\frac{\pi(y)}{P_{t'}},
 \qquad t=\tau(x),\quad t'=\tau(y).
\]
Its likelihood ratio relative to $\pi(x)\pi(y)$ is
$J(t,t')/(P_tP_{t'})$, which depends only on the types.
Grouping the mutual-information sum by type therefore gives
\[
 I(X_j;X_{j+1})
 =\sum_{t,t'}J(t,t')\log\frac{J(t,t')}{P_tP_{t'}}
 =I(T_j;T_{j+1})=:c.
\]
\hypertarget{app-positive-pair-mi}{}
To see why this constant is positive, the displayed sum is
$\KLDivergence{J}{(P_tP_{t'})_{t,t'}}$.
KL is zero exactly when its two laws agree.  Such equality would make
$J$ a product table and hence rank one, but
\[
 \det J=(1-3p/2)(p/2)-(p/2)^2
       =\frac{p(1-2p)}2>0.
\]
It depends only on $p$, so \eqref{app:eq:forest-tc-identity} gives
$\mathrm{TC}(P)=(N-1)c=\Theta(N)$.

\emph{5. Dual total correlation.}
The chain rule and conditional Markov property give, respectively,
\[
 H(X)=H(X_1)+\sum_{i=2}^NH(X_i\mid X_{i-1})
\]
and
\[
 \sum_iH(X_i\mid X_{-i})
 =H(X_1\mid X_2)
  +\sum_{i=2}^{N-1}H(X_i\mid X_{i-1},X_{i+1})
  +H(X_N\mid X_{N-1}).
\]
Subtracting cancels the final-coordinate term and yields
\begin{align}
 \mathrm{DTC}(P)
 &=H(X_1)-H(X_1\mid X_2)\notag\\
 &\quad+\sum_{i=2}^{N-1}
   \bigl[H(X_i\mid X_{i-1})
         -H(X_i\mid X_{i-1},X_{i+1})\bigr]\notag\\
 &=I(X_1;X_2)+\sum_{i=2}^{N-1}
                   I(X_i;X_{i+1}\mid X_{i-1}).
 \label{app:eq:path-dtc-identity}
\end{align}
To evaluate an interior term, fix the previous type $t$.
The conditional joint table of $(T_i,T_{i+1})$ is
\[
 P(T_i=u,T_{i+1}=v\mid T_{i-1}=t)
 =\frac{J(t,u)}{P_t}\frac{J(u,v)}{P_u}.
\]
This is obtained from $J$ by multiplying row $u$ by the positive factor
$J(t,u)/(P_tP_u)$.  Its determinant is consequently
\[
 \frac{J(t,\mathsf A)J(t,\mathsf B)}
      {P_t^2P_{\mathsf A}P_{\mathsf B}}\det J>0.
\]
The two conditional type variables are therefore dependent for either
$t$, and their conditional mutual information is a positive constant
depending only on $p$.

It remains to relate this type calculation to the token term.
Conditioning on $X_{i-1}=x$ affects the future law only through $\tau(x)=t$.
For this calculation write $\Pr_t(\cdot)=\Pr(\cdot\mid T_{i-1}=t)$.
For tokens $y,z$ of types $u,v$, the emission factors cancel as follows:
\begin{align*}
 &\frac{\Pr_t(X_i=y,X_{i+1}=z)}
        {\Pr_t(X_i=y)\Pr_t(X_{i+1}=z)}\\
 &\quad=
 \frac{\Pr_t(T_i=u,T_{i+1}=v)\,[\pi(y)/P_u]\,[\pi(z)/P_v]}
      {\Pr_t(T_i=u)\,[\pi(y)/P_u]\,
       \Pr_t(T_{i+1}=v)\,[\pi(z)/P_v]}\\
 &\quad=
 \frac{\Pr_t(T_i=u,T_{i+1}=v)}
      {\Pr_t(T_i=u)\Pr_t(T_{i+1}=v)}.
\end{align*}
Grouping the token sum by $(u,v)$ therefore gives
\begin{align*}
 &I(X_i;X_{i+1}\mid T_{i-1}=t)\\
 &=\sum_{u,v}\Pr_t(T_i=u,T_{i+1}=v)
   \log\frac{\Pr_t(T_i=u,T_{i+1}=v)}
             {\Pr_t(T_i=u)\Pr_t(T_{i+1}=v)}\\
 &\qquad\times
   \underbrace{\left(\sum_{y:\tau(y)=u}\frac{\pi(y)}{P_u}\right)}_{=1}
   \underbrace{\left(\sum_{z:\tau(z)=v}\frac{\pi(z)}{P_v}\right)}_{=1}\\
 &=I(T_i;T_{i+1}\mid T_{i-1}=t).
\end{align*}
The two sums equal one because $P_u=\sum_{\tau(y)=u}\pi(y)$ and similarly
for $P_v$.  The conditional future law is the same for every previous token
of type $t$, so averaging over the previous token, or equivalently over its
type with weights $P_t$, yields
\[
 I(X_i;X_{i+1}\mid X_{i-1})
 =I(T_i;T_{i+1}\mid T_{i-1}).
\]
Every interior term is the same positive constant by stationarity and is
at most $\log2$, since the types are binary.
Together with $I(X_1;X_2)=c>0$, the identity above proves
$\mathrm{DTC}(P)=\Theta(N)$.

All constants used for the signal and dependence depend only on
$p=2^{-\RankTailExponent}$.  Each main benchmark has
$\EdgeSignal\to0$, so eventually $\EdgeSignal\le\omega_\star(p)$;
its prescribed vocabulary sequence supplies the remaining asymptotic
part of (RF\textcolor{red}{; \cref{app:ass:rate-regime}}).  This proves membership in those benchmark classes.
\end{proof}

\paragraph{What the substitution establishes.}
For this path, the TC/DTC sufficient counts listed in
Table~\ref{tab:resource-comparison} become
$\widetilde O(1+N/\varepsilon)$ at forward-KL tolerance $\varepsilon$.
These displayed guarantees alone do not yield a sublinear count at fixed
accuracy.  They do not establish that those samplers require linear work:
exact singleton updates already give an $N$-call cap.
Our case-(ii) benchmark also applies, with the balanced count in
\eqref{main:eq:balanced-upper} on its sublinear range.
The calculation concerns one class member, not every forest or every
dependence-adaptive schedule.

\begin{proposalcontext}{Other dependence measures}
No corresponding quantitative comparison is made here for entropy-based
bounds or refined unmasking growth complexity.
\end{proposalcontext}

\subsection{Effective total correlation on the same path}
\label{app:comparison-effective-tc}

To compare with \citet{dmitriev2026efficient}, let $Y(t)$ be an
independently masked copy of $X\sim P$: each $Y_i(t)$ equals $X_i$
with probability $e^{-t}$ and equals $\MASK$ otherwise.
Write $Y_{-(i,j)}(t)$ for all coordinates except $i,j$.
Their effective total correlation, written $\EffectiveTC$ here, is
\begin{equation}
 \EffectiveTC(P)
 :=\int_0^\infty\min\{1,t\}\MaskedPairInformation_P(t)\,dt,
 \qquad
 \MaskedPairInformation_P(t)
 :=\sum_{i\ne j}I\bigl(Y_i(t);Y_j(t)\mid Y_{-(i,j)}(t)\bigr).
 \label{app:eq:effective-tc-definition}
\end{equation}
These are equation~(16) of
\href{https://arxiv.org/html/2602.15008v2}{arXiv:2602.15008v2}.
Its Lemma~16 gives
$\EffectiveTC(P)\le\min\{\mathrm{TC}(P),\mathrm{DTC}(P)\}$.
Large TC and DTC alone do not imply large effective TC: that paper's
Proposition~5 supplies a counterexample. We therefore evaluate
$\EffectiveTC$ on our path directly.

\paragraph{Intuition.}
For an interior edge, reveal its two endpoints and their two outer
neighbors. This occurs with probability $e^{-4t}$.
Fixing those neighbors isolates a dependent two-vertex conditional law,
regardless of other reveals. Summing over the $N-3$ interior edges gives
$\MaskedPairInformation_P(t)\gtrsim_{\RankTailExponent}Ne^{-4t}$;
integrating over a fixed time interval keeps a linear contribution.

\begin{corollary}[Linear effective TC on the comparison path]
\label{app:cor:path-effective-tc}
For the path law in Proposition~\ref{app:prop:path-in-class},
with fixed $\RankTailExponent>1$,
$\EffectiveTC(P)=\Theta(N)$ as $N\to\infty$.
The constants depend only on $\RankTailExponent$, not on $\VocabSize$.
\end{corollary}

\begin{proof}
\begin{samepage}
\emph{1. Expand over the visible coordinates.}
Take $N\ge4$, fix $t>0$, and put $\lambda=e^{-t}$.
For a pair $i\ne j$, the visible set outside the pair is
$S\subseteq[N]\setminus\{i,j\}$ with probability
$\lambda^{|S|}(1-\lambda)^{N-2-|S|}$.
Put $M_k:=\mathbf 1\{Y_k(t)\ne\MASK\}$, so the $M_k$ are independent
$\operatorname{Ber}(\lambda)$ variables, independent of $X$.
Fix an exterior mask pattern with visible set $S$ and revealed values
$X_S=x_S$.  The conditioning event is
\[
 \mathcal C=\{M_k=\mathbf 1\{k\in S\}\ \text{for every }k\notin\{i,j\},
                \ X_S=x_S\}.
\]
For this calculation abbreviate $Y_k(t)$ to $Y_k$.
Since $M_i$ is determined by $Y_i$ and $M_j$ by $Y_j$, two applications
of the conditional chain rule give
\begin{align*}
 I(Y_i;Y_j\mid\mathcal C)
 &=I(M_i;Y_j\mid\mathcal C)
   +I(Y_i;Y_j\mid M_i,\mathcal C)\\
 &=I(Y_i;M_j\mid M_i,\mathcal C)
   +I(Y_i;Y_j\mid M_i,M_j,\mathcal C)\\
 &=\lambda^2 I(X_i;X_j\mid X_S=x_S).
\end{align*}
Here $I(M_i;Y_j\mid\mathcal C)=0$ and
$I(Y_i;M_j\mid M_i,\mathcal C)=0$ by mask independence.
The last conditional mutual information averages over $(M_i,M_j)$.
It is zero whenever either endpoint is masked, because that $Y$ value is
constant; when both are visible it equals
$I(X_i;X_j\mid X_S=x_S)$, with weight $\lambda^2$.
Averaging over $x_S$ and then over exterior mask patterns gives
\begin{equation}
 I\bigl(Y_i(t);Y_j(t)\mid Y_{-(i,j)}(t)\bigr)
 =\lambda^2
   \sum_{S\subseteq[N]\setminus\{i,j\}}
    \lambda^{|S|}(1-\lambda)^{N-2-|S|}
    I(X_i;X_j\mid X_S).
 \label{app:eq:masked-pair-information}
\end{equation}
\end{samepage}

\emph{2. Retain a positive dependence across each interior edge.}
Recall the two types $T_i=\tau(X_i)$ and their transition matrix
$\bar K(a,b)=J(a,b)/P_a$ from the proof of
Proposition~\ref{app:prop:path-in-class}.
All entries are positive, and
$\det\bar K=\det J/(P_{\mathsf A}P_{\mathsf B})>0$.
For an interior pair $i,i+1$, where $2\le i\le N-2$, fix the outer
types $T_{i-1}=a,T_{i+2}=b$.
The conditional two-by-two table is
\begin{align*}
 &\Pr(T_i=u,T_{i+1}=v\mid T_{i-1}=a,T_{i+2}=b)
 =\frac{\bar K(a,u)\bar K(u,v)\bar K(v,b)}{Z_{a,b}},\\
 &Z_{a,b}:=\sum_{u,v}\bar K(a,u)\bar K(u,v)\bar K(v,b)>0.
\end{align*}
Here and below type indices range over $\{\mathsf A,\mathsf B\}$.
This table is obtained from $\bar K$ by multiplying row $u$ by
$\bar K(a,u)$, column $v$ by $\bar K(v,b)$, and the whole table
by $Z_{a,b}^{-1}$. Its determinant is
\[
 \frac{\bar K(a,\mathsf A)\bar K(a,\mathsf B)
       \bar K(\mathsf A,b)\bar K(\mathsf B,b)}
      {Z_{a,b}^{\,2}}\det\bar K>0.
\]
Thus the conditional pair has positive mutual information for each
boundary type pair, by the \hyperlink{app-positive-pair-mi}{KL and rank-one
criterion} used above.  There are only four such pairs.  By stationarity,
the positive constant
\[
 c_{\mathrm{pair}}
 :=\min_{a,b}I(T_2;T_3\mid T_1=a,T_4=b)>0
\]
depends only on $p=2^{-\RankTailExponent}$.

If $\{i-1,i+2\}\subseteq S\subseteq[N]\setminus\{i,i+1\}$,
path factorization shows that other
observed coordinates affect only factors outside this pair.
Conditional on $X_S=x_S$, the type pair therefore has the displayed
table with $a=\tau(x_{i-1})$, $b=\tau(x_{i+2})$.
Applying data processing to the two type maps at each $x_S$ and then
averaging gives
\begin{align*}
 I(X_i;X_{i+1}\mid X_S)
 &=\sum_{x_S}P(X_S=x_S)
       I(X_i;X_{i+1}\mid X_S=x_S)\\
 &\ge\sum_{x_S}P(X_S=x_S)
       I(T_i;T_{i+1}\mid X_S=x_S)\\
 &\ge\sum_{x_S}P(X_S=x_S)c_{\mathrm{pair}}
 =c_{\mathrm{pair}}.
\end{align*}
The total mask weight of these sets is
\[
 \sum_{\substack{S\subseteq[N]\setminus\{i,i+1\}\\
                 \{i-1,i+2\}\subseteq S}}
   \lambda^{|S|}(1-\lambda)^{N-2-|S|}
 =\lambda^2\bigl(\lambda+(1-\lambda)\bigr)^{N-4}
 =\lambda^2.
\]
The first $\lambda^2$ requires the two outer neighbors to be visible;
each remaining coordinate contributes $\lambda+(1-\lambda)=1$.
Substituting into \eqref{app:eq:masked-pair-information} yields
$I(Y_i(t);Y_{i+1}(t)\mid Y_{-(i,i+1)}(t))
 \ge c_{\mathrm{pair}}\lambda^4$.

\emph{3. Sum the edges and integrate.}
There are $N-3$ such edges and two orientations per edge in
$\MaskedPairInformation_P(t)$. Nonnegativity of the other terms gives
\[
 \MaskedPairInformation_P(t)\ge
 2(N-3)c_{\mathrm{pair}}e^{-4t}.
\]
Hence
\[
 \EffectiveTC(P)\ge
 2(N-3)c_{\mathrm{pair}}
 \int_0^\infty\min\{1,t\}e^{-4t}\,dt
 =\Omega_{\RankTailExponent}(N).
\]
The integral is a positive finite absolute constant.
The reverse bound follows from
$\EffectiveTC(P)\le\mathrm{DTC}(P)=O_{\RankTailExponent}(N)$,
using the cited Lemma~16 and Proposition~\ref{app:prop:path-in-class}.
\end{proof}

\section{Source locators and conventions for Table~\ref{tab:resource-comparison}}
\label{app:oracle-comparison}

This appendix is a verification aid for Table~\ref{tab:resource-comparison},
not a separate survey.  For each external row, it records the linked version,
the theorem or definition supporting the entry, and only the conversion needed
to read the table.  Numbering below refers to the linked version and may differ
from another bibliographic version.

\paragraph{Common conventions.}
For an external randomized schedule $M$ with conditional output law $Q_M$,
\emph{Mean KL/TV} denotes $\mathbb E_M D(P,Q_M)$, not
$D(P,\mathbb E_MQ_M)$; an unqualified TV guarantee concerns the randomized
output law.  The marker \emph{exp.} applies to both resource
columns, while \emph{same} means that the submission bound also bounds depth.
Appendix~\ref{app:comparison-states} gives the conversion from conditional-row
queries or unmasking updates to our masked-state submissions.  Fixed accuracy
and polynomial vocabulary are imposed only when comparing growth in $N$.
Uniform counterfactual-row error and an average under a data--mask or trajectory
law are different oracle assumptions; Table~\ref{tab:resource-comparison}
does not identify them.

\subsection{Distribution-general sampler rows}
\label{app:comparison-exact-oracles}

\begin{description}
\item[\Citet{anari2024parallel}.]
\href{https://arxiv.org/pdf/2408.09442v1}{arXiv:2408.09442v1},
Theorem~2 and Algorithm~4: exact conditional rows and exact output, with
$O(N)$ expected row queries and $\widetilde O(N^{2/3})$ expected depth for
polynomial vocabulary.  Appendix~\ref{app:comparison-states} supplies the
submission conversion used in the table.

\item[\Citet{anari2026autospeculation}.]
\href{https://arxiv.org/pdf/2511.07869v1}{arXiv:2511.07869v1},
Theorem~27 and footnote~3: $O(N\log N)$ expected full-row queries and
$O(\sqrt{N\log\VocabSize\,(\log N)^3})$ expected depth.
Remark~32, equation~(25), gives the noisy-row TV entry:
\begin{equation}
 e\le\min\{\tau/N,N^{-3/2}\}
 \quad\Longrightarrow\quad
 \mathrm{TV}(P,Q)\le\tau.
 \label{app:eq:autospec-native-tv}
\end{equation}
The condition is uniform TV accuracy of the normalized rows; the conclusion is
TV for the output law.  Exact rows give exact output.
\end{description}

\subsection{Dependence-based schedule rows}
\label{app:comparison-unmasking-errors}

\begin{description}
\item[\Citet{li2025convergence}.]
\href{https://arxiv.org/html/2505.21400v2}{arXiv:2505.21400v2},
Definition~1, equation~(7), defines the predictor error averaged over target
data, random masks, and the weighted time index.  Theorem~1, equation~(8), and
Corollary~1, equation~(9), give for near-balanced blocks
\[
 \mathbb E_M\mathrm{KL}(P\Vert Q_M)
 \le \frac{C}{K}\bigl(\mathrm{TC}(P)+\mathrm{DTC}(P)\bigr)
      +\varepsilon_{\rm train},
\]
where $M$ is the random schedule and $K$ is its number of updates.  The table
allocates a constant accuracy budget to $\varepsilon_{\rm train}$ and
substitutes the path calculation in
Proposition~\ref{app:prop:path-in-class}; no uniform counterfactual-row
condition is inferred from this average.

\item[\Citet{chen2026optimal}.]
\href{https://arxiv.org/html/2511.04647v2}{arXiv:2511.04647v2},
Definition~1.1 specifies schedule-averaged forward KL, Definition~2.1 defines
one partial-assignment query returning all remaining conditional rows, and
Theorem~1.9 gives, for a supplied TC or DTC upper bound $\widehat T$,
\begin{equation}
 \mathbb E_M\mathrm{KL}(P\Vert Q_M)\le\varepsilon,
 \qquad
 K\le2+(1+\log N)(1+\lceil\widehat T/\varepsilon\rceil).
 \label{app:eq:chen-native-bound}
\end{equation}
The table uses exact rows and a constant-factor supplied bound.  On the path of
Proposition~\ref{app:prop:path-in-class}, both TC and DTC are $\Theta(N)$.
\end{description}

\paragraph{Baselines and this paper.}
Exact singleton sampling uses $N$ submissions and depth $N$ by the chain rule.
For marker $a$, one exact product batch has forward KL
$\mathrm{TC}(P)$, which is linear on the path of
Proposition~\ref{app:prop:path-in-class}; it is only a resource baseline.
Our two rows are the balanced bounds in
Theorem~\ref{main:thm:upper-curve} and
\eqref{main:eq:balanced-upper}: case (i) uses uniform row squared-Hellinger
error and reports mean TV, while case (ii) uses uniform row forward KL and
reports mean forward KL.  Their exact finite statements are
Theorems~\ref{app:thm:finite-upper} and
\ref{app:thm:finite-kl-upper}.

\subsection{Supplementary citation map}
\label{app:related-work-supplement}

The following citations supply context only; they are not used to derive a row
of Table~\ref{tab:resource-comparison}.

{\color{red}
\paragraph{Modeling background.}
\begin{itemize}
\item \emph{Forest approximation.}
\citet{liu2011forestdensity} study nonparametric density estimation using
forests, allowing the true continuous distribution to lie outside that
family (\href{https://jmlr.org/papers/v12/liu11a.html}{abstract}).
This motivates forests as an approximation class; our target instead
satisfies exact forest factorization.

\item \emph{Prediction without exact structure recovery.}
\citet{bresler2020predictions} study tree Ising prediction using
\emph{small-set TV}, which compares marginals on small subsets
(\href{https://arxiv.org/abs/1604.06749}{abstract}).
This differs from the global TV objective of
\citet{daskalakis2021treeising} and our case (i).

\item \emph{Learned conditional distributions.}
\citet{heckerman2000dependency} allow separately learned local conditionals
that need not be compatible with a joint law
(\href{https://www.jmlr.org/papers/volume1/heckerman00a/heckerman00a.pdf\#page=7}{Section 3,
pp.~55--56}). Their inference procedure uses ordered pseudo-Gibbs updates;
our interface supplies masked-state rows for irreversible commits.
\end{itemize}

\paragraph{Algorithmic antecedents.}
\label{app:algorithmic-antecedents}
The construction in Section~\ref{sec:main-algorithm} connects the following
design principles to hidden-forest sampling.
\begin{itemize}
\item \emph{Shared evaluations (Steps 2--3).}
\citet{coleman1983estimation} use graph coloring to reduce function
evaluations for sparse Jacobian estimation (see the paper's
\href{https://epubs.siam.org/doi/10.1137/0720013}{abstract}).
Our packed probes share evaluations across source--readout pairs;
the separation conditions (R)/(S) in Section~\ref{sec:main-packed-screen}
make their conditional replies reproduce single-source tests.

\item \emph{Hashing and majority recovery (Step 3).}
\citet{bshouty2018exact},
\href{https://arxiv.org/pdf/1706.06934v1\#page=17}{arXiv v1, Section 4.3,
Theorem 8, Figure 1, steps 1--3}, hash variables, lift identified buckets
back to variables, and aggregate by majority.
Our bucket tests use conditional-row diameters; we recover low-degree
neighborhoods rather than a Boolean function.

\item \emph{Peeling (Step 4).}
\citet{becker2010referee},
\href{https://arxiv.org/pdf/1009.4447v2\#page=9}{arXiv v2, Section 3.1, p.~9},
reconstruct a forest by decoding leaves from exact degrees and neighbor-ID
sums, then removing them and updating these summaries.
Here, low-degree reports instead select singleton commits that shrink
the high-degree core (\cref{main:eq:peeling-contraction}), without requiring
complete neighborhoods for the committed vertices.

\item \emph{Separator layers (Steps 5--6).}
\citet{iyer1988optimal} relate vertex ranking to separator-tree height
(see their
\href{https://research.ibm.com/publications/optimal-node-ranking-of-trees}{abstract}).
In our centroid schedule, earlier layers have higher ranks: their committed
vertices separate same-layer vertices into distinct residual components.
We use centroid recursion, not their optimal-ranking algorithm.
\end{itemize}
}

\paragraph{\textcolor{red}{Additional context.}}
Diffusion foundations and discrete or language formulations include
\citet{sohldickstein2015deep,ho2020denoising,song2021score,
campbell2022continuous,li2022diffusionlm,he2023diffusionbert}.
Parallel-time and complementary complexity perspectives include
\citet{shih2023parallel,chen2024sublinear,yao2026parallelintime,
feng2025theoretical,jiang2026optimal,zhang2026generation,liang2026sharp}.

Adaptive order, lookahead, confidence, and schedule analyses include
\citet{li2024promises,kim2025train,hayakawa2026demystifying,
lee2025lookahead,fu2025bits,lavenant2025error,zhao2026adaptation,
cai2026confidence,jys}.
For the specific round statement cited in the main text, the locator for
\citet{fu2025bits} is
\href{https://arxiv.org/html/2511.21103v1}{arXiv:2511.21103v1},
Assumption~3.1 and Theorem~3.2.
Dependence-based, remasking, and within-batch modeling perspectives include
\citet{dmitriev2026efficient,dmitriev2026remasking,wainwright2026geometry,
bansal2026randomwalks,hayakawa2025distillation,liu2024discrete,
lezama2022discrete}.

Related oracle, learning, and convergence settings include approximate-density
and conditional-sampling access
\citep{golowich2026testtime,canonne2015conditional,canonne2021random,
chen2021junta,blanca2023coordinate}, distributed sampling
\citep{feng2020local}, learned-denoiser error
\citep{wakasugi2025state}, and discrete-diffusion convergence
\citep{zhang2025convergence,ren2025discrete,liang2025absorb}.
These works use objectives or access models different from the fixed
singleton-conditional interface studied here; no quantitative conversion is
claimed beyond the four source rows above.

\subsection{Reading Table~\ref{tab:resource-comparison}}
\label{app:comparison-table-notes}

\begin{itemize}
\item The external entries are sufficient upper bounds, not lower bounds on
their samplers.  The two Anari rows retain their distribution-general
guarantees.  For the TC/DTC rows, Appendix~\ref{app:comparison-path}
substitutes one path with
$\mathrm{TC}(P)=\mathrm{DTC}(P)=\Theta(N)$; this does not imply that those
samplers require linear work.

\item Marker $b$ is exactly the pair of row-precision requirements in
\eqref{app:eq:autospec-native-tv}.  For
\citet{chen2026optimal}, $\widehat T$ is a supplied TC or DTC upper bound.
For the Li--Cai row, $K=N$ removes the factorization term in its Theorem~1;
for the Chen et al. row, exact singleton sampling gives the $N$-call cap.

\item The scaling comparison fixes positive output tolerance and polynomial
vocabulary.  Constant allocations to source-specific learning or sampling
error do not change the displayed powers of $N$.

\item Our rows use $\RankTailExponent>1$ and the balanced examples of
\eqref{main:eq:balanced-upper}.  Their common leading resource power is
$2/3+1/(9\RankTailExponent)<1$; the target class and oracle assumptions remain
part of the comparison.
\end{itemize}

\section{Ideal-target fixed-accuracy protocol}
\label{app:experiments}

\subsection{Targets}
\label{app:experiment-targets}

Let $\Vocab=\{0,\ldots,2047\}$, with $\phi(0)=1$, $\phi(1)=-1$,
and $\phi(a)=0$ otherwise. On a hidden forest
$\TargetForest=(\PositionSet,\TargetEdges)$, the target law is
\begin{equation}
 \TargetLaw(x)=\frac{1}{Z_F}
 \prod_{i=1}^N\psi_i(x_i)
 \prod_{\{i,j\}\in\TargetEdges}
 \bigl(1+w_{ij}\phi(x_i)\phi(x_j)\bigr),
 \qquad x\in\Vocab^N,
 \label{app:eq:experiment-target}
\end{equation}
where $Z_F$ is the normalizer. For independent
$h_i\sim\operatorname{Unif}[-1,1]$, its unary factors are
\begin{equation}
 \psi_i(a)=\frac{1}{0.4e^{h_i}+0.4e^{-h_i}+0.2}
 \begin{cases}
 0.4e^{h_i},&a=0,\\
 0.4e^{-h_i},&a=1,\\
 0.2/2046,&a\in\{2,\ldots,2047\}.
 \end{cases}
 \label{app:eq:experiment-unary}
\end{equation}
For each target draw, a master field vector is shared by prefixes across
sizes and forest families, and vertex labels are randomly permuted.
Edge weights are given in Table~\ref{tab:experiment-targets}.
The frozen conditional oracle is exact,
$\OracleRow{y}{j}=\ExactRow{y}{j}$, and the sampler is not given the edges.
We test $N\in\{8192,10240,12288,14336,16384\}$.

\begin{table}[htbp]
 \centering
 \caption{Ideal target families and fixed absolute accuracy thresholds.
 Here $\Delta_\star(N)=\max\{16,\lceil\sqrt N\rceil\}$, and
 $K_\star=10^{-10}L_0$ uses the reference in
 Appendix~\ref{app:experiment-accuracy}.}
 \label{tab:experiment-targets}
 \begin{tabular}{@{}llcc@{}}
 \toprule
 Family & Structure before relabeling & $w_F(N)$ & $K_\star$ \\
 \midrule
 Matching & Disjoint pairs & $0.5$ & $6.375289\times10^{-9}$ \\
 Path & One path & $0.25$ & $2.801962\times10^{-9}$ \\
 Binary tree & Complete binary tree & $1/6$ & $1.226162\times10^{-9}$ \\
 Growing stars & At most $\Delta_\star(N)$ leaves per star
   & $0.5/\Delta_\star(N)$ & $5.138930\times10^{-12}$ \\
 \bottomrule
 \end{tabular}
\end{table}

\subsection{Resources and accuracy}
\label{app:experiment-accuracy}

\begin{equation}
 Q_{\mathrm{tot}}=Q_{\mathrm{pre}}+\CounterfactualQueries+\CommitRounds,
 \label{app:eq:experiment-total-rounds}
\end{equation}
counts preprocessing submissions, counterfactual submissions, and nonempty
commit rounds. Each submission is charged once, including repeated states;
rows sharing a submission share its cost. For random sampling with $B$
balanced batches, $Q_{\mathrm{tot}}=\CommitRounds=B$.
The proposal has no external $Q_{\mathrm{tot}}\le N$ cap.

At history $H_{t-1}=(G_{t-1},x_{G_{t-1}})$ before batch $B_t$, let $y_t$ reveal
$x_{G_{t-1}}$ and mask the other positions, and put
$q_{t,i}=\OracleRow{y_t}{i}$ and
$P_t=\TargetLaw(X_{B_t}\in\cdot\mid X_{G_{t-1}}=x_{G_{t-1}})$,
using the commit index of \cref{app:def:admissible-algorithm}. We measure
\begin{equation}
 K=\sum_{t=1}^{\CommitRounds}
 \KLDivergence{\bigotimes_{i\in B_t}q_{t,i}}{P_t}.
 \label{app:eq:experiment-kl}
\end{equation}
This is a product-to-joint sum along the sampled history, distinct from
the target-to-output risk $\KLSeedRisk$.
Passive evaluation adds no oracle rounds: after conditioning and
marginalization, boundary factors of arity at most four are enumerated
exactly; larger factors use 128 independent Monte Carlo draws each.
$\operatorname{SE}_{\mathrm{MC}}$ is the conditional standard error of
the total estimate. A completed run passes the empirical criterion
\begin{equation}
 K+2\operatorname{SE}_{\mathrm{MC}}\le K_\star,
 \qquad K_\star=10^{-10}L_0.
 \label{app:eq:experiment-accuracy}
\end{equation}
For each forest family, the reference target $P^{(0)}$ has $N=2048$ and
\begin{equation}
 L_0=\KLDivergence{\bigotimes_{i=1}^{2048}P_i^{(0)}}{P^{(0)}},
 \qquad P_i^{(0)}=\mathcal L_{P^{(0)}}(X_i).
 \label{app:eq:experiment-reference-kl}
\end{equation}
Thus each absolute threshold in Table~\ref{tab:experiment-targets}
is fixed across sizes and methods.

\subsection{Sampling and evaluation}
\label{app:experiment-proposal}

\paragraph{Proposal.}
\textcolor{red}{We tune color counts, degree cutoffs, and thresholds empirically,
without enforcing the theoretical parameter prescriptions
(\cref{app:algorithm-overview,app:eq:random-color-parameters}); the target family
also need not satisfy (RT; \cref{app:eq:RT}) with
$\ResponseConstant=\ResponseExponent=1$.}
Within each $(\text{forest},N)$ cell, we screen \textcolor{red}{candidate settings},
then evaluate six low-cost diverse candidates on target
draws $171,172,173$ crossed with sampler streams $271,272$.
We freeze the minimum-mean-$Q_{\mathrm{tot}}$ candidate that completes
and passes on all six pairs at development threshold $10^{-3}L_0$.
Screening and development use respectively 32 and 64 Monte Carlo draws
per nonenumerated factor; evaluation uses 128.
All 120 fresh proposal evaluation runs have structural exact zero
discrepancy, $K=\operatorname{SE}_{\mathrm{MC}}=0$.

\paragraph{Random baseline.}
\label{app:experiment-random}
For each evaluation run, fix a random permutation and commit random stream
across budgets; budget $B$ partitions the permutation into $B$ balanced
nonempty batches. Since adjacent partitions need not be nested, among
observed budgets $b_1<\cdots<b_m=N$ select the first passing suffix:
\begin{equation}
 B_{\mathrm{safe}}=b_{k_\star},\qquad
 k_\star=\min\bigl\{k:\text{every observed }b_\ell,\ \ell\ge k,
                    \text{ passes \eqref{app:eq:experiment-accuracy}}\bigr\}.
 \label{app:eq:experiment-random-suffix}
\end{equation}
The adjacent unsafe--safe bracket is refined to width at most $0.05N$
(observed maximum: $0.03125N$); $B=N$ is the exact singleton endpoint.
This is per-run selection from each evaluation curve, not a globally frozen
budget. Reported cost is $B_{\mathrm{safe}}$ and excludes search costs,
just as proposal development is excluded.

\paragraph{Replicates.}
\label{app:experiment-replicates}
Each point uses target draws $196,197,198$, crossed with sampler streams
$292,293$ for the proposal and disjoint streams $294,295$ for random.
Means and min--max ranges use these same six runs.

\paragraph{Computation.}
\label{app:experiment-reproducibility}
Eight workers ran on an Apple-silicon Mac mini (macOS 15.7.7,
Python 3.13.15, NumPy 2.2.2) and a Linux 5.4.0 host (Python 3.10.12,
NumPy 1.26.4; one BLAS/OpenMP thread per worker).
The proposal/random study took about 64.7 minutes; the per-run random
budget searches took about 12.0 minutes in total.

\end{document}